\documentclass[letterpaper]{article} 
\usepackage[preprint]{aaai2027}
\usepackage[hyphens]{url}  
\usepackage{graphicx} 
\usepackage{natbib}  
\usepackage{caption} 
\usepackage{algorithm}
\usepackage{algorithmic}
\usepackage{makecell}
\usepackage{newfloat}
\usepackage{listings}
\DeclareCaptionStyle{ruled}{labelfont=normalfont,labelsep=colon,strut=off} 
\floatstyle{ruled}
\newfloat{listing}{tb}{lst}{}
\floatname{listing}{Listing}

\usepackage{booktabs}
\usepackage{arydshln} 

\usepackage{subfigure}

\usepackage{amsmath} 

\usepackage{amsmath}
\usepackage{amssymb}
\usepackage{mathtools}
\usepackage{amsthm}

\theoremstyle{plain}
\newtheorem{theorem}{Theorem}[section]
\newtheorem{proposition}[theorem]{Proposition}
\newtheorem{lemma}[theorem]{Lemma}
\newtheorem{corollary}[theorem]{Corollary}
\theoremstyle{definition}
\newtheorem{definition}[theorem]{Definition}
\newtheorem{assumption}[theorem]{Assumption}
\theoremstyle{remark}
\newtheorem{remark}[theorem]{Remark}

\usepackage{multirow}

\newcommand{\cO}{\mathcal{O}}
\newcommand{\cA}{\mathcal{A}}
\newcommand{\cS}{\mathcal{S}}
\newcommand{\cM}{\mathcal{M}}
\newcommand{\cC}{\mathcal{C}}
\newcommand{\cX}{\mathcal{X}}
\newcommand{\cP}{\mathcal{P}}
\newcommand{\cI}{\mathcal{I}}
\newcommand{\cR}{\mathcal{R}}

\newcommand{\bP}{\mathbb{P}}
\newcommand{\bE}{\mathbb{E}}
\newcommand{\Id}{\mathrm{Id}}
\newcommand{\Sym}{\mathrm{Sym}}
\newcommand{\Hol}{\mathrm{Hol}}

\newcommand{\argmax}{\operatorname*{arg\,max}}
\newcommand{\argmin}{\operatorname*{arg\,min}}
\newcommand{\Holo}{\mathrm{Holo}}

\DeclareMathOperator{\src}{src}
\DeclareMathOperator{\tgt}{tgt}

\DeclareMathOperator{\ab}{ab}

\title{Minimal Markovization via Stable Quotients in Holonomy-Cover Decision Processes}

\author{
    Zuyuan Zhang\textsuperscript{\rm 1},
    Yongshan Chen\textsuperscript{\rm 2},
    Mahdi Imani\textsuperscript{\rm 2},
    Tian Lan\textsuperscript{\rm 1}
}

\affiliations{
    \textsuperscript{\rm 1}Department of Electrical and Computer Engineering,
    The George Washington University\\
    \textsuperscript{\rm 2}Department of Electrical and Computer Engineering,
    Northeastern University\\
    zuyuan.zhang@gwu.edu,
    chen.yongs@northeastern.edu,
    m.imani@northeastern.edu,
    tlan@gwu.edu
}

\newcommand{\PDFmode}{0}

\ifcase\PDFmode
\or
  \usepackage[1-9,10000]{pagesel}
\or
  \usepackage[10-]{pagesel}
\fi

\begin{document}

 \maketitle

\begin{abstract}
An agent acting under partial observability must retain a recursively updateable statistic of history that restores the Markov property, but the smallest such statistic is generally unknown. We characterize this minimal Markov sufficient statistic for holonomy-cover decision processes, a structured POMDP class in which the visible dynamics are Markov and every realized visible transition applies a fixed permutation to a hidden mode. 
In particular, we construct the stable quotient, the coarsest observation-wise abstraction preserving one-step rewards and quotient successors, and prove that the pair of the current observation and stable class forms an exact finite Markov state. When the current class is correctly initialized, exact class tracking requires exactly the minimal
memory symbols, in the sense that under reachability and pairwise decision separation at a maximizing observation, no arbitrary finite-memory controller can use fewer. Under resettable diagnostics, nearest-prototype class inference has exponentially decaying error, 
and a calibrate-then-restart reduction transfers finite-MDP guarantees to the recovered state. 
The results enable \emph{Holonomy Memory Reinforcement Learning}. It represents memory by the current stable class, updates it through ordered edge transports, identifies local class coordinates when diagnostics are available, and applies a standard finite-MDP RL backbone after synchronization.
Experiments recover an exact compression from 
raw states to 
quotient states and achieve perfect paired-order accuracy with three decision-time memory states, matching the quotient oracle and outperforming the non-oracle baselines.

\end{abstract}

\section{Introduction}
\label{sec:intro}

Reinforcement learning (RL) relies on a Markov state that makes future rewards and action-conditioned transitions depend only on the present~\cite{sutton1998reinforcement}. Under partial observability, the current observation alone may not be sufficient to predict the future. Partially observable Markov decision process (POMDP) theory restores the Markov property through belief states--posterior distributions over latent states conditioned on the interaction history~\cite{aastrom1965optimal,smallwood1973optimal,monahan1982state,kaelbling1998planning}. In practice, this belief is often approximated by encoding the action--observation history with sequential or memory-augmented neural architectures~\cite{meuleau2013learning}. Although these methods can learn useful history representations from a computational perspective, they generally do not characterize the minimal Markov sufficient statistic--the smallest recursively updateable representation of history that preserves the Markovian property.

We study a structured family of POMDPs through the lens of transport theory and identify its minimal Markov sufficient statistic. We model partial observability as hidden-mode transport over the graph of visible transitions. Each latent state consists of an observed component and an unobserved mode. The observed component evolves according to a Markov transition rule that is independent of the hidden mode. Whenever an action leads from one observation to the next, the realized transition applies a fixed permutation to the hidden mode. A trajectory therefore updates the hidden mode by composing these transition-specific transports in temporal order. We call the resulting model a \emph{holonomy-cover decision process} (HCDP). Its visible dynamics are Markov, but two histories ending at the same observation can induce different hidden modes---and hence different rewards or optimal decisions---because their ordered transports differ.

Our approach departs from existing work that primarily designs expressive or efficient history representations rather than characterizing the smallest Markov sufficient statistic. 
Recurrent, Transformer, and linear-recurrent networks can encode ordered histories in POMDPs, but their learned representations do not generally guarantee that the resulting memory-augmented process is Markov~\cite{hausknecht2015deep,parisotto2020stabilizing,ni2021recurrent,morad2024recurrent,zhanghodgeflow,zhanggeometric}. Belief compression, predictive-state representations, and world models focus on summarizing posterior or predictive information~\cite{roy2005finding,igl2018deep,littman2001predictive,singh2012predictive,hafner2019learning,hafner2019dream,zhang2024modeling,zhang2025learning}.  Bisimulation and homomorphism methods provide principled aggregation only after a Markov state is identified
~\cite{givan2003equivalence,ravindran2003smdp,ferns2004metrics,abel2016near,gelada2019deepmdp,zhang2020learning,efroni2022provable,lamb2022guaranteed}, while reward machines and regular decision processes introduce finite automata for history dependence, but do not directly isolate the observation-wise minimal Markov state generated by ordered hidden transport~\cite{bresina1996heuristic,thiebaux2006decision,icarte2018using,camacho2019ltl,icarte2022reward,toro2019learning,brafman2019regular,ronca2021efficient,cipollone2023provably,deb2024tractable,zhang2026matrix}. Appendix~\ref{app:positioning_table} offers a detailed comparison.

We analyze an HCDP by refining the hidden modes separately within each observation fiber. Two modes remain in the same class only if they agree on every immediate action reward and every feasible edge transports them into the same successor classes. The key proof is a monotone finite refinement: the operator stabilizes at a fixed point $\Pi^\star$ (Lemma~\ref{lem:stable_fp_main}), the exactness criterion characterizes every reward-and-successor-preserving abstraction (Proposition~\ref{prop:stable_exact_equiv_main}), and the resulting factorization shows that every exact observation-wise abstraction refines the stable quotient (Theorem~\ref{thm:coarsest_exact_main}). Consequently, the pair $(o,c)$ of the current observation $o$ and stable quotient class $c=[i]_o$ is an exact value-preserving finite Markov state (Theorem~\ref{thm:markovization_main}). Within this HCDP abstraction class, it is therefore shown to be the minimal Markov sufficient statistic. When the current class is correctly initialized, exactly $\max_o|\cC_o|$ reusable memory symbols are necessary and sufficient for recursive class tracking (Corollary~\ref{cor:min_memory_main}). If the initial class is unknown, the generic exact observable information state is instead a posterior over stable classes (Proposition~\ref{prop:quotient_belief_main}).

The analysis enables \emph{Holonomy Memory Reinforcement Learning} (HMRL). HMRL represents memory by the current stable class, updates it through ordered edge transports, identifies local class coordinates when diagnostics are available, and applies a standard finite-MDP RL backbone after synchronization.
The quotient is structural, but learning it requires additional information. We prove that arbitrary passive logs do not uniformly identify stable classes and transports (Theorem~\ref{thm:passive_nonident_main}); under resettable calibrated diagnostics, class error decays exponentially (Theorem~\ref{thm:local_class_inf_main}) and transition-anchored estimates eventually recover edge transports up to observation-wise relabeling (Theorem~\ref{thm:transport_local_main}). Once local class tracking is synchronized, the learned process is isomorphic to the canonical quotient MDP (Theorem~\ref{thm:eventual_exact_reduction_main}), and a calibrate-then-restart condition transfers the original finite-MDP guarantee to the recovered state (Corollary~\ref{cor:backbone_transfer_main}).

Our main contributions are:
\begin{itemize}
    \item We formulate HCDPs and identify their stable quotient as the coarsest exact observation-wise Markov state, with a matching $\max_o|\cC_o|$ memory characterization for exact known-class tracking.
    \item We separate passive nonidentifiability from diagnostic recovery, identify classes and transports up to observation-wise relabeling, and give the conditions under which a finite-MDP RL guarantee transfers to the recovered state.
    \item We prove a nonabelian lower bound for count-factorized memories and show that ordered stable-class tracking avoids this loss with finite memory.
\end{itemize}

Experiments support these claims. ChainCover is compressed from $216$ raw states to $25$ quotient states without reward, transition, or optimal-value inconsistency. On nonabelian LoopGuess, HMRL matches the quotient oracle, achieves $1.000$ final success and perfect paired-order accuracy, and uses only three decision-time memory states.

\section{Preliminaries}
\label{sec:prelim}

This section separates three objects that are easy to conflate: the hidden Markov state, the observable support path, and the controller's recursively updated memory. The first is Markov but unavailable, the second is observable but may be insufficient, and the third is implementable but need not be minimal. We fix their notation before imposing the transport structure.
Let $\cS$, $\cO$, and $\cA$ be finite, and consider
$\cP=(\cS,\cO,\cA,P,r,h,\gamma,\rho_0)$ with
$P(\cdot\mid s,a)\in\Delta(\cS)$,
$r:\cS\times\cA\to\mathbb R$ bounded,
$h:\cS\to\cO$, $\gamma\in(0,1)$, and
$\rho_0\in\Delta(\cS)$. Here
$o_t=h(s_t)$ and $s_{t+1}\sim P(\cdot\mid s_t,a_t)$.
For $o\in\cO$, let $\cS_o=h^{-1}(\{o\})$. A length-$t$ observable history is
$
H_t=(o_0,a_0,\ldots,a_{t-1},o_t),
$
and a history-dependent policy is
$
\boldsymbol\pi=(\pi_t)_{t\geq0},
\qquad
\pi_t(\cdot\mid H_t)\in\Delta(\cA).
$
Its discounted return is
$
J(\boldsymbol\pi)
=
\bE^{\boldsymbol\pi}
\left[
\sum_{t\geq0}\gamma^t r(s_t,a_t)
\right].
$
The posterior $b_t(\cdot\mid H_t)$ is supported on $\cS_{o_t}$, but different
histories ending at $o_t$ may induce different posteriors. Thus the visible
state need not be Markov; a full belief is sufficient but generally
continuous and finer than control requires.

The action-labeled observation support graph $X=(\cO,E)$ has
$
E=
\left\{
(o\xrightarrow{a}o'):
\begin{array}{l}
P(s'\mid s,a)>0\text{ for some }s\in\cS_o,\\[-1mm]
s'\in\cS_{o'}
\end{array}
\right\}.
$
For a path $\omega=e_1\cdots e_k$, write $\src(\omega)$ and
$\tgt(\omega)$ for its endpoints and $\omega_2\circ\omega_1$ for execution
of $\omega_1$ followed by $\omega_2$. A support path is latent realizable
only when one sequence $s_j\in\cS_{o_j}$ witnesses all of its edges; edgewise
feasibility alone need not provide such a sequence. Supplementary
Lemma~\ref{lem:history_support_path} relates positive-probability histories
to realizable paths.

A deterministic finite-memory controller $(\cM,\iota,U,\pi_{\cM})$ obeys
$
m_0=\iota(o_0),
m_{t+1}=U(m_t,o_t,a_t,o_{t+1}),
\pi_{\cM}(\cdot\mid o_t,m_t)\in\Delta(\cA).
$
Equivalently, a recursive encoder $F$ satisfies
$
F(H_t\cdot(a_t,o_{t+1}))
=
U(F(H_t),o_t,a_t,o_{t+1}).
$
Histories merged by $F$ therefore remain merged after every common observable continuation, so finite memory induces a finite right congruence. For evaluation, the analyst may augment the latent state to $(s_t,m_t)$, which is Markov under a fixed controller; this does not grant the controller access to $s_t$, and pointwise Bellman maximization over $(s,m)$ would generally be a latent-state oracle. Unless stated otherwise, realized rewards are training signals rather than within-episode controller observations. Reward-aware histories, stochastic emissions, and the formal augmentation results are in Supplementary Appendix~\ref{app:prelim_details}. We use $[n]=\{1,\ldots,n\}$, $S_n=\Sym([n])$, $\Delta(\cX)$ for the simplex on finite $\cX$, and $\mathbf1\{\cdot\}$ for an indicator.

\section{Holonomy-cover decision processes}
\label{sec:model}

This section introduces the transport structure and establishes the minimal exact Markov state of an HCDP. Lemma~\ref{lem:hcdp_path_lifting_main} gives every feasible visible path a unique hidden lift, and Lemma~\ref{lem:stable_fp_main} shows that reward-initialized partition refinement reaches the coarsest stable family in finitely many splits. Proposition~\ref{prop:stable_exact_equiv_main} characterizes exact abstractions, Theorem~\ref{thm:coarsest_exact_main} proves that every exact observation-wise abstraction refines the stable quotient, and Theorem~\ref{thm:markovization_main} turns the observation--class pair into a value-preserving finite MDP. Lemma~\ref{lem:directed_quotient_holonomy_main} identifies the finite permutation group induced by closed directed walks, Proposition~\ref{prop:quotient_belief_main} gives the exact posterior state when the current class is unknown, and Corollary~\ref{cor:min_memory_main} proves that known-class tracking uses the minimal alphabet size $\max_o|\cC_o|$. These results separate the environment-side minimal statistic from the controller-side problem of observing or inferring its current value.

\begin{definition}[Holonomy-cover decision process]
\label{def:hcdp_main}
Fix finite $\cO,\cA$ and $n\geq1$. An HCDP has latent state $\cS=\cO\times[n]$, observation $h(o,i)=o$, base kernel $P_O(\cdot\mid o,a)\in\Delta(\cO)$, an edge permutation $\sigma_{o,a,o'}\in S_n$ whenever $P_O(o'\mid o,a)>0$, mean reward $R:\cO\times[n]\times\cA\to\mathbb R$, discount $\gamma\in(0,1)$, and initial law $\rho_0\in\Delta(\cO\times[n])$. Its transition is
$
P_H((o',i')\mid(o,i),a)
=
P_O(o'\mid o,a)
\mathbf1\{i'=\sigma_{o,a,o'}(i)\}.
$
\end{definition}

The visible successor distribution $P_O(\cdot\mid o,a)$ is independent of the layer, while the realized edge deterministically transports that layer. Thus visible dynamics are Markov although rewards may remain path dependent. For statistical statements we observe
$
Y_t=R(o_t,i_t,a_t)+\xi_t
$
and impose the bounded conditional sub-Gaussian condition in Supplementary
Assumption~\ref{ass:hcdp_noise_main}.

\begin{definition}[Directed path transport]
\label{def:directed_transport_main}
For a feasible edge $e=(o\xrightarrow{a}o')$, write
$\sigma_e:=\sigma_{o,a,o'}$. For $\omega=e_1\cdots e_k$, where
$e_j=(o_{j-1}\xrightarrow{a_{j-1}}o_j)$, define
$
\sigma_\omega
=
\sigma_{e_k}\circ\cdots\circ\sigma_{e_1},
\qquad
\sigma_{\varnothing_o}
=
\operatorname{id}_{[n]}.
$
Hence
$
\sigma_{\omega_2\circ\omega_1}
=
\sigma_{\omega_2}\circ\sigma_{\omega_1};
$
the rightmost map acts first.
\end{definition}

\begin{lemma}[Canonical HCDP path lifting]
\label{lem:hcdp_path_lifting_main}
From any $(o_0,i_0)$, executing the actions of a directed path $\omega$
yields that visible path with probability
$
\prod_{j=1}^k
P_O(o_j\mid o_{j-1},a_{j-1})
>0,
$
and, conditional on it, the terminal layer is
$i_k=\sigma_\omega(i_0)$.
\end{lemma}

Each feasible path therefore has a unique hidden lift from every initial
layer. A fundamental-group interpretation additionally requires paired
inverse edges, as stated in Supplementary
Appendix~\ref{app:hcdp_topological_holonomy}.

\begin{definition}[Stable partition operator]
\label{def:stable_family_main}
A partition family is $\Pi=\{\Pi_o\}_{o\in\cO}$ on $[n]$; write
$i\equiv_o^\Pi j$ and let $\Pi\preceq\Pi'$ mean that every block of $\Pi$ is contained in a block of $\Pi'$; thus $\Pi$ is the finer partition family.
Define
$
i\equiv_o^{\mathcal T(\Pi)}j
\Longleftrightarrow{}
R(o,i,a)=R(o,j,a)
\quad\forall a,
\sigma_{o,a,o'}(i)
\equiv_{o'}^\Pi
\sigma_{o,a,o'}(j),
\forall(a,o'):
P_O(o'\mid o,a)>0.
$
A family is stable when $\mathcal T(\Pi)=\Pi$.
\end{definition}

The reward partition $\Pi^{(0)}$ first separates layers with different immediate control consequences. The operator then propagates those distinctions backward through every feasible one-step continuation, so a merge survives only when it is recursively closed under future quotient evolution. Iteration stops exactly when no current or future reward distinction requires another split.

\begin{lemma}[Finite stabilization]
\label{lem:stable_fp_main}
$\mathcal T$ is monotone. Starting from the reward partition
$\Pi^{(0)}$ and iterating
$
\Pi^{(t+1)}=\mathcal T(\Pi^{(t)})
$
yields $\Pi^{(t+1)}\preceq\Pi^{(t)}$ and stabilizes after at most
$|\cO|(n-1)$ strict refinements at the coarsest stable family $\Pi^\star$.
\end{lemma}

Every strict iteration increases the total number of blocks, which gives the
finite bound. Exactness and fixed-point stability are separated below: an
exact refinement may retain unnecessary distinctions, whereas the fixed point
is maximally compressed.

\begin{definition}[Exact class abstraction]
\label{def:exact_abstraction_main}
Surjections $q_o:[n]\to\widehat{\cC}_o$ are exact if there are functions $\widehat R(o,\cdot,\cdot): \widehat\cC_o\times\cA\to\mathbb R$ and, for every feasible $(o,a,o')$, a map $\widehat\tau_{o,a,o'}: \widehat\cC_o\to\widehat\cC_{o'}$ such that
$
R(o,i,a)
=
\widehat R(o,q_o(i),a),
q_{o'}(\sigma_{o,a,o'}(i))
=
\widehat\tau_{o,a,o'}(q_o(i)).
$
\end{definition}

\begin{proposition}[Exactness criterion]
\label{prop:stable_exact_equiv_main}
The quotient maps of a partition family $\Pi$ define an exact abstraction if
and only if
$
\Pi\preceq\mathcal T(\Pi).
$
Thus every stable family is exact; an exact family is stable precisely when
$\Pi=\mathcal T(\Pi)$.
\end{proposition}

\begin{theorem}[Coarsest exact class abstraction]
\label{thm:coarsest_exact_main}
Let $\cC_o=[n]/\Pi_o^\star$ and $q_o^\star(i)=[i]_o$. Then $q^\star$ is
exact. For every exact $q_o:[n]\to\widehat{\cC}_o$, there is a unique
surjection $\kappa_o:\widehat{\cC}_o\to\cC_o$ such that
$
q_o^\star=\kappa_o\circ q_o.
$
Thus every exact abstraction refines the stable quotient.
\end{theorem}

The factorization allows another exact abstraction to split a stable class
but forbids it from merging layers separated by $\Pi^\star$.

\begin{theorem}[Exact Markovization]
\label{thm:markovization_main}
For every feasible $(o,a,o')$,
$
\tau_{o,a,o'}([i]_o)
=
[\sigma_{o,a,o'}(i)]_{o'},
\qquad
\bar R(o,[i]_o,a)
=
R(o,i,a)
$
are representative independent. Hence
$
\bar\cS
=
\{(o,c):c\in\cC_o\}
$
is an MDP with
$
\bar P((o',c')\mid(o,c),a)
=
P_O(o'\mid o,a)
\mathbf1\{c'=\tau_{o,a,o'}(c)\},
\bar\rho_0(o,c)
=
\sum_{i:[i]_o=c}\rho_0(o,i).
$
For every stationary quotient policy $\bar\pi$, its lift
$
\pi^\uparrow(a\mid o,i)
=
\bar\pi(a\mid o,[i]_o)
$
satisfies
$
V_{\pi^\uparrow}(o,i)
=
\bar V_{\bar\pi}(o,[i]_o);
$
in particular,
$
V^\star_{\rm full}(o,i)
=
\bar V^\star(o,[i]_o).
$
\end{theorem}

Henceforth write
$
c_t:=q_{o_t}^\star(i_t)=[i_t]_{o_t}\in\cC_{o_t}
$
for the canonical stable class of the realized latent layer.
This theorem compresses the fully observed latent HCDP. It does not make the
class observable in the original POMDP; the controller must know, infer, or
maintain it.

For a path $\omega$, let $\tau_\omega$ be the ordered composition of its edge
maps; Supplementary
Corollary~\ref{cor:path_transport_descends_main} gives
$
\tau_\omega([i]_{\src(\omega)})
=
[\sigma_\omega(i)]_{\tgt(\omega)}.
$

\begin{lemma}[Directed quotient holonomy]
\label{lem:directed_quotient_holonomy_main}
For every closed directed walk $\omega$ based at $o$,
$\tau_\omega\in\Sym(\cC_o)$, and
$
\Hol_o^{\rm dir}
=
\{
\tau_\omega:
\omega\text{ closed at }o
\}
$
is a finite subgroup of $\Sym(\cC_o)$.
\end{lemma}

A nonidentity holonomy therefore distinguishes histories that return to the same observation. The quotient solves the structural aggregation problem, but implementability depends on the initial information: a known class propagates deterministically, whereas an unknown class must be represented by a posterior over quotient classes.

\begin{proposition}[Belief reduction over stable classes]
\label{prop:quotient_belief_main}
When reward samples are not within-episode observations, initialize
$
\beta_0(\cdot)=\bP(c_0\in\cdot\mid o_0)
$
and define, for $c\in\cC_{o_t}$,
$
\beta_t(c)
=
\bP(c_t=c\mid H_t).
$
It obeys
$
\beta_{t+1}(c')
=
\sum_{c:
\tau_{o_t,a_t,o_{t+1}}(c)=c'}
\beta_t(c),
\bE[R(o_t,i_t,a_t)\mid H_t,a_t]
=
\sum_c
\beta_t(c)\bar R(o_t,c,a_t).
$
Hence $(o_t,\beta_t)$ is exact; if $\beta_t=\delta_{c_t}$, then
$
c_{t+1}
=
\tau_{o_t,a_t,o_{t+1}}(c_t)
$
tracks the class exactly.
\end{proposition}

If the initial class is unknown, the generic exact information state is the posterior $\beta_t$ on $\cC_{o_t}$. Its ambient simplex is continuous, although a fixed finite HCDP may generate only a finite reachable orbit of beliefs. If the class is initialized correctly, deterministic transport keeps the posterior degenerate and supplies a finite recursive controller state. The reward-observed update and the formal known-class memory lemma are in Supplementary Appendix~\ref{app:deferred_structural}.

\begin{corollary}[Minimal exact class-tracking memory]
\label{cor:min_memory_main}
Let
$
m_{\rm exact}^\star
=
\max_{o\in\cO}|\cC_o|.
$
Among all finite alphabets $\cM$ for which there exist injective encoders
$\eta_o:\cC_o\to\cM$ and an update $U$ satisfying
$
U(\eta_o(c),o,a,o')
=
\eta_{o'}(\tau_{o,a,o'}(c))
$
for every feasible edge and every $c\in\cC_o$, the minimum alphabet
cardinality is
$
|\cM|_{\min}=m_{\rm exact}^\star.
$
The same memory symbols may be reused across different observations.
\end{corollary}

This is a known-class representational statement. Pairwise decision
separation yields an unconditional controller lower bound in
Corollary~\ref{cor:decision_lower_bound_main}.

\section{Identifying classes and transports from diagnostics}
\label{sec:ident}

This section characterizes when the latent stable quotient can be identified. Theorem~\ref{thm:passive_nonident_main} rules out uniform recovery from unrestricted passive logs, whereas Lemma~\ref{lem:finite_distinguishing_witness_main} shows that every pair of distinct stable classes has a finite separating continuation. Under resettable calibrated diagnostics, Theorem~\ref{thm:local_class_inf_main} gives an exponentially decaying local classification error, and Theorem~\ref{thm:transport_local_main} shows that transition-anchored row-wise estimates eventually recover every covered edge map up to observation-wise relabeling. Lemma~\ref{lem:exact_local_tracking_main} proves exact propagation once the initial label and used transports are correct, while Proposition~\ref{prop:eventual_sync_state_main} upgrades summable diagnostic errors and eventual map recovery to permanent class synchronization after an almost surely finite time. Thus the quotient is identified only in local gauges, which is sufficient because gauge-equivalent coordinates define isomorphic quotient MDPs.

The structural identification problem concerns the local label sets and edge maps. Once the local state is synchronized, the base kernel and mean rewards can be estimated from ordinary transition and reward samples; the initial law depends on the chosen restart distribution. For comparison with the canonical quotient, package a full local-coordinate description as label sets $\{\widehat\cC_o\}_{o\in\cO}$, a base kernel $\widehat P_O$, rewards $\widehat R$, edge maps $\{\widehat\tau_e\}_{e\in E}$, and an initial law $\widehat\rho_0$.

\begin{definition}[Observation-wise gauge equivalence]
\label{def:gauge_equiv_main}
A learned description and the canonical quotient are gauge equivalent when
$\widehat P_O=P_O$
and there are bijections
$\lambda_o:\widehat\cC_o\to\cC_o$ such that
$
\widehat R(o,\hat c,a)
=
\bar R(o,\lambda_o(\hat c),a),
\lambda_{o'}\circ\widehat\tau_e
=
\tau_e\circ\lambda_o,
\widehat\rho_0(o,\hat c)
=
\bar\rho_0(o,\lambda_o(\hat c)).
$
\end{definition}

The commuting transport equation makes the two quotient MDPs isomorphic; no
cross-observation agreement between the numerical labels is required.

\begin{theorem}[Passive nonidentifiability]
\label{thm:passive_nonident_main}
There are two finite HCDPs and a common behavior policy whose stable quotient
descriptions are not gauge equivalent but whose complete observable logs
$
(o_0,a_0,Y_0,o_1,\ldots)
$
have the same law. Therefore unrestricted passive logs admit no uniformly
consistent estimator of stable classes, transports, and the quotient MDP over
all HCDPs and behavior policies.
\end{theorem}

The construction omits a distinguishing action. Accordingly, passive HMRL is
conditional, whereas the following result uses resettable diagnostics.

\begin{lemma}[Finite distinguishing witnesses]
\label{lem:finite_distinguishing_witness_main}
Let
$
L_\star
=
\min\{L:\Pi^{(L)}=\Pi^\star\}.
$
For distinct $c,c'\in\cC_o$, some path $\omega$ of length at most $L_\star$
and action $b$ satisfy
$
\bar R(\tgt(\omega),\tau_\omega(c),b)
\ne
\bar R(\tgt(\omega),\tau_\omega(c'),b).
$
\end{lemma}

This supplies finite separating experiments; the statistical interface below
adds repeatability and a quantitative margin.

\begin{definition}[Diagnostic fingerprint]
\label{def:diag_test_main}
Let
$
\cO_+
=
\{o\in\cO:|\cC_o|\geq2\}.
$
If $\cO_+=\varnothing$, every quotient fiber is a singleton and no
classification probe is required. Otherwise, a bounded-horizon protocol $u$
starting from $(o,i)$ returns a scalar $W(o,i;u)$ with mean $G(o,i;u)$.
For each $o\in\cO_+$, choose
$
U_{\rm probe}(o)=\{u_1,\ldots,u_{d_o}\}.
$
Assume $G$ is constant on stable classes, write
$\bar G(o,c;u)=G(o,i;u)$ for any $i\in c$, and define
$
\phi_o(c)
=
(\bar G(o,c;u_\ell))_{\ell=1}^{d_o},
\qquad
\Delta
=
\min_{\substack{o\in\cO_+\\c,c'\in\cC_o,\ c\neq c'}}
\|\phi_o(c)-\phi_o(c')\|_2
>0.
$
Singleton fibers use their unique label and require no probes.
\end{definition}

Assume the interface can restore the same latent checkpoint and generate
conditionally independent repetitions with sub-Gaussian variance proxy
$\nu^2$. It supplies a local
label set $\widehat\cC_o$ and prototypes satisfying
$
\|
\widetilde\phi_o(\hat c)
-
\phi_o(\lambda_o(\hat c))
\|_2
\le
\Delta/8.
$
These assumptions are stated separately in Supplementary
Appendix~\ref{app:deferred_identification}; they constitute calibrated
classification, not unsupervised discovery from passive data.

\begin{definition}[Nearest-prototype inference]
\label{def:emp_fp_main}
With $m$ repetitions per coordinate, set
$
\widehat G(o,i;u_\ell)
=
\frac1m
\sum_{s=1}^m
W^{(s)}(o,i;u_\ell),
\widehat\phi_o(i)
=
(\widehat G(o,i;u_\ell))_{\ell=1}^{d_o},
\widehat c
\in
\argmin_{\hat c\in\widehat\cC_o}
\|
\widehat\phi_o(i)
-
\widetilde\phi_o(\hat c)
\|_2,
$
with deterministic tie breaking.
\end{definition}

\begin{theorem}[Local class inference]
\label{thm:local_class_inf_main}
Under Assumptions~\ref{ass:diag_class_stable_main}, \ref{ass:fp_sep_main}, \ref{ass:diag_repeat_main}, \ref{ass:diag_subg_main}, and \ref{ass:local_proto_main}, for $o\in\cO_+$, $i\in[n]$, and $c=[i]_o$, conditional on the restored start state $s_{\rm start}=(o,i)$,
$
\bP(
\lambda_o(\widehat c)\ne c
\mid
s_{\rm start}=(o,i)
)
\le
2d_o
\exp\!\left(
-\frac{m\Delta^2}{128\nu^2d_o}
\right).
$
\end{theorem}

For $o\notin\cO_+$, the inferred class is deterministically correct without diagnostic sampling.
A logarithmic, summable repetition schedule therefore makes only finitely many
classification errors almost surely; see Supplementary
Corollaries~\ref{cor:local_sample_complexity_main}--%
\ref{cor:eventual_local_class_main}.

Local class labels alone do not determine how labels propagate between
observations. Anchored source--target diagnostics estimate this map without
requiring a common numerical naming of classes.

\begin{assumption}[Anchored edge diagnostics]
\label{ass:transition_anchor_main}
For a sampled edge $e=(o\xrightarrow{a}o')$, the interface preserves the exact source checkpoint, executes the edge, and diagnoses the paired source and target checkpoints.
\end{assumption}

\begin{definition}[Row-wise transport estimator]
\label{def:rowwise_transport_main}
Let $N_e[\hat c,\hat c']$ count paired inferred labels and set
$
\widehat\tau_e(\hat c)
\in
\argmax_{\hat c'\in\widehat\cC_{o'}}
N_e[\hat c,\hat c'],
$
with deterministic tie breaking.
\end{definition}

The true map in local coordinates is
$
\tau_e^\lambda
=
\lambda_{o'}^{-1}
\circ\tau_e
\circ\lambda_o;
$
it need not be injective, so recovery is row-wise.

\begin{theorem}[Eventual transport recovery]
\label{thm:transport_local_main}
Fix the gauges from Assumption~\ref{ass:local_proto_main}. Under Assumptions~\ref{ass:transition_anchor_main} and \ref{ass:edge_cov_main}, if only finitely many source or target classification errors occur almost surely, then for every feasible $e=(o\xrightarrow{a}o')$ and $\hat c\in\widehat\cC_o$, there is an almost surely finite time after which, permanently,
$
\widehat\tau_e(\hat c)
=
\lambda_{o'}^{-1}
(
\tau_e(\lambda_o(\hat c))
).
$
\end{theorem}

\begin{definition}[Local-coordinate class tracker]
\label{def:local_class_tracker_main}
At an exogenous initialization, diagnostics choose
$\widehat c_{t_0}\in\widehat\cC_{o_{t_0}}$. Until the next reset,
$
\widehat c_{t+1}
=
\widehat\tau_{(o_t\xrightarrow{a_t}o_{t+1})}
(\widehat c_t).
$
A reset discards the propagated label and triggers a new diagnostic
initialization.
\end{definition}

\begin{lemma}[Exact local tracking]
\label{lem:exact_local_tracking_main}
On a reset-free interval, if
$
\lambda_{o_{t_0}}(\widehat c_{t_0})=c_{t_0}
$
and every used edge satisfies
$
\widehat\tau_e
=
\lambda_{\tgt(e)}^{-1}
\circ\tau_e
\circ\lambda_{\src(e)},
$
then
$
\lambda_{o_t}(\widehat c_t)=c_t
$
for all $t\ge t_0$ on that interval.
\end{lemma}

\begin{proposition}[Eventual exact quotient tracking]
\label{prop:eventual_sync_state_main}
If exogenous initialization times are almost surely unbounded, their diagnostic error probabilities are summable, every edge--class transport entry used by the tracker is eventually correct, and the tracker uses the preceding recursion, then for fixed local gauges there is an almost surely finite $T$ such that
$
\lambda_{o_t}(\widehat c_t)=c_t
$
for every $t\ge T$.
\end{proposition}

All formal access, coverage, estimator, and exact-tracking conditions are stated in Supplementary Appendix~\ref{app:deferred_identification}.

\section{Quotient-lifted reinforcement learning}
\label{sec:algo}

This section converts synchronized local class labels into a valid state for an RL backbone. Theorem~\ref{thm:eventual_exact_reduction_main} proves that, whenever the current label and all relevant edge maps are correct, the learned observation--label process has the exact quotient transition law, conditional reward mean, and noise control, and is isomorphic to the canonical quotient MDP. Corollary~\ref{cor:diag_to_exact_reduction_main} combines the diagnostic recovery results with this theorem to obtain eventual ambient correctness, while emphasizing that eventual correctness alone does not automatically justify an arbitrary learning theorem. Corollary~\ref{cor:backbone_transfer_main} gives the required transfer condition: restart the finite-MDP backbone at a learned-information stopping time using only post-calibration data, or use a guarantee explicitly robust to the finite corrupted prefix and induced algorithmic state. This distinguishes the proved diagnostic HMRL reduction from passive variants, which remain conditional because of Theorem~\ref{thm:passive_nonident_main}.
Let
$
\widehat\cS
=
\bigsqcup_o
(\{o\}\times\widehat\cC_o)
$
and
$
\widehat z_t
=
(o_t,\widehat c_t).
$
The proof-level coordinate map is
$
\Lambda_\lambda(o,\hat c)
=
(o,\lambda_o(\hat c)).
$

\begin{definition}[Local quotient MDP]
\label{def:quotient_lifted_state_main}
The transport $\tau_e^\lambda$ is the exact canonical map written in local
coordinates, whereas $\widehat\tau_e$ is its learned estimator. The hats on
$\widehat P^\lambda$ and $\widehat R^\lambda$ indicate the local state space,
not statistical estimation. Define
$
\tau_e^\lambda
=
\lambda_{\tgt(e)}^{-1}
\circ\tau_e
\circ\lambda_{\src(e)},
$
$
\widehat R^\lambda(o,\hat c,a)
=
\bar R(o,\lambda_o(\hat c),a),
$
and, for $\hat z=(o,\hat c)$ and $\hat z'=(o',\hat c')$,
$
\widehat P^\lambda(\hat z'\mid\hat z,a)
=
\begin{cases}
P_O(o'\mid o,a)
\mathbf 1\!\left\{
\hat c'=\tau_{o,a,o'}^\lambda(\hat c)
\right\},
& P_O(o'\mid o,a)>0,\\[1mm]
0,
& P_O(o'\mid o,a)=0.
\end{cases}
$
\end{definition}

The local initial law is
$
\widehat\rho_0^\lambda(o,\hat c)
=
\bar\rho_0(o,\lambda_o(\hat c)).
$
Through $\Lambda_\lambda$, stationary policies and values transfer by
$
\bar\pi(a\mid o,c)
=
\widehat\pi(a\mid o,\lambda_o^{-1}(c))
$
and
$
\widehat V_{\widehat\pi}(o,\hat c)
=
\bar V_{\bar\pi}(o,\lambda_o(\hat c)).
$
Thus the learned labels may be used directly; the gauges are proof devices,
not algorithmic inputs.

\begin{theorem}[Exact reduction after synchronization]
\label{thm:eventual_exact_reduction_main}
Let $\mathcal G_t^-$ be the pre-transition sigma-field generated by the
latent and observable histories, the current tracker and transport tables, and
the chosen action $a_t$, before $(Y_t,o_{t+1})$ is observed. Let
$\widehat{\mathcal F}_t^-\subseteq\mathcal G_t^-$ be the corresponding
learned-information subfiltration, which excludes the latent class.

Let $A_t$ be the event that
$
\lambda_{o_t}(\widehat c_t)=c_t
$
and, for every $o'$ with $P_O(o'\mid o_t,a_t)>0$,
$
\widehat\tau_{o_t,a_t,o'}
=
\lambda_{o'}^{-1}
\circ
\tau_{o_t,a_t,o'}
\circ
\lambda_{o_t}.
$
On $A_t$, for every $\hat z'\in\widehat\cS$,
$
\bP(
\widehat z_{t+1}=\hat z'
\mid
\mathcal G_t^-
)
=
\widehat P^\lambda(
\hat z'
\mid
\widehat z_t,a_t
),
$
and
$
\bE[
Y_t
\mid
\mathcal G_t^-
]
=
\widehat R^\lambda(
\widehat z_t,a_t
).
$
The centered reward remains conditionally sub-Gaussian with variance proxy
$\sigma_Y^2$, and $\Lambda_\lambda$ is an MDP isomorphism to the canonical
quotient.

If an almost surely finite $T_0$ satisfies $A_t$ for every $t\geq T_0$,
these laws hold after $T_0$ relative to the ambient filtration
$(\mathcal G_t^-)$. If, in addition, $T_0$ is a stopping time for
$(\widehat{\mathcal F}_t^-)$, then the same laws hold relative to the
post-$T_0$ learned-information filtration. The distribution at $T_0$ need not
equal $\widehat\rho_0^\lambda$.
\end{theorem}

The theorem identifies the transition kernel, conditional reward mean, and
martingale noise. A stationary reward-output kernel additionally requires the
full conditional law of $Y_t$ to depend only on the quotient state and action.

\begin{corollary}[Diagnostics imply eventual exact reduction]
\label{cor:diag_to_exact_reduction_main}
Under Theorem~\ref{thm:transport_local_main} and
Proposition~\ref{prop:eventual_sync_state_main}, fixed gauges and an almost
surely finite $T_0$ satisfy $A_t$ for every $t\geq T_0$, giving eventual
ambient correctness in
Theorem~\ref{thm:eventual_exact_reduction_main}. A finite-MDP backbone may be
invoked directly only after a certified calibration-and-restart stopping time,
as in Corollary~\ref{cor:pac_calibration_main}, or under a separate theorem
formulated for the ambient filtration and robust to a finite corrupted prefix.
\end{corollary}

\begin{corollary}[Transfer of an MDP-backbone guarantee]
\label{cor:backbone_transfer_main}
Suppose there is an almost surely finite stopping time $T$ for the learned-information filtration such that $A_t$ holds for every $t\geq T$. Any finite-MDP guarantee transfers with its original exploration, sampling, step-size, and approximation conditions if the backbone is restarted at $T$ using only post-$T$ data and its theorem applies to the resulting initial-state law, or uniformly over initial states. The same conclusion also holds when the backbone theorem is explicitly invariant to the finite pre-$T$ prefix and the induced internal algorithmic state. Without one of these conditions, eventual representation correctness alone is insufficient.
\end{corollary}

The finite-sample calibration corollary, diagnostic and passive algorithms,
and conditional passive reduction are in Supplementary
Appendix~\ref{app:algo_details}. In the passive regime, exact reduction is
therefore conditional on eventual class correctness under fixed local gauges;
Theorem~\ref{thm:passive_nonident_main} precludes an unconditional claim from
unrestricted logs.

\section{Nonabelian loop memory barrier}
\label{sec:nonabelian}
This section establishes when loop order is intrinsically control relevant. Theorem~\ref{thm:abelian_barrier_main} proves that any terminal memory factoring through signed loop counts assigns the same action distribution to equal-count histories and therefore incurs a fixed continuation-value loss when their transported classes are decision separated; Corollary~\ref{cor:commutator_barrier_main} applies this obstruction to a zero-count commutator. Corollary~\ref{cor:decision_lower_bound_main} removes the count-based architectural restriction and shows that pairwise decision separation forces any exact finite-memory controller to distinguish every reachable stable class. Proposition~\ref{prop:min_nonabelian_main} realizes the gap in a three-layer, two-action HCDP, while Proposition~\ref{prop:ordered_tracking_nonabelian_main} shows that ordered transport tracking is exact using only the current finite class rather than the complete loop word. Proofs and construction details are given in Appendices~\ref{app:proofs} and~\ref{app:nonabelian_details}.

Fix a base observation $o$.

\begin{definition}[Executable reversible loop alphabet]
\label{def:reversible_loop_alphabet_main}
A family of directed closed walks
$\{\ell_j^+,\ell_j^-\}_{j=1}^r$ based at $o$ is reversible when
$
\tau_{\ell_j^-}
=
\tau_{\ell_j^+}^{-1}.
$
The inverse transport need not reverse physical edges: every permutation has finite order, so its inverse is a positive power and may be executed by repeated positive traversal.
\end{definition}

These loops induce a representation of the free group
$
F_r
=
\langle x_1,\ldots,x_r\rangle,
$
$
\rho_o:F_r\to\Sym(\cC_o),
\qquad
\rho_o(x_j^{\pm1})
=
\tau_{\ell_j^{\pm}},
$
with the algebraic composition convention
$
\rho_o(uv)=\rho_o(u)\circ\rho_o(v),
$
so the loop represented by $v$ is executed first and the rightmost transport acts first. For reachable $c_0$, write

\begin{definition}[Abelianized loop memory]
\label{def:abelianized_memory_main}
The abelianization
$
\ab:F_r\to\mathbb Z^r
$
records signed generator counts. A surrogate $S:F_r\to\cM$ is abelianized
when
$
S=f\circ\ab,
$
equivalently
$
\ab(w)=\ab(w')
\Longrightarrow
S(w)=S(w').
$
\end{definition}

It includes signed-count vectors and arbitrary downstream readouts, but not
general order-sensitive recurrent or attention encoders.

Let $\bar Q^\star$ be the quotient optimal action value and
$
\cA^\star(o,c)
=
\argmax_a
\bar Q^\star(o,c,a).
$

\begin{definition}[Decision separation]
\label{def:decision_separated_main}
Classes $c,c'\in\cC_o$ are decision separated when
$
\cA^\star(o,c)
\cap
\cA^\star(o,c')
=
\varnothing.
$
For $p\in\Delta(\cA)$ and $d\in\{c,c'\}$, let
$
\ell(p,d)
=
\bar V^\star(o,d)
-
\sum_a
p(a)\bar Q^\star(o,d,a),
$
and define
$
\varepsilon_{\rm dec}(o;c,c')
=
\min_{p\in\Delta(\cA)}
\max_{d\in\{c,c'\}}
\ell(p,d)
>0.
$
\end{definition}

The strict positivity follows because $\Delta(\cA)$ is compact, $\ell(\cdot,d)$ is continuous, and the two disjoint optimal-action faces admit no common zero-loss distribution.

Noncommuting raw transports alone do not imply a control obstruction: the
commutator must act differently on reachable stable classes, and those
classes must demand incompatible decisions.

\begin{theorem}[Abelianization barrier]
\label{thm:abelian_barrier_main}
Suppose $\ab(w)=\ab(w')$ but $c_w$ and $c_{w'}$ are decision separated. Any controller whose terminal memory has the form
$
m(w)=g_o(\ab(w)),
\qquad
g_o:\mathbb Z^r\to\cM,
$
uses the same action distribution after both histories; at least one therefore
incurs conditional optimality loss at least
$
\varepsilon_{\rm dec}(o;c_w,c_{w'}).
$
Such memory cannot be exact on a reachable history set containing both
histories.
\end{theorem}

\begin{corollary}[Commutator barrier]
\label{cor:commutator_barrier_main}
If a reachable $c_0$ is decision separated from
$\rho_o([u,v])(c_0)$, then every abelianized memory merges the commutator
$[u,v]$ with the empty word because both have abelianization zero, and cannot
be exact at both resulting classes.
\end{corollary}

In particular, a commutator and the empty word both have zero abelianization,
although the commutator may act nontrivially on stable classes. Any encoder
$
A(w)
=
\sum_j
N_j(w)v_j
$
followed by an arbitrary readout also factors through $\ab$ and inherits the
same barrier. The formal count-additive corollary is in Supplementary Appendix~\ref{app:nonabelian_details}. The preceding result targets count-based encoders. The next corollary removes that architectural restriction: pairwise decision separation forces any exact finite-memory controller to distinguish all reachable classes at the observation.

\begin{corollary}[Arbitrary finite-memory lower bound]
\label{cor:decision_lower_bound_main}
Fix $o$. If every distinct pair in $\cC_o$ is decision separated and every
class is reached by a positive-probability history ending at $o$, then any
finite-memory controller attaining $\bar V^\star(o_t,c_t)$ after every such
history satisfies
$
|\cM|
\ge
|\cC_o|.
$
At an observation attaining $m_{\rm exact}^\star$, this gives
$
|\cM|
\ge
m_{\rm exact}^\star.
$
\end{corollary}

\begin{proposition}[Three-layer nonabelian HCDP]
\label{prop:min_nonabelian_main}
With one observation, actions $a,b$, $n=3$, transports
$
\sigma_a=(12),
\qquad
\sigma_b=(123),
$
and, suppressing the fixed observation in $R$, reward vectors
$
(R(i,a),R(i,b))_{i=1}^3
=
((1,0),(0,1),(0,0)),
$
the stable quotient is discrete and its transports generate $S_3$. An
executable commutator sends class $1$ to class $2$; for
$0<\gamma<1/2$, their unique optimal actions are $a$ and $b$, respectively.
\end{proposition}

Thus the obstruction is finite.

\begin{proposition}[Ordered stable-class tracking]
\label{prop:ordered_tracking_nonabelian_main}
Starting from known $c$, exact tracking maps $w$ to $\rho_o(w)(c)$ and
therefore distinguishes any $w,w'$ with different transported classes, even
when
$
\ab(w)=\ab(w').
$
It stores only the finite current class, not the complete word.
\end{proposition}

The executable construction and count-additive specialization are in
Supplementary Appendix~\ref{app:nonabelian_details}.

\section{Experiments}
\label{sec:exp}

We evaluate four finite consequences of the theory in controlled HCDPs: exact stable-quotient recovery, concentration of diagnostic identification, failure of commutative count memory under nonabelian transport, and quotient-level control with finite memory. These experiments are intended as falsifiable implementation checks under known model assumptions rather than evidence that arbitrary POMDP benchmarks admit an exact finite quotient. Complete environment definitions, protocols, metrics, robustness sweeps, and hyperparameters are given in Appendix~\ref{app:exp_details}.

\paragraph{Environments and methods.}
\textbf{ChainCover} is an acyclic depth-$D$ construction in which three reward-relevant classes are revealed only at a terminal query, forcing stable refinement to propagate distinctions backward while merging a nuisance coordinate. \textbf{LoopGuess} contains two visible loop excursions with quotient transports $\alpha=(12)$ and $\beta=(123)$, which generate the nonabelian group $S_3$; in particular, same-count words such as $\alpha\beta$ and $\beta\alpha$ can induce different terminal classes. All control methods use the same tabular Q-learning backbone and differ only in their state representation: observation only (Obs-Q), commutative loop counts (Count-Q), the complete loop word (History-Q), the privileged raw latent state (Raw-Q), the true quotient class (Quotient-Q), or the class propagated by calibrated transports (HMRL-D). HMRL-D uses the simulator-provided calibrated diagnostic interface, completes transport calibration, and then restarts Q-learning, as required by Corollary~\ref{cor:backbone_transfer_main}.

\paragraph{Structural and diagnostic checks.}
For ChainCover with $D=6$ and nuisance size $b=8$, exhaustive refinement compresses $216$ raw states to the analytical $25$ quotient states and stabilizes at round $7$. The recovered quotient has zero reward inconsistency, zero successor-class inconsistency, and zero optimal-value gap, while all $24$ one-pair strict coarsenings violate exactness. Rejecting all tested one-pair merges provides a local irreducibility check beyond mere sufficiency and complements the analytical coarseness result. In the noisy-prototype diagnostic experiment, both class and edge-transport errors decrease with repeated probes and are zero in the reported trials by $m=8$ at the primary noise level $\sigma_{\mathrm{diag}}=0.5$; see Figure~\ref{fig:hmrl_main_validation}. These results separately verify quotient construction and the calibrated interface used by HMRL-D. Complete structural and diagnostic measurements are reported in Appendices~\ref{app:exact_audit}--\ref{app:transport_recovery_exp}.

\begin{figure}[t]
    \centering
    \includegraphics[width=0.485\linewidth]{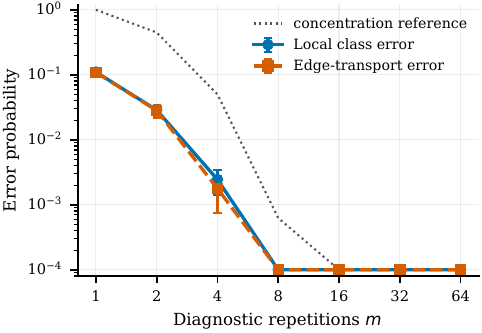}
    \hfill
    \includegraphics[width=0.485\linewidth]{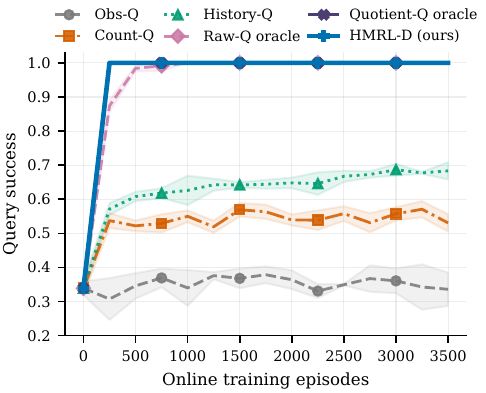}
    \caption{\textbf{Diagnostic recovery and online control.}
    Left: repeated probes reduce class and edge-transport errors at
    $\sigma_{\mathrm{diag}}=0.5$, with zero observed error by $m=8$.
    Right: after calibration and restart, HMRL-D matches the quotient-oracle
    learning curve, whereas less informative or uncompressed representations
    remain limited under the same online budget.
    shading denotes one standard deviation
    where visible.}
    \label{fig:hmrl_main_validation}
\end{figure}

\begin{figure}[t]
    \centering
    \includegraphics[width=0.485\linewidth]{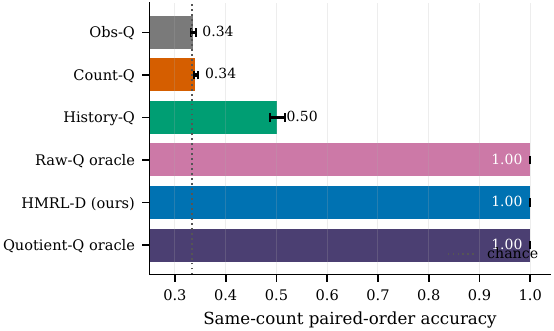}
    \hfill
    \includegraphics[width=0.485\linewidth]{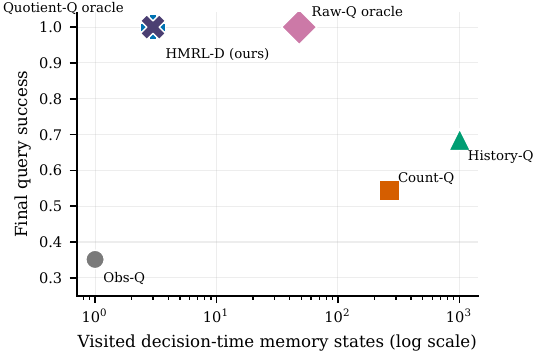}
    \caption{\textbf{Ordered-memory necessity and memory efficiency.}
    Left: on equal-count histories with different ordered transports,
    Count-Q remains near the $1/3$ chance level, while HMRL-D and the exact
    oracles attain perfect accuracy.
    Right: HMRL-D matches quotient-oracle success using three decision-time
    states, compared with $48$ for Raw-Q and more than $10^3$ visited states
    for History-Q.}
    \label{fig:hmrl_order_memory}
\end{figure}

\begin{table}[t]
\centering
\caption{\textbf{Primary LoopGuess comparison.}
Values are mean $\pm$ standard deviation.
Oracle rows are references; bold marks the best deployable result.}
\label{tab:loopguess_main}
\resizebox{\columnwidth}{!}{%
\begin{tabular}{lccc}
\toprule
Method
& Final success
& Paired-order accuracy
& Decision-time states \\
\midrule
Obs-Q
& 0.351 $\pm$ 0.055
& 0.336 $\pm$ 0.005
& 1.0 $\pm$ 0.0 \\
Count-Q
& 0.543 $\pm$ 0.014
& 0.342 $\pm$ 0.003
& 266.7 $\pm$ 6.5 \\
History-Q
& 0.686 $\pm$ 0.007
& 0.502 $\pm$ 0.014
& 1005.3 $\pm$ 16.6 \\
\midrule
Raw-Q oracle
& 1.000 $\pm$ 0.000
& 1.000 $\pm$ 0.000
& 48.0 $\pm$ 0.0 \\
\midrule
HMRL-D (ours)
& \textbf{1.000 $\pm$ 0.000}
& \textbf{1.000 $\pm$ 0.000}
& \textbf{3.0 $\pm$ 0.0} \\
Quotient-Q oracle
& 1.000 $\pm$ 0.000
& 1.000 $\pm$ 0.000
& 3.0 $\pm$ 0.0 \\
\bottomrule
\end{tabular}}
\end{table}

\paragraph{Control results.}
The primary LoopGuess setting uses nuisance multiplicity $b=16$ and query probability $p_{\mathrm{query}}=0.1$. HMRL-D matches Quotient-Q in final success, normalized learning-curve AUC ($0.976\pm0.001$), and paired-order accuracy while using only three decision-time class states. Raw-Q also reaches perfect final success but retains all $48$ class--nuisance combinations and has lower AUC ($0.966\pm0.002$). Count-Q improves ordinary success over Obs-Q but remains at chance on same-count order pairs, directly exhibiting the information loss predicted by the nonabelian barrier. History-Q retains order in principle, but its visited decision-time memory exceeds $10^3$ states and learns more slowly under the fixed budget. HMRL-D requires $216$ calibration interactions, reported separately from online control. Figures~\ref{fig:hmrl_main_validation}
and~\ref{fig:hmrl_order_memory}, together with
Table~\ref{tab:loopguess_main}, provide the primary evidence;
Appendix~\ref{app:exp_full_results} reports complete metrics and robustness
sweeps over nuisance multiplicity and expected history length.

\section{Conclusion}
\label{sec:conclusion}

We presented HMRL and identified, for holonomy-cover decision processes, the stable observation--class quotient as the minimal Markov sufficient statistic among exact observation-wise abstractions. This quotient preserves precisely the reward and transition distinctions required for Bellman control while discarding irrelevant latent differences. Under a resettable, locally calibrated diagnostic interface, HMRL identifies local classes and quotient transports up to observation-wise relabeling and applies standard finite-MDP RL to the resulting state. We further showed that commutative summaries such as loop counts fail when hidden transports are nonabelian. Experiments verified exact state compression, quotient recovery, and order-sensitive control: ChainCover was reduced from $216$ states to $25$ exact quotient states, while HMRL achieved $1.000$ success and perfect paired-order accuracy on LoopGuess. ChainCover and LoopGuess thus test structural exactness and order-sensitive control, respectively. Extending quotient discovery to noisy, continuous, and passive-data settings remains an important direction.

\nocite{*}
\bibliography{ref}

\clearpage
\appendix

\section{Positioning relative to prior work}
\label{app:positioning_table}

Table~\ref{tab:positioning} compares the scope of formal guarantees rather than empirical expressive power. A checkmark denotes a general guarantee under the method-specific assumptions cited in Section~\ref{sec:intro}; a parenthesized checkmark denotes a guarantee available only for particular variants or additional assumptions; $\times$ denotes no general guarantee; and ``--'' means that the property is not the method's primary target. Observation aliasing refers to distinct latent states sharing the same current observation while differing in rewards or future behavior.

\begin{table*}[t]
\centering
\scriptsize
\caption{Guarantee-level comparison of memory and state representations for control under partial observability.}
\label{tab:positioning}
\setlength{\tabcolsep}{3.0pt}
\renewcommand{\arraystretch}{1.10}
\begin{tabular}{@{}>{\raggedright\arraybackslash}p{0.19\textwidth}cccccc>{\raggedright\arraybackslash}p{0.19\textwidth}@{}}
\toprule
Approach
& \makecell{Obs.\\aliasing}
& \makecell{Exact\\Markov state}
& \makecell{Minimal\\finite state}
& \makecell{Order\\sensitive}
& \makecell{Unknown\\initial mode}
& \makecell{Structural\\recovery}
& Main requirement \\
\midrule
Belief-state planning
& $\checkmark$ & $\checkmark$ & $\times$ & $\checkmark$ & $\checkmark$ & $\times$ & known model and initial belief \\
\hdashline
Recurrent or Transformer memory
& $\checkmark$ & $\times$ & $\times$ & $\checkmark$ & $(\checkmark)$ & $\times$ & architecture and training data \\
\hdashline
Linear-recurrent or SSM memory
& $\checkmark$ & $\times$ & $\times$ & $\checkmark$ & $(\checkmark)$ & $\times$ & architecture and training data \\
\hdashline
Bisimulation or homomorphisms
& $\times$ & $\checkmark$ & $(\checkmark)$ & -- & -- & $(\checkmark)$ & an available Markov state \\
\hdashline
Rich-observation state discovery
& $\times$ & $\checkmark$ & $\checkmark$ & -- & -- & $\checkmark$ & decodable observations \\
\hdashline
PSRs or revealing POMDPs
& $\checkmark$ & $\checkmark$ & $\times$ & $\checkmark$ & $\checkmark$ & $\checkmark$ & rank or revealing conditions \\
\hdashline
Reward machines
& $(\checkmark)$ & $(\checkmark)$ & $\times$ & $\checkmark$ & $\times$ & $(\checkmark)$ & event-labeling interface \\
\hdashline
RDP learning
& $\checkmark$ & $\checkmark$ & $\checkmark$ & $\checkmark$ & $\times$ & $\checkmark$ & finite regularity and coverage \\
\hdashline
\textbf{HMRL (ours)}
& $\checkmark$ & $\checkmark$ & $\checkmark$ & $\checkmark$ & $\checkmark$ & $\checkmark$ & resettable calibrated diagnostics \\
\bottomrule
\end{tabular}
\end{table*}

\section{Deferred technical preliminaries}
\label{app:prelim_details}

We relate observable histories to latent-realizable support paths and formalize the right congruence induced by finite memory. We then treat reward-aware histories, deterministic and stochastic observation augmentations, and the Bellman equation for a fixed implementable controller, keeping controller evaluation separate from latent-state oracle optimization.

\begin{lemma}[Observable histories and realizable support paths]
\label{lem:history_support_path}
Let $\boldsymbol{\pi}$ be any history-dependent policy.

First, if an observable history
\[
H_t
=
(o_0,a_0,o_1,\ldots,a_{t-1},o_t)
\]
has positive probability under $\boldsymbol{\pi}$, then its associated visible path
\[
\omega(H_t)
=
(o_0\xrightarrow{a_0}o_1)
\cdots
(o_{t-1}\xrightarrow{a_{t-1}}o_t)
\]
is a $\rho_0$-realizable directed support path in $X$.

Conversely, suppose that $\omega(H_t)$ is $\rho_0$-realizable and that
\[
\pi_j(a_j\mid H_j)>0,
\qquad j=0,\ldots,t-1,
\]
along the corresponding observable prefixes. Then $H_t$ has positive probability under $\boldsymbol{\pi}$.

In particular, an arbitrary directed support path in $X$ need not be latent-realizable, because the edge-wise existential witnesses in the definition of $E$ need not be compatible with one another.
\end{lemma}

\paragraph{Finite-memory history quotients.}
Recall the deterministic history encoder $F$ induced by a finite-memory controller in Section~\ref{sec:prelim}. The following result formalizes the recursive indistinguishability imposed by any finite memory state.

\begin{lemma}[Finite memory induces a right-congruent history quotient]
\label{lem:memory_right_congruence_main}
For histories $H,H'$ ending at the same observation $o$, define
$H\sim_{\cM}H'$ if and only if $F(H)=F(H')$. If
$H\sim_{\cM}H'$, then
\[
\pi_{\cM}(\cdot\mid o,F(H))
=
\pi_{\cM}(\cdot\mid o,F(H'))
\]
and, for every common observable extension $(a,o')$,
\[
H\cdot(a,o')
\sim_{\cM}
H'\cdot(a,o').
\]
Hence equivalence is preserved under every finite common observable continuation. For each terminal observation, the quotient has at most $|\cM|$ classes, and jointly over terminal observations it has at most $|\cO||\cM|$ classes.
\end{lemma}

\paragraph{Reward-aware histories and transition outputs.}
The main text uses observation--action histories
\[
H_t
=
(o_0,a_0,o_1,\ldots,a_{t-1},o_t)
\]
as the controller's within-episode information. If the reward observed after action $a_t$ is also allowed to influence the next action, the appropriate history is
\[
H_t^Y
=
(o_0,a_0,Y_0,o_1,\ldots,a_{t-1},Y_{t-1},o_t).
\]
A reward-aware finite-memory controller therefore uses an update of the form
\[
U_Y:
\cM\times\cO\times\cA\times\mathcal Y\times\cO
\to\cM,
\]
with
\[
m_{t+1}
=
U_Y(m_t,o_t,a_t,Y_t,o_{t+1}),
\]
where $\mathcal Y$ is the reward-observation space.

In the deterministic-reward model,
\[
Y_t=r(s_t,a_t).
\]
It is tempting to define an augmented observation by
\[
\bar o_{t+1}=(o_{t+1},Y_t).
\]
However, $\bar o_{t+1}$ is generally not a deterministic function of $s_{t+1}$, because $Y_t$ depends on the preceding state--action pair $(s_t,a_t)$. Therefore this notation does not, by itself, define another deterministic-observation POMDP on the original latent state space.

One exact finite-state reformulation augments the latent state with the preceding state--action pair. Let
\[
\cR:=r(\cS\times\cA)
\]
be the finite reward range and let $\bot$ be a distinguished initial symbol. Define
\[
\bar{\cS}
=
\cS\times
\bigl(
\{\bot\}\cup(\cS\times\cA)
\bigr)
\]
and
\[
\bar{\cO}
=
\cO\times
\bigl(
\{\bot\}\cup\cR
\bigr).
\]
The initial augmented state is $(s_0,\bot)$ with distribution
\[
\bar\rho_0(s,\bot)=\rho_0(s).
\]
For an augmented state $(s,\eta)$ and action $a$, sample
\[
s'\sim P(\cdot\mid s,a)
\]
and set the next augmented state to
\[
(s',(s,a)).
\]
Equivalently,
\[
\bar P\bigl((s',(s,a))\mid(s,\eta),a\bigr)
=
P(s'\mid s,a),
\]
with zero probability assigned to every other next augmented state.
Define the deterministic observation map by
\[
\bar h(s,\bot)
=
(h(s),\bot)
\]
and
\[
\bar h(s',(s,a))
=
\bigl(h(s'),r(s,a)\bigr).
\]
At decision time $t\geq 1$, the augmented observation contains the current visible observation $o_t$ and the reward $r_{t-1}$ produced by the preceding transition. Thus this augmented POMDP generates exactly the same reward-aware information sequence as $H_t^Y$.

For stochastic reward observations, it is usually cleaner to use a transition-output kernel. Let
$
G(\mathrm{d}y\mid s,a,s')
$
denote the conditional law of the observed reward output after the transition from $s$ to $s'$ under action $a$. One step then consists of
$
s_{t+1}\sim P(\cdot\mid s_t,a_t),
\qquad
Y_t\sim G(\cdot\mid s_t,a_t,s_{t+1}),
\qquad
o_{t+1}=h(s_{t+1}).
$
The controller updates its memory from the complete observable transition
$
(o_t,a_t,Y_t,o_{t+1}).
$

The distinction between mean-reward abstraction and reward-observation abstraction is important. If reward samples are used only to estimate expected return, an exact quotient need only preserve the conditional mean reward. If reward samples are themselves used as online observations for subsequent control, then two states may be merged only if they induce the same relevant reward-output law. More generally, the quotient must preserve the joint conditional law of
$
(Y_t,o_{t+1},c_{t+1})
$
given the current quotient state and action. Equality of conditional means alone is insufficient when the realized reward can reveal additional hidden-state information.


\paragraph{Latent augmentation for deterministic observations.}
The next result records the analyst's Markov augmentation for a fixed finite-memory controller. It is an evaluation construction and does not grant the controller access to the latent state.

\begin{lemma}[Latent augmented Markov property]
\label{lem:memory_aug_mdp_main}
For a fixed controller $(\cM,\iota,U,\pi_{\cM})$, define an action-conditioned process on $\cS\times\cM$ by
$
\widetilde P((s',m')\mid(s,m),a)
=
P(s'\mid s,a)
\mathbf{1}
\left\{
m'=U(m,h(s),a,h(s'))
\right\},
$
with reward and initial distribution
$
\widetilde r((s,m),a)=r(s,a),
\qquad
\widetilde\rho_0(s,m)
=
\rho_0(s)\mathbf{1}\{m=\iota(h(s))\}.
$
Under
$
\widetilde\pi(a\mid s,m)
=
\pi_{\cM}(a\mid h(s),m),
$
the process $(s_t,m_t)$ is Markov and has exactly the same expected discounted return as the original finite-memory controller.
\end{lemma}

\paragraph{Stochastic emissions.}
A stochastic-observation POMDP replaces the deterministic map
\[
h:\cS\to\cO
\]
by an emission kernel
\[
Z(\cdot\mid s)\in\Delta(\cO)
\]
and is written as
\[
\cP
=
(\cS,\cO,\cA,P,r,Z,\gamma,\rho_0).
\]
The deterministic fiber $\cS_o$ is replaced by the support fiber
\[
\cS_o^Z
:=
\{s\in\cS:Z(o\mid s)>0\}.
\]
The observation support graph has edge set
$
E_Z
:=
\{
(o,a,o')\in\cO\times\cA\times\cO:
\exists\,s,s'\in\cS\text{ such that}
Z(o\mid s)>0,\;
P(s'\mid s,a)>0,\;
Z(o'\mid s')>0
\}.
$
As in the deterministic case, this graph records local visible support and may contain composable paths that admit no globally consistent latent realization.

There is one important change to the latent augmentation. Because the current observation $o_t$ is no longer determined by $s_t$, the Markov analysis state for a finite-memory controller is
$
(s_t,o_t,m_t),
$
rather than merely $(s_t,m_t)$. Given $(s,o,m)$ and action $a$, its action-conditioned transition kernel is
$
\widehat P
\bigl(
(s',o',m')
\mid
(s,o,m),a
\bigr)
=
P(s'\mid s,a)
Z(o'\mid s')
\mathbf{1}
\left\{
m'=U(m,o,a,o')
\right\}.
$
The corresponding initial distribution is
\[
\widehat\rho_0(s,o,m)
=
\rho_0(s)
Z(o\mid s)
\mathbf{1}
\left\{
m=\iota(o)
\right\}.
\]
Thus the conceptual separation remains unchanged: the analyst can construct a latent Markov augmentation, while the controller has access only to the observable coordinates $(o,m)$.

\paragraph{Bellman evaluation for a fixed finite-memory controller.}
We now record the Bellman equation corresponding to Lemma~\ref{lem:memory_aug_mdp_main}. The result evaluates a fixed implementable controller; it does not optimize over policies that observe the latent state.

\begin{lemma}[Bellman evaluation for a fixed finite-memory controller]
\label{lem:finite_memory_bellman_eval_app}
Fix a finite-memory controller
\[
(\cM,\iota,U,\pi_{\cM})
\]
in the deterministic-observation POMDP of Section~\ref{sec:prelim}. Define the operator $T_{\pi_{\cM},U}$ on bounded functions
\[
V:\cS\times\cM\to\mathbb{R}
\]
by
$
(T_{\pi_{\cM},U}V)(s,m)
=
\sum_{a\in\cA}
\pi_{\cM}(a\mid h(s),m)
\Biggl[
r(s,a)
+
\gamma
\sum_{s'\in\cS}
P(s'\mid s,a)
V\Bigl(
s',
U(m,h(s),a,h(s'))
\Bigr)
\Biggr].
$
Then $T_{\pi_{\cM},U}$ is a $\gamma$-contraction under the sup norm and has a unique fixed point
\[
V_{\pi_{\cM},U}.
\]
The expected discounted return from the original initial distribution is
\[
J(\pi_{\cM},U,\iota)
=
\sum_{s_0\in\cS}
\rho_0(s_0)
V_{\pi_{\cM},U}
\bigl(
s_0,\iota(h(s_0))
\bigr).
\]
\end{lemma}

\begin{remark}[Why latent-state Bellman optimality is an oracle problem]
A pointwise Bellman optimality operator on $\cS\times\cM$ would choose
\[
\argmax_{a\in\cA}
Q((s,m),a)
\]
separately for each latent state $s$. Such an operator permits different actions at two latent states $s$ and $\tilde s$ satisfying
$
h(s)=h(\tilde s)
$
and having the same memory value $m$. The resulting policy observes the latent state and is therefore not, in general, implementable by a finite-memory controller whose action must be a function only of $(h(s),m)$.

For a fixed memory update $U$, admissible policy optimization remains coupled across all latent states with the same observable pair $(o,m)$. It is not equivalent to unrestricted statewise Bellman maximization on $\cS\times\cM$. The quotient theory in the main text resolves this obstruction by constructing an observation--memory state on which rewards and future class dynamics are independent of the hidden representative.
\end{remark}

The concentration tools used later to convert finite-sample diagnostic guarantees into eventual almost-sure correctness are recorded next.

\begin{definition}[$\sigma^2$-sub-Gaussian random variable]
A random variable $\xi$ is called $\sigma^2$-sub-Gaussian if
\[
\bE\!\left[\exp\!\bigl(\lambda(\xi-\bE[\xi])\bigr)\right]
\le
\exp(\lambda^2\sigma^2/2)
\qquad
\text{for all }\lambda\in\mathbb{R}.
\]
\end{definition}

\begin{lemma}[Empirical mean concentration]
\label{lem:subg_mean_app}
Let $\xi_1,\dots,\xi_m$ be i.i.d.\ $\sigma^2$-sub-Gaussian random variables with mean $\mu$. Then for any $\varepsilon>0$,
\[
\Pr\!\left(
\left|
\frac{1}{m}\sum_{i=1}^m \xi_i-\mu
\right|
\ge \varepsilon
\right)
\le
2\exp\!\left(-\frac{m\varepsilon^2}{2\sigma^2}\right).
\]
\end{lemma}

\begin{corollary}[Coordinate-wise vector concentration]
\label{cor:vector_subg_mean_app}
Let $X_1,\dots,X_m\in\mathbb{R}^d$ be i.i.d.\ random vectors with mean $\mu\in\mathbb{R}^d$, and assume that each coordinate $(X_i)_j$ is $\sigma^2$-sub-Gaussian. Then for every $\varepsilon>0$,
\[
\Pr\!\left(
\left\|
\frac{1}{m}\sum_{i=1}^m X_i-\mu
\right\|_\infty
\ge \varepsilon
\right)
\le
2d\exp\!\left(-\frac{m\varepsilon^2}{2\sigma^2}\right).
\]
\end{corollary}

\begin{corollary}[Finite-family uniform concentration]
\label{cor:finite_family_union_app}
Let $\{\hat\mu_t^\alpha:\alpha\in\cI_t\}$ be a finite family of empirical estimators indexed by a finite set $\cI_t$. Suppose that for each fixed $\alpha\in\cI_t$,
\[
\Pr\!\left(
\|\hat\mu_t^\alpha-\mu^\alpha\|_\infty\ge \varepsilon_t
\right)
\le \delta_t.
\]
Then
\[
\Pr\!\left(
\exists\,\alpha\in\cI_t:
\|\hat\mu_t^\alpha-\mu^\alpha\|_\infty\ge \varepsilon_t
\right)
\le
|\cI_t|\delta_t.
\]
In particular, if $\sum_{t\ge 1}|\cI_t|\delta_t<\infty$, then the corresponding uniform failure event occurs only finitely often almost surely.
\end{corollary}

\begin{lemma}[Summable failure implies eventual correctness]
\label{lem:bc_wrapper_main}
Let $(E_t)_{t\ge 1}$ be a sequence of events. If
\[
\sum_{t\ge 1}\Pr(E_t)<\infty,
\]
then
\[
\Pr(E_t\ \mathrm{i.o.})=0.
\]
In particular, if at round $t$ a finite identification procedure fails with probability at most $\delta_t$ and $\sum_t \delta_t<\infty$, then only finitely many rounds are incorrect almost surely.
\end{lemma}

\subsection{Finite-sample and eventual diagnostic consequences}
\label{app:local_diagnostic_consequences}

The exponential classification bound in Theorem~\ref{thm:local_class_inf_main} yields the following fixed-confidence and almost-sure consequences.

\begin{corollary}[Diagnostic sample complexity]
\label{cor:local_sample_complexity_main}
For any non-singleton fiber $\cC_o$ and any $\delta\in(0,1)$, it suffices that
\[
m
\geq
\frac{128\nu^2d_o}{\Delta^2}
\log\!\left(
\frac{2d_o}{\delta}
\right)
\]
to guarantee
\[
\Pr\!\left(
\lambda_o(\widehat c)\neq[i]_o
\middle|\,
s_{\mathrm{start}}=(o,i)
\right)
\leq\delta.
\]
For a singleton fiber, the error probability is zero without diagnostic sampling.
\end{corollary}

\begin{corollary}[Summable schedules imply eventual local correctness]
\label{cor:eventual_local_class_main}
Fix a non-singleton fiber $o\in\cO_+$. At successive classification events
$k=2,3,\ldots$ for $o$, suppose that, for some $\eta>0$,
\[
m_{o,k}
\geq
\frac{128\nu^2d_o}{\Delta^2}
\bigl[
(1+\eta)\log k+\log(2d_o)
\bigr].
\]
Then
\[
\Pr\!\left(
\lambda_o(\widehat c_{o,k})\neq c_{o,k}
\right)
\leq
k^{-(1+\eta)}.
\]
Hence the error probabilities are summable, so the first Borel--Cantelli lemma implies that almost surely only finitely many classification errors occur at $o$. If every required observation follows such a schedule, then the finiteness of $\cO$ implies almost surely only finitely many classification errors system-wide.
\end{corollary}

\section{Deferred structural and identification statements}
\label{app:deferred_structural}

This appendix states the complete technical conditions and secondary results
summarized in Sections~\ref{sec:model} and~\ref{sec:ident}.

\subsection{Stable-quotient details}

\begin{assumption}[Bounded mean rewards and reward noise]
\label{ass:hcdp_noise_main}
There exist constants $r_{\max}>0$ and $\sigma_Y>0$ such that
$|R(o,i,a)|\leq r_{\max}$ for all
$(o,i,a)\in\cO\times[n]\times\cA$.
For statistical identification, the observed reward is
$Y_t=R(o_t,i_t,a_t)+\xi_t$.
Let $\mathcal G_t^-$ be the sigma-algebra generated by the latent and
observable histories, every algorithmic state measurable from that history,
and the current action, immediately before observing $(Y_t,o_{t+1})$.
Then, for every $\lambda\in\mathbb{R}$,
\[
\bE[\xi_t\mid\mathcal G_t^-]=0,
\qquad
\bE\!\left[
e^{\lambda\xi_t}\mid\mathcal G_t^-
\right]
\leq
\exp\!\left(
\frac{\lambda^2\sigma_Y^2}{2}
\right).
\]
\end{assumption}

\begin{corollary}[Path transport descends to stable classes]
\label{cor:path_transport_descends_main}
For a directed path $\omega=e_1\cdots e_k$ from $o$ to $o'$, define
$\tau_\omega:=\tau_{e_k}\circ\cdots\circ\tau_{e_1}$.
Then
\[
\tau_\omega([i]_o)
=
[\sigma_\omega(i)]_{o'}
\]
for every $i\in[n]$.
\end{corollary}

\begin{lemma}[Exact quotient-class memory update]
\label{lem:quotient_memory_update_main}
If the current class $c_t\in\cC_{o_t}$ is known, then the update
\[
c_{t+1}:=\tau_{o_t,a_t,o_{t+1}}(c_t)
\]
makes $(o_t,c_t)$ evolve according to the quotient MDP of
Theorem~\ref{thm:markovization_main}. Hence a correctly initialized or
inferred class is a recursively maintainable finite memory state.
\end{lemma}

\subsection{Diagnostic assumptions and quotient descriptions}
\label{app:deferred_identification}

\begin{definition}[Stable quotient description]
\label{def:quotient_description_main}
The stable quotient description is
\[
\mathfrak Q
=
\bigl(
\{\cC_o\}_{o\in\cO},
P_O,
\bar R,
\{\tau_e\}_{e\in E},
\bar\rho_0
\bigr),
\]
where $\cC_o=[n]/\Pi_o^\star$, $P_O$ is the observation kernel,
$\bar R(o,c,a)$ is the class-level mean reward,
$\tau_e:\cC_o\to\cC_{o'}$ for every feasible edge
$e=(o\xrightarrow{a}o')$, and
\[
\bar\rho_0(o,c)=\sum_{i:[i]_o=c}\rho_0(o,i).
\]
\end{definition}

\begin{assumption}[Class-stable diagnostic means]
\label{ass:diag_class_stable_main}
For every $o\in\cO$, $u\in U_{\mathrm{probe}}(o)$, and $i,j\in[n]$ with
$[i]_o=[j]_o$,
\[
G(o,i;u)=G(o,j;u).
\]
\end{assumption}

\begin{definition}[Class fingerprint]
\label{def:class_fp_main}
For
$U_{\mathrm{probe}}(o)=\{u_1,\ldots,u_{d_o}\}$,
define
\[
\phi_o(c)
:=
\bigl(
\bar G(o,c;u_\ell)
\bigr)_{\ell=1}^{d_o}
\in\mathbb R^{d_o}.
\]
Singleton fibers require no diagnostic coordinate; otherwise $d_o\geq1$.
\end{definition}

\begin{assumption}[Fingerprint separability]
\label{ass:fp_sep_main}
Each $\phi_o$ is injective, and, omitting singleton fibers,
\[
\Delta
:=
\min_{\substack{o\in\cO,\;c,c'\in\cC_o\\c\neq c'}}
\left\|
\phi_o(c)-\phi_o(c')
\right\|_2
>0.
\]
\end{assumption}

\begin{assumption}[Resettable diagnostic repeatability]
\label{ass:diag_repeat_main}
For every admissible $(o,i,u)$ and $m\geq1$, the interface can restore the
same latent start state and produce conditionally independent repetitions
\[
W^{(1)}(o,i;u),\ldots,W^{(m)}(o,i;u)
\]
with the conditional law of $W(o,i;u)$. The value of $i$ need not be
revealed.
\end{assumption}

\begin{assumption}[Uniform sub-Gaussian diagnostic noise]
\label{ass:diag_subg_main}
There exists $\nu>0$ such that, for every admissible $(o,i,u)$,
repetition $s$, and $\eta\in\mathbb R$,
$
\bE\!\left[
\exp\!\left(
\eta\bigl[W^{(s)}(o,i;u)-G(o,i;u)\bigr]
\right)
\middle|
s_{\mathrm{start}}=(o,i)
\right]
\leq
\exp\!\left(\frac{\eta^2\nu^2}{2}\right).
$
\end{assumption}

\begin{assumption}[Locally calibrated class prototypes]
\label{ass:local_proto_main}
For every $o\in\cO_+$, there are a finite local label set
$\widehat\cC_o$, a bijection
$\lambda_o:\widehat\cC_o\to\cC_o$, and prototypes satisfying
\[
\max_{\hat c\in\widehat{\cC}_o}
\left\|
\widetilde\phi_o(\hat c)
-
\phi_o\bigl(\lambda_o(\hat c)\bigr)
\right\|_2
\leq
\frac{\Delta}{8}.
\]
For $o\notin\cO_+$, set $\widehat\cC_o$ to a singleton and let
$\lambda_o$ be its unique bijection to $\cC_o$; no prototype is required.
\end{assumption}

\begin{assumption}[Coverage of edge--class pairs]
\label{ass:edge_cov_main}
For every feasible $e=(o\xrightarrow{a}o')$ and $c\in\cC_o$, the protocol
produces infinitely many transition-anchored samples with true source class
$c$ and realized edge $e$.
\end{assumption}

\section{Directed transport and topological holonomy}
\label{app:hcdp_topological_holonomy}

The main theory is formulated using executable directed paths. This is the appropriate control-theoretic object because a formal reverse traversal of an observation edge need not correspond to an action that the environment can execute. A representation of the fundamental group of an undirected graph requires additional inverse structure. This appendix states that structure explicitly and shows how both raw-layer and quotient-layer topological holonomy follow from it.

Let
\[
X=(\cO,E)
\]
be the directed action-labeled support graph of an HCDP. For each
\[
e\in E,
\]
write
\[
\src(e)\in\cO,
\qquad
\tgt(e)\in\cO
\]
for its source and target, and let
\[
\sigma_e\in S_n
\]
be its raw layer transport.

\begin{assumption}[Paired inverse-consistent directed support]
\label{ass:inverse_consistency_app}
The directed edge set $E$ is equipped with an involution
\[
e\longmapsto\bar e
\]
such that, for every $e\in E$,
\[
\overline{\bar e}=e,
\]
\[
\src(\bar e)=\tgt(e),
\qquad
\tgt(\bar e)=\src(e),
\]
and
\[
\sigma_{\bar e}=\sigma_e^{-1}.
\]
The edges $e$ and $\bar e$ may carry different action labels; the notation $\bar e$ denotes an inverse edge, not an inverse element of the action set.
\end{assumption}

Under Assumption~\ref{ass:inverse_consistency_app}, define
\[
X_{\mathrm{und}}
\]
to be the undirected multigraph containing one undirected edge for each pair
\[
\{e,\bar e\}.
\]
An oriented traversal of this undirected edge is represented by either $e$ or $\bar e$. Immediate backtracking has identity transport because
\[
\sigma_{\bar e}\circ\sigma_e
=
\Id
\]
and
\[
\sigma_e\circ\sigma_{\bar e}
=
\Id.
\]

A based loop in $|X_{\mathrm{und}}|$ may be represented by an oriented edge word
\[
w=e_1e_2\cdots e_k.
\]
Reduction of the word removes adjacent pairs of the form
\[
e\bar e
\qquad\text{or}\qquad
\bar e e.
\]
Define its raw transport by
\[
\sigma_w
=
\sigma_{e_k}\circ\cdots\circ\sigma_{e_1}.
\]

\begin{proposition}[Raw topological holonomy]
\label{prop:raw_topological_holonomy_app}
Under Assumption~\ref{ass:inverse_consistency_app}, for every base observation $o\in\cO$, the assignment
\[
\Holo_o^{\mathrm{raw}}:
\pi_1(|X_{\mathrm{und}}|,o)
\to S_n,
\qquad
\Holo_o^{\mathrm{raw}}([w])
=
\sigma_w,
\]
is a well-defined group homomorphism.
Its image
\[
G_o^{\mathrm{raw}}
:=
\Holo_o^{\mathrm{raw}}
\bigl(
\pi_1(|X_{\mathrm{und}}|,o)
\bigr)
\leq S_n
\]
is the raw topological holonomy group at $o$.
\end{proposition}

We next pass from raw layers to the stable quotient classes
\[
\cC_o=[n]/\Pi_o^\star.
\]
For every directed edge
\[
e:o\to o',
\]
let
\[
\tau_e:\cC_o\to\cC_{o'}
\]
be the quotient transport defined in Theorem~\ref{thm:markovization_main}.

\begin{lemma}[Inverse consistency descends to the stable quotient]
\label{lem:quotient_inverse_consistency_app}
Under Assumption~\ref{ass:inverse_consistency_app}, the quotient transports along paired edges satisfy
\[
\tau_{\bar e}\circ\tau_e
=
\Id_{\cC_{\src(e)}}
\]
and
\[
\tau_e\circ\tau_{\bar e}
=
\Id_{\cC_{\tgt(e)}}.
\]
Consequently,
\[
\tau_{\bar e}
=
\tau_e^{-1},
\]
and every quotient transport along an edge of $X_{\mathrm{und}}$ is a bijection.
\end{lemma}

For an oriented loop word
\[
w=e_1e_2\cdots e_k
\]
based at $o$, define
\[
\tau_w
=
\tau_{e_k}\circ\cdots\circ\tau_{e_1}
:
\cC_o\to\cC_o.
\]

\begin{proposition}[Quotient topological holonomy]
\label{prop:quotient_topological_holonomy_app}
Under Assumption~\ref{ass:inverse_consistency_app}, for every base observation $o\in\cO$, the assignment
$
\Holo_o^{\mathrm{quot}}:
\pi_1(|X_{\mathrm{und}}|,o)
\to\Sym(\cC_o),
\qquad
\Holo_o^{\mathrm{quot}}([w])
=
\tau_w,
$
is a well-defined group homomorphism.
Its image
\[
G_o^{\mathrm{quot}}
:=
\Holo_o^{\mathrm{quot}}
\bigl(
\pi_1(|X_{\mathrm{und}}|,o)
\bigr)
\leq
\Sym(\cC_o)
\]
is the quotient topological holonomy group at $o$.
\end{proposition}

\begin{proposition}[Change of base observation]
\label{prop:topological_basepoint_change_app}
Assume Assumption~\ref{ass:inverse_consistency_app}. Let $o,o'\in\cO$ lie in the same connected component of $X_{\mathrm{und}}$, and let $\eta$ be an oriented path from $o$ to $o'$. Then
\[
G_{o'}^{\mathrm{raw}}
=
\sigma_\eta
G_o^{\mathrm{raw}}
\sigma_\eta^{-1}
\]
and
\[
G_{o'}^{\mathrm{quot}}
=
\tau_\eta
G_o^{\mathrm{quot}}
\tau_\eta^{-1}.
\]
Thus changing the base observation changes the holonomy group only by conjugation through the transport along the chosen reference path.
\end{proposition}

The directed theory in the main text does not require Assumption~\ref{ass:inverse_consistency_app}. Without paired inverse edges, executable directed closed walks still induce the directed quotient holonomy monoid---and, in the finite closed-walk setting of Lemma~\ref{lem:directed_quotient_holonomy_main}, a finite permutation group---but the transport cannot in general be interpreted as a representation of the fundamental group of an undirected graph. The topological formulation is therefore a strengthened special case of the directed control-theoretic formulation, not a prerequisite for the stable quotient or for exact Markovization.

\subsection{Optional tree-based gauge fixing}
\label{app:tree_gauge_fixing}
Observation-wise local coordinates are already sufficient for exact control. A spanning-tree synchronization step is useful only when one wants to express several observation fibers relative to one chosen root coordinate system, for example when visualizing transport matrices or comparing loop compositions.

\begin{proposition}[Optional gauge fixing on a bijective tree scaffold]
\label{prop:sync_tree_main}
Fix a connected component of the underlying undirected observation graph and choose a root observation
\[
o_{\mathrm{root}}.
\]
Suppose there exists an undirected spanning tree $T$ such that, for each tree adjacency $\{o,o'\}$, one has selected a feasible directed edge in one of the two directions whose quotient transport is bijective.

After the corresponding local edge maps have been recovered exactly, an arbitrary labeling of the root class set can be propagated uniquely along the tree, using a recovered bijection when traversing it in its selected direction and its inverse when traversing it in the opposite direction. This produces a root-relative coordinate convention at every observation.

The resulting convention is unique relative to the chosen root labeling, the selected directed tree edges, and the tree $T$. Every off-tree edge can then be expressed in the root-relative coordinates and yields a cycle-consistency or holonomy diagnostic.
\end{proposition}

Proposition~\ref{prop:sync_tree_main} requires bijective tree transports and therefore may fail when neighboring observations have different numbers of stable classes. This does not obstruct exact quotient control, because Definition~\ref{def:quotient_lifted_state_main} already provides the required global state space in observation-wise local coordinates, while Lemma~\ref{lem:gauge_policy_transfer_main} establishes policy and value transfer between local and canonical coordinates. Tree synchronization is an optional gauge choice, not an identifiability requirement.

\section{Algorithmic templates and implementation details}
\label{app:algo_details}

This appendix gives concrete wrappers for the two regimes in Section~\ref{sec:algo}. The diagnostic wrapper uses resettable class fingerprints and transition-anchored transport calibration. The passive wrapper uses a history encoder and predictive consistency but carries no unconditional quotient-recovery guarantee. Both algorithms use observation-wise local states
$
(o,\widehat c),
$
rather than unknown canonical labels. Observation-wise gauge maps appear only in the correctness proofs.

\subsection{Diagnostic quotient-lifted reinforcement learning}

The diagnostic interface provides the locally calibrated label sets and
prototypes from Assumption~\ref{ass:local_proto_main}. HMRL-D first runs a
diagnostic calibration phase that estimates every required edge--class row.
It then freezes the recovered maps, reinitializes the RL backbone, diagnoses
each episode's initial class, and propagates that class between resets. This
separation implements Corollary~\ref{cor:backbone_transfer_main}.

\begin{algorithm}[t]
\caption{HMRL-D: diagnostic quotient-lifted RL in local coordinates}
\label{alg:hmrl_diag_app}
\begin{algorithmic}[1]
\STATE \textbf{Input:}
local label sets $\{\widehat{\cC}_o\}_{o\in\cO}$;
calibrated prototypes $\{\widetilde\phi_o\}_{o\in\cO}$;
diagnostic protocols $\{U_{\mathrm{probe}}(o)\}_{o\in\cO}$;
transition-anchored diagnostic interface;
edge-count updater;
RL backbone $(\mathrm{RLSelect},\mathrm{RLUpdate})$.
\STATE Initialize local edge-count matrices $\{N_e\}_{e\in E}$ and local edge maps $\{\widehat\tau_e\}_{e\in E}$.
\STATE Run transition-anchored diagnostics until every required edge--class
row satisfies the chosen calibration criterion.
\STATE Freeze $\{\widehat\tau_e\}_{e\in E}$ and reinitialize the RL backbone,
including its parameters, replay data, counters, and optimizer state.
\STATE Reset the environment and obtain a resettable checkpoint with initial observation $o_0$.
\STATE Diagnose the checkpoint and infer
\[
\widehat c_0\in\widehat{\cC}_{o_0}.
\]
\FOR{$t=0,1,2,\ldots$}
    \STATE Form the local-coordinate RL state
    \[
    \widehat z_t=(o_t,\widehat c_t).
    \]
    \STATE Select
    \[
    a_t\leftarrow\mathrm{RLSelect}(\widehat z_t).
    \]
    \STATE Preserve or clone the source checkpoint when an edge-calibration sample is scheduled.
    \STATE Execute $a_t$ and observe
    \[
    (Y_t,o_{t+1},\mathrm{done}).
    \]
    \IF{a post-calibration diagnostic audit is scheduled}
        \STATE Diagnose the preserved source checkpoint and the resulting target checkpoint and record whether the frozen map is consistent; do not update the RL state or the frozen transport table.
    \ENDIF
    \STATE Propagate
        \[
            \widehat c_{t+1}
            =
            \widehat\tau_{e_t}(\widehat c_t).
        \]
    \STATE Form
    \[
    \widehat z_{t+1}
    =
    (o_{t+1},\widehat c_{t+1}).
    \]
    \STATE Run
    \[
    \mathrm{RLUpdate}
    (
    \widehat z_t,
    a_t,
    Y_t,
    \widehat z_{t+1},
    \mathrm{done}
    ).
    \]
    \IF{$\mathrm{done}$}
        \STATE Reset the environment and obtain a new resettable checkpoint.
        \STATE Diagnose the new initial checkpoint to obtain
        \[
        (o_{t+1},\widehat c_{t+1}).
        \]
    \ENDIF
\ENDFOR
\end{algorithmic}
\end{algorithm}

Algorithm~\ref{alg:hmrl_diag_app} never uses the unknown proof-level gauges
$\lambda_o$. Because the displayed wrapper calibrates first and restarts the
backbone, its online state stream is isomorphic to the canonical quotient-state
stream on the calibration-success event of Corollary~\ref{cor:pac_calibration_main}. An interleaved implementation is also possible, but then a backbone theorem applies only when it is explicitly robust to a finite corrupted prefix, as stated in
Corollary~\ref{cor:backbone_transfer_main}.

The finite pre-identification portion of the RL updates may contain non-Markov state labels. A backbone theorem applies directly only if the RL component is restarted after a certified diagnostic phase or if its convergence theorem is robust to a finite number of incorrect early updates, as stated in Corollary~\ref{cor:backbone_transfer_main}.

\subsection{Passive quotient-lifted reinforcement learning}

The passive regime has no latent checkpoint, no calibrated fingerprints, and no oracle class count. Its local labels must therefore be treated as learned representation variables rather than identified quotient classes.

Under the main information convention of this paper, the online encoder uses the observation--action history
\[
H_t
=
(o_0,a_0,o_1,\ldots,a_{t-1},o_t).
\]
Reward samples may be used as training targets but are not fed back as within-episode observations. A reward-aware variant may instead use $H_t^Y$, but then the stronger reward-output abstraction described in Appendix~\ref{app:prelim_details} is required.

\begin{definition}[Passive local-memory learner]
\label{def:passive_learner_app}
A passive local-memory learner maintains:
\begin{enumerate}
    \item adaptive local label sets
    \[
    \{\widehat{\cC}_o\}_{o\in\cO};
    \]

    \item a history encoder
    \[
    q_\psi(\widehat c_t\mid H_t),
    \qquad
    \widehat c_t\in\widehat{\cC}_{o_t};
    \]

    \item a class-conditioned mean-reward model
    \[
    \widehat R_\theta(o,\widehat c,a);
    \]

    \item a base observation model
    \[
    \widehat P_{O,\theta}(o'\mid o,a),
    \]
    which is not conditioned on $\widehat c$;

    \item a local edge-transition model
    \[
    \widehat T_\theta
    (
    \widehat c'
    \mid
    o,\widehat c,a,o'
    );
    \]

    \item a split--merge or compression mechanism for adapting the local label sets.
\end{enumerate}
\end{definition}

A generic passive objective may be written as
$
\mathcal L
=
\mathcal L_{\mathrm{rew}}
+
\lambda_{\mathrm{base}}
\mathcal L_{\mathrm{base}}
+
\lambda_{\mathrm{edge}}
\mathcal L_{\mathrm{edge}}
+
\lambda_{\mathrm{ord}}
\mathcal L_{\mathrm{ord}}
+
\lambda_{\mathrm{comp}}
\mathcal L_{\mathrm{comp}}.
$
Here:

\begin{itemize}
    \item $\mathcal L_{\mathrm{rew}}$ measures class-conditioned reward-prediction error;

    \item $\mathcal L_{\mathrm{base}}$ estimates the class-independent kernel $P_O$;

    \item $\mathcal L_{\mathrm{edge}}$ encourages deterministic consistency of inferred labels across each realized edge;

    \item $\mathcal L_{\mathrm{ord}}$ compares ordered compositions of learned edge maps on observed paths and loops;

    \item $\mathcal L_{\mathrm{comp}}$ discourages unnecessary local classes.
\end{itemize}

The ordered-composition loss must preserve temporal order. Replacing it by a commutative edge-count loss can erase the nonabelian distinctions studied in Section~\ref{sec:nonabelian}.

\begin{algorithm}[t]
\caption{A-HMRL: passive quotient-lifted RL from standard tuples}
\label{alg:ahmrl_passive_app}
\begin{algorithmic}[1]
\STATE \textbf{Input:}
history encoder $q_\psi$;
models $(\widehat R_\theta,\widehat P_{O,\theta},\widehat T_\theta)$;
label-refinement routine $\mathrm{RefineLabels}$;
replay buffer $\mathcal B$;
RL backbone $(\mathrm{RLSelect},\mathrm{RLUpdate})$.
\STATE Initialize one or more local labels for each encountered observation and initialize all model and RL parameters.
\STATE Reset the environment, observe $o_0$, and set
\[
H_0=(o_0).
\]
\FOR{$t=0,1,2,\ldots$}
    \STATE Infer
    \[
    q_\psi(\cdot\mid H_t)
    \quad\text{on}\quad
    \widehat{\cC}_{o_t},
    \]
    and select or sample a current local label $\widehat c_t$.
    \STATE Form
    \[
    \widehat z_t=(o_t,\widehat c_t).
    \]
    \STATE Select
    \[
    a_t\leftarrow\mathrm{RLSelect}(\widehat z_t).
    \]
    \STATE Execute $a_t$ and observe
    \[
    (Y_t,o_{t+1},\mathrm{done}).
    \]
    \STATE Update the observation--action history
    \[
    H_{t+1}
    =
    H_t\cdot(a_t,o_{t+1}).
    \]
    \STATE Infer a provisional next label
    \[
    \widehat c_{t+1}
    \sim
    q_\psi(\cdot\mid H_{t+1}).
    \]
    \STATE Store
    \[
    (
    H_t,
    \widehat c_t,
    a_t,
    Y_t,
    o_{t+1},
    \widehat c_{t+1}
    )
    \]
    and posterior information in $\mathcal B$.
    \STATE Update the representation models by minimizing
    $
    \mathcal L_{\mathrm{rew}}
    +
    \lambda_{\mathrm{base}}
    \mathcal L_{\mathrm{base}}
    +
    \lambda_{\mathrm{edge}}
    \mathcal L_{\mathrm{edge}}
    +
    \lambda_{\mathrm{ord}}
    \mathcal L_{\mathrm{ord}}
    +
    \lambda_{\mathrm{comp}}
    \mathcal L_{\mathrm{comp}}.
    $
    \STATE Periodically run $\mathrm{RefineLabels}$ to split persistently incompatible labels and merge predictively equivalent labels.
    \STATE Form
    \[
    \widehat z_{t+1}
    =
    (o_{t+1},\widehat c_{t+1})
    \]
    and run
    \[
    \mathrm{RLUpdate}
    (
    \widehat z_t,
    a_t,
    Y_t,
    \widehat z_{t+1},
    \mathrm{done}
    ).
    \]
    \IF{$\mathrm{done}$}
        \STATE Reset the environment, observe a new initial observation, and reinitialize the history.
    \ENDIF
\ENDFOR
\end{algorithmic}
\end{algorithm}

Algorithm~\ref{alg:ahmrl_passive_app} does not call a synchronization routine. Because the observation is part of the state, the local label sets already define one global disjoint-union state space. The algorithm also carries no unconditional identification guarantee. Its exact interpretation is the conditional one in Theorem~\ref{thm:passive_conditional_reduction_main}.

\subsection{Auxiliary quotient-lifting results}

The following results formalize policy/value invariance under observation-wise gauge changes and the conditional exact-reduction claim used by the passive learner.

\begin{lemma}[Policy and value transfer under local gauges]
\label{lem:gauge_policy_transfer_main}
Let $\widehat\pi$ be a stationary policy on $\widehat{\cS}$, and define
\[
\bar\pi(a\mid o,c)
:=
\widehat\pi\bigl(a\mid o,\lambda_o^{-1}(c)\bigr).
\]
Then, for every $(o,\hat c)\in\widehat{\cS}$,
\[
\widehat V_{\widehat\pi}(o,\hat c)
=
\bar V_{\bar\pi}\bigl(o,\lambda_o(\hat c)\bigr),
\qquad
\widehat V^\star(o,\hat c)
=
\bar V^\star\bigl(o,\lambda_o(\hat c)\bigr).
\]
Hence $\widehat\pi$ is optimal if and only if its gauge-transformed policy
$\bar\pi$ is optimal.
\end{lemma}

\begin{corollary}[Finite-sample calibration]
\label{cor:pac_calibration_main}
Assume $\cO_+\neq\varnothing$; if every fiber is a singleton, calibration is
trivial. Fix
$\delta_{\rm proto},\delta_{\rm cls}\in(0,1)$, and suppose the
prototype-accuracy condition of
Assumption~\ref{ass:local_proto_main} holds on an event
$\mathcal E_{\rm proto}$ with
\[
\bP(\mathcal E_{\rm proto})
\geq
1-\delta_{\rm proto}.
\]
Let $K$ be the total number of nontrivial diagnostic classifications performed
up to the backbone's horizon, counting each initial-class inference and both
members of every transition-anchored source--target pair. Suppose the
calibration count tables start at zero. Conditional on
$\mathcal E_{\rm proto}$, if every classification uses
\[
m
\geq
\frac{128\nu^2d}{\Delta^2}
\log\!\left(
\frac{2dK}{\delta_{\rm cls}}
\right),
\qquad
d:=\max_{o\in\cO_+}d_o,
\]
repetitions per coordinate, and every feasible edge--class pair receives at
least one anchored sample during calibration, then, with probability at least
\[
1-\delta_{\rm proto}-\delta_{\rm cls},
\]
the prototypes and every classification are correct. On that event, every
recovered map satisfies
\[
\widehat\tau_e
=
\lambda_{\tgt(e)}^{-1}
\circ
\tau_e
\circ
\lambda_{\src(e)},
\]
every episode is initialized and tracked exactly, and the synchronization event
of Theorem~\ref{thm:eventual_exact_reduction_main} holds at every
post-calibration time. Consequently, any finite-MDP guarantee for a backbone
restarted after calibration holds with its failure probability increased by
at most
$\delta_{\rm proto}+\delta_{\rm cls}$, provided that guarantee applies to the
post-restart initial-state law.
\end{corollary}

\begin{theorem}[Conditional quotient reduction in the passive regime]
\label{thm:passive_conditional_reduction_main}
Suppose a passive learner produces
$\widehat c_t\in\widehat{\cC}_{o_t}$, and let
$\widehat z_t=(o_t,\widehat c_t)$. Assume that there exist fixed bijections
$\lambda_o:\widehat{\cC}_o\to\cC_o$ and an almost surely finite stopping time
$T_0$ for the ambient pre-transition filtration $(\mathcal G_t^-)_{t\geq0}$
such that
\[
\lambda_{o_t}(\widehat c_t)=c_t
\qquad
\text{for every }t\geq T_0
\quad\text{almost surely}.
\]
Then, for every deterministic $t\geq0$, every
$\hat z'\in\widehat{\cS}$, and every $\eta\in\mathbb R$, on the event
$\{T_0\leq t\}$,
\[
\bP\!\left(
\widehat z_{t+1}=\hat z'
\mid
\mathcal G_t^-
\right)
=
\widehat P^\lambda\!\left(
\hat z'
\mid
\widehat z_t,a_t
\right),
\]
\[
\bE[Y_t\mid\mathcal G_t^-]
=
\widehat R^\lambda(\widehat z_t,a_t),
\]
and
\[
\bE\!\left[
\exp\!\left(
\eta\bigl[
Y_t-\widehat R^\lambda(\widehat z_t,a_t)
\bigr]
\right)
\middle|
\mathcal G_t^-
\right]
\leq
\exp\!\left(
\frac{\eta^2\sigma_Y^2}{2}
\right).
\]
Moreover,
$\Lambda_\lambda(o,\hat c)=(o,\lambda_o(\hat c))$
is an MDP isomorphism from the local-coordinate quotient MDP to the canonical
stable-quotient MDP. If $T_0$ is also a stopping time for the
learned-information filtration
$(\widehat{\mathcal F}_t^-)_{t\geq0}$, then the same transition,
conditional-reward, and conditional sub-Gaussian conclusions hold relative to
$\widehat{\mathcal F}_t^-$ after $T_0$. Without these stopping-time and eventual
class-correctness conditions, pathwise eventual agreement alone does not imply
that the passive learned state is Markov or value preserving.
\end{theorem}

\section{Explicit nonabelian construction}
\label{app:nonabelian_details}

This appendix instantiates Proposition~\ref{prop:min_nonabelian_main} and makes the executable commutator calculation explicit. The construction uses one visible observation, two executable actions, and three hidden layers. It therefore shows that the nonabelian memory barrier does not require a complicated observation graph.

\begin{corollary}[Barrier for count-additive loop encoders]
\label{cor:additive_barrier_main}

Let
\[
N_j(w):=(\ab(w))_j
\]
denote the signed exponent sum of generator $x_j$ in $w$.

If an encoder computes
\[
A(w)
=
\sum_{j=1}^{r}N_j(w)v_j
\]
and its readout depends only on $A(w)$, then its memory factors through
$\ab:F_r\to\mathbb Z^r$. Therefore
Theorem~\ref{thm:abelian_barrier_main} applies whenever its
decision-separation condition holds.
\end{corollary}

Let
$
\cO=\{o\},
\qquad
\cA=\{a,b\},
\qquad
[n]=\{1,2,3\},
$
with
\[
P_O(o\mid o,a)
=
P_O(o\mid o,b)
=
1.
\]
Assign
\[
\alpha
:=
\sigma_a
=
(12)
\]
and
\[
\beta
:=
\sigma_b
=
(123).
\]
The generated transport group is
\[
\langle\alpha,\beta\rangle
=
S_3.
\]

The inverse transports are executable using the same two actions:
\[
\alpha^{-1}
=
\alpha,
\qquad
\beta^{-1}
=
\beta^2.
\]
Thus no additional inverse action is required.

Under the function-composition convention of Section~\ref{sec:nonabelian}, the commutator is
\[
[\alpha,\beta]
=
\alpha\beta\alpha^{-1}\beta^{-1}.
\]
A direct calculation gives
\[
[\alpha,\beta]
=
(123).
\]
An executable positive action sequence implementing this transport is obtained by replacing
\[
\alpha^{-1}
\quad\text{with}\quad
\alpha
\]
and
\[
\beta^{-1}
\quad\text{with}\quad
\beta^2.
\]
Reading temporal execution from right to left under the composition convention, one such chronological action sequence is
\[
b,b,a,b,a.
\]
Its total transport is
\[
\alpha\circ\beta\circ\alpha\circ\beta\circ\beta
=
[\alpha,\beta].
\]

In particular,
\[
[\alpha,\beta](1)=2.
\]

Choose the reward table
\[
\begin{array}{c|cc}
\text{layer} & a & b\\
\hline
1 & 1 & 0\\
2 & 0 & 1\\
3 & 0 & 0
\end{array}
\]
and let
\[
0<\gamma<\frac{1}{2}.
\]
The three immediate reward vectors are distinct:
\[
(1,0),
\qquad
(0,1),
\qquad
(0,0).
\]
Therefore the initial reward partition is already the discrete partition, and the stable quotient consists of the three singleton classes
\[
\{1\},
\qquad
\{2\},
\qquad
\{3\}.
\]

At class \(1\), action \(a\) is uniquely optimal. Indeed, both actions send layer \(1\) to layer \(2\), so their continuation values are identical while
\[
R(o,1,a)-R(o,1,b)=1.
\]

At class \(2\), action \(b\) is uniquely optimal for
\[
\gamma<\frac{1}{2}.
\]
The immediate advantage of \(b\) is \(1\), while the largest possible discounted continuation-value disadvantage is at most
\[
\frac{\gamma}{1-\gamma}<1.
\]
Hence
\[
\bar Q^\star(o,\{2\},b)
>
\bar Q^\star(o,\{2\},a).
\]

The empty word leaves class \(1\) unchanged, whereas the commutator sends it to class \(2\):
\[
\rho_o(e)(\{1\})
=
\{1\},
\]
\[
\rho_o([x_a,x_b])(\{1\})
=
\{2\}.
\]
However,
\[
\ab(e)
=
\ab([x_a,x_b])
=
0.
\]
Thus every memory surrogate that factors through abelianization assigns the same memory value to these two histories, even though their unique optimal actions are different. This verifies the hypotheses of Corollary~\ref{cor:commutator_barrier_main}.

\section{Scope remarks on extensions}
\label{app:extensions}

This appendix records brief scope remarks beyond the finite deterministic HCDP framework studied in the main text. These remarks clarify conceptual reach but do not enlarge the theorem claims of the paper.

\begin{remark}[Operator-valued transport]
A natural extension replaces edge permutations by stochastic transport kernels. In that setting, path transport is operator-valued and loop effects generate an operator semigroup rather than, in general, a permutation group. The quotient perspective remains meaningful, but both identification and exact recovery become substantially harder because edge recovery is no longer combinatorial.
\end{remark}

\begin{remark}[Continuous observation spaces]
A second extension replaces the finite observation graph by a finite-resolution skeleton extracted from a continuous observation space. At each fixed resolution one obtains a finite transport-control model to which the structural quotient theory applies. A genuine continuum theory would additionally require discretization-stability, refinement-consistency, and sample-complexity arguments, which are outside the scope of the present paper.
\end{remark}

\begin{remark}[Conceptual continuity]
Both extensions preserve the same conceptual separation as the main text: history acts through path transport, control-relevant memory is obtained by quotienting predictive redundancies, and reinforcement learning becomes standard once the correct augmented state is available.
\end{remark}

\section{Experimental Details and Additional Results}
\label{app:exp_details}

This appendix gives the complete controlled environments, diagnostic protocol, baselines, metrics, hyperparameters, and measured results for Section~\ref{sec:exp}. The experiments are designed to validate finite consequences of the theory under known HCDP assumptions. They do not claim that generic public POMDP benchmarks possess a finite exact holonomy quotient.

\subsection{Experimental Claim Boundary}
\label{app:exp_claim_boundary}

The experiments test four claims.

\begin{enumerate}
    \item \textbf{Exact quotient construction.}
    Stable refinement should recover the known quotient of ChainCover, preserve the exact optimal value, and reject every attempted one-pair strict coarsening.

    \item \textbf{Diagnostic identification.}
    Repeated class-stable probes should reduce local class error, and transition-anchored source--target diagnostics should recover deterministic quotient transports under the simulator's fixed local class labeling.

    \item \textbf{Nonabelian order dependence.}
    A memory based only on commutative loop counts should fail on equal-count histories whose ordered transport differs.

    \item \textbf{Finite quotient control.}
    Once calibrated, HMRL-D should approach the quotient oracle with a fixed three-state class memory, rather than storing the raw nuisance coordinate or the complete loop word.
\end{enumerate}

The experiments do not attempt to prove the theorems. They test implementation-level consequences that would be false if the partition, diagnostic, transport, or quotient-lifting procedures were implemented incorrectly.

\subsection{Controlled HCDP Constructions}
\label{app:controlled_hcdps}

Both environment families use a raw hidden layer
$
(c,u)\in[3]\times[b],
$
which is identified with the layer set $[3b]$ through any fixed bijection. The coordinate
$
c\in[3]
$
is control relevant, whereas
$
u\in[b]
$
is a nuisance coordinate. Rewards depend on $c$ but never on $u$.

Define
\[
\alpha=(12),
\qquad
\beta=(123)
\]
as permutations of $[3]$. For the nuisance coordinate, define the cyclic permutations
\[
\nu_{+}(u)
=
1+(u\bmod b)
\]
and
\[
\nu_{-}(u)
=
1+((u-2)\bmod b).
\]
When $b=1$, both nuisance transports reduce to the identity.

\subsubsection{ChainCover}

For a depth parameter $D\geq 1$, ChainCover has observation space
\[
\cO_{\mathrm{chain}}
=
\{o_0,o_1,\ldots,o_D,q,\dagger\}
\]
and action space
\[
\cA
=
\{g_1,g_2,g_3\}.
\]

At every chain observation $o_d$, all actions have zero reward and induce the same visible transition. Specifically,
\[
P_O(o_{d+1}\mid o_d,a)=1,
\qquad
d=0,\ldots,D-1,
\]
and
\[
P_O(q\mid o_D,a)=1.
\]
The class transport alternates between $\alpha$ and $\beta$:
\[
\pi_d
=
\begin{cases}
\alpha, & d\text{ even},\\
\beta, & d\text{ odd}.
\end{cases}
\]
The full raw transport is
\[
\sigma_{o_d,a,o_{d+1}}(c,u)
=
\begin{cases}
(\pi_d(c),\nu_{+}(u)), & d<D,\\
(\pi_D(c),\nu_{-}(u)), & d=D,
\end{cases}
\]
where the second line denotes the edge from $o_D$ to $q$.

At the query observation $q$, action $g_j$ gives reward
\[
R(q,(c,u),g_j)
=
\mathbf{1}\{c=j\},
\]
and the process transitions to the terminal observation:
\[
P_O(\dagger\mid q,g_j)=1.
\]
The query-to-terminal transport is the identity. At $\dagger$, every action has zero reward and the process self-loops with identity transport.

The stable quotient is known analytically. At every observation
\[
o\in
\{o_0,\ldots,o_D,q\},
\]
the stable classes are
\[
C_o^{(j)}
=
\{(j,u):u\in[b]\},
\qquad
j=1,2,3.
\]
At $\dagger$, all raw layers form one class. Hence
\[
|\cS_{\mathrm{raw}}|
=
3b(D+3)
\]
and
\[
|\bar{\cS}|
=
3(D+2)+1.
\]

The reward partition separates the three classes at $q$ immediately. One additional application of the stability operator propagates this distinction backward by one chain edge. Therefore the ground-truth stabilization depth is
\[
L_\star=D+1.
\]

The primary audit uses
\[
D=6,
\qquad
b=8.
\]
Its exact ground-truth counts are
\[
|\cS_{\mathrm{raw}}|=216,
\qquad
|\bar{\cS}|=25,
\qquad
L_\star=7.
\]

\subsubsection{LoopGuess}

LoopGuess has observation space
\[
\cO_{\mathrm{loop}}
=
\{h,x_\alpha,x_\beta,q,\dagger\}
\]
and action space
\[
\cA
=
\{g_1,g_2,g_3\}.
\]

The observation $h$ is the common loop base. For every action $a$,
\[
P_O(q\mid h,a)
=
p_{\mathrm{query}},
\]
\[
P_O(x_\alpha\mid h,a)
=
\frac{1-p_{\mathrm{query}}}{2},
\]
and
\[
P_O(x_\beta\mid h,a)
=
\frac{1-p_{\mathrm{query}}}{2}.
\]
All three outgoing transports from $h$ are the identity.

The two excursion observations return deterministically to $h$:
\[
P_O(h\mid x_\alpha,a)=1,
\qquad
P_O(h\mid x_\beta,a)=1.
\]
Their raw transports are
\[
\sigma_{x_\alpha,a,h}(c,u)
=
(\alpha(c),\nu_{+}(u))
\]
and
\[
\sigma_{x_\beta,a,h}(c,u)
=
(\beta(c),\nu_{-}(u)).
\]

Thus the visible excursions
\[
h\to x_\alpha\to h
\]
and
\[
h\to x_\beta\to h
\]
are closed walks based at $h$ whose quotient transports are $\alpha$ and $\beta$, respectively.

At the query observation,
\[
R(q,(c,u),g_j)
=
\mathbf{1}\{c=j\},
\]
and
\[
P_O(\dagger\mid q,g_j)=1.
\]
All other rewards are zero. The initial distribution is
\[
\rho_0(h,(1,u))
=
\frac{1}{b},
\qquad
u\in[b].
\]

At each of
\[
h,x_\alpha,x_\beta,q,
\]
the stable quotient contains the three classes
\[
C_o^{(j)}
=
\{(j,u):u\in[b]\}.
\]
At $\dagger$, all layers merge. Consequently,
\[
|\cS_{\mathrm{raw}}|
=
15b
\]
and
\[
|\bar{\cS}|=13.
\]

The primary control setting uses
\[
b=16,
\qquad
p_{\mathrm{query}}=0.1.
\]
The expected number of completed loops before a query is
\[
\frac{1-p_{\mathrm{query}}}{p_{\mathrm{query}}}
=
9.
\]

For numerical safety, an episode is forced to query after at most $48$ completed loops. For the primary setting $p_{\mathrm{query}}=0.1$, the probability that an uncapped geometric episode would reach this limit is
\[
(1-p_{\mathrm{query}})^{48}
=
0.9^{48}
\approx
6.36\times 10^{-3}.
\]
The measured primary cap frequency is
$0.0058\pm0.0010$,
which is consistent with this probability. At
$p_{\mathrm{query}}=0.05$,
the measured cap frequency increases to
$0.085\pm0.010$;
therefore that sensitivity setting should be interpreted as the configured capped process rather than as an effectively uncapped geometric process.

\subsection{Exact Stable-Quotient Audit}
\label{app:exact_audit}

The exact audit enumerates every raw state, action, and feasible visible successor. Starting from the reward partition, it repeatedly applies
\[
\Pi^{(t+1)}
=
\mathcal T(\Pi^{(t)})
\]
until no partition changes.

For an inferred partition family $\widehat\Pi$, define
\[
\epsilon_R(\widehat\Pi)
=
\max_{\substack{o,a,\\
i\equiv_o^{\widehat\Pi}j}}
\left|
R(o,i,a)-R(o,j,a)
\right|.
\]

Define the successor inconsistency by
\[
\epsilon_T(\widehat\Pi)
=
\max_{\substack{
e:o\to o',\\
i\equiv_o^{\widehat\Pi}j
}}
\mathbf{1}
\left\{
[\sigma_e(i)]_{o'}^{\widehat\Pi}
\neq
[\sigma_e(j)]_{o'}^{\widehat\Pi}
\right\}.
\]

The quotient reward and transition kernel are constructed from $\widehat\Pi$, and value iteration is run on both the raw latent MDP and quotient MDP until the sup-norm Bellman residual is below
\[
10^{-12}.
\]
The value-preservation metric is
\[
\Delta_V
=
\max_{o,i}
\left|
V^\star_{\mathrm{full}}(o,i)
-
\bar V^\star
\left(
o,[i]_o^{\widehat\Pi}
\right)
\right|.
\]

To audit coarseness, for every nonterminal observation and every unordered pair of distinct stable classes, we form the partition obtained by merging exactly that pair while leaving every other block unchanged. A strict coarsening is rejected if
\[
\epsilon_R>0
\quad\text{or}\quad
\epsilon_T>0.
\]

\begin{table}[t]
\centering
\caption{\textbf{Exact ChainCover structural audit.}
Measured quantities are obtained by exhaustive enumeration and agree with the analytical construction.}
\label{tab:chain_exact_audit}
\begin{tabular}{lcc}
\toprule
Metric & Analytical & Measured \\
\midrule
Raw latent states & 216 & 216 \\
Stable quotient states & 25 & 25 \\
Compression ratio & 8.64$\times$ & 8.64$\times$ \\
Stabilization round & 7 & 7 \\
Reward inconsistency $\epsilon_R$ & 0 & 0 \\
Transition inconsistency $\epsilon_T$ & 0 & 0 \\
Optimal-value gap $\Delta_V$ & 0 & 0 \\
Rejected strict coarsenings & 24/24 & 24/24 \\
\bottomrule
\end{tabular}
\end{table}

\begin{figure}[t]
    \centering
    \includegraphics[width=\columnwidth]{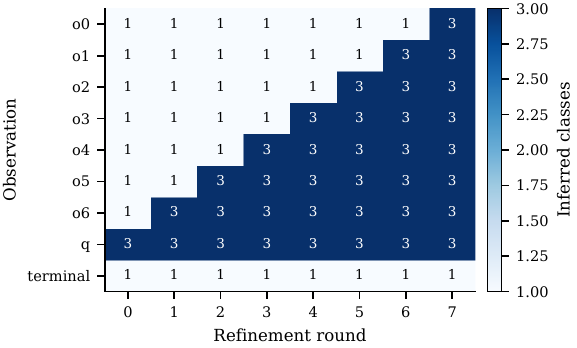}
    \caption{\textbf{Backward propagation of stable class distinctions in ChainCover.}
    The reward partition separates the three classes at the query observation in round $0$.
    Each subsequent refinement round propagates this distinction backward by one chain position, and the partition stabilizes at round $7$, as predicted for depth $D=6$.}
    \label{fig:chain_refinement_full}
\end{figure}

\subsection{Diagnostic Fingerprints}
\label{app:diagnostic_protocol}

The standalone recovery experiment assigns the three classes the fixed prototype vectors
$
v_1=(0.85,0,0.35),\qquad
v_2=(-0.42,0.74,-0.20),\qquad
v_3=(-0.42,-0.74,-0.20).
$
Their minimum pairwise Euclidean separation is
\[
\Delta
=
\min_{i\neq j}\|v_i-v_j\|_2
=
1.48.
\]
A diagnostic observation from class $c$ is
\[
Y=v_c+\xi,
\qquad
\xi\sim\mathcal N(0,\sigma_{\mathrm{diag}}^2 I_3).
\]
For each trial, the evaluator averages
$m\in\{1,2,4,8,16,32,64\}$
independent observations and applies nearest-prototype decoding:
\[
\widehat c
=
\argmin_{j\in[3]}
\left\|
\frac{1}{m}\sum_{\ell=1}^{m}Y_\ell-v_j
\right\|_2.
\]
The tested noise levels are
\[
\sigma_{\mathrm{diag}}
\in
\{0.25,0.50,1.00\}.
\]
For every $(\sigma_{\mathrm{diag}},m)$ configuration, each repeated evaluation uses $500$ trials from each of the three classes, giving $1{,}500$ class trials in total. Because the released diagnostic sweep uses a fixed prototype labeling rather than independently permuted local gauges, the reported quantity is direct classification error rather than a separately tested gauge-alignment error.

\subsection{Transition-Anchored Transport Recovery}
\label{app:transport_recovery_exp}

For each generator
$e\in\{\alpha,\beta\}$
and source class $c$, the target class is $\tau_e(c)$. A transport trial generates $m$ independent noisy observations of the target prototype $v_{\tau_e(c)}$, averages them, and applies the same nearest-prototype decoder. The row-wise transport error is
\[
\operatorname{Err}_{\mathrm{transport}}
=
\frac{1}{6N_e}
\sum_{e\in\{\alpha,\beta\}}
\sum_{c=1}^{3}
\sum_{r=1}^{N_e}
\mathbf{1}
\left\{
\widehat c_{e,c,r}
\neq
\tau_e(c)
\right\},
\]
where $N_e=250$ trials are used for each source-class and generator pair. Thus, each repeated evaluation at a given $(\sigma_{\mathrm{diag}},m)$ setting contains $6N_e=1{,}500$ transport trials.

\begin{table*}[t]
\centering
\caption{\textbf{Diagnostic and transport-recovery errors.}
Entries are reported as mean $\pm$ standard deviation.
Each repeated evaluation uses $1{,}500$ class trials and $1{,}500$ transport trials.}
\label{tab:diagnostic_full}
\resizebox{\textwidth}{!}{%
\begin{tabular}{c ccc ccc}
\toprule
& \multicolumn{3}{c}{Class error}
& \multicolumn{3}{c}{Transport error} \\
\cmidrule(lr){2-4}
\cmidrule(lr){5-7}
$m$
& $\sigma=0.25$
& $\sigma=0.50$
& $\sigma=1.00$
& $\sigma=0.25$
& $\sigma=0.50$
& $\sigma=1.00$ \\
\midrule
1
& 0.0024 $\pm$ 0.0010
& 0.1104 $\pm$ 0.0110
& 0.3436 $\pm$ 0.0090
& 0.0029 $\pm$ 0.0010
& 0.1087 $\pm$ 0.0131
& 0.3467 $\pm$ 0.0127 \\
2
& 0.0000 $\pm$ 0.0000
& 0.0287 $\pm$ 0.0027
& 0.2198 $\pm$ 0.0128
& 0.0000 $\pm$ 0.0000
& 0.0282 $\pm$ 0.0062
& 0.2278 $\pm$ 0.0088 \\
4
& 0.0000 $\pm$ 0.0000
& 0.0024 $\pm$ 0.0010
& 0.1113 $\pm$ 0.0052
& 0.0000 $\pm$ 0.0000
& 0.0018 $\pm$ 0.0010
& 0.1127 $\pm$ 0.0064 \\
8
& 0.0000 $\pm$ 0.0000
& 0.0000 $\pm$ 0.0000
& 0.0260 $\pm$ 0.0007
& 0.0000 $\pm$ 0.0000
& 0.0000 $\pm$ 0.0000
& 0.0282 $\pm$ 0.0034 \\
16
& 0.0000 $\pm$ 0.0000
& 0.0000 $\pm$ 0.0000
& 0.0018 $\pm$ 0.0010
& 0.0000 $\pm$ 0.0000
& 0.0000 $\pm$ 0.0000
& 0.0011 $\pm$ 0.0010 \\
32
& 0.0000 $\pm$ 0.0000
& 0.0000 $\pm$ 0.0000
& 0.0000 $\pm$ 0.0000
& 0.0000 $\pm$ 0.0000
& 0.0000 $\pm$ 0.0000
& 0.0000 $\pm$ 0.0000 \\
64
& 0.0000 $\pm$ 0.0000
& 0.0000 $\pm$ 0.0000
& 0.0000 $\pm$ 0.0000
& 0.0000 $\pm$ 0.0000
& 0.0000 $\pm$ 0.0000
& 0.0000 $\pm$ 0.0000 \\
\bottomrule
\end{tabular}}
\end{table*}

At the primary noise level $\sigma_{\mathrm{diag}}=0.5$, class and transport errors decrease from approximately $0.11$ at $m=1$ to below $0.003$ at $m=4$, and both become zero in the reported trials by $m=8$. At the harder noise level $\sigma_{\mathrm{diag}}=1.0$, the same monotone pattern is observed, with both errors reaching zero by $m=32$.

\subsection{Control Baselines}
\label{app:control_baselines}

All tabular methods use the same Q-learning update,
$
Q_{t+1}(z_t,a_t)
=
Q_t(z_t,a_t)
+
\eta_t
\left[
Y_t
+
\gamma
\max_{a'}
Q_t(z_{t+1},a')
-
Q_t(z_t,a_t)
\right].
$

Only the state descriptor differs.

\paragraph{Obs-Q.}
The state is the current visible observation:
\[
z_t=o_t.
\]
At the query observation, all histories are therefore merged.

\paragraph{Count-Q.}
The state contains the current observation and the commutative loop-count vector:
\[
z_t
=
\left(
o_t,
N_\alpha(t),
N_\beta(t)
\right).
\]
The order of completed loops is discarded. This baseline directly instantiates an abelianized loop-memory surrogate.

\paragraph{History-Q.}
The state contains the complete visible loop word:
\[
z_t
=
(o_t,w_t),
\qquad
w_t\in\{\alpha,\beta\}^{*}.
\]
Because LoopGuess uses a fixed initial quotient class, this representation is
exact on the truncated experimental trajectories, but its number of possible
states grows exponentially with loop length.

\paragraph{Raw-Q oracle.}
The privileged state is
\[
z_t
=
(o_t,c_t,u_t).
\]
This oracle retains the nuisance coordinate and therefore uses up to $3b$ latent memory values at a fixed nonterminal observation.

\paragraph{HMRL-D.}
HMRL-D first estimates local class prototypes and the two local loop maps
$
\widehat\tau_\alpha,
\quad
\widehat\tau_\beta.
$
After calibration, all Q-tables and optimizer statistics are reinitialized. The initial local class is fixed by the calibrated initial checkpoint, and subsequent classes are propagated as
\[
\widehat c_{t+1}
=
\widehat\tau_e(\widehat c_t)
\]
after each completed loop. The control state is
\[
z_t
=
(o_t,\widehat c_t).
\]

\paragraph{Quotient-Q oracle.}
The privileged quotient state is
\[
z_t
=
(o_t,c_t).
\]
This is the smallest exact class-tracking oracle and supplies the reference control curve.

\subsection{Training and Evaluation Protocol}
\label{app:training_protocol}

Unless otherwise stated, the primary parameters are:

\begin{table}[t]
\centering
\caption{\textbf{Hyperparameters used for the reported experimental artifacts.}}
\label{tab:exp_hyperparameters}
\begin{tabular}{ll}
\toprule
Parameter & Value \\
\midrule
ChainCover depth $D$ & $6$ \\
ChainCover nuisance size & $8$ \\
LoopGuess nuisance size & $16$ \\
Query probability & $0.10$ \\
Discount factor $\gamma$ & $0.99$ \\
Online training episodes & $3{,}500$ \\
Evaluation interval & $250$ \\
Evaluation episodes/checkpoint & $500$ \\
Final evaluation episodes & $2{,}000$ \\
Q-learning rate & $0.20$ \\
Initial/final exploration & $0.90/0.03$ \\
Exploration-decay episodes & $1{,}200$ \\
Maximum completed loops & $48$ \\
Diagnostic noise & $0.50$ \\
Prototype repetitions & $24$ \\
Edge samples per source row & $24$ \\
Sweep training episodes & $1{,}800$ \\
Sweep final evaluation episodes & $1{,}000$ \\
\bottomrule
\end{tabular}
\end{table}

The exploration rate is linearly annealed from $0.90$ to $0.03$ over the first $1{,}200$ online episodes and remains fixed thereafter. Evaluation uses greedy actions on independently sampled environment trajectories.

The standalone recovery sweep in Appendix~\ref{app:diagnostic_protocol} and the LoopGuess control calibrator use separate fixed synthetic prototype sets. For LoopGuess control, HMRL-D uses the four-dimensional fingerprints
$
\mu_1=(1.25,-0.25,0.55,-0.90),
$
$
\mu_2=(-0.85,1.10,-0.20,0.65),
\quad
\mu_3=(0.15,-0.75,1.20,0.35).
$
Each prototype is estimated from $24$ repeated observations. Each of the six source rows corresponding to the two generators and three source classes uses $24$ independent edge observations. Counting one complete fingerprint observation as one resettable diagnostic interaction gives
\[
N_{\mathrm{calibration}}
=
3\times24+6\times24
=
216.
\]
After calibration, the Q-table is reinitialized before online training.

\subsection{Evaluation Metrics}
\label{app:exp_metrics}

\paragraph{Query success.}
For an evaluation episode $k$, let
\[
S_k
=
\mathbf{1}
\left\{
\text{the selected query action equals the terminal class}
\right\}.
\]
The success rate is
\[
\operatorname{Success}
=
\frac{1}{N_{\mathrm{eval}}}
\sum_{k=1}^{N_{\mathrm{eval}}}
S_k.
\]

\paragraph{Normalized learning-curve AUC.}
Let $x_j$ be the number of completed online-training episodes at checkpoint $j$, and let $s_j$ be the corresponding evaluation success. The implementation reports the normalized trapezoidal area
\[
\operatorname{AUC}
=
\frac{1}{x_K-x_0}
\sum_{j=0}^{K-1}
\frac{s_j+s_{j+1}}{2}
\left(x_{j+1}-x_j\right).
\]
Because $s_j\in[0,1]$, the normalized AUC also lies in $[0,1]$.

\paragraph{Paired-order accuracy.}
Let
\[
\mathcal P
=
\left\{
(w,w'):
\begin{array}{l}
N_\alpha(w)=N_\alpha(w'),\\
N_\beta(w)=N_\beta(w'),\\
\rho(w)(1)\neq\rho(w')(1)
\end{array}
\right\}
\]
be a finite evaluation set of same-count, different-transport word pairs.

For each word, the environment is rolled forward to the query observation without policy intervention. The method then chooses a query action using the memory produced by that history. Paired-order accuracy is the fraction of individual histories in these pairs on which the correct query action is selected.

\paragraph{Visited memory states.}
For each method, we count the number of distinct memory values observed at the query observation during final evaluation. The observation coordinate is held fixed, so this metric measures only the memory alphabet actually used to distinguish hidden conditions at the decision point.

\paragraph{Truncation frequency.}
We report
\[
f_{\mathrm{cap}}
=
\frac{
\#\{\text{episodes forced to query at the loop cap}\}
}{
\#\{\text{all episodes}\}
}.
\]
A non-negligible value means that the reported results characterize the capped
process; the cap should be increased only when an uncapped geometric-horizon
interpretation is required.

\subsection{Complete Control Results}
\label{app:exp_full_results}

We report two controlled robustness sweeps. The nuisance sweep varies the number of reward-irrelevant latent alternatives while holding the quotient fixed. The query-probability sweep changes the expected visible-history length before the terminal decision. All entries are reported as mean $\pm$ standard deviation.

\begin{figure}[t]
    \centering
    \includegraphics[width=\columnwidth]{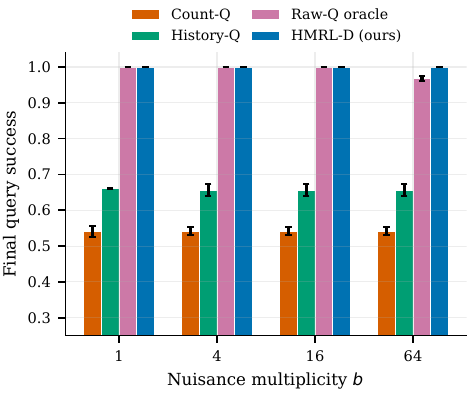}
    \caption{\textbf{Robustness to nuisance multiplicity.}
    HMRL-D preserves perfect final query success while retaining the same three-state decision memory.
    Raw-Q becomes harder to train at $b=64$ because it distinguishes all nuisance values, whereas Count-Q and History-Q remain limited by their respective representation deficiencies.}
    \label{fig:loopguess_nuisance_robustness}
\end{figure}

\begin{table*}[t]
\centering
\caption{\textbf{Final LoopGuess success across nuisance multiplicities.}
Bold marks the best deployable method; oracle methods are reference rows.}
\label{tab:loopguess_nuisance_full}
\resizebox{\textwidth}{!}{%
\begin{tabular}{lcccc}
\toprule
Method & $b=1$ & $b=4$ & $b=16$ & $b=64$ \\
\midrule
Obs-Q
& 0.352 $\pm$ 0.026
& 0.364 $\pm$ 0.026
& 0.364 $\pm$ 0.026
& 0.364 $\pm$ 0.026 \\
Count-Q
& 0.541 $\pm$ 0.016
& 0.542 $\pm$ 0.010
& 0.542 $\pm$ 0.010
& 0.542 $\pm$ 0.010 \\
History-Q
& 0.660 $\pm$ 0.001
& 0.656 $\pm$ 0.017
& 0.656 $\pm$ 0.017
& 0.656 $\pm$ 0.017 \\
Raw-Q oracle
& 1.000 $\pm$ 0.000
& 1.000 $\pm$ 0.000
& 1.000 $\pm$ 0.000
& 0.967 $\pm$ 0.007 \\
HMRL-D (ours)
& \textbf{1.000 $\pm$ 0.000}
& \textbf{1.000 $\pm$ 0.000}
& \textbf{1.000 $\pm$ 0.000}
& \textbf{1.000 $\pm$ 0.000} \\
Quotient-Q oracle
& 1.000 $\pm$ 0.000
& 1.000 $\pm$ 0.000
& 1.000 $\pm$ 0.000
& 1.000 $\pm$ 0.000 \\
\bottomrule
\end{tabular}}
\end{table*}

\begin{figure}[t]
    \centering
    \includegraphics[width=\columnwidth]{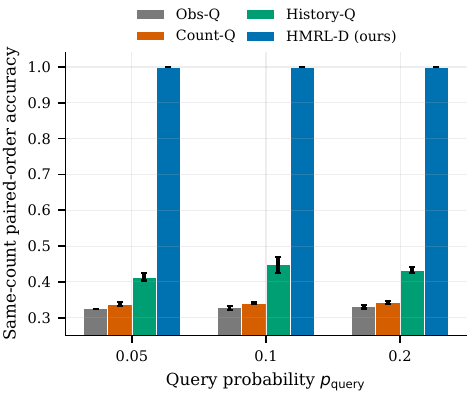}
    \caption{\textbf{Paired-order robustness across query probabilities.}
    HMRL-D remains exact as the expected visible history becomes longer.
    Count-Q remains close to the $1/3$ chance level because loop counts cannot recover ordering, while History-Q only partially resolves the larger collection of observed words.}
    \label{fig:loopguess_query_robustness}
\end{figure}

\begin{table*}[t]
\centering
\caption{\textbf{Paired-order accuracy across query probabilities.}
A smaller $p_{\mathrm{query}}$ produces longer loop histories on average.
Bold marks the best deployable method.}
\label{tab:loopguess_query_full}
\resizebox{\textwidth}{!}{%
\begin{tabular}{lccc}
\toprule
Method
& $p_{\mathrm{query}}=0.05$
& $p_{\mathrm{query}}=0.10$
& $p_{\mathrm{query}}=0.20$ \\
\midrule
Obs-Q
& 0.325 $\pm$ 0.000
& 0.328 $\pm$ 0.005
& 0.331 $\pm$ 0.005 \\
Count-Q
& 0.338 $\pm$ 0.006
& 0.340 $\pm$ 0.003
& 0.342 $\pm$ 0.004 \\
History-Q
& 0.414 $\pm$ 0.010
& 0.448 $\pm$ 0.023
& 0.433 $\pm$ 0.008 \\
Raw-Q oracle
& 1.000 $\pm$ 0.000
& 1.000 $\pm$ 0.000
& 1.000 $\pm$ 0.000 \\
HMRL-D (ours)
& \textbf{1.000 $\pm$ 0.000}
& \textbf{1.000 $\pm$ 0.000}
& \textbf{1.000 $\pm$ 0.000} \\
Quotient-Q oracle
& 1.000 $\pm$ 0.000
& 1.000 $\pm$ 0.000
& 1.000 $\pm$ 0.000 \\
\bottomrule
\end{tabular}}
\end{table*}

\begin{figure}[t]
    \centering
    \includegraphics[width=\columnwidth]{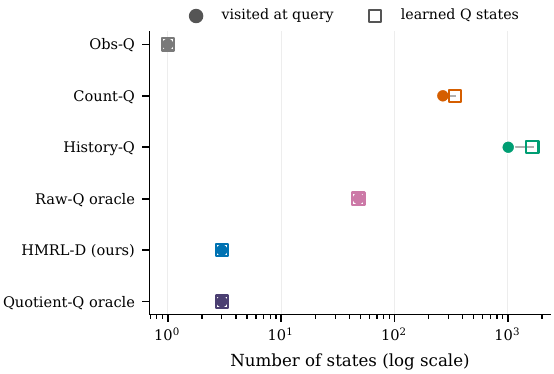}
    \caption{\textbf{Decision-time and learned-state complexity.}
    Filled markers report memory states visited at the query observation, while open markers report the total number of learned Q-states.
    HMRL-D and Quotient-Q use three states under both measures.
    Count-Q and especially History-Q create substantially larger learned state spaces.}
    \label{fig:loopguess_state_complexity}
\end{figure}

The nuisance sweep confirms that the quotient representation is invariant to an increasing number of control-irrelevant latent alternatives. In contrast, Raw-Q must learn separate values for each nuisance state and falls below perfect success at $b=64$ under the fixed sweep budget. The query-probability sweep provides a direct test of ordered transport: Count-Q remains near chance for all three settings, whereas HMRL-D maintains perfect paired-order accuracy.

\section{Proofs}
\label{app:proofs}

This section collects the proofs in the order in which the corresponding
results are used in the main text and the preceding appendices. Each result has
a single dedicated proof subsection.

\subsection{Proof of Lemma~\ref{lem:memory_right_congruence_main}}

\begin{proof}
Fix an observation $o\in\cO$, and let $H$ and $H'$ be two histories ending at $o$ such that $H\sim_{\cM}H'$. By the definition of $\sim_{\cM}$, this assumption is exactly
\[
F(H)=F(H').
\]
Let
\[
m:=F(H)=F(H')\in\cM
\]
denote their common memory value. The finite-memory controller selects its action distribution only as a function of the current observation and the current memory state. Therefore, for every action $\bar a\in\cA$,
\[
\pi_{\cM}(\bar a\mid o,F(H))
=
\pi_{\cM}(\bar a\mid o,m)
=
\pi_{\cM}(\bar a\mid o,F(H')).
\]
The two distributions consequently assign the same probability to every element of the finite action space $\cA$. Hence they are equal as probability distributions in $\Delta(\cA)$, and
\[
\pi_{\cM}(\cdot\mid o,F(H))
=
\pi_{\cM}(\cdot\mid o,F(H')).
\]

Now let $(a,o')$ be any common observable extension of $H$ and $H'$. Because both histories end at the same observation $o$, the recursive definition of the history encoder gives
\[
F\bigl(H\cdot(a,o')\bigr)
=
U\bigl(F(H),o,a,o'\bigr)
\]
and, similarly,
\[
F\bigl(H'\cdot(a,o')\bigr)
=
U\bigl(F(H'),o,a,o'\bigr).
\]
Substituting $F(H)=F(H')=m$ into the two identities yields
\[
\begin{aligned}
F\bigl(H\cdot(a,o')\bigr)
&=
U\bigl(F(H),o,a,o'\bigr)\\
&=
U(m,o,a,o')\\
&=
U\bigl(F(H'),o,a,o'\bigr)\\
&=
F\bigl(H'\cdot(a,o')\bigr).
\end{aligned}
\]
The two extended histories both end at the observation $o'$. The defining condition of $\sim_{\cM}$ therefore applies to them, and the preceding equality implies
\[
H\cdot(a,o')
\sim_{\cM}
H'\cdot(a,o').
\]

The same reasoning shows that equivalence is preserved not only by one observable extension but by every finite common observable continuation. To make this explicit, let
\[
\xi
=
\bigl((a_0,o_1),\ldots,(a_{k-1},o_k)\bigr)
\]
be a common observable continuation from $o_0:=o$. Define the corresponding history prefixes recursively by
\[
H_0:=H,
\qquad
H'_0:=H',
\]
and, for $j=0,\ldots,k-1$, by
\[
H_{j+1}:=H_j\cdot(a_j,o_{j+1}),
\qquad
H'_{j+1}:=H'_j\cdot(a_j,o_{j+1}).
\]
The histories $H_j$ and $H'_j$ end at the same observation $o_j$ for every $j$. At the initial prefix,
\[
F(H_0)=F(H)=F(H')=F(H'_0).
\]
Suppose that, for some $j<k$,
\[
F(H_j)=F(H'_j).
\]
Applying the recursive encoder update to the next common extension gives
\[
F(H_{j+1})
=
U\bigl(F(H_j),o_j,a_j,o_{j+1}\bigr)
\]
and
\[
F(H'_{j+1})
=
U\bigl(F(H'_j),o_j,a_j,o_{j+1}\bigr).
\]
Since the memory arguments in these two expressions are equal by the induction hypothesis, all four arguments supplied to $U$ are identical. Because $U$ is a deterministic function,
\[
\begin{aligned}
F(H_{j+1})
&=
U\bigl(F(H_j),o_j,a_j,o_{j+1}\bigr)\\
&=
U\bigl(F(H'_j),o_j,a_j,o_{j+1}\bigr)\\
&=
F(H'_{j+1}).
\end{aligned}
\]
Induction therefore proves that
\[
F(H_j)=F(H'_j)
\qquad
\text{for every }j=0,\ldots,k.
\]
Thus, after every prefix of the same observable continuation, the two histories remain in the same memory class. In particular, at every corresponding future decision point, they have the same terminal observation and the same memory value, so the controller again uses the same action distribution.

Finally, for each fixed observation $o$, the relation $\sim_{\cM}$ on histories ending at $o$ is an equivalence relation. Reflexivity follows from
\[
F(H)=F(H);
\]
symmetry follows because
\[
F(H)=F(H')
\quad\Longrightarrow\quad
F(H')=F(H);
\]
and transitivity follows because
$
F(H)=F(H')
\quad\text{and}\quad
F(H')=F(H'')
\quad\Longrightarrow\quad
F(H)=F(H'').
$
Since every equivalence class is determined by a memory value in the finite set $\cM$, there are at most $|\cM|$ such classes among histories ending at any fixed observation. When histories are indexed jointly by their terminal observation and memory value, the total number of classes is at most
$
|\cO|\,|\cM|.
$
The preservation of these classes under every common observable continuation is exactly the required right-congruence property.
\end{proof}

\subsection{Proof of Lemma~\ref{lem:memory_aug_mdp_main}}

\begin{proof}
Let
\[
Z_t:=(s_t,m_t)\in\cS\times\cM
\]
denote the latent state--memory pair. We first derive its action-conditioned transition law. Fix a time $t\geq 0$, a current augmented state $(s,m)\in\cS\times\cM$, an action $a\in\cA$, and a candidate next augmented state $(s',m')\in\cS\times\cM$. Conditional on
\[
(s_t,m_t)=(s,m)
\qquad\text{and}\qquad
a_t=a,
\]
the latent state evolves according to the original transition kernel. Therefore,
\[
\bP\left(
s_{t+1}=s'
\mid
s_t=s,m_t=m,a_t=a
\right)
=
P(s'\mid s,a).
\]
The presence of $m_t$ in the conditioning does not change this probability because, once $s_t$ and $a_t$ are fixed, the transition of the latent state is governed by the Markov kernel $P(\cdot\mid s_t,a_t)$ and is independent of the preceding history and the controller's internal memory.

The memory state does not evolve independently. Since the current and next observations are
\[
o_t=h(s_t)
\qquad\text{and}\qquad
o_{t+1}=h(s_{t+1}),
\]
respectively, the memory update rule gives
\[
m_{t+1}
=
U\bigl(m_t,h(s_t),a_t,h(s_{t+1})\bigr).
\]
Consequently, on the event
\[
\{s_t=s,m_t=m,a_t=a,s_{t+1}=s'\},
\]
the next memory state is uniquely determined by
\[
m_{t+1}
=
U\bigl(m,h(s),a,h(s')\bigr).
\]
Using the conditional product rule, we obtain
$
\bP\left(
s_{t+1}=s',m_{t+1}=m'
\mid
s_t=s,m_t=m,a_t=a
\right)
\quad=
\bP\left(
s_{t+1}=s'
\mid
s_t=s,m_t=m,a_t=a
\right)
\qquad\quad\times
\bP\left(
m_{t+1}=m'
\mid
s_t=s,m_t=m,a_t=a,s_{t+1}=s'
\right).
$
The first factor is
\[
P(s'\mid s,a).
\]
The second factor is equal to one if
\[
m'
=
U\bigl(m,h(s),a,h(s')\bigr),
\]
and is equal to zero otherwise. Hence
\[
\begin{aligned}
&\bP\left(
s_{t+1}=s',m_{t+1}=m'
\mid
s_t=s,m_t=m,a_t=a
\right)\\
&\quad=
P(s'\mid s,a)
\mathbf{1}
\left\{
m'
=
U\bigl(m,h(s),a,h(s')\bigr)
\right\}.
\end{aligned}
\]
Therefore, the action-conditioned transition kernel on the augmented state space is
$
\widetilde P\bigl((s',m')\mid(s,m),a\bigr)
=
P(s'\mid s,a)
\mathbf{1}
\left\{
m'
=
U\bigl(m,h(s),a,h(s')\bigr)
\right\}.
$

This expression is a valid probability kernel. It is clearly nonnegative. Moreover, for every fixed $(s,m)\in\cS\times\cM$ and $a\in\cA$,
\[
\begin{aligned}
&\sum_{s'\in\cS}
\sum_{m'\in\cM}
\widetilde P\bigl((s',m')\mid(s,m),a\bigr)\\
&\quad=
\sum_{s'\in\cS}
\sum_{m'\in\cM}
P(s'\mid s,a)
\mathbf{1}
\left\{
m'
=
U\bigl(m,h(s),a,h(s')\bigr)
\right\}\\
&\quad=
\sum_{s'\in\cS}
P(s'\mid s,a)
\sum_{m'\in\cM}
\mathbf{1}
\left\{
m'
=
U\bigl(m,h(s),a,h(s')\bigr)
\right\}.
\end{aligned}
\]
For each fixed $s'$, the deterministic update map $U$ produces exactly one memory state
\[
U\bigl(m,h(s),a,h(s')\bigr)\in\cM.
\]
Thus,
\[
\sum_{m'\in\cM}
\mathbf{1}
\left\{
m'
=
U\bigl(m,h(s),a,h(s')\bigr)
\right\}
=
1.
\]
It follows that
\[
\begin{aligned}
\sum_{s'\in\cS}
\sum_{m'\in\cM}
\widetilde P\bigl((s',m')\mid(s,m),a\bigr)
&=
\sum_{s'\in\cS}P(s'\mid s,a)\\
&=
1.
\end{aligned}
\]

We next derive the initial distribution of the augmented process. The original latent state is initialized according to
\[
s_0\sim\rho_0.
\]
Its initial observation is
\[
o_0=h(s_0),
\]
and the memory state is initialized by
\[
m_0=\iota(o_0)=\iota(h(s_0)).
\]
Therefore, for every $(s,m)\in\cS\times\cM$,
\[
\begin{aligned}
\bP(s_0=s,m_0=m)
&=
\bP\bigl(s_0=s,\iota(h(s_0))=m\bigr)\\
&=
\bP(s_0=s)
\mathbf{1}\{m=\iota(h(s))\}\\
&=
\rho_0(s)
\mathbf{1}\{m=\iota(h(s))\}.
\end{aligned}
\]
Hence the initial distribution of the augmented process is
\[
\widetilde\rho_0(s,m)
=
\rho_0(s)
\mathbf{1}\{m=\iota(h(s))\}.
\]
This is also a valid probability distribution because
\[
\begin{aligned}
\sum_{s\in\cS}
\sum_{m\in\cM}
\widetilde\rho_0(s,m)
&=
\sum_{s\in\cS}
\rho_0(s)
\sum_{m\in\cM}
\mathbf{1}\{m=\iota(h(s))\}\\
&=
\sum_{s\in\cS}\rho_0(s)\\
&=
1.
\end{aligned}
\]

The memory variable records information used by the controller, but it does not modify the physical reward generated by the latent state and action. Thus, when the augmented state is $(s,m)$ and action $a$ is selected, the augmented reward is
\[
\widetilde r((s,m),a)
=
r(s,a).
\]

Consider now the lifted policy
\[
\widetilde\pi(a\mid s,m)
=
\pi_{\cM}(a\mid h(s),m).
\]
For every $(s,m)\in\cS\times\cM$,
\[
\begin{aligned}
\sum_{a\in\cA}
\widetilde\pi(a\mid s,m)
&=
\sum_{a\in\cA}
\pi_{\cM}(a\mid h(s),m)\\
&=
1,
\end{aligned}
\]
so $\widetilde\pi$ is a valid stochastic policy on the augmented state space. Although it is written as a function of $s$, it depends on $s$ only through the observable value $h(s)$, and therefore does not provide the controller with any additional latent-state information.

To prove the Markov property, let
\[
\mathcal{F}_t
:=
\sigma\bigl(
Z_0,a_0,Z_1,a_1,\ldots,a_{t-1},Z_t
\bigr)
\]
be the augmented history immediately before the action at time $t$ is selected. Suppose that
\[
Z_t=(s,m).
\]
Under the original finite-memory controller, the conditional distribution of the action is
\[
\begin{aligned}
\bP(a_t=a\mid\mathcal{F}_t)
&=
\pi_{\cM}(a\mid h(s_t),m_t)\\
&=
\pi_{\cM}(a\mid h(s),m)\\
&=
\widetilde\pi(a\mid s,m).
\end{aligned}
\]
Thus, the conditional law of $a_t$ depends on the entire augmented history only through the present augmented state $Z_t$.

For any candidate next augmented state $(s',m')$, the conditional law of total probability gives
\[
\begin{aligned}
&\bP\bigl(
Z_{t+1}=(s',m')
\mid
\mathcal{F}_t
\bigr)\\
&\quad=
\sum_{a\in\cA}
\bP(a_t=a\mid\mathcal{F}_t)
\bP\bigl(
Z_{t+1}=(s',m')
\mid
\mathcal{F}_t,a_t=a
\bigr).
\end{aligned}
\]
Given $Z_t=(s,m)$ and $a_t=a$, the transition law derived above depends on no earlier variable. Hence
\[
\bP\bigl(
Z_{t+1}=(s',m')
\mid
\mathcal{F}_t,a_t=a
\bigr)
=
\widetilde P\bigl((s',m')\mid(s,m),a\bigr).
\]
Substituting the action law and the action-conditioned transition law yields
\[
\begin{aligned}
&\bP\bigl(
Z_{t+1}=(s',m')
\mid
\mathcal{F}_t
\bigr)\\
&\quad=
\sum_{a\in\cA}
\widetilde\pi(a\mid s,m)
\widetilde P\bigl((s',m')\mid(s,m),a\bigr).
\end{aligned}
\]
Define the induced closed-loop kernel by
$
\widetilde P_{\widetilde\pi}
\bigl((s',m')\mid(s,m)\bigr)
:=
\sum_{a\in\cA}
\widetilde\pi(a\mid s,m)
\widetilde P\bigl((s',m')\mid(s,m),a\bigr).
$
Then
\[
\bP\bigl(
Z_{t+1}=(s',m')
\mid
\mathcal{F}_t
\bigr)
=
\widetilde P_{\widetilde\pi}
\bigl((s',m')\mid Z_t\bigr).
\]
The right-hand side depends on the augmented history only through $Z_t$. Therefore,
\[
Z_t=(s_t,m_t)
\]
is a time-homogeneous Markov process under $\widetilde\pi$.

It remains to prove that the augmented process has exactly the same expected discounted return as the original finite-memory controller. Fix an arbitrary finite horizon $T\geq 1$ and an arbitrary augmented trajectory
$
\tau_T
=
(
(s_0,m_0),a_0,
(s_1,m_1),a_1,...,(s_{T-1},m_{T-1}),a_{T-1},
(s_T,m_T)).
$
Under the original finite-memory controller, the initial state has probability $\rho_0(s_0)$, and the initial memory state must satisfy
\[
m_0=\iota(h(s_0)).
\]
At every time $t$, the action is generated according to
\[
\pi_{\cM}(a_t\mid h(s_t),m_t),
\]
the next latent state is generated according to
\[
P(s_{t+1}\mid s_t,a_t),
\]
and the next memory state must satisfy
\[
m_{t+1}
=
U\bigl(
m_t,h(s_t),a_t,h(s_{t+1})
\bigr).
\]
The chain rule therefore gives
$
\bP_{\mathrm{orig}}(\tau_T)
=
\rho_0(s_0)
\mathbf{1}\{m_0=\iota(h(s_0))\}
\quad\times
\prod_{t=0}^{T-1}
\Bigl[
\pi_{\cM}(a_t\mid h(s_t),m_t)
P(s_{t+1}\mid s_t,a_t)\hspace{34mm}\times
\mathbf{1}
\left\{
m_{t+1}
=
U\bigl(
m_t,h(s_t),a_t,h(s_{t+1})
\bigr)
\right\}
\Bigr].
$

Under the augmented process initialized from $\widetilde\rho_0$ and controlled by $\widetilde\pi$, the probability of the same trajectory is
$
\bP_{\mathrm{aug}}(\tau_T)
=
\widetilde\rho_0(s_0,m_0)
\times
\prod_{t=0}^{T-1}
\Bigl[
\widetilde\pi(a_t\mid s_t,m_t)
\widetilde P\bigl(
(s_{t+1},m_{t+1})
\mid
(s_t,m_t),a_t
\bigr)
\Bigr].
$
Substituting the definitions
\[
\widetilde\rho_0(s_0,m_0)
=
\rho_0(s_0)
\mathbf{1}\{m_0=\iota(h(s_0))\},
\]
\[
\widetilde\pi(a_t\mid s_t,m_t)
=
\pi_{\cM}(a_t\mid h(s_t),m_t),
\]
and
$
\widetilde P\bigl(
(s_{t+1},m_{t+1})
\mid
(s_t,m_t),a_t
\bigr)
=
P(s_{t+1}\mid s_t,a_t)
\mathbf{1}
\left\{
m_{t+1}
=
U\bigl(
m_t,h(s_t),a_t,h(s_{t+1})
\bigr)
\right\}
$
into the preceding expression gives
$
\bP_{\mathrm{aug}}(\tau_T)
=
\rho_0(s_0)
\mathbf{1}\{m_0=\iota(h(s_0))\}
\quad\times
\prod_{t=0}^{T-1}
\Bigl[
\pi_{\cM}(a_t\mid h(s_t),m_t)
P(s_{t+1}\mid s_t,a_t)
\hspace{34mm}\times
\mathbf{1}
\left\{
m_{t+1}
=
U\bigl(
m_t,h(s_t),a_t,h(s_{t+1})
\bigr)
\right\}
\Bigr]
=
\bP_{\mathrm{orig}}(\tau_T).
$
Since the trajectory $\tau_T$ was arbitrary, the original finite-memory controller and the augmented Markov process induce exactly the same probability distribution on every finite augmented trajectory.

The reward functions also agree pointwise along every trajectory:
\[
\widetilde r((s_t,m_t),a_t)
=
r(s_t,a_t).
\]
It follows that, for every finite horizon $T$,
\[
\begin{aligned}
&\bE_{\mathrm{orig}}
\left[
\sum_{t=0}^{T-1}
\gamma^t r(s_t,a_t)
\right]\\
&\quad=
\bE_{\mathrm{aug}}^{\widetilde\pi}
\left[
\sum_{t=0}^{T-1}
\gamma^t
\widetilde r((s_t,m_t),a_t)
\right].
\end{aligned}
\]

Because the reward function is bounded, define
\[
R_{\max}
:=
\max_{(s,a)\in\cS\times\cA}
|r(s,a)|
<
\infty.
\]
For every trajectory and every $T$,
\[
\begin{aligned}
\left|
\sum_{t=0}^{T-1}
\gamma^t r(s_t,a_t)
\right|
&\leq
\sum_{t=0}^{T-1}
\gamma^t|r(s_t,a_t)|\\
&\leq
R_{\max}
\sum_{t=0}^{T-1}\gamma^t\\
&=
R_{\max}
\frac{1-\gamma^T}{1-\gamma}\\
&\leq
\frac{R_{\max}}{1-\gamma}.
\end{aligned}
\]
The same bound holds for the augmented reward sums because
\[
\widetilde r((s_t,m_t),a_t)=r(s_t,a_t).
\]
Since $\gamma\in(0,1)$, the truncated discounted returns converge pointwise to the corresponding infinite discounted returns. The common deterministic bound
$
\frac{R_{\max}}{1-\gamma}
$
is integrable. Therefore, by the dominated convergence theorem,
\[
\begin{aligned}
&\bE_{\mathrm{orig}}
\left[
\sum_{t=0}^{\infty}
\gamma^t r(s_t,a_t)
\right]\\
&\quad=
\lim_{T\to\infty}
\bE_{\mathrm{orig}}
\left[
\sum_{t=0}^{T-1}
\gamma^t r(s_t,a_t)
\right]\\
&\quad=
\lim_{T\to\infty}
\bE_{\mathrm{aug}}^{\widetilde\pi}
\left[
\sum_{t=0}^{T-1}
\gamma^t
\widetilde r((s_t,m_t),a_t)
\right]\\
&\quad=
\bE_{\mathrm{aug}}^{\widetilde\pi}
\left[
\sum_{t=0}^{\infty}
\gamma^t
\widetilde r((s_t,m_t),a_t)
\right].
\end{aligned}
\]
Thus, the process on $\cS\times\cM$ has the stated transition kernel, reward function, and initial distribution. Under the lifted policy $\widetilde\pi$, the process $(s_t,m_t)$ is Markov and has exactly the same expected discounted return as the original finite-memory controller.
\end{proof}

\subsection{Proof of Lemma~\ref{lem:hcdp_path_lifting_main}}

\begin{proof}
Throughout the proof, probabilities are taken under the experiment that starts from the fixed latent state $(o_0,i_0)$ and applies the prescribed action sequence $a_0,\ldots,a_{k-1}$. For $j=0,\ldots,k$, define the visible-prefix event
\[
\mathcal E_j
:=
\{O_1=o_1,\ldots,O_j=o_j\},
\qquad
\mathcal E_0:=\Omega,
\]
where $(O_t,I_t)$ denotes the random latent state and $A_t$ denotes the random action at time $t$, while $o_0,\ldots,o_k$ are the fixed vertices of the prescribed path $\omega$. Under the experiment fixed above, $A_t=a_t$ almost surely for $t=0,\ldots,k-1$. Also define the prefix transports
\[
\tau_0:=\operatorname{id}_{[n]},
\qquad
\tau_j
:=
\sigma_{e_j}\circ\cdots\circ\sigma_{e_1}
\quad
\text{for }j=1,\ldots,k.
\]
Thus $\tau_k=\sigma_\omega$ by Definition~\ref{def:directed_transport_main}. We will prove simultaneously, for every $j=0,\ldots,k$, that
\[
\bP(\mathcal E_j)
=
\prod_{\ell=1}^{j}
P_O(o_\ell\mid o_{\ell-1},a_{\ell-1})
\]
and
\[
\bP\bigl(I_j=\tau_j(i_0)\mid\mathcal E_j\bigr)=1.
\]
The conditional probability in the second display is well defined because the first identity will imply that $\bP(\mathcal E_j)>0$.

We first derive the relevant one-step consequences of the HCDP transition kernel. Fix a latent state $(o,i)$, an action $a$, and an observation $o'$ satisfying
\[
P_O(o'\mid o,a)>0.
\]
Marginalizing the latent transition kernel over the next hidden layer gives
\[
\begin{aligned}
&\bP\bigl(
O_{t+1}=o'
\mid
(O_t,I_t)=(o,i),A_t=a
\bigr)\\
&\quad=
\sum_{r=1}^{n}
P_H\bigl((o',r)\mid(o,i),a\bigr)\\
&\quad=
\sum_{r=1}^{n}
P_O(o'\mid o,a)
\mathbf{1}\!\left\{
r=\sigma_{o,a,o'}(i)
\right\}\\
&\quad=
P_O(o'\mid o,a)
\sum_{r=1}^{n}
\mathbf{1}\!\left\{
r=\sigma_{o,a,o'}(i)
\right\}\\
&\quad=
P_O(o'\mid o,a).
\end{aligned}
\]
The final equality holds because $\sigma_{o,a,o'}(i)$ is a single well-defined element of $[n]$, so exactly one term in the sum of indicators is equal to one. In particular, the probability of observing $o'$ depends only on the current observation $o$ and the applied action $a$, and does not depend on the current hidden layer $i$.

Moreover, for any candidate next layer $r\in[n]$, the assumption
$
P_O(o'\mid o,a)>0
$
allows us to condition on the event $\{O_{t+1}=o'\}$. Using the definition of conditional probability and the preceding marginal identity, we obtain
\[
\begin{aligned}
&\bP\bigl(
I_{t+1}=r
\mid
(O_t,I_t)=(o,i),
A_t=a,
O_{t+1}=o'
\bigr)\\
&\quad=
\frac{
\bP\bigl(
O_{t+1}=o',
I_{t+1}=r
\mid
(O_t,I_t)=(o,i),
A_t=a
\bigr)
}{
\bP\bigl(
O_{t+1}=o'
\mid
(O_t,I_t)=(o,i),
A_t=a
\bigr)
}\\
&\quad=
\frac{
P_H\bigl((o',r)\mid(o,i),a\bigr)
}{
P_O(o'\mid o,a)
}\\
&\quad=
\frac{
P_O(o'\mid o,a)
\mathbf{1}\!\left\{
r=\sigma_{o,a,o'}(i)
\right\}
}{
P_O(o'\mid o,a)
}\\
&\quad=
\mathbf{1}\!\left\{
r=\sigma_{o,a,o'}(i)
\right\}.
\end{aligned}
\]
Consequently, once the visible edge
$
o\xrightarrow{a}o'
$
has occurred, the next hidden layer is uniquely determined:
\[
I_{t+1}
=
\sigma_{o,a,o'}(i)
\]
with conditional probability one.

We now establish the two claimed prefix identities by induction on $j$. For $j=0$, the event $\mathcal E_0=\Omega$ has probability
\[
\bP(\mathcal E_0)=1,
\]
which agrees with the empty product
\[
\prod_{\ell=1}^{0}
P_O(o_\ell\mid o_{\ell-1},a_{\ell-1})
=1.
\]
Furthermore,
\[
I_0=i_0
=
\operatorname{id}_{[n]}(i_0)
=
\tau_0(i_0)
\]
with probability one. Thus both statements hold at $j=0$.

Suppose that both statements hold for some $j-1$, where
\[
1\leq j\leq k.
\]
By the induction hypothesis,
\[
\bP(\mathcal E_{j-1})
=
\prod_{\ell=1}^{j-1}
P_O(o_\ell\mid o_{\ell-1},a_{\ell-1})
>0,
\]
and, conditional on $\mathcal E_{j-1}$,
\[
I_{j-1}
=
\tau_{j-1}(i_0)
\]
with probability one. By the definition of $\mathcal E_{j-1}$, the visible observation at time $j-1$ is also equal to $o_{j-1}$ on this event. Hence, conditional on $\mathcal E_{j-1}$, the latent state at time $j-1$ is almost surely
\[
(O_{j-1},I_{j-1})
=
\bigl(
o_{j-1},
\tau_{j-1}(i_0)
\bigr).
\]
The prescribed action at that time is $A_{j-1}=a_{j-1}$. Applying the one-step visible marginal identity derived above therefore gives
\[
\begin{aligned}
&\bP(\mathcal E_j\mid\mathcal E_{j-1})\\
=&
\bP\bigl(
O_j=o_j
\mid
\mathcal E_{j-1}
\bigr)\\
=&
\bP\Bigl(
O_j=o_j
\mathrel{\Big|}
(O_{j-1},I_{j-1})\\
&=
\bigl(
o_{j-1},
\tau_{j-1}(i_0)
\bigr),
A_{j-1}=a_{j-1}
\Bigr)\\
=&
P_O(o_j\mid o_{j-1},a_{j-1}).
\end{aligned}
\]
Since
\[
\mathcal E_j
=
\mathcal E_{j-1}
\cap
\{O_j=o_j\},
\]
the multiplication rule for probabilities yields
\[
\begin{aligned}
\bP(\mathcal E_j)
&=
\bP(\mathcal E_{j-1})
\bP(\mathcal E_j\mid\mathcal E_{j-1})\\
&=
\left(
\prod_{\ell=1}^{j-1}
P_O(o_\ell\mid o_{\ell-1},a_{\ell-1})
\right)
P_O(o_j\mid o_{j-1},a_{j-1})\\
&=
\prod_{\ell=1}^{j}
P_O(o_\ell\mid o_{\ell-1},a_{\ell-1}).
\end{aligned}
\]
Because
\[
e_j
=
(o_{j-1}\xrightarrow{a_{j-1}}o_j)
\]
is an edge of the directed support graph $X$, the definition of $X$ implies
\[
P_O(o_j\mid o_{j-1},a_{j-1})>0.
\]
Every factor in the product for $\bP(\mathcal E_{j-1})$ is also strictly positive by the induction hypothesis. Therefore,
\[
\bP(\mathcal E_j)>0.
\]

It remains to determine the hidden layer after the $j$th visible edge. Conditional on $\mathcal E_{j-1}$, the current hidden layer is almost surely $\tau_{j-1}(i_0)$. Conditional further on the next prescribed observation being $O_j=o_j$, the one-step conditional layer identity gives
\[
\begin{aligned}
I_j
&=
\sigma_{o_{j-1},a_{j-1},o_j}
\bigl(
\tau_{j-1}(i_0)
\bigr)\\
&=
\sigma_{e_j}
\bigl(
\tau_{j-1}(i_0)
\bigr)\\
&=
\bigl(
\sigma_{e_j}\circ\tau_{j-1}
\bigr)(i_0).
\end{aligned}
\]
By the definition of the prefix transports,
\[
\begin{aligned}
\sigma_{e_j}\circ\tau_{j-1}
&=
\sigma_{e_j}
\circ
\bigl(
\sigma_{e_{j-1}}
\circ\cdots\circ
\sigma_{e_1}
\bigr)\\
&=
\sigma_{e_j}
\circ\sigma_{e_{j-1}}
\circ\cdots\circ\sigma_{e_1}\\
&=
\tau_j.
\end{aligned}
\]
It follows that
\[
I_j=\tau_j(i_0)
\]
with conditional probability one on $\mathcal E_j$, or equivalently,
\[
\bP\bigl(
I_j=\tau_j(i_0)
\mid
\mathcal E_j
\bigr)=1.
\]
This completes the induction.

Taking $j=k$ in the probability identity gives
\[
\bP(\mathcal E_k)
=
\prod_{j=1}^{k}
P_O(o_j\mid o_{j-1},a_{j-1}).
\]
Since every $e_j$ is an edge of $X$, every factor in this product is strictly positive, and hence
\[
\prod_{j=1}^{k}
P_O(o_j\mid o_{j-1},a_{j-1})
>0.
\]
Finally, taking $j=k$ in the hidden-layer identity and using the definition of directed path transport gives
\[
\begin{aligned}
I_k
&=
\tau_k(i_0)\\
&=
\bigl(
\sigma_{e_k}
\circ\cdots\circ
\sigma_{e_1}
\bigr)(i_0)\\
&=
\sigma_\omega(i_0)
\end{aligned}
\]
with conditional probability one on $\mathcal E_k$. Therefore, executing $a_0,\ldots,a_{k-1}$ produces the prescribed visible path with the stated positive product probability, and conditional on that visible path, its hidden lift is unique and has terminal layer
\[
i_k=\sigma_\omega(i_0).
\]
\end{proof}

\subsection{Proof of Lemma~\ref{lem:stable_fp_main}}

\begin{proof}
We first verify that the definition of $\mathcal T$ indeed produces a partition family. Fix a partition family $\Pi$ and an observation $o\in\cO$, and define a binary relation on $[n]$ by declaring that $i$ and $j$ are related precisely when
\[
R(o,i,a)=R(o,j,a)
\quad\text{for every }a\in\cA
\]
and
$
\sigma_{o,a,o'}(i)
\equiv_{o'}^\Pi
\sigma_{o,a,o'}(j)
\quad\text{for every }(a,o')\text{ such that }
P_O(o'\mid o,a)>0.
$
This relation is reflexive because equality of rewards is reflexive and each relation $\equiv_{o'}^\Pi$ is reflexive. It is symmetric because equality is symmetric and each $\equiv_{o'}^\Pi$ is symmetric. To verify transitivity, suppose that $i$ is related to $j$ and that $j$ is related to $\ell$. Then, for every $a\in\cA$,
\[
R(o,i,a)=R(o,j,a)=R(o,\ell,a),
\]
so
\[
R(o,i,a)=R(o,\ell,a).
\]
Moreover, for every feasible pair $(a,o')$ satisfying $P_O(o'\mid o,a)>0$, we have
\[
\sigma_{o,a,o'}(i)
\equiv_{o'}^\Pi
\sigma_{o,a,o'}(j)
\]
and
\[
\sigma_{o,a,o'}(j)
\equiv_{o'}^\Pi
\sigma_{o,a,o'}(\ell).
\]
The transitivity of $\equiv_{o'}^\Pi$ therefore gives
\[
\sigma_{o,a,o'}(i)
\equiv_{o'}^\Pi
\sigma_{o,a,o'}(\ell).
\]
Thus the relation is an equivalence relation on $[n]$. Since this holds for every $o\in\cO$, the family
\[
\mathcal T(\Pi)
=
\{\mathcal T(\Pi)_o\}_{o\in\cO}
\]
is well defined as a partition family.

We next prove monotonicity. Let $\Pi\preceq\Pi'$. By the definition of refinement, for every observation $o'\in\cO$ and every $x,y\in[n]$,
\[
x\equiv_{o'}^\Pi y
\quad\Longrightarrow\quad
x\equiv_{o'}^{\Pi'}y.
\]
Fix $o\in\cO$ and suppose that
\[
i\equiv_o^{\mathcal T(\Pi)}j.
\]
Expanding the definition of $\mathcal T(\Pi)$, we obtain
\[
R(o,i,a)=R(o,j,a)
\quad\text{for every }a\in\cA
\]
and, for every $(a,o')$ with $P_O(o'\mid o,a)>0$,
\[
\sigma_{o,a,o'}(i)
\equiv_{o'}^\Pi
\sigma_{o,a,o'}(j).
\]
Because $\Pi\preceq\Pi'$, the latter equivalence implies
\[
\sigma_{o,a,o'}(i)
\equiv_{o'}^{\Pi'}
\sigma_{o,a,o'}(j)
\]
for every such feasible pair $(a,o')$. The reward equalities are unchanged, and hence the defining conditions of $\mathcal T(\Pi')$ are satisfied. Therefore
\[
i\equiv_o^{\mathcal T(\Pi')}j.
\]
We have proved, for every $o\in\cO$, that every block of $\mathcal T(\Pi)_o$ is contained in a block of $\mathcal T(\Pi')_o$. Consequently,
\[
\mathcal T(\Pi)
\preceq
\mathcal T(\Pi'),
\]
which establishes the claimed monotonicity.

We now consider the sequence initialized at the reward partition. By definition,
\[
i\equiv_o^{\Pi^{(0)}}j
\quad\Longleftrightarrow\quad
R(o,i,a)=R(o,j,a)
\quad\text{for every }a\in\cA.
\]
If
\[
i\equiv_o^{\Pi^{(1)}}j,
\]
then $\Pi^{(1)}=\mathcal T(\Pi^{(0)})$, and the first condition in the definition of $\mathcal T$ gives
\[
R(o,i,a)=R(o,j,a)
\quad\text{for every }a\in\cA.
\]
It follows directly from the definition of $\Pi^{(0)}$ that
\[
i\equiv_o^{\Pi^{(0)}}j.
\]
Thus every block of $\Pi_o^{(1)}$ is contained in a block of $\Pi_o^{(0)}$ for every $o\in\cO$, and therefore
\[
\Pi^{(1)}
\preceq
\Pi^{(0)}.
\]
Suppose inductively that, for some $t\geq1$,
\[
\Pi^{(t)}
\preceq
\Pi^{(t-1)}.
\]
Applying the monotonicity of $\mathcal T$ yields
\[
\mathcal T(\Pi^{(t)})
\preceq
\mathcal T(\Pi^{(t-1)}).
\]
Using the recursion
\[
\Pi^{(s+1)}
=
\mathcal T(\Pi^{(s)}),
\]
the preceding relation becomes
\[
\Pi^{(t+1)}
\preceq
\Pi^{(t)}.
\]
Together with the base case $\Pi^{(1)}\preceq\Pi^{(0)}$, induction gives
\[
\Pi^{(t+1)}
\preceq
\Pi^{(t)}
\quad\text{for every }t\geq0.
\]

To prove finite stabilization with the stated bound, let
\[
b_o^{(t)}
:=
\lvert\Pi_o^{(t)}\rvert
\]
denote the number of blocks of the partition at observation $o$, and define the total block count
\[
B^{(t)}
:=
\sum_{o\in\cO}
b_o^{(t)}.
\]
Since $\Pi_o^{(t+1)}$ refines $\Pi_o^{(t)}$, every block of $\Pi_o^{(t+1)}$ is contained in exactly one block of $\Pi_o^{(t)}$. It follows that
\[
b_o^{(t+1)}
\geq
b_o^{(t)}.
\]
If the refinement is strict at observation $o$, then this inequality is strict. Indeed, suppose that $\Pi_o^{(t+1)}$ strictly refines $\Pi_o^{(t)}$ but that
\[
b_o^{(t+1)}
=
b_o^{(t)}.
\]
Every block of the finer partition $\Pi_o^{(t+1)}$ is contained in a unique block of $\Pi_o^{(t)}$, and every block of $\Pi_o^{(t)}$ contains at least one block of $\Pi_o^{(t+1)}$, because the finer blocks cover all of $[n]$. Since the two finite collections contain the same number of blocks, every block of $\Pi_o^{(t)}$ must contain exactly one block of $\Pi_o^{(t+1)}$. If $C$ is a block of $\Pi_o^{(t)}$ and $D$ is the unique block of $\Pi_o^{(t+1)}$ contained in $C$, then every element of $C$ must belong to some finer block contained in $C$. Since $D$ is the only such finer block, we have
\[
C\subseteq D.
\]
The refinement relation already gives
\[
D\subseteq C,
\]
and hence $C=D$. This holds for every block, so the two partitions are equal, contradicting the assumed strictness. Therefore a strict refinement at observation $o$ necessarily satisfies
\[
b_o^{(t+1)}
\geq
b_o^{(t)}+1.
\]

Consequently, whenever the partition family changes,
\[
\Pi^{(t+1)}
\neq
\Pi^{(t)},
\]
the relation
\[
\Pi^{(t+1)}
\preceq
\Pi^{(t)}
\]
implies that there is at least one observation $o$ at which $\Pi_o^{(t+1)}$ strictly refines $\Pi_o^{(t)}$. Hence
\[
B^{(t+1)}
\geq
B^{(t)}+1.
\]
Every partition of the nonempty set $[n]$ has at least one block and at most $n$ blocks. Therefore, for every $t\geq0$,
\[
1
\leq
b_o^{(t)}
\leq
n
\quad\text{for every }o\in\cO,
\]
and summing over $o$ gives
\[
\lvert\cO\rvert
\leq
B^{(0)}
\leq
B^{(t)}
\leq
\lvert\cO\rvert n.
\]
If the sequence underwent $N$ strict refinements, the total block count would increase by at least $N$, so
\[
B^{(0)}+N
\leq
\lvert\cO\rvert n.
\]
It follows that
\[
N
\leq
\lvert\cO\rvert n-B^{(0)}
\leq
\lvert\cO\rvert n-\lvert\cO\rvert
=
\lvert\cO\rvert(n-1).
\]
Thus the sequence can undergo at most
\[
\lvert\cO\rvert(n-1)
\]
strict refinements. Therefore there exists an index $t_\star$, reached after at most that many strict refinements, such that
\[
\Pi^{(t_\star+1)}
=
\Pi^{(t_\star)}.
\]
Define
\[
\Pi^\star
:=
\Pi^{(t_\star)}.
\]
Using the recursive definition of the sequence, we then obtain
\[
\begin{aligned}
\mathcal T(\Pi^\star)
&=
\mathcal T(\Pi^{(t_\star)})\\
&=
\Pi^{(t_\star+1)}\\
&=
\Pi^{(t_\star)}\\
&=
\Pi^\star.
\end{aligned}
\]
Hence $\Pi^\star$ is a fixed point of $\mathcal T$. Moreover, once two consecutive iterates are equal, all subsequent iterates remain equal. Indeed,
\[
\begin{aligned}
\Pi^{(t_\star+2)}
&=
\mathcal T(\Pi^{(t_\star+1)})\\
&=
\mathcal T(\Pi^\star)\\
&=
\Pi^\star,
\end{aligned}
\]
and repeating the same argument gives
\[
\Pi^{(t)}
=
\Pi^\star
\quad\text{for every }t\geq t_\star.
\]

It remains to prove the maximality statement among stable families. Let $\Theta$ be any stable partition family, so that
\[
\mathcal T(\Theta)
=
\Theta.
\]
Fix an observation $o\in\cO$ and layers $i,j\in[n]$ satisfying
\[
i\equiv_o^\Theta j.
\]
By stability,
\[
i\equiv_o^{\mathcal T(\Theta)}j.
\]
The first condition in the definition of $\mathcal T(\Theta)$ therefore implies
\[
R(o,i,a)
=
R(o,j,a)
\quad\text{for every }a\in\cA.
\]
By the definition of the reward partition $\Pi^{(0)}$, this is equivalent to
\[
i\equiv_o^{\Pi^{(0)}}j.
\]
Since this implication holds for every $o\in\cO$ and every $i,j\in[n]$, every block of $\Theta_o$ is contained in a block of $\Pi_o^{(0)}$. Hence
\[
\Theta
\preceq
\Pi^{(0)}.
\]

We now show by induction that
\[
\Theta
\preceq
\Pi^{(t)}
\quad\text{for every }t\geq0.
\]
The statement holds at $t=0$ by the preceding argument. Suppose that
\[
\Theta
\preceq
\Pi^{(t)}
\]
for some $t\geq0$. Monotonicity of $\mathcal T$ gives
\[
\mathcal T(\Theta)
\preceq
\mathcal T(\Pi^{(t)}).
\]
Since $\Theta$ is stable and $\Pi^{(t+1)}=\mathcal T(\Pi^{(t)})$, this becomes
\[
\Theta
=
\mathcal T(\Theta)
\preceq
\mathcal T(\Pi^{(t)})
=
\Pi^{(t+1)}.
\]
The induction is therefore complete, and
\[
\Theta
\preceq
\Pi^{(t)}
\quad\text{for every }t\geq0.
\]
In particular, evaluating this relation at the stabilization index $t_\star$ yields
\[
\Theta
\preceq
\Pi^{(t_\star)}
=
\Pi^\star.
\]
Thus every stable partition family refines $\Pi^\star$. Since $\Pi^\star$ is itself stable, it is an element of the collection of stable partition families, and every other stable family lies below it in the refinement order. By the convention in Definition~\ref{def:stable_family_main}, a smaller partition family is finer and a larger partition family is coarser. Therefore $\Pi^\star$ is the greatest stable family with respect to $\preceq$, equivalently the coarsest stable partition family.
\end{proof}

\subsection{Proof of Proposition~\ref{prop:stable_exact_equiv_main}}

\begin{proof}
Assume first that the quotient maps
\[
q_o^\Pi:[n]\longrightarrow[n]/\Pi_o,
\qquad
q_o^\Pi(i)=[i]_o^\Pi,
\]
define an exact class abstraction in the sense of
Definition~\ref{def:exact_abstraction_main}. Thus, for every
$o\in\cO$, there exists a class-level mean reward
\[
\widehat R:
\left\{
(o,\hat c,a):
o\in\cO,\ 
\hat c\in[n]/\Pi_o,\ 
a\in\cA
\right\}
\longrightarrow\mathbb R,
\]
and, for every feasible triple $(o,a,o')$ satisfying
\[
P_O(o'\mid o,a)>0,
\]
there exists a deterministic class transport
\[
\widehat\tau_{o,a,o'}:
[n]/\Pi_o
\longrightarrow
[n]/\Pi_{o'},
\]
such that, for every layer $i\in[n]$,
\[
R(o,i,a)
=
\widehat R\bigl(o,q_o^\Pi(i),a\bigr)
\]
and
\[
q_{o'}^\Pi\!\left(
\sigma_{o,a,o'}(i)
\right)
=
\widehat\tau_{o,a,o'}\!\left(
q_o^\Pi(i)
\right).
\]

Fix an arbitrary observation $o\in\cO$ and arbitrary layers
$i,j\in[n]$ satisfying
\[
i\equiv_o^\Pi j.
\]
By the definition of the quotient map, two layers have the same
quotient label precisely when they belong to the same block of
$\Pi_o$. Hence
\[
i\equiv_o^\Pi j
\quad\Longleftrightarrow\quad
q_o^\Pi(i)=q_o^\Pi(j),
\]
and in particular
\[
q_o^\Pi(i)=q_o^\Pi(j).
\]
For any action $a\in\cA$, the exact reward identity gives
\[
R(o,i,a)
=
\widehat R\bigl(o,q_o^\Pi(i),a\bigr)
\]
and
\[
R(o,j,a)
=
\widehat R\bigl(o,q_o^\Pi(j),a\bigr).
\]
Since the quotient labels are equal, these two expressions satisfy
\[
\begin{aligned}
R(o,i,a)
&=
\widehat R\bigl(o,q_o^\Pi(i),a\bigr)\\
&=
\widehat R\bigl(o,q_o^\Pi(j),a\bigr)\\
&=
R(o,j,a).
\end{aligned}
\]
Because $a\in\cA$ was arbitrary, we have shown that
\[
R(o,i,a)=R(o,j,a)
\quad\text{for every }a\in\cA.
\]

Now fix any action--successor pair $(a,o')$ satisfying
\[
P_O(o'\mid o,a)>0.
\]
The transport identity in
Definition~\ref{def:exact_abstraction_main} gives
\[
q_{o'}^\Pi\!\left(
\sigma_{o,a,o'}(i)
\right)
=
\widehat\tau_{o,a,o'}\!\left(
q_o^\Pi(i)
\right)
\]
and
\[
q_{o'}^\Pi\!\left(
\sigma_{o,a,o'}(j)
\right)
=
\widehat\tau_{o,a,o'}\!\left(
q_o^\Pi(j)
\right).
\]
Using again the equality
\[
q_o^\Pi(i)=q_o^\Pi(j),
\]
we obtain
\[
\begin{aligned}
q_{o'}^\Pi\!\left(
\sigma_{o,a,o'}(i)
\right)
&=
\widehat\tau_{o,a,o'}\!\left(
q_o^\Pi(i)
\right)\\
&=
\widehat\tau_{o,a,o'}\!\left(
q_o^\Pi(j)
\right)\\
&=
q_{o'}^\Pi\!\left(
\sigma_{o,a,o'}(j)
\right).
\end{aligned}
\]
Equality of the images under $q_{o'}^\Pi$ is equivalent to membership
in the same block of $\Pi_{o'}$. Therefore
\[
\sigma_{o,a,o'}(i)
\equiv_{o'}^\Pi
\sigma_{o,a,o'}(j).
\]
Since the feasible pair $(a,o')$ was arbitrary, this relation holds
for every $(a,o')$ such that
\[
P_O(o'\mid o,a)>0.
\]

We have therefore established both conditions
\[
R(o,i,a)=R(o,j,a)
\quad
\text{for every }a\in\cA
\]
and
$
\sigma_{o,a,o'}(i)
\equiv_{o'}^\Pi
\sigma_{o,a,o'}(j)
\quad
\text{for every }(a,o')
\text{ with }P_O(o'\mid o,a)>0.
$
By the definition of the stability operator in
Definition~\ref{def:stable_family_main}, these two conditions are
equivalent to
\[
i\equiv_o^{\mathcal T(\Pi)}j.
\]
Consequently,
\[
i\equiv_o^\Pi j
\quad\Longrightarrow\quad
i\equiv_o^{\mathcal T(\Pi)}j.
\]
Since $o\in\cO$ and $i,j\in[n]$ were arbitrary, every block of
$\Pi_o$ is contained in a block of $\mathcal T(\Pi)_o$ for every
$o\in\cO$. By the definition of the refinement order, this means
\[
\Pi\preceq\mathcal T(\Pi).
\]

Conversely, assume that
\[
\Pi\preceq\mathcal T(\Pi).
\]
For every observation $o\in\cO$, define the class-label set by
\[
\widehat{\cC}_o
:=
[n]/\Pi_o
\]
and define the abstraction map to be the canonical quotient map
\[
q_o^\Pi:[n]\longrightarrow\widehat{\cC}_o,
\qquad
q_o^\Pi(i):=[i]_o^\Pi.
\]
Because $[n]$ is finite, the set $[n]/\Pi_o$ of blocks of $\Pi_o$ is
finite. Moreover, $q_o^\Pi$ is surjective: for every class
$C\in[n]/\Pi_o$, the block $C$ is nonempty, so one may choose
$i\in C$, and then
\[
q_o^\Pi(i)=C.
\]

We define the proposed class-level reward by
\[
\widehat R^\Pi
\bigl(
o,[i]_o^\Pi,a
\bigr)
:=
R(o,i,a).
\]
We also define, for every feasible triple $(o,a,o')$ satisfying
\[
P_O(o'\mid o,a)>0,
\]
the proposed class transport by
\[
\widehat\tau_{o,a,o'}^\Pi
\bigl(
[i]_o^\Pi
\bigr)
:=
\bigl[
\sigma_{o,a,o'}(i)
\bigr]_{o'}^\Pi.
\]
To obtain an exact class abstraction, it is essential to show that
these definitions do not depend on the representative $i$ chosen
from the quotient class.

Suppose that two representatives $i,j\in[n]$ determine the same
class at observation $o$, so that
\[
[i]_o^\Pi=[j]_o^\Pi.
\]
By the definition of quotient classes, this equality is equivalent
to
\[
i\equiv_o^\Pi j.
\]
The assumption
\[
\Pi\preceq\mathcal T(\Pi)
\]
means that every $\Pi_o$-block is contained in a
$\mathcal T(\Pi)_o$-block. Therefore
\[
i\equiv_o^\Pi j
\quad\Longrightarrow\quad
i\equiv_o^{\mathcal T(\Pi)}j,
\]
and hence
\[
i\equiv_o^{\mathcal T(\Pi)}j.
\]
Expanding the definition of $\mathcal T(\Pi)$ gives
\[
R(o,i,a)=R(o,j,a)
\quad\text{for every }a\in\cA.
\]
It follows that
\[
\begin{aligned}
\widehat R^\Pi
\bigl(
o,[i]_o^\Pi,a
\bigr)
&=
R(o,i,a)\\
&=
R(o,j,a)\\
&=
\widehat R^\Pi
\bigl(
o,[j]_o^\Pi,a
\bigr).
\end{aligned}
\]
Thus the definition of $\widehat R^\Pi$ is independent of the
representative, and the class-level reward is well defined.

The same relation
\[
i\equiv_o^{\mathcal T(\Pi)}j
\]
also implies, for every feasible $(a,o')$,
\[
\sigma_{o,a,o'}(i)
\equiv_{o'}^\Pi
\sigma_{o,a,o'}(j).
\]
Equivalence under $\Pi_{o'}$ is precisely equality of the
corresponding quotient classes. Therefore
\[
\bigl[
\sigma_{o,a,o'}(i)
\bigr]_{o'}^\Pi
=
\bigl[
\sigma_{o,a,o'}(j)
\bigr]_{o'}^\Pi.
\]
Consequently,
\[
\begin{aligned}
\widehat\tau_{o,a,o'}^\Pi
\bigl(
[i]_o^\Pi
\bigr)
&=
\bigl[
\sigma_{o,a,o'}(i)
\bigr]_{o'}^\Pi\\
&=
\bigl[
\sigma_{o,a,o'}(j)
\bigr]_{o'}^\Pi\\
&=
\widehat\tau_{o,a,o'}^\Pi
\bigl(
[j]_o^\Pi
\bigr).
\end{aligned}
\]
Thus $\widehat\tau_{o,a,o'}^\Pi$ is also independent of the chosen
representative and is a well-defined deterministic map
\[
\widehat\tau_{o,a,o'}^\Pi:
[n]/\Pi_o
\longrightarrow
[n]/\Pi_{o'}.
\]

It remains to verify the two exactness identities. For every
$o\in\cO$, $i\in[n]$, and $a\in\cA$, the definition of
$\widehat R^\Pi$ gives
\[
\begin{aligned}
\widehat R^\Pi
\bigl(
o,q_o^\Pi(i),a
\bigr)
&=
\widehat R^\Pi
\bigl(
o,[i]_o^\Pi,a
\bigr)\\
&=
R(o,i,a).
\end{aligned}
\]
Equivalently,
\[
R(o,i,a)
=
\widehat R^\Pi
\bigl(
o,q_o^\Pi(i),a
\bigr).
\]
Similarly, for every feasible $(o,a,o')$,
\[
\begin{aligned}
\widehat\tau_{o,a,o'}^\Pi
\bigl(
q_o^\Pi(i)
\bigr)
&=
\widehat\tau_{o,a,o'}^\Pi
\bigl(
[i]_o^\Pi
\bigr)\\
&=
\bigl[
\sigma_{o,a,o'}(i)
\bigr]_{o'}^\Pi\\
&=
q_{o'}^\Pi\!\left(
\sigma_{o,a,o'}(i)
\right).
\end{aligned}
\]
Equivalently,
\[
q_{o'}^\Pi\!\left(
\sigma_{o,a,o'}(i)
\right)
=
\widehat\tau_{o,a,o'}^\Pi
\bigl(
q_o^\Pi(i)
\bigr).
\]
The quotient maps are surjective, the class label sets are finite,
and the class reward and class transports satisfy exactly the two
identities in Definition~\ref{def:exact_abstraction_main}. Hence
$\{q_o^\Pi\}_{o\in\cO}$ defines an exact class abstraction.

Combining the two implications proves
$
\{q_o^\Pi\}_{o\in\cO}
\text{ defines an exact class abstraction}
\quad\Longleftrightarrow\quad
\Pi\preceq\mathcal T(\Pi).
$

Suppose now that $\Pi$ is stable. By
Definition~\ref{def:stable_family_main},
\[
\mathcal T(\Pi)=\Pi.
\]
Since the refinement order is reflexive,
\[
\Pi\preceq\Pi,
\]
and therefore
\[
\Pi
\preceq
\mathcal T(\Pi).
\]
The equivalence just established implies that every stable family
defines an exact class abstraction.

Exactness alone, however, requires only
\[
\Pi\preceq\mathcal T(\Pi);
\]
it does not require equality. Thus it is possible that
\[
\Pi\prec\mathcal T(\Pi),
\]
in which case $\Pi$ distinguishes some layers that
$\mathcal T(\Pi)$ can safely merge. For completeness, this strict
case can be realized by an explicit HCDP. Let
\[
n=2,
\qquad
\cO=\{o_0,o_1\},
\qquad
\cA=\{a\},
\]
and define the visible transition kernel by
\[
P_O(o_1\mid o_0,a)=1,
\qquad
P_O(o_1\mid o_1,a)=1.
\]
Assign the identity permutation to both feasible edges:
\[
\sigma_{o_0,a,o_1}
=
\operatorname{id}_{[2]},
\qquad
\sigma_{o_1,a,o_1}
=
\operatorname{id}_{[2]},
\]
and let every mean reward be zero:
\[
R(o,i,a)=0
\quad
\text{for every }o\in\cO
\text{ and }i\in[2].
\]
Consider the partition family
\[
\Pi_{o_0}
=
\bigl\{
\{1\},\{2\}
\bigr\},
\qquad
\Pi_{o_1}
=
\bigl\{
\{1,2\}
\bigr\}.
\]

At observation $o_1$, layers $1$ and $2$ have equal rewards. Their
successors under the only feasible edge are
\[
\sigma_{o_1,a,o_1}(1)=1,
\qquad
\sigma_{o_1,a,o_1}(2)=2,
\]
and these successors are equivalent under $\Pi_{o_1}$ because
\[
1\equiv_{o_1}^\Pi2.
\]
It follows from the definition of $\mathcal T$ that
\[
1\equiv_{o_1}^{\mathcal T(\Pi)}2,
\]
and hence
\[
\mathcal T(\Pi)_{o_1}
=
\bigl\{
\{1,2\}
\bigr\}.
\]

At observation $o_0$, layers $1$ and $2$ again have equal rewards.
Their successors under the only feasible edge are
\[
\sigma_{o_0,a,o_1}(1)=1,
\qquad
\sigma_{o_0,a,o_1}(2)=2.
\]
Both successors lie in the common block $\{1,2\}$ of $\Pi_{o_1}$,
so
\[
\sigma_{o_0,a,o_1}(1)
\equiv_{o_1}^\Pi
\sigma_{o_0,a,o_1}(2).
\]
Therefore
\[
1\equiv_{o_0}^{\mathcal T(\Pi)}2,
\]
and hence
\[
\mathcal T(\Pi)_{o_0}
=
\bigl\{
\{1,2\}
\bigr\}.
\]
We have consequently obtained
\[
\Pi\preceq\mathcal T(\Pi),
\]
because the two singleton blocks at $o_0$ refine the single block of
$\mathcal T(\Pi)_{o_0}$, while the partitions at $o_1$ are equal.
However,
\[
\Pi\neq\mathcal T(\Pi),
\]
because
\[
\Pi_{o_0}
=
\bigl\{
\{1\},\{2\}
\bigr\}
\neq
\bigl\{
\{1,2\}
\bigr\}
=
\mathcal T(\Pi)_{o_0}.
\]
The established criterion shows that this $\Pi$ is exact, but it is
not stable. Its separation of layers $1$ and $2$ at $o_0$ is
unnecessary: the two layers have identical immediate rewards and
are transported, under every feasible continuation, into the same
target quotient class.

Finally, stability was defined precisely as the fixed-point
condition
\[
\mathcal T(\Pi)=\Pi.
\]
Therefore, for an exact family, as for any partition family,
\[
\Pi
\text{ is stable}
\quad\Longleftrightarrow\quad
\Pi=\mathcal T(\Pi).
\]
This proves all assertions of the proposition.
\end{proof}

\subsection{Proof of Theorem~\ref{thm:coarsest_exact_main}}

\begin{proof}
By Lemma~\ref{lem:stable_fp_main}, the family $\Pi^\star$ obtained from the refinement iteration is a fixed point of the stability operator, so
\[
\mathcal T(\Pi^\star)=\Pi^\star.
\]
In particular,
\[
\Pi^\star\preceq\mathcal T(\Pi^\star).
\]
For each $o\in\cO$, the map
\[
q_o^\star:[n]\longrightarrow[n]/\Pi_o^\star,
\qquad
q_o^\star(i)=[i]_o,
\]
is exactly the canonical quotient map associated with the partition $\Pi_o^\star$. Proposition~\ref{prop:stable_exact_equiv_main}, applied with $\Pi=\Pi^\star$, therefore implies that the family $q^\star=\{q_o^\star\}_{o\in\cO}$ defines an exact class abstraction. This proves the existence assertion.

Let now
\[
q_o:[n]\longrightarrow\widehat{\cC}_o,
\qquad o\in\cO,
\]
be any other exact class abstraction. By Definition~\ref{def:exact_abstraction_main}, each $q_o$ is surjective, and there exist a class-level mean reward $\widehat R$ and, for every feasible triple $(o,a,o')$, a deterministic class transport
\[
\widehat\tau_{o,a,o'}:
\widehat{\cC}_o\longrightarrow\widehat{\cC}_{o'}
\]
such that, for every $i\in[n]$,
\[
R(o,i,a)=\widehat R(o,q_o(i),a)
\]
and
\[
q_{o'}\!\left(\sigma_{o,a,o'}(i)\right)
=
\widehat\tau_{o,a,o'}\!\left(q_o(i)\right)
\]
whenever $P_O(o'\mid o,a)>0$. Let $\Pi^q=\{\Pi_o^q\}_{o\in\cO}$ be the kernel partition family induced by $q$, namely
\[
i\equiv_o^{\Pi^q}j
\quad\Longleftrightarrow\quad
q_o(i)=q_o(j).
\]
We first show that this induced partition family is exact-compatible. Fix $o\in\cO$ and suppose that
\[
i\equiv_o^{\Pi^q}j.
\]
Then
\[
q_o(i)=q_o(j).
\]
For every action $a\in\cA$, exactness of $q$ gives
\[
\begin{aligned}
R(o,i,a)
&=\widehat R(o,q_o(i),a)\\
&=\widehat R(o,q_o(j),a)\\
&=R(o,j,a).
\end{aligned}
\]
Thus $i$ and $j$ have identical immediate mean rewards for every action. Now fix any $(a,o')$ satisfying
\[
P_O(o'\mid o,a)>0.
\]
Using the exact transport identity and again using $q_o(i)=q_o(j)$, we obtain
\[
\begin{aligned}
q_{o'}\!\left(\sigma_{o,a,o'}(i)\right)
&=\widehat\tau_{o,a,o'}\!\left(q_o(i)\right)\\
&=\widehat\tau_{o,a,o'}\!\left(q_o(j)\right)\\
&=q_{o'}\!\left(\sigma_{o,a,o'}(j)\right).
\end{aligned}
\]
By the definition of the kernel partition at $o'$, this equality is equivalent to
\[
\sigma_{o,a,o'}(i)
\equiv_{o'}^{\Pi^q}
\sigma_{o,a,o'}(j).
\]
Since the feasible pair $(a,o')$ was arbitrary, the two defining conditions of $\mathcal T(\Pi^q)$ are satisfied. Therefore
\[
i\equiv_o^{\mathcal T(\Pi^q)}j.
\]
We have proved that
\[
i\equiv_o^{\Pi^q}j
\quad\Longrightarrow\quad
i\equiv_o^{\mathcal T(\Pi^q)}j
\]
for every $o\in\cO$ and every $i,j\in[n]$. Equivalently,
\[
\Pi^q\preceq\mathcal T(\Pi^q).
\]

We next compare $\Pi^q$ with the fixed-point sequence from Lemma~\ref{lem:stable_fp_main}. Recall that $\Pi^{(0)}$ is the reward partition, so
\[
\begin{aligned}
i\equiv_o^{\Pi^{(0)}}j
\quad\Longleftrightarrow\quad&
R(o,i,a)=R(o,j,a)\\
&\text{for every }a\in\cA.
\end{aligned}
\]
If $i\equiv_o^{\Pi^q}j$, the reward calculation above shows that
\[
R(o,i,a)=R(o,j,a)
\quad\text{for every }a\in\cA.
\]
Hence
\[
i\equiv_o^{\Pi^{(0)}}j.
\]
This implication holds at every observation, and therefore
\[
\Pi^q\preceq\Pi^{(0)}.
\]
We claim that, more generally,
\[
\Pi^q\preceq\Pi^{(t)}
\quad\text{for every }t\geq0.
\]
The claim has just been established for $t=0$. Suppose that it holds for some $t\geq0$. Monotonicity of $\mathcal T$, proved in Lemma~\ref{lem:stable_fp_main}, gives
\[
\mathcal T(\Pi^q)
\preceq
\mathcal T(\Pi^{(t)}).
\]
Since $\Pi^q$ is exact-compatible and the iteration satisfies $\Pi^{(t+1)}=\mathcal T(\Pi^{(t)})$, we have
\[
\Pi^q
\preceq
\mathcal T(\Pi^q)
\preceq
\mathcal T(\Pi^{(t)})
=
\Pi^{(t+1)}.
\]
Thus
\[
\Pi^q\preceq\Pi^{(t+1)},
\]
and induction proves the claim for every $t\geq0$.

By Lemma~\ref{lem:stable_fp_main}, there exists a finite index $t_\star$ such that
\[
\Pi^{(t_\star)}=\Pi^\star.
\]
Evaluating the preceding refinement relation at $t=t_\star$ yields
\[
\Pi^q\preceq\Pi^\star.
\]
Unpacking this refinement relation, for every $o\in\cO$ and every $i,j\in[n]$,
\[
i\equiv_o^{\Pi^q}j
\quad\Longrightarrow\quad
i\equiv_o^{\Pi^\star}j.
\]
By the definitions of $\Pi^q$ and $q_o^\star$, this is precisely
\[
q_o(i)=q_o(j)
\quad\Longrightarrow\quad
q_o^\star(i)=q_o^\star(j).
\]
This proves that no exact abstraction can merge two layers that are separated by the stable quotient.

We now construct the factor map. Fix $o\in\cO$. For each $\hat c\in\widehat{\cC}_o$, surjectivity of $q_o$ guarantees that the fiber
\[
q_o^{-1}(\{\hat c\})
=
\{i\in[n]:q_o(i)=\hat c\}
\]
is nonempty. Choose any $i\in[n]$ with $q_o(i)=\hat c$ and define
\[
\kappa_o(\hat c)
:=
q_o^\star(i).
\]
This definition is independent of the chosen representative. Indeed, if $i,j\in[n]$ both satisfy
\[
q_o(i)=\hat c=q_o(j),
\]
then the implication proved above gives
\[
q_o^\star(i)=q_o^\star(j).
\]
Thus every representative of the fiber of $\hat c$ produces the same value, and
\[
\kappa_o:
\widehat{\cC}_o\longrightarrow\cC_o
\]
is a well-defined map.

For every $i\in[n]$, taking $\hat c=q_o(i)$ in the definition of $\kappa_o$ gives
\[
\begin{aligned}
(\kappa_o\circ q_o)(i)
&=\kappa_o(q_o(i))\\
&=q_o^\star(i).
\end{aligned}
\]
Therefore
\[
q_o^\star=\kappa_o\circ q_o.
\]
The map $\kappa_o$ is surjective. To see this, let $c\in\cC_o=[n]/\Pi_o^\star$ be arbitrary. Since $c$ is a nonempty block of the partition $\Pi_o^\star$, there exists $i\in[n]$ such that
\[
q_o^\star(i)=c.
\]
The factorization identity then gives
\[
c
=q_o^\star(i)
=\kappa_o(q_o(i)),
\]
so $c$ lies in the image of $\kappa_o$. Because $c$ was arbitrary,
\[
\kappa_o(\widehat{\cC}_o)=\cC_o.
\]

Finally, suppose that another map
\[
\lambda_o:
\widehat{\cC}_o\longrightarrow\cC_o
\]
satisfies
\[
q_o^\star=\lambda_o\circ q_o.
\]
Let $\hat c\in\widehat{\cC}_o$. Since $q_o$ is surjective, there exists $i\in[n]$ with
\[
q_o(i)=\hat c.
\]
Then
\[
\begin{aligned}
\lambda_o(\hat c)
&=\lambda_o(q_o(i))\\
&=q_o^\star(i)\\
&=\kappa_o(q_o(i))\\
&=\kappa_o(\hat c).
\end{aligned}
\]
Thus $\lambda_o=\kappa_o$, proving uniqueness. Conversely, the existence of any factorization $q_o^\star=\kappa_o\circ q_o$ immediately implies that
\[
q_o(i)=q_o(j)
\quad\Longrightarrow\quad
\begin{aligned}[t]
q_o^\star(i)
&=\kappa_o(q_o(i))\\
&=\kappa_o(q_o(j))\\
&=q_o^\star(j),
\end{aligned}
\]
so the implication and the factorization statements are indeed equivalent.

We have shown that the kernel partition $\Pi^q$ of every exact abstraction satisfies
\[
\Pi^q\preceq\Pi^\star,
\]
while $\Pi^\star$ itself defines the exact abstraction $q^\star$. Hence every exact observation-wise abstraction refines the stable quotient. Equivalently, every exact label set admits a surjection onto the corresponding stable label set; in particular,
\[
|\cC_o|
\leq
|\widehat{\cC}_o|
\quad\text{for every }o\in\cO.
\]
Therefore $\Pi^\star$ is the coarsest exact observation-wise partition family, and $q^\star$ is the coarsest exact class abstraction.
\end{proof}

\subsection{Proof of Theorem~\ref{thm:markovization_main}}

\begin{proof}
Recall that $\Pi^\star$ is stable, and hence
\[
\mathcal T(\Pi^\star)=\Pi^\star.
\]
Fix an observation $o\in\cO$ and two layers $i,j\in[n]$ belonging to the same stable class, so that
\[
[i]_o=[j]_o,
\]
or equivalently,
\[
i\equiv_o^{\Pi^\star}j.
\]
Since $\Pi^\star=\mathcal T(\Pi^\star)$, this also means
\[
i\equiv_o^{\mathcal T(\Pi^\star)}j.
\]
Expanding the definition of the stability operator gives, for every action $a\in\cA$,
\[
R(o,i,a)=R(o,j,a),
\]
and, for every $o'\in\cO$ satisfying $P_O(o'\mid o,a)>0$,
\[
\sigma_{o,a,o'}(i)
\equiv_{o'}^{\Pi^\star}
\sigma_{o,a,o'}(j).
\]
The latter equivalence is precisely the equality of the corresponding stable classes:
\[
[\sigma_{o,a,o'}(i)]_{o'}
=
[\sigma_{o,a,o'}(j)]_{o'}.
\]
It follows that replacing $i$ by any other representative $j$ of the same class $[i]_o$ leaves both
\[
R(o,i,a)
\]
and
\[
[\sigma_{o,a,o'}(i)]_{o'}
\]
unchanged. Therefore the formulas
\[
\bar R(o,[i]_o,a):=R(o,i,a)
\]
and
\[
\tau_{o,a,o'}([i]_o)
:=
[\sigma_{o,a,o'}(i)]_{o'}
\]
are independent of the representative and define, respectively, a function
\[
\bar R:
\bar\cS\times\cA\longrightarrow\mathbb R
\]
and, for every feasible triple $(o,a,o')$, a deterministic map
\[
\tau_{o,a,o'}:
\cC_o\longrightarrow\cC_{o'}.
\]
For a triple satisfying $P_O(o'\mid o,a)=0$, neither the raw transport nor the quotient transport affects any transition probability. For notational uniformity in sums over all $o'\in\cO$, we may extend both $\sigma_{o,a,o'}$ and $\tau_{o,a,o'}$ arbitrarily to such zero-probability triples; these arbitrary values are always multiplied by $P_O(o'\mid o,a)=0$ and therefore do not change any kernel, expectation, or value function.

We next verify that the displayed formulas define a finite discounted MDP. The state space
\[
\bar\cS
=
\{(o,c):o\in\cO,\ c\in\cC_o\}
\]
is finite because $\cO$ and $[n]$ are finite. The reward $\bar R$ is bounded by the same bound as the latent reward, since for every $(o,c)\in\bar\cS$ one may choose $i\in[n]$ with $c=[i]_o$ and obtain
\[
|\bar R(o,c,a)|
=
|R(o,i,a)|
\leq
r_{\max}.
\]
For every $(o,c)\in\bar\cS$ and $a\in\cA$, the proposed transition probabilities are nonnegative because
\[
P_O(o'\mid o,a)\geq0
\]
and the indicator is nonnegative. Their total mass is
\[
\begin{aligned}
&\sum_{o'\in\cO}
\sum_{c'\in\cC_{o'}}
\bar P\bigl((o',c')\mid(o,c),a\bigr)\\
&\quad=
\sum_{o'\in\cO}
\sum_{c'\in\cC_{o'}}
P_O(o'\mid o,a)
\mathbf 1\!\left\{
c'=\tau_{o,a,o'}(c)
\right\}\\
&\quad=
\sum_{o'\in\cO}
P_O(o'\mid o,a)
\sum_{c'\in\cC_{o'}}
\mathbf 1\!\left\{
c'=\tau_{o,a,o'}(c)
\right\}.
\end{aligned}
\]
For each $o'$ with $P_O(o'\mid o,a)>0$, the element $\tau_{o,a,o'}(c)$ is a single well-defined member of $\cC_{o'}$, so exactly one term in the inner sum is equal to one and all other terms are zero. For $o'$ with $P_O(o'\mid o,a)=0$, the entire corresponding summand vanishes. Hence
\[
\begin{aligned}
&\sum_{o'\in\cO}
\sum_{c'\in\cC_{o'}}
\bar P\bigl((o',c')\mid(o,c),a\bigr)\\
&\quad=
\sum_{o'\in\cO}
P_O(o'\mid o,a)\\
&\quad=1,
\end{aligned}
\]
because $P_O(\cdot\mid o,a)$ is a probability distribution on $\cO$. Thus $\bar P(\cdot\mid(o,c),a)$ is a probability distribution on $\bar\cS$.

The proposed initial law is also nonnegative. Since the classes in $\cC_o=[n]/\Pi_o^\star$ form a partition of $[n]$, each $i\in[n]$ belongs to exactly one class $c\in\cC_o$. Therefore
\[
\begin{aligned}
\sum_{(o,c)\in\bar\cS}
\bar\rho_0(o,c)
&=
\sum_{o\in\cO}
\sum_{c\in\cC_o}
\sum_{i:\,[i]_o=c}
\rho_0(o,i)\\
&=
\sum_{o\in\cO}
\sum_{i=1}^{n}
\rho_0(o,i)\\
&=1.
\end{aligned}
\]
Hence $\bar\rho_0\in\Delta(\bar\cS)$. Together with the original discount factor $\gamma\in(0,1)$, these facts show that
\[
(\bar\cS,\cA,\bar P,\bar R,\gamma,\bar\rho_0)
\]
is a well-defined finite discounted MDP.

The quotient kernel is exactly the pushforward of the latent HCDP kernel under the map
\[
q^\star(o,i):=(o,[i]_o).
\]
Indeed, fix $(o,i)\in\cO\times[n]$, put $c=[i]_o$, and fix $a\in\cA$, $o'\in\cO$, and $c'\in\cC_{o'}$. The probability that the latent next state has observation $o'$ and stable class $c'$ is
\[
\begin{aligned}
&\sum_{i':\,[i']_{o'}=c'}
P_H\bigl((o',i')\mid(o,i),a\bigr)\\
&\quad=
\sum_{i':\,[i']_{o'}=c'}
P_O(o'\mid o,a)
\mathbf 1\!\left\{
i'=\sigma_{o,a,o'}(i)
\right\}\\
&\quad=
P_O(o'\mid o,a)
\sum_{i':\,[i']_{o'}=c'}
\mathbf 1\!\left\{
i'=\sigma_{o,a,o'}(i)
\right\}\\
&\quad=
P_O(o'\mid o,a)
\mathbf 1\!\left\{
[\sigma_{o,a,o'}(i)]_{o'}=c'
\right\}\\
&\quad=
P_O(o'\mid o,a)
\mathbf 1\!\left\{
\tau_{o,a,o'}([i]_o)=c'
\right\}\\
&\quad=
\bar P\bigl((o',c')\mid(o,[i]_o),a\bigr).
\end{aligned}
\]
The fourth equality holds because the deterministic successor layer $\sigma_{o,a,o'}(i)$ either belongs to the class $c'$ or does not, and the fifth equality is the definition of the quotient transport. This calculation makes explicit that, once the current stable class is known, the conditional law of the next observation--class pair is independent of the latent representative $i$.

Let now $\bar\pi$ be an arbitrary stationary policy on the quotient MDP, and define its lift to the fully observed latent state space by
\[
\pi^\uparrow(a\mid o,i)
:=
\bar\pi(a\mid o,[i]_o).
\]
This is a valid stationary policy because, for every $(o,i)$,
\[
\pi^\uparrow(a\mid o,i)\geq0
\]
for all $a$, and
\[
\begin{aligned}
\sum_{a\in\cA}
\pi^\uparrow(a\mid o,i)
&=
\sum_{a\in\cA}
\bar\pi(a\mid o,[i]_o)\\
&=1.
\end{aligned}
\]
To compare the two value functions, let $\mathbb B(\bar\cS)$ denote the bounded real-valued functions on $\bar\cS$, and define the lifting operator
\[
L:\mathbb B(\bar\cS)
\longrightarrow
\mathbb B(\cO\times[n])
\]
by
\[
(Lf)(o,i):=f(o,[i]_o).
\]
Define the policy-evaluation operator of the latent HCDP under $\pi^\uparrow$ by
$
(T_{\pi^\uparrow}V)(o,i)
:=
\sum_{a\in\cA}
\pi^\uparrow(a\mid o,i)
\Biggl[
R(o,i,a)
+
\gamma
\sum_{o'\in\cO}
\sum_{i'=1}^{n}
P_H\bigl((o',i')\mid(o,i),a\bigr)
V(o',i')
\Biggr],
$
and define the quotient policy-evaluation operator by
$
(\bar T_{\bar\pi}f)(o,c)
:=
\sum_{a\in\cA}
\bar\pi(a\mid o,c)
\Biggl[
\bar R(o,c,a)
+
\gamma
\sum_{o'\in\cO}
\sum_{c'\in\cC_{o'}}
\bar P\bigl((o',c')\mid(o,c),a\bigr)
f(o',c')
\Biggr].
$
Both operators are $\gamma$-contractions under the sup norm. For example, for bounded $V,W$ and every $(o,i)$,
\[
\begin{aligned}
&\left|
(T_{\pi^\uparrow}V)(o,i)
-
(T_{\pi^\uparrow}W)(o,i)
\right|\\
&\quad\leq
\gamma
\sum_{a\in\cA}
\pi^\uparrow(a\mid o,i)\\
&\sum_{o'\in\cO}
\sum_{i'=1}^{n}
P_H\bigl((o',i')\mid(o,i),a\bigr)
\left|V(o',i')-W(o',i')\right|\\
&\quad\leq
\gamma\|V-W\|_\infty
\sum_{a\in\cA}
\pi^\uparrow(a\mid o,i)\\
&\sum_{o'\in\cO}
\sum_{i'=1}^{n}
P_H\bigl((o',i')\mid(o,i),a\bigr)\\
&\quad=
\gamma\|V-W\|_\infty.
\end{aligned}
\]
Taking the supremum over $(o,i)$ gives
\[
\|T_{\pi^\uparrow}V-T_{\pi^\uparrow}W\|_\infty
\leq
\gamma\|V-W\|_\infty.
\]
The same calculation with $\bar P$ and $\bar\pi$ gives
\[
\|\bar T_{\bar\pi}f-\bar T_{\bar\pi}g\|_\infty
\leq
\gamma\|f-g\|_\infty.
\]
Consequently, $T_{\pi^\uparrow}$ and $\bar T_{\bar\pi}$ have unique bounded fixed points, namely $V_{\pi^\uparrow}$ and $\bar V_{\bar\pi}$.

The lifting operator intertwines these two Bellman evaluation operators. To verify this identity, fix $f\in\mathbb B(\bar\cS)$ and $(o,i)\in\cO\times[n]$. Substituting the definitions of $\pi^\uparrow$, $P_H$, and $L$ gives
\[
\begin{aligned}
&(T_{\pi^\uparrow}Lf)(o,i)\\
&\quad=
\sum_{a\in\cA}
\bar\pi(a\mid o,[i]_o)
\Biggl[
R(o,i,a)\\
&+
\gamma
\sum_{o'\in\cO}
\sum_{i'=1}^{n}
P_O(o'\mid o,a)
\mathbf 1\!\left\{
i'=\sigma_{o,a,o'}(i)
\right\}
f(o',[i']_{o'})
\Biggr]\\
&\quad=
\sum_{a\in\cA}
\bar\pi(a\mid o,[i]_o)
\Biggl[
R(o,i,a)\\
&+
\gamma
\sum_{o'\in\cO}
P_O(o'\mid o,a)
f\bigl(o',[\sigma_{o,a,o'}(i)]_{o'}\bigr)
\Biggr]\\
&\quad=
\sum_{a\in\cA}
\bar\pi(a\mid o,[i]_o)
\Biggl[
\bar R(o,[i]_o,a)\\
&+
\gamma
\sum_{o'\in\cO}
P_O(o'\mid o,a)
f\bigl(o',\tau_{o,a,o'}([i]_o)\bigr)
\Biggr].
\end{aligned}
\]
For fixed $o'$, the definition of $\bar P$ implies
\[
\begin{aligned}
&\sum_{c'\in\cC_{o'}}
\bar P\bigl((o',c')\mid(o,[i]_o),a\bigr)
f(o',c')\\
&\quad=
\sum_{c'\in\cC_{o'}}
P_O(o'\mid o,a)
\mathbf 1\!\left\{
c'=\tau_{o,a,o'}([i]_o)
\right\}
f(o',c')\\
&\quad=
P_O(o'\mid o,a)
f\bigl(o',\tau_{o,a,o'}([i]_o)\bigr).
\end{aligned}
\]
Substituting this equality into the preceding Bellman expression yields
\[
\begin{aligned}
(T_{\pi^\uparrow}Lf)(o,i)
&=
(\bar T_{\bar\pi}f)(o,[i]_o)\\
&=
(L\bar T_{\bar\pi}f)(o,i).
\end{aligned}
\]
Thus, as an operator identity,
\[
T_{\pi^\uparrow}L
=
L\bar T_{\bar\pi}.
\]
Since $\bar V_{\bar\pi}$ is the fixed point of $\bar T_{\bar\pi}$,
\[
\bar T_{\bar\pi}\bar V_{\bar\pi}
=
\bar V_{\bar\pi}.
\]
Applying the intertwining identity gives
\[
\begin{aligned}
T_{\pi^\uparrow}(L\bar V_{\bar\pi})
&=
L(\bar T_{\bar\pi}\bar V_{\bar\pi})\\
&=
L\bar V_{\bar\pi}.
\end{aligned}
\]
Therefore $L\bar V_{\bar\pi}$ is a fixed point of $T_{\pi^\uparrow}$. By uniqueness of the fixed point of $T_{\pi^\uparrow}$,
\[
V_{\pi^\uparrow}
=
L\bar V_{\bar\pi}.
\]
Evaluating this identity at $(o,i)$ gives
\[
V_{\pi^\uparrow}(o,i)
=
\bar V_{\bar\pi}(o,[i]_o),
\]
which proves the claimed policy-value preservation.

It remains to establish the optimality assertion. Let the Bellman optimality operator of the fully observed latent HCDP be
$
(T_{\mathrm{full}}^\star V)(o,i)
:=
\max_{a\in\cA}
\Biggl[
R(o,i,a)
+
\gamma
\sum_{o'\in\cO}
\sum_{i'=1}^{n}
P_H\bigl((o',i')\mid(o,i),a\bigr)
V(o',i')
\Biggr],
$
and let the quotient Bellman optimality operator be
$
(\bar T^\star f)(o,c)
:=
\max_{a\in\cA}
\Biggl[
\bar R(o,c,a)
+
\gamma
\sum_{o'\in\cO}
\sum_{c'\in\cC_{o'}}
\bar P\bigl((o',c')\mid(o,c),a\bigr)
f(o',c')
\Biggr].
$
Because both state spaces and the action set are finite and $\gamma\in(0,1)$, these operators are $\gamma$-contractions under the sup norm. Their unique fixed points are, respectively, the fully observed latent optimal value $V_{\mathrm{full}}^\star$ and the quotient optimal value $\bar V^\star$.

The same calculation used for policy evaluation, now performed before taking the maximum over actions, shows that for every $f\in\mathbb B(\bar\cS)$,
\[
T_{\mathrm{full}}^\star Lf
=
L\bar T^\star f.
\]
Indeed, for every $(o,i)$,
\[
\begin{aligned}
&(T_{\mathrm{full}}^\star Lf)(o,i)\\
&\quad=
\max_{a\in\cA}
\Biggl[
R(o,i,a)\\
&+
\gamma
\sum_{o'\in\cO}
P_O(o'\mid o,a)
f\bigl(o',[\sigma_{o,a,o'}(i)]_{o'}\bigr)
\Biggr]\\
&\quad=
\max_{a\in\cA}
\Biggl[
\bar R(o,[i]_o,a)\\
&+
\gamma
\sum_{o'\in\cO}
P_O(o'\mid o,a)
f\bigl(o',\tau_{o,a,o'}([i]_o)\bigr)
\Biggr]\\
&\quad=
(\bar T^\star f)(o,[i]_o)\\
&\quad=
(L\bar T^\star f)(o,i).
\end{aligned}
\]
Since $\bar V^\star$ is the unique fixed point of $\bar T^\star$,
\[
\bar T^\star\bar V^\star
=
\bar V^\star.
\]
Consequently,
\[
\begin{aligned}
T_{\mathrm{full}}^\star(L\bar V^\star)
&=
L(\bar T^\star\bar V^\star)\\
&=
L\bar V^\star.
\end{aligned}
\]
Thus $L\bar V^\star$ is a fixed point of $T_{\mathrm{full}}^\star$. Uniqueness of the fully observed optimal fixed point implies
\[
V_{\mathrm{full}}^\star
=
L\bar V^\star,
\]
and therefore, for every $(o,i)$,
\[
V_{\mathrm{full}}^\star(o,i)
=
\bar V^\star(o,[i]_o).
\]

Finally, because the quotient MDP is finite and discounted, there exists a stationary deterministic quotient policy $\bar\pi^\star$ attaining the maximum in the Bellman optimality equation at every quotient state. Define its lift by
\[
(\pi^\star)^\uparrow(a\mid o,i)
:=
\bar\pi^\star(a\mid o,[i]_o).
\]
The already proved policy-value identity gives
\[
\begin{aligned}
V_{(\pi^\star)^\uparrow}(o,i)
&=
\bar V_{\bar\pi^\star}(o,[i]_o)\\
&=
\bar V^\star(o,[i]_o)\\
&=
V_{\mathrm{full}}^\star(o,i).
\end{aligned}
\]
Hence $(\pi^\star)^\uparrow$ is optimal for the fully observed latent HCDP. By construction, whenever $[i]_o=[j]_o$,
$
(\pi^\star)^\uparrow(\cdot\mid o,i)
=
\bar\pi^\star(\cdot\mid o,[i]_o)
=
\bar\pi^\star(\cdot\mid o,[j]_o)
=
(\pi^\star)^\uparrow(\cdot\mid o,j).
$
Thus a fully observed latent optimal policy can be chosen constant on stable classes, completing the proof.
\end{proof}

\subsection{Proof of Corollary~\ref{cor:path_transport_descends_main}}

\begin{proof}
Write the directed path as
\[
\omega=e_1\cdots e_k,
\qquad
e_j=(o_{j-1}\xrightarrow{a_{j-1}}o_j),
\]
where $o_0=o$ and $o_k=o'$. For each edge $e_j$, use the abbreviations
\[
\sigma_{e_j}:=\sigma_{o_{j-1},a_{j-1},o_j}
\qquad\text{and}\qquad
\tau_{e_j}:=\tau_{o_{j-1},a_{j-1},o_j}.
\]
Because $e_j$ is an edge of the directed support graph, it is feasible, and hence
\[
P_O(o_j\mid o_{j-1},a_{j-1})>0.
\]
Theorem~\ref{thm:markovization_main} therefore applies to every $e_j$ and shows that
\[
\tau_{e_j}:\cC_{o_{j-1}}\longrightarrow\cC_{o_j}
\]
is a well-defined map satisfying, for every $r\in[n]$,
\[
\begin{aligned}
\tau_{e_j}([r]_{o_{j-1}})
&=
\tau_{o_{j-1},a_{j-1},o_j}([r]_{o_{j-1}})\\
&=
[\sigma_{o_{j-1},a_{j-1},o_j}(r)]_{o_j}\\
&=
[\sigma_{e_j}(r)]_{o_j}.
\end{aligned}
\]
Thus each raw one-edge transport descends through the stable quotient. Moreover, the codomain $\cC_{o_j}$ of $\tau_{e_j}$ is exactly the domain of $\tau_{e_{j+1}}$. Consequently, the composition
\[
\tau_\omega
=
\tau_{e_k}\circ\cdots\circ\tau_{e_1}:
\cC_{o_0}\longrightarrow\cC_{o_k}
\]
is well defined.

For $j=0,\ldots,k$, let
\[
\omega^{(j)}:=e_1\cdots e_j
\]
denote the length-$j$ prefix of $\omega$, with
\[
\omega^{(0)}:=\varnothing_{o_0}.
\]
Define
\[
\sigma_{\omega^{(0)}}:=\operatorname{id}_{[n]},
\qquad
\tau_{\omega^{(0)}}:=\operatorname{id}_{\cC_{o_0}},
\]
and, for $j\geq1$, define
\[
\sigma_{\omega^{(j)}}
:=
\sigma_{e_j}\circ\cdots\circ\sigma_{e_1},
\qquad
\tau_{\omega^{(j)}}
:=
\tau_{e_j}\circ\cdots\circ\tau_{e_1}.
\]
We prove that, for every $j=0,\ldots,k$ and every $i\in[n]$,
\[
\tau_{\omega^{(j)}}([i]_{o_0})
=
[\sigma_{\omega^{(j)}}(i)]_{o_j}.
\]
For $j=0$, both prefix transports are identity maps, and therefore
\[
\begin{aligned}
\tau_{\omega^{(0)}}([i]_{o_0})
&=
\operatorname{id}_{\cC_{o_0}}([i]_{o_0})\\
&=
[i]_{o_0}\\
&=
[\operatorname{id}_{[n]}(i)]_{o_0}\\
&=
[\sigma_{\omega^{(0)}}(i)]_{o_0}.
\end{aligned}
\]
Hence the claimed identity holds for the empty prefix.

Now fix $j\in\{1,\ldots,k\}$ and suppose that the identity holds for the prefix $\omega^{(j-1)}$. Thus, for every $i\in[n]$,
\[
\tau_{\omega^{(j-1)}}([i]_{o_0})
=
[\sigma_{\omega^{(j-1)}}(i)]_{o_{j-1}}.
\]
By the definitions of the prefix transports,
\[
\tau_{\omega^{(j)}}
=
\tau_{e_j}\circ\tau_{\omega^{(j-1)}}
\]
and
\[
\sigma_{\omega^{(j)}}
=
\sigma_{e_j}\circ\sigma_{\omega^{(j-1)}}.
\]
Applying the quotient prefix transport to $[i]_{o_0}$ and then using the induction hypothesis gives
\[
\begin{aligned}
\tau_{\omega^{(j)}}([i]_{o_0})
&=
(\tau_{e_j}\circ\tau_{\omega^{(j-1)}})([i]_{o_0})\\
&=
\tau_{e_j}\!\left(
\tau_{\omega^{(j-1)}}([i]_{o_0})
\right)\\
&=
\tau_{e_j}\!\left(
[\sigma_{\omega^{(j-1)}}(i)]_{o_{j-1}}
\right).
\end{aligned}
\]
The element
\[
\sigma_{\omega^{(j-1)}}(i)
\]
belongs to $[n]$, so the one-edge descent identity for $e_j$ may be applied with
\[
r=\sigma_{\omega^{(j-1)}}(i).
\]
It follows that
\[
\begin{aligned}
\tau_{e_j}\!\left(
[\sigma_{\omega^{(j-1)}}(i)]_{o_{j-1}}
\right)
&=
\left[
\sigma_{e_j}\!\left(
\sigma_{\omega^{(j-1)}}(i)
\right)
\right]_{o_j}\\
&=
\left[
(\sigma_{e_j}\circ\sigma_{\omega^{(j-1)}})(i)
\right]_{o_j}\\
&=
[\sigma_{\omega^{(j)}}(i)]_{o_j}.
\end{aligned}
\]
Combining the preceding equalities yields
\[
\tau_{\omega^{(j)}}([i]_{o_0})
=
[\sigma_{\omega^{(j)}}(i)]_{o_j}.
\]
This completes the induction over all path prefixes.

Taking $j=k$, and using
\[
o_0=o,
\qquad
o_k=o',
\qquad
\omega^{(k)}=\omega,
\]
together with
\[
\tau_\omega
=
\tau_{e_k}\circ\cdots\circ\tau_{e_1}
\]
and
\[
\sigma_\omega
=
\sigma_{e_k}\circ\cdots\circ\sigma_{e_1},
\]
we obtain, for every $i\in[n]$,
\[
\tau_\omega([i]_o)
=
[\sigma_\omega(i)]_{o'},
\]
as required.
\end{proof}

\subsection{Proof of Lemma~\ref{lem:directed_quotient_holonomy_main}}

\begin{proof}
Fix an observation $o\in\cO$ and a directed closed walk $\omega$ based at $o$. Since the initial and terminal observations of $\omega$ are both $o$, its raw transport is a self-map
\[
\sigma_\omega:[n]\longrightarrow[n].
\]
By Definition~\ref{def:directed_transport_main}, $\sigma_\omega$ is a composition of edge permutations. More explicitly, if
\[
\omega=e_1\cdots e_k,
\]
then
\[
\sigma_\omega
=
\sigma_{e_k}\circ\cdots\circ\sigma_{e_1}.
\]
Every $\sigma_{e_j}$ belongs to $S_n$ and is therefore bijective. A composition of bijections is bijective, so
\[
\sigma_\omega\in S_n.
\]
Let
\[
q_o^\star:[n]\longrightarrow\cC_o,
\qquad
q_o^\star(i)=[i]_o,
\]
be the canonical stable quotient map. Corollary~\ref{cor:path_transport_descends_main} gives, for every $i\in[n]$,
\[
\tau_\omega\bigl(q_o^\star(i)\bigr)
=
\tau_\omega([i]_o)
=
[\sigma_\omega(i)]_o
=
q_o^\star\bigl(\sigma_\omega(i)\bigr).
\]
Thus the raw and quotient transports satisfy the commuting relation
\[
\tau_\omega\circ q_o^\star
=
q_o^\star\circ\sigma_\omega.
\]

We first show directly that $\tau_\omega$ is surjective. Let $c'\in\cC_o$ be arbitrary. Since $q_o^\star$ is the quotient map onto
\[
\cC_o=[n]/\Pi_o^\star,
\]
it is surjective. Hence there exists $j\in[n]$ such that
\[
q_o^\star(j)=c'.
\]
Because $\sigma_\omega$ is bijective, its inverse $\sigma_\omega^{-1}$ exists. Define
\[
i:=\sigma_\omega^{-1}(j)
\]
and let
\[
c:=q_o^\star(i)\in\cC_o.
\]
By the definition of $i$,
\[
\sigma_\omega(i)
=
\sigma_\omega\bigl(\sigma_\omega^{-1}(j)\bigr)
=
j.
\]
Using the commuting relation, we therefore obtain
\[
\begin{aligned}
\tau_\omega(c)
&=
\tau_\omega\bigl(q_o^\star(i)\bigr)\\
&=
q_o^\star\bigl(\sigma_\omega(i)\bigr)\\
&=
q_o^\star(j)\\
&=
c'.
\end{aligned}
\]
Since $c'\in\cC_o$ was arbitrary, every element of $\cC_o$ has a preimage under $\tau_\omega$. Thus
\[
\tau_\omega:\cC_o\longrightarrow\cC_o
\]
is surjective.

The set $\cC_o$ is finite because it is the set of blocks of a partition of the finite set $[n]$. A surjective self-map of a finite set is injective. Indeed, suppose for contradiction that $\tau_\omega$ were not injective. Then there would exist two distinct classes $c_1,c_2\in\cC_o$ such that
\[
\tau_\omega(c_1)=\tau_\omega(c_2).
\]
Consequently, the image of $\tau_\omega$ would contain at most
\[
|\cC_o|-1
\]
distinct elements, contradicting the already established surjectivity
\[
\tau_\omega(\cC_o)=\cC_o.
\]
Therefore $\tau_\omega$ is injective as well as surjective. Hence it is bijective, and
\[
\tau_\omega\in\Sym(\cC_o).
\]

Now consider the collection
\[
\Hol_o^{\mathrm{dir}}
=
\left\{
\tau_\omega:
\omega\text{ is a directed closed walk based at }o
\right\}.
\]
The empty walk $\varnothing_o$ is a directed closed walk based at $o$. By Definition~\ref{def:directed_transport_main},
\[
\sigma_{\varnothing_o}
=
\operatorname{id}_{[n]}.
\]
Applying Corollary~\ref{cor:path_transport_descends_main} to the empty walk gives, for every $i\in[n]$,
\[
\begin{aligned}
\tau_{\varnothing_o}([i]_o)
&=
[\sigma_{\varnothing_o}(i)]_o\\
&=
[\operatorname{id}_{[n]}(i)]_o\\
&=
[i]_o.
\end{aligned}
\]
Every element of $\cC_o$ is a block of $\Pi_o^\star$ and therefore has the form $[i]_o$ for some $i\in[n]$. It follows that
\[
\tau_{\varnothing_o}
=
\operatorname{id}_{\cC_o}.
\]
Consequently,
\[
\operatorname{id}_{\cC_o}
\in
\Hol_o^{\mathrm{dir}}.
\]

Let $\omega_1$ and $\omega_2$ be arbitrary directed closed walks based at $o$, where $\omega_1$ is executed first and $\omega_2$ is executed second. Their concatenation
\[
\omega_2\circ\omega_1
\]
is again a directed walk beginning at $o$ and ending at $o$, and hence is a directed closed walk based at $o$. By the composition convention in Definition~\ref{def:directed_transport_main},
\[
\sigma_{\omega_2\circ\omega_1}
=
\sigma_{\omega_2}\circ\sigma_{\omega_1}.
\]
For every $i\in[n]$, Corollary~\ref{cor:path_transport_descends_main} therefore gives
\[
\begin{aligned}
\tau_{\omega_2\circ\omega_1}([i]_o)
&=
[\sigma_{\omega_2\circ\omega_1}(i)]_o\\
&=
[(\sigma_{\omega_2}\circ\sigma_{\omega_1})(i)]_o\\
&=
[\sigma_{\omega_2}(\sigma_{\omega_1}(i))]_o.
\end{aligned}
\]
Applying the same corollary first to $\omega_1$ and then to $\omega_2$ gives
\[
\tau_{\omega_1}([i]_o)
=
[\sigma_{\omega_1}(i)]_o
\]
and
\[
\tau_{\omega_2}\bigl([\sigma_{\omega_1}(i)]_o\bigr)
=
[\sigma_{\omega_2}(\sigma_{\omega_1}(i))]_o.
\]
Combining these identities yields
\[
\begin{aligned}
\tau_{\omega_2\circ\omega_1}([i]_o)
&=
[\sigma_{\omega_2}(\sigma_{\omega_1}(i))]_o\\
&=
\tau_{\omega_2}\bigl([\sigma_{\omega_1}(i)]_o\bigr)\\
&=
\tau_{\omega_2}\bigl(\tau_{\omega_1}([i]_o)\bigr)\\
&=
(\tau_{\omega_2}\circ\tau_{\omega_1})([i]_o).
\end{aligned}
\]
Since every element of $\cC_o$ has the form $[i]_o$, equality on all such classes implies equality of the two maps. Therefore
\[
\tau_{\omega_2\circ\omega_1}
=
\tau_{\omega_2}\circ\tau_{\omega_1}.
\]
Because $\omega_2\circ\omega_1$ is again a directed closed walk based at $o$, it follows that
\[
\tau_{\omega_2}\circ\tau_{\omega_1}
\in
\Hol_o^{\mathrm{dir}}.
\]
Thus $\Hol_o^{\mathrm{dir}}$ is closed under composition. Together with the presence of the identity, this shows that $\Hol_o^{\mathrm{dir}}$ is a submonoid of $\Sym(\cC_o)$.

It remains to verify that the inverse of every element also belongs to $\Hol_o^{\mathrm{dir}}$. Fix an arbitrary element
\[
g\in\Hol_o^{\mathrm{dir}}.
\]
By definition, there exists a directed closed walk $\omega$ based at $o$ such that
\[
g=\tau_\omega.
\]
The first part of the proof shows that
\[
g\in\Sym(\cC_o).
\]
Since $\cC_o$ is finite, the symmetric group $\Sym(\cC_o)$ is finite. Therefore the infinite sequence
$
\operatorname{id}_{\cC_o},
\quad
g,
\quad
g^2,
\quad
g^3,
\quad
\ldots
$
cannot consist of pairwise distinct permutations. Hence there exist integers
\[
0\leq r<s
\]
such that
\[
g^r=g^s.
\]
Because $g$ is bijective, $g^r$ is also bijective. We may therefore compose both sides on the left with the inverse permutation $(g^r)^{-1}$ inside $\Sym(\cC_o)$, obtaining
\[
\begin{aligned}
(g^r)^{-1}\circ g^r
&=
(g^r)^{-1}\circ g^s,\\
\operatorname{id}_{\cC_o}
&=
g^{s-r}.
\end{aligned}
\]
Set
\[
m:=s-r.
\]
Then $m\geq1$ and
\[
g^m
=
\operatorname{id}_{\cC_o}.
\]
It follows that
\[
\begin{aligned}
g\circ g^{m-1}
&=
g^m\\
&=
\operatorname{id}_{\cC_o},
\end{aligned}
\]
and likewise
\[
\begin{aligned}
g^{m-1}\circ g
&=
g^m\\
&=
\operatorname{id}_{\cC_o}.
\end{aligned}
\]
Therefore
\[
g^{-1}=g^{m-1}.
\]

Because $\Hol_o^{\mathrm{dir}}$ contains the identity and is closed under composition, it contains every nonnegative power of $g$. In particular,
\[
g^{m-1}\in\Hol_o^{\mathrm{dir}}.
\]
More explicitly, for each integer $\ell\geq1$, let
\[
\omega^{\langle\ell\rangle}
\]
denote the directed closed walk obtained by executing $\omega$ successively $\ell$ times, and let
\[
\omega^{\langle0\rangle}:=\varnothing_o.
\]
Repeated application of the composition identity gives
\[
\tau_{\omega^{\langle\ell\rangle}}
=
(\tau_\omega)^\ell
=
g^\ell.
\]
Consequently,
\[
g^{-1}
=
g^{m-1}
=
\tau_{\omega^{\langle m-1\rangle}}
\in
\Hol_o^{\mathrm{dir}}.
\]
When $m=1$, the exponent $m-1$ is zero and the inverse is induced by the empty walk. This construction does not require an executable edge-wise reversal of $\omega$; the inverse permutation is realized algebraically by finitely many repetitions of the same directed closed walk.

We have shown that $\Hol_o^{\mathrm{dir}}$ contains the identity permutation, is closed under composition, and contains the inverse of every one of its elements. Therefore
\[
\Hol_o^{\mathrm{dir}}
\leq
\Sym(\cC_o).
\]
Finally, because $\cC_o$ is finite,
\[
|\Hol_o^{\mathrm{dir}}|
\leq
|\Sym(\cC_o)|
=
|\cC_o|!
<
\infty.
\]
Hence $\Hol_o^{\mathrm{dir}}$ is a finite subgroup of $\Sym(\cC_o)$, with the empty walk providing the identity.
\end{proof}

\subsection{Proof of Proposition~\ref{prop:quotient_belief_main}}

\begin{proof}
Fix a time $t$ and a positive-probability realization of the controller history
\[
H_t=(o_0,a_0,\ldots,a_{t-1},o_t).
\]
All conditional probabilities below are taken under the policy generating the interaction. Let
\[
c_t:=[i_t]_{o_t}\in\cC_{o_t}
\]
be the stable quotient class of the current latent layer, so that
\[
\beta_t(c)=\bP(c_t=c\mid H_t),
\qquad c\in\cC_{o_t}.
\]
The action $a_t$ is selected according to a distribution that is measurable with respect to $H_t$. Consequently, conditional on $H_t$, the policy randomization used to choose $a_t$ does not depend on the hidden layer $i_t$ or on its class $c_t$. More explicitly, for every action $a\in\cA$ satisfying $\bP(a_t=a\mid H_t)>0$ and every $c\in\cC_{o_t}$,
\[
\begin{aligned}
&\bP(c_t=c\mid H_t,a_t=a)\\
&=
\frac{\bP(c_t=c,a_t=a\mid H_t)}
{\bP(a_t=a\mid H_t)}\\
&=
\frac{
\bP(c_t=c\mid H_t)
\bP(a_t=a\mid H_t,c_t=c)
}{
\bP(a_t=a\mid H_t)
}\\
&=
\frac{
\bP(c_t=c\mid H_t)
\bP(a_t=a\mid H_t)
}{
\bP(a_t=a\mid H_t)
}\\
&=
\bP(c_t=c\mid H_t)\\
&=
\beta_t(c).
\end{aligned}
\]
The third equality follows from the definition of an admissible history-dependent policy: once $H_t$ is given, the conditional distribution of $a_t$ is $\pi_t(\cdot\mid H_t)$ and cannot depend additionally on the unobserved variable $c_t$. Thus conditioning on the selected action does not change the current class belief. In the notation of the proposition,
\[
\bP(c_t=c\mid H_t,a_t)=\beta_t(c)
\]
for every action value that can be selected after $H_t$.

Now fix an observed successor $o_{t+1}$ satisfying
\[
P_O(o_{t+1}\mid o_t,a_t)>0.
\]
The HCDP transition law implies that the probability of this visible successor is independent of the current raw layer. Indeed, for every $i\in[n]$,
\[
\begin{aligned}
&\bP(o_{t+1}\mid H_t,a_t,i_t=i)\\
&\quad=
\sum_{i'=1}^{n}
P_H\bigl((o_{t+1},i')\mid(o_t,i),a_t\bigr)\\
&\quad=
\sum_{i'=1}^{n}
P_O(o_{t+1}\mid o_t,a_t)
\mathbf{1}\!\left\{
i'=\sigma_{o_t,a_t,o_{t+1}}(i)
\right\}\\
&\quad=
P_O(o_{t+1}\mid o_t,a_t)
\sum_{i'=1}^{n}
\mathbf{1}\!\left\{
i'=\sigma_{o_t,a_t,o_{t+1}}(i)
\right\}\\
&\quad=
P_O(o_{t+1}\mid o_t,a_t).
\end{aligned}
\]
The final equality holds because
\[
\sigma_{o_t,a_t,o_{t+1}}(i)\in[n]
\]
is a single well-defined layer, so exactly one term in the indicator sum is equal to one.

Because the right-hand side does not depend on $i$, the joint probability of the current class and the observed successor can be computed without conditioning on a possibly null class event. For every $c\in\cC_{o_t}$,
\[
\begin{aligned}
&\bP(c_t=c,o_{t+1}\mid H_t,a_t)\\
&\quad=
\sum_{i:\,[i]_{o_t}=c}
\bP(i_t=i,o_{t+1}\mid H_t,a_t)\\
&\quad=
\sum_{i:\,[i]_{o_t}=c}
\bP(o_{t+1}\mid H_t,a_t,i_t=i)
\bP(i_t=i\mid H_t,a_t)\\
&\quad=
P_O(o_{t+1}\mid o_t,a_t)
\sum_{i:\,[i]_{o_t}=c}
\bP(i_t=i\mid H_t,a_t)\\
&\quad=
P_O(o_{t+1}\mid o_t,a_t)
\bP(c_t=c\mid H_t,a_t)\\
&\quad=
P_O(o_{t+1}\mid o_t,a_t)\beta_t(c).
\end{aligned}
\]
Summing this identity over all current classes gives
\[
\begin{aligned}
\bP(o_{t+1}\mid H_t,a_t)
&=
\sum_{d\in\cC_{o_t}}
\bP(c_t=d,o_{t+1}\mid H_t,a_t)\\
&=
\sum_{d\in\cC_{o_t}}
P_O(o_{t+1}\mid o_t,a_t)\beta_t(d)\\
&=
P_O(o_{t+1}\mid o_t,a_t)
\sum_{d\in\cC_{o_t}}\beta_t(d)\\
&=
P_O(o_{t+1}\mid o_t,a_t),
\end{aligned}
\]
because $\beta_t$ is a probability distribution on $\cC_{o_t}$.

Bayes' rule now shows that observing $o_{t+1}$ introduces no class-dependent likelihood correction. For every $c\in\cC_{o_t}$,
\[
\begin{aligned}
&\bP(c_t=c\mid H_t,a_t,o_{t+1})\\
&\quad=
\frac{
\bP(c_t=c,o_{t+1}\mid H_t,a_t)
}{
\bP(o_{t+1}\mid H_t,a_t)
}\\
&\quad=
\frac{
P_O(o_{t+1}\mid o_t,a_t)\beta_t(c)
}{
P_O(o_{t+1}\mid o_t,a_t)
}\\
&\quad=
\beta_t(c).
\end{aligned}
\]
The positivity of
\[
P_O(o_{t+1}\mid o_t,a_t)
\]
guarantees that this conditional probability is well defined. The assumption that reward samples are not used as within-episode controller observations is used at this point. Under the convention of the proposition, the information available at the next decision time is
\[
H_{t+1}
=
H_t\cdot(a_t,o_{t+1}),
\]
and therefore no realized reward sample is included in the conditioning information. If a reward sample were included, its class-conditional likelihood could produce an additional Bayes correction, which would require the stronger reward-output abstraction discussed in Appendix~\ref{app:prelim_details}.

By Theorem~\ref{thm:markovization_main}, the stable class transition associated with the observed feasible edge is deterministic. More precisely,
\[
\begin{aligned}
c_{t+1}
&=
[i_{t+1}]_{o_{t+1}}\\
&=
[\sigma_{o_t,a_t,o_{t+1}}(i_t)]_{o_{t+1}}\\
&=
\tau_{o_t,a_t,o_{t+1}}([i_t]_{o_t})\\
&=
\tau_{o_t,a_t,o_{t+1}}(c_t)
\end{aligned}
\]
almost surely conditional on $(H_t,a_t,o_{t+1})$. Equivalently, for every $c\in\cC_{o_t}$ and every $c'\in\cC_{o_{t+1}}$,
\[
\bP(c_{t+1}=c'\mid H_t,a_t,o_{t+1},c_t=c)
=
\mathbf{1}\!\left\{
c'=\tau_{o_t,a_t,o_{t+1}}(c)
\right\}
\]
whenever the conditioning event has positive probability. In the zero-probability case, the corresponding term in the law-of-total-probability expansion below vanishes.

Using the conditional law of total probability over the finite set $\cC_{o_t}$, we obtain
\[
\begin{aligned}
\beta_{t+1}(c')
&=
\bP(c_{t+1}=c'\mid H_{t+1})\\
&=
\bP(c_{t+1}=c'\mid H_t,a_t,o_{t+1})\\
&=
\sum_{c\in\cC_{o_t}}
\bP(c_{t+1}=c'\mid H_t,a_t,o_{t+1},c_t=c)\\
&\bP(c_t=c\mid H_t,a_t,o_{t+1})\\
&=
\sum_{c\in\cC_{o_t}}
\mathbf{1}\!\left\{
c'=\tau_{o_t,a_t,o_{t+1}}(c)
\right\}
\beta_t(c)\\
&=
\sum_{\substack{c\in\cC_{o_t}:\\
\tau_{o_t,a_t,o_{t+1}}(c)=c'}}
\beta_t(c).
\end{aligned}
\]
This formula remains valid when the open-edge class transport is not injective. If several current classes are mapped to the same successor class, the posterior masses of all those classes are added.

The right-hand side also defines a normalized probability distribution. Indeed,
\[
\begin{aligned}
\sum_{c'\in\cC_{o_{t+1}}}\beta_{t+1}(c')
&=
\sum_{c'\in\cC_{o_{t+1}}}
\sum_{c\in\cC_{o_t}}
\mathbf{1}\!\left\{
c'=\tau_{o_t,a_t,o_{t+1}}(c)
\right\}
\beta_t(c)\\
&=
\sum_{c\in\cC_{o_t}}
\beta_t(c)
\sum_{c'\in\cC_{o_{t+1}}}
\mathbf{1}\!\left\{
c'=\tau_{o_t,a_t,o_{t+1}}(c)
\right\}\\
&=
\sum_{c\in\cC_{o_t}}\beta_t(c)\\
&=
1,
\end{aligned}
\]
because $\tau_{o_t,a_t,o_{t+1}}(c)$ is one uniquely defined element of $\cC_{o_{t+1}}$ for each $c\in\cC_{o_t}$.

The stable quotient also makes the conditional mean reward depend only on the current class. By the representative-independent reward identity established in Theorem~\ref{thm:markovization_main},
\[
\begin{aligned}
R(o_t,i_t,a_t)
&=
\bar R(o_t,[i_t]_{o_t},a_t)\\
&=
\bar R(o_t,c_t,a_t)
\end{aligned}
\]
almost surely. Hence
\[
\begin{aligned}
&\bE\!\left[
R(o_t,i_t,a_t)\mid H_t,a_t
\right]\\
&=
\bE\!\left[
\bar R(o_t,c_t,a_t)\mid H_t,a_t
\right]\\
&=
\sum_{c\in\cC_{o_t}}
\bar R(o_t,c,a_t)
\bP(c_t=c\mid H_t,a_t)\\
&=
\sum_{c\in\cC_{o_t}}
\beta_t(c)\bar R(o_t,c,a_t).
\end{aligned}
\]
This proves the two displayed identities in the proposition.

To make the belief-state sufficiency statement explicit, define, for every feasible triple $(o,a,o')$, the deterministic pushforward operator
\[
\Phi_{o,a,o'}:
\Delta(\cC_o)
\longrightarrow
\Delta(\cC_{o'})
\]
by
\[
\bigl(\Phi_{o,a,o'}\beta\bigr)(c')
:=
\sum_{\substack{c\in\cC_o:\\
\tau_{o,a,o'}(c)=c'}}
\beta(c).
\]
The update established above can then be written as
\[
\beta_{t+1}
=
\Phi_{o_t,a_t,o_{t+1}}\beta_t.
\]
Moreover, for every $o'\in\cO$,
\[
\bP(o_{t+1}=o'\mid H_t,a_t=a)
=
P_O(o'\mid o_t,a).
\]
Define the compressed belief-state space by
\[
\mathfrak B
:=
\left\{
(o,\beta):
o\in\cO,\ 
\beta\in\Delta(\cC_o)
\right\}.
\]
For any current realization $(o_t,\beta_t)$ and action $a$, the conditional transition law of the next compressed belief state is therefore
\[
\begin{aligned}
&\bP\bigl(
(o_{t+1},\beta_{t+1})=(o',\beta')
\mid H_t,a_t=a
\bigr)\\
&\quad=
P_O(o'\mid o_t,a)
\mathbf{1}\!\left\{
\beta'
=
\Phi_{o_t,a,o'}\beta_t
\right\}.
\end{aligned}
\]
The corresponding conditional mean reward is
\[
\bar R_{\mathrm{bel}}(o_t,\beta_t,a)
:=
\sum_{c\in\cC_{o_t}}
\beta_t(c)\bar R(o_t,c,a).
\]
Both the controlled transition law and the conditional mean reward depend on the complete observable history $H_t$ only through the pair $(o_t,\beta_t)$. Since the next belief is obtained recursively from the current belief, the selected action, and the next observation, iteration of this update reproduces exactly the posterior distribution of the stable quotient class after every future observable continuation. Therefore $(o_t,\beta_t)$ is an exact belief state after the raw hidden layers have been compressed to stable classes, for the conditional-mean discounted control objective considered in the paper.

Finally, suppose that the current stable class is known exactly, so that
\[
\beta_t=\delta_{c_t}.
\]
For every $c'\in\cC_{o_{t+1}}$, the belief update gives
\[
\begin{aligned}
\beta_{t+1}(c')
&=
\sum_{\substack{c\in\cC_{o_t}:\\
\tau_{o_t,a_t,o_{t+1}}(c)=c'}}
\delta_{c_t}(c)\\
&=
\mathbf{1}\!\left\{
\tau_{o_t,a_t,o_{t+1}}(c_t)=c'
\right\}\\
&=
\delta_{\tau_{o_t,a_t,o_{t+1}}(c_t)}(c').
\end{aligned}
\]
Since the two probability distributions agree at every
$c'\in\cC_{o_{t+1}}$, we conclude that
\[
\beta_{t+1}
=
\delta_{\tau_{o_t,a_t,o_{t+1}}(c_t)}.
\]
Thus exact knowledge of the stable class, once available, remains exactly and recursively maintainable.
\end{proof}

\subsection{Proof of Lemma~\ref{lem:quotient_memory_update_main}}

\begin{proof}
Let
\[
C_t:=[i_t]_{o_t}\in\cC_{o_t}
\]
denote the actual stable class of the latent layer at time $t$, and let
$M_t\in\cC_{o_t}$ denote the class label maintained by the controller.
Suppose that the maintained label is correct at time $t$, so that
\[
M_t=C_t=[i_t]_{o_t}.
\]
Fix a realized action $a_t$ and a realized successor observation
$o_{t+1}$ having positive conditional probability. Then
\[
P_O(o_{t+1}\mid o_t,a_t)>0,
\]
so the triple $(o_t,a_t,o_{t+1})$ is feasible and the transport maps
\[
\sigma_{o_t,a_t,o_{t+1}}:[n]\to[n]
\qquad\text{and}\qquad
\tau_{o_t,a_t,o_{t+1}}:
\cC_{o_t}\to\cC_{o_{t+1}}
\]
are well defined. By the HCDP transition kernel, for every candidate
layer $j\in[n]$,
\[
\begin{aligned}
&\bP\bigl(
i_{t+1}=j,
o_{t+1}
\mid H_t,i_t,a_t
\bigr)\\
&\quad=
P_H\bigl(
(o_{t+1},j)\mid(o_t,i_t),a_t
\bigr)\\
&\quad=
P_O(o_{t+1}\mid o_t,a_t)
\mathbf{1}\!\left\{
j=\sigma_{o_t,a_t,o_{t+1}}(i_t)
\right\}.
\end{aligned}
\]
Marginalizing over the candidate next layer gives
\[
\begin{aligned}
&\bP(o_{t+1}\mid H_t,i_t,a_t)\\
&=
\sum_{j=1}^{n}
P_H\bigl(
(o_{t+1},j)\mid(o_t,i_t),a_t
\bigr)\\
&=
P_O(o_{t+1}\mid o_t,a_t)
\sum_{j=1}^{n}
\mathbf{1}\!\left\{
j=\sigma_{o_t,a_t,o_{t+1}}(i_t)
\right\}\\
&=
P_O(o_{t+1}\mid o_t,a_t).
\end{aligned}
\]
The final equality holds because
$\sigma_{o_t,a_t,o_{t+1}}(i_t)$ is a single element of $[n]$, and
hence exactly one term in the indicator sum equals one. Since the
observed successor has positive probability, division by
$P_O(o_{t+1}\mid o_t,a_t)$ is valid. Therefore
\[
\begin{aligned}
&\bP\bigl(
i_{t+1}=j
\mid H_t,i_t,a_t,o_{t+1}
\bigr)\\
&\quad=
\frac{
P_O(o_{t+1}\mid o_t,a_t)
\mathbf{1}\!\left\{
j=\sigma_{o_t,a_t,o_{t+1}}(i_t)
\right\}
}{
P_O(o_{t+1}\mid o_t,a_t)
}\\
&\quad=
\mathbf{1}\!\left\{
j=\sigma_{o_t,a_t,o_{t+1}}(i_t)
\right\}.
\end{aligned}
\]
Thus, conditional on the observed transition,
\[
i_{t+1}
=
\sigma_{o_t,a_t,o_{t+1}}(i_t)
\]
almost surely. Taking stable classes and using the representative-independent
transport identity from Theorem~\ref{thm:markovization_main}, we obtain
\[
\begin{aligned}
C_{t+1}
&=
[i_{t+1}]_{o_{t+1}}\\
&=
\bigl[
\sigma_{o_t,a_t,o_{t+1}}(i_t)
\bigr]_{o_{t+1}}\\
&=
\tau_{o_t,a_t,o_{t+1}}([i_t]_{o_t})\\
&=
\tau_{o_t,a_t,o_{t+1}}(C_t)\\
&=
\tau_{o_t,a_t,o_{t+1}}(M_t)
\end{aligned}
\]
almost surely. Consequently, if the controller performs the update
\[
M_{t+1}
:=
\tau_{o_t,a_t,o_{t+1}}(M_t),
\]
then
\[
M_{t+1}=C_{t+1}
\]
almost surely. Repeating this argument inductively shows that, for any
time $t_0$, correctness at time $t_0$ implies
\[
M_t=C_t
\qquad
\text{for every }t\geq t_0
\]
almost surely along every realized continuation. Indeed, the equality is
true at $t=t_0$ by assumption, and the one-step implication just proved
carries correctness from each time $t$ to time $t+1$.

We now verify that the correctly tracked observation--class pair has
exactly the transition law of the quotient MDP. Fix a
positive-probability conditioning event on which
\[
o_t=o,
\qquad
C_t=c,
\qquad
a_t=a.
\]
For any $o'\in\cO$ and $c'\in\cC_{o'}$, the conditional law of total
probability over the possible raw layers contained in the class $c$
gives
\[
\begin{aligned}
&\bP\bigl(
o_{t+1}=o',
C_{t+1}=c'
\mid H_t,C_t=c,a_t=a
\bigr)\\
&\quad=
\sum_{i:\,[i]_o=c}
\bP(i_t=i\mid H_t,C_t=c,a_t=a)\\
&\qquad\qquad\cdot
\sum_{j:\,[j]_{o'}=c'}
P_H\bigl((o',j)\mid(o,i),a\bigr).
\end{aligned}
\]
For every representative $i$ satisfying $[i]_o=c$, the inner sum is
\[
\begin{aligned}
&\sum_{j:\,[j]_{o'}=c'}
P_H\bigl((o',j)\mid(o,i),a\bigr)\\
&\quad=
\sum_{j:\,[j]_{o'}=c'}
P_O(o'\mid o,a)
\mathbf{1}\!\left\{
j=\sigma_{o,a,o'}(i)
\right\}\\
&\quad=
P_O(o'\mid o,a)
\mathbf{1}\!\left\{
[\sigma_{o,a,o'}(i)]_{o'}=c'
\right\}\\
&\quad=
P_O(o'\mid o,a)
\mathbf{1}\!\left\{
\tau_{o,a,o'}([i]_o)=c'
\right\}\\
&\quad=
P_O(o'\mid o,a)
\mathbf{1}\!\left\{
\tau_{o,a,o'}(c)=c'
\right\}.
\end{aligned}
\]
When $P_O(o'\mid o,a)=0$, both sides are zero and the value assigned
to the corresponding transport is immaterial. When
$P_O(o'\mid o,a)>0$, the fourth equality follows directly from
Theorem~\ref{thm:markovization_main}. In either case, the final
expression is independent of the representative $i$. Substituting it
into the preceding total-probability expansion yields
\[
\begin{aligned}
&\bP\bigl(
o_{t+1}=o',
C_{t+1}=c'
\mid H_t,C_t=c,a_t=a
\bigr)\\
&\quad=
P_O(o'\mid o,a)
\mathbf{1}\!\left\{
\tau_{o,a,o'}(c)=c'
\right\}\\
&\sum_{i:\,[i]_o=c}
\bP(i_t=i\mid H_t,C_t=c,a_t=a)\\
&\quad=
P_O(o'\mid o,a)
\mathbf{1}\!\left\{
c'=\tau_{o,a,o'}(c)
\right\}\\
&\quad=
\bar P\bigl((o',c')\mid(o,c),a\bigr).
\end{aligned}
\]
Thus the conditional distribution of the next observation--class pair
depends on the complete past only through the current pair $(o_t,C_t)$
and the selected action, and it coincides exactly with the transition
kernel of the quotient MDP.

The one-step mean reward also agrees exactly with the quotient reward.
Since
\[
C_t=[i_t]_{o_t},
\]
the representative-independent reward identity from
Theorem~\ref{thm:markovization_main} gives
\[
\begin{aligned}
R(o_t,i_t,a_t)
&=
\bar R(o_t,[i_t]_{o_t},a_t)\\
&=
\bar R(o_t,C_t,a_t)
\end{aligned}
\]
almost surely. Hence both the controlled transition kernel and the
mean reward of the correctly tracked process $(o_t,M_t)$ agree with
those of the quotient MDP whenever $M_t=C_t$.

If the class is correctly initialized at time zero, so that
\[
M_0=C_0=[i_0]_{o_0},
\]
then its initial joint distribution is also the quotient initial
distribution. Indeed, for every $(o,c)\in\bar\cS$,
\[
\begin{aligned}
\bP(o_0=o,M_0=c)
&=
\bP(o_0=o,C_0=c)\\
&=
\sum_{i:\,[i]_o=c}
\bP(o_0=o,i_0=i)\\
&=
\sum_{i:\,[i]_o=c}
\rho_0(o,i)\\
&=
\bar\rho_0(o,c).
\end{aligned}
\]
Combining this initial-law identity with the transition and reward
identities proves that the process $(o_t,M_t)$ started from a correct
initial class has exactly the controlled law and mean-reward structure
of the quotient MDP in Theorem~\ref{thm:markovization_main}. If the
class is instead inferred correctly at a later time $t_0$, the same
argument applies conditionally from time $t_0$ onward, with the
inferred class serving as the current quotient state.

Finally, the update is a finite-memory recursion. The tagged class
alphabet
\[
\cM_{\mathrm{quot}}
:=
\{(o,c):o\in\cO,\ c\in\cC_o\}
=
\bar\cS
\]
is finite, and for every feasible observed transition one may define
\[
U_{\mathrm{quot}}
\bigl((o,c),o,a,o'\bigr)
:=
\bigl(
o',
\tau_{o,a,o'}(c)
\bigr).
\]
For zero-probability transitions, $U_{\mathrm{quot}}$ may be defined
arbitrarily because those transitions are never realized. This
recursion uses only the current stored class, the current observation,
the selected action, and the newly observed successor. Since
correctness is preserved by every feasible update, a class that is
correctly initialized or correctly inferred once is a recursively
maintainable finite memory state.
\end{proof}

\subsection{Proof of Corollary~\ref{cor:min_memory_main}}

\begin{proof}
For every observation $o\in\cO$, the stable class set
\[
\cC_o=[n]/\Pi_o^\star
\]
is finite and nonempty. Since $\cO$ is finite, the quantity
\[
m_{\mathrm{exact}}^\star
=
\max_{o\in\cO}|\cC_o|
\]
is a well-defined positive integer. We first prove that no alphabet satisfying the requirements of the corollary can contain fewer than $m_{\mathrm{exact}}^\star$ symbols. Fix an arbitrary observation $o\in\cO$. Because
\[
\eta_o:\cC_o\longrightarrow\cM
\]
is injective, its restriction to the finite set $\cC_o$ maps distinct classes to distinct memory symbols. Hence
\[
|\eta_o(\cC_o)|=|\cC_o|.
\]
Since $\eta_o(\cC_o)\subseteq\cM$, it follows that
\[
|\cC_o|
=
|\eta_o(\cC_o)|
\leq
|\cM|.
\]
This inequality holds for every $o\in\cO$, and therefore
\[
|\cM|
\geq
\max_{o\in\cO}|\cC_o|
=
m_{\mathrm{exact}}^\star.
\]
Equivalently, if $|\cM|<|\cC_o|$ for some observation $o$, then the pigeonhole principle would give distinct classes $c,d\in\cC_o$ with
\[
\eta_o(c)=\eta_o(d),
\]
contradicting injectivity. Thus, when the visible observation is fixed at $o$, at least $|\cC_o|$ different memory symbols are necessary to retain exact knowledge of the stable class. This requirement is intrinsic to exact class tracking: by Theorem~\ref{thm:coarsest_exact_main}, two distinct stable classes cannot be merged by any exact observation-wise abstraction.

We now construct an alphabet of exactly $m_{\mathrm{exact}}^\star$ symbols satisfying all the required update identities. Let
\[
\cM^\star
:=
[m_{\mathrm{exact}}^\star]
=
\{1,\ldots,m_{\mathrm{exact}}^\star\}.
\]
For each $o\in\cO$, write
\[
r_o:=|\cC_o|.
\]
By the definition of $m_{\mathrm{exact}}^\star$,
\[
r_o\leq m_{\mathrm{exact}}^\star.
\]
Choose an enumeration
\[
\cC_o
=
\{c_{o,1},\ldots,c_{o,r_o}\}
\]
and define
\[
\eta_o^\star(c_{o,\ell})
:=
\ell,
\qquad
\ell=1,\ldots,r_o.
\]
For each fixed $o$, the labels $1,\ldots,r_o$ are distinct, so
\[
\eta_o^\star:\cC_o\longrightarrow\cM^\star
\]
is injective. The same numerical symbol may be used by the encoders associated with different observations; for example, the symbol $1$ may encode $c_{o,1}$ at observation $o$ and $c_{\widetilde o,1}$ at observation $\widetilde o$. This causes no ambiguity because the maintained controller state is the pair consisting of the current observation and the memory symbol, and the update map receives the current observation $o$ as an explicit argument.

We define a total update map
\[
U^\star:
\cM^\star\times\cO\times\cA\times\cO
\longrightarrow
\cM^\star.
\]
Fix a tuple $(m,o,a,o')$. If
\[
P_O(o'\mid o,a)>0
\qquad\text{and}\qquad
m\in\eta_o^\star(\cC_o),
\]
then the injectivity of $\eta_o^\star$ implies that there is a unique class $c\in\cC_o$ such that
\[
\eta_o^\star(c)=m.
\]
For such a tuple, define
\[
U^\star(m,o,a,o')
:=
\eta_{o'}^\star\!\left(
\tau_{o,a,o'}(c)
\right).
\]
This value belongs to $\cM^\star$ because
\[
\tau_{o,a,o'}(c)\in\cC_{o'}
\]
and $\eta_{o'}^\star$ maps $\cC_{o'}$ into $\cM^\star$. The definition is unambiguous because the class $c$ satisfying $\eta_o^\star(c)=m$ is unique. If either
\[
P_O(o'\mid o,a)=0
\]
or
\[
m\notin\eta_o^\star(\cC_o),
\]
define, for example,
\[
U^\star(m,o,a,o'):=1.
\]
These arbitrary values do not affect exact class tracking: a zero-probability transition is never realized, and a correctly encoded class symbol always lies in $\eta_o^\star(\cC_o)$.

Let now $(o,a,o')$ be feasible and let $c\in\cC_o$. Put
\[
m:=\eta_o^\star(c).
\]
Then $m\in\eta_o^\star(\cC_o)$, and the unique class mapped to $m$ by $\eta_o^\star$ is precisely $c$. Therefore the defining clause for $U^\star$ gives
\[
U^\star\bigl(\eta_o^\star(c),o,a,o'\bigr)
=
\eta_{o'}^\star\!\left(
\tau_{o,a,o'}(c)
\right).
\]
Thus $U^\star$ satisfies the required compatibility identity for every feasible transition and every source class.

To verify explicitly that this compatibility identity recursively maintains the class, suppose that at some time $t$ the true class $c_t\in\cC_{o_t}$ is known and the memory is initialized as
\[
m_t=\eta_{o_t}^\star(c_t).
\]
By Lemma~\ref{lem:quotient_memory_update_main}, after observing a feasible transition $(o_t,a_t,o_{t+1})$, the true next class is
\[
c_{t+1}
=
\tau_{o_t,a_t,o_{t+1}}(c_t).
\]
The memory update then satisfies
\[
\begin{aligned}
m_{t+1}
&=
U^\star(m_t,o_t,a_t,o_{t+1})\\
&=
U^\star\bigl(
\eta_{o_t}^\star(c_t),
o_t,a_t,o_{t+1}
\bigr)\\
&=
\eta_{o_{t+1}}^\star\!\left(
\tau_{o_t,a_t,o_{t+1}}(c_t)
\right)\\
&=
\eta_{o_{t+1}}^\star(c_{t+1}).
\end{aligned}
\]
Hence a correct encoded class remains correct after one update. Repeating the same argument inductively gives
\[
m_s=\eta_{o_s}^\star(c_s)
\qquad
\text{for every later time }s.
\]
Therefore an alphabet of size
\[
|\cM^\star|
=
m_{\mathrm{exact}}^\star
\]
is sufficient.

Combining the lower bound valid for every admissible alphabet with the explicit construction above yields
\[
|\cM|_{\min}
=
m_{\mathrm{exact}}^\star
=
\max_{o\in\cO}|\cC_o|.
\]
The bound is a maximum rather than the sum
\[
\sum_{o\in\cO}|\cC_o|
\]
because class labels need only be injective within each fixed observation fiber. The observation coordinate identifies which encoder $\eta_o^\star$ and which interpretation of a reused memory symbol are active, so the same alphabet symbols may be reused across different observations.
\end{proof}

\subsection{Proof of Theorem~\ref{thm:passive_nonident_main}}

\begin{proof}
Let the common observation and action spaces be
\[
\cO=\{o\},
\qquad
\cA=\{a,b\},
\]
and let both HCDPs have two raw layers, so that $n=2$. We construct two
models, denoted by $\mathsf M^{\mathrm A}$ and $\mathsf M^{\mathrm B}$.
They have the same discount factor $\gamma\in(0,1)$, the same initial
distribution
\[
\rho_0(o,1)=\rho_0(o,2)=\frac12,
\]
the same deterministic observation kernel
\[
P_O(o\mid o,a)=P_O(o\mid o,b)=1,
\]
and the same raw transport on every feasible edge,
\[
\sigma_{o,a,o}
=
\sigma_{o,b,o}
=
\operatorname{id}_{[2]}.
\]
Thus, under either action, the visible observation remains $o$ and the raw
layer remains unchanged. The two models differ only in their mean reward
functions. In $\mathsf M^{\mathrm A}$, set
\[
R^{\mathrm A}(o,1,a)
=
R^{\mathrm A}(o,2,a)
=
R^{\mathrm A}(o,1,b)
=
R^{\mathrm A}(o,2,b)
=
0.
\]
In $\mathsf M^{\mathrm B}$, set
$
R^{\mathrm B}(o,1,a)
=
R^{\mathrm B}(o,2,a)
=
0,
\qquad
R^{\mathrm B}(o,1,b)=0,
\qquad
R^{\mathrm B}(o,2,b)=1.
$
Both are finite HCDPs in the sense of Definition~\ref{def:hcdp_main}. We may
take the reward noise to be identically zero in both models, namely
$\xi_t=0$ almost surely for every $t$. This choice satisfies
Assumption~\ref{ass:hcdp_noise_main}, because it is conditionally mean zero
and, for every $\lambda\in\mathbb R$,
\[
\bE\!\left[
e^{\lambda\xi_t}
\mid
\mathcal G_t^-
\right]
=
1
\leq
\exp\!\left(
\frac{\lambda^2\sigma_Y^2}{2}
\right).
\]

We first compute the two stable quotients. In $\mathsf M^{\mathrm A}$, layers
$1$ and $2$ have the same mean reward under every action. Hence the reward
partition at the unique observation is
\[
\Pi_o^{\mathrm A}
=
\bigl\{\{1,2\}\bigr\}.
\]
Because both raw edge transports are the identity, for each
$u\in\{a,b\}$ we have
\[
\sigma_{o,u,o}(1)=1,
\qquad
\sigma_{o,u,o}(2)=2,
\]
and these two successor layers belong to the same block of
$\Pi_o^{\mathrm A}$. Therefore the defining conditions of the stability
operator give
\[
1
\equiv_o^{\mathcal T(\Pi^{\mathrm A})}
2.
\]
There is only one block to begin with, so this implies
\[
\mathcal T(\Pi^{\mathrm A})
=
\Pi^{\mathrm A}.
\]
Thus $\Pi^{\mathrm A}$ is stable. Since no partition can be strictly coarser
than the one-block partition, it is the coarsest stable family, and
consequently
\[
\Pi_o^{\mathrm A,\star}
=
\bigl\{\{1,2\}\bigr\},
\qquad
\cC_o^{\mathrm A}
=
\{c_\star\},
\qquad
c_\star=\{1,2\}.
\]
The quotient transports in $\mathsf M^{\mathrm A}$ are necessarily
\[
\tau_{o,a,o}^{\mathrm A}(c_\star)
=
\tau_{o,b,o}^{\mathrm A}(c_\star)
=
c_\star,
\]
the quotient rewards satisfy
\[
\bar R^{\mathrm A}(o,c_\star,a)
=
\bar R^{\mathrm A}(o,c_\star,b)
=
0,
\]
and the quotient initial distribution satisfies
\[
\bar\rho_0^{\mathrm A}(o,c_\star)
=
1.
\]

In $\mathsf M^{\mathrm B}$, the two layers have different mean rewards under
action $b$, because
\[
R^{\mathrm B}(o,1,b)
=
0
\neq
1
=
R^{\mathrm B}(o,2,b).
\]
Therefore they cannot belong to the same block of the reward partition, and
the reward partition is the discrete partition
\[
\Pi_o^{\mathrm B}
=
\bigl\{\{1\},\{2\}\bigr\}.
\]
Applying $\mathcal T$ cannot merge the two singleton blocks, since equality of
rewards for every action is a necessary condition in
Definition~\ref{def:stable_family_main}, and that condition already fails for
action $b$. Hence
\[
\mathcal T(\Pi^{\mathrm B})
=
\Pi^{\mathrm B}.
\]
It follows that
$
\Pi_o^{\mathrm B,\star}
=
\bigl\{\{1\},\{2\}\bigr\},
\qquad
\cC_o^{\mathrm B}
=
\{c_1,c_2\},
\qquad
c_1=\{1\},
\quad
c_2=\{2\}.
$
Since the raw transports are identities, the quotient transports are
identities as well:
\[
\tau_{o,a,o}^{\mathrm B}(c_j)
=
\tau_{o,b,o}^{\mathrm B}(c_j)
=
c_j,
\qquad
j\in\{1,2\}.
\]
The quotient rewards under action $a$ are
\[
\bar R^{\mathrm B}(o,c_1,a)
=
\bar R^{\mathrm B}(o,c_2,a)
=
0,
\]
whereas under action $b$ they are
\[
\bar R^{\mathrm B}(o,c_1,b)=0,
\qquad
\bar R^{\mathrm B}(o,c_2,b)=1.
\]
The quotient initial distribution is
\[
\bar\rho_0^{\mathrm B}(o,c_1)
=
\bar\rho_0^{\mathrm B}(o,c_2)
=
\frac12.
\]
The two stable quotient descriptions cannot be observation-wise gauge
equivalent. Indeed, Definition~\ref{def:gauge_equiv_main} would require a
bijection between the class sets at the unique observation, whereas
\[
|\cC_o^{\mathrm A}|=1
\qquad\text{and}\qquad
|\cC_o^{\mathrm B}|=2.
\]
No bijection from a one-element set to a two-element set exists, so the
required observation-wise gauge map cannot exist.

Now define the common behavior policy by
$
\mu_t(a\mid H_t)=1,
\qquad
\mu_t(b\mid H_t)=0,
\quad
\text{for every }t\geq0
\text{ and every observable history }H_t.
$
Under either model, $o_0=o$ almost surely. Suppose that at some time $t$ the
latent state is $(o,i_t)$. Since the behavior policy chooses $a$ almost surely,
the deterministic base kernel and identity transport imply
\[
o_{t+1}=o,
\qquad
i_{t+1}
=
\sigma_{o,a,o}(i_t)
=
i_t
\qquad
\text{almost surely}.
\]
Moreover, both models assign mean reward zero to action $a$ at both layers,
and the reward noise is zero, so under $\mathsf M^{\mathrm A}$,
\[
Y_t
=
R^{\mathrm A}(o,i_t,a)+\xi_t
=
0,
\]
and under $\mathsf M^{\mathrm B}$,
\[
Y_t
=
R^{\mathrm B}(o,i_t,a)+\xi_t
=
0.
\]
Induction over $t$ therefore gives, under both models,
\[
L_\infty
=
(o,a,0,o,a,0,o,a,0,\ldots)
\qquad
\text{almost surely}.
\]
Writing $\bP_{\mathrm A}^{\mu}$ and $\bP_{\mathrm B}^{\mu}$ for the two
induced probability laws, we have the stronger identity
\[
\mathcal L_{\bP_{\mathrm A}^{\mu}}(L_\infty)
=
\delta_{(o,a,0,o,a,0,\ldots)}
=
\mathcal L_{\bP_{\mathrm B}^{\mu}}(L_\infty).
\]
Thus even the complete infinite passive log has exactly the same distribution
in the two models.

It remains to derive the statistical impossibility. Let $L_T$ be any finite
prefix of $L_\infty$, and let $\widehat{\mathfrak Q}_T$ be any estimator
measurable with respect to $L_T$. The same argument also covers a randomized
estimator by adjoining to $L_T$ an auxiliary random seed having the same
independent law under both models. Consider the gauge-invariant functional
\[
K(\mathfrak Q)
:=
|\cC_o|,
\]
which is the number of stable classes at the unique observation, and define
\[
\widehat K_T
:=
K(\widehat{\mathfrak Q}_T).
\]
For the two constructed models,
\[
K(\mathfrak Q^{\mathrm A})=1,
\qquad
K(\mathfrak Q^{\mathrm B})=2.
\]
Because $L_T$ has the same law under $\bP_{\mathrm A}^{\mu}$ and
$\bP_{\mathrm B}^{\mu}$, the estimator output also has the same law. In
particular, if
\[
p_T
:=
\bP_{\mathrm A}^{\mu}(\widehat K_T=1)
=
\bP_{\mathrm B}^{\mu}(\widehat K_T=1)
\]
and
\[
q_T
:=
\bP_{\mathrm A}^{\mu}(\widehat K_T=2)
=
\bP_{\mathrm B}^{\mu}(\widehat K_T=2),
\]
then the events $\{\widehat K_T=1\}$ and $\{\widehat K_T=2\}$ are disjoint,
so
\[
p_T+q_T\leq1.
\]
Consequently,
\[
\begin{aligned}
\max\Bigl\{
\bP_{\mathrm A}^{\mu}(\widehat K_T\neq1),
\bP_{\mathrm B}^{\mu}(\widehat K_T\neq2)
\Bigr\}
&=
\max\{1-p_T,1-q_T\}\\
&\geq
1-\frac{p_T+q_T}{2}\\
&\geq
\frac12.
\end{aligned}
\]
This lower bound holds for every sample length $T$. Hence no estimator can
consistently recover the stable class structure at both
$\mathsf M^{\mathrm A}$ and $\mathsf M^{\mathrm B}$ under the common behavior
policy $\mu$; a fortiori, no estimator can be uniformly consistent over the
full HCDP class and all behavior policies.

Finally, the same obstruction applies to quotient transports and to the
quotient MDP. A quotient-transport description includes the source and target
class sets on which each $\tau_e$ is defined, and therefore determines
$|\cC_o|$. Likewise, by Theorem~\ref{thm:markovization_main}, the quotient MDP
has state space
\[
\bar{\cS}
=
\bigl\{
(o',c):
o'\in\cO,\;
c\in\cC_{o'}
\bigr\}.
\]
In the present one-observation construction,
\[
|\bar{\cS}|
=
|\cC_o|.
\]
Thus a uniformly consistent estimator of either the quotient transports or
the quotient MDP would induce a uniformly consistent estimator of
$K(\mathfrak Q)$, contradicting the lower bound above. Therefore unrestricted
passive logs do not uniformly identify stable classes, quotient transports, or
the quotient MDP.
\end{proof}

\subsection{Proof of Lemma~\ref{lem:finite_distinguishing_witness_main}}

\begin{proof}
We prove a slightly stronger statement at the level of the refinement sequence. For every integer $t\geq0$, every observation $o\in\cO$, and every pair of layers $i,j\in[n]$ satisfying
\[
i\not\equiv_o^{\Pi^{(t)}}j,
\]
there exist a directed path $\omega$ beginning at $o$, of length at most $t$, and terminating at some observation $\bar o$, together with an action $b\in\cA$, such that
\[
R\bigl(\bar o,\sigma_\omega(i),b\bigr)
\neq
R\bigl(\bar o,\sigma_\omega(j),b\bigr).
\]
The path is allowed to be empty. For the empty path at $o$, we use the convention
\[
\sigma_{\varnothing_o}
=
\operatorname{id}_{[n]}.
\]

We establish this statement by induction on $t$. Consider first $t=0$. By the definition of the initial reward partition $\Pi^{(0)}$,
\[
i\equiv_o^{\Pi^{(0)}}j
\quad\Longleftrightarrow\quad
R(o,i,a)=R(o,j,a)
\quad
\text{for every }a\in\cA.
\]
Therefore, if
\[
i\not\equiv_o^{\Pi^{(0)}}j,
\]
then it is not true that the two layers have equal rewards under every action. Consequently, there exists an action $b\in\cA$ such that
\[
R(o,i,b)\neq R(o,j,b).
\]
Choose
\[
\omega=\varnothing_o,
\qquad
\bar o=o.
\]
Since the empty-path transport is the identity,
\[
\sigma_{\varnothing_o}(i)=i,
\qquad
\sigma_{\varnothing_o}(j)=j,
\]
and hence
\[
R\bigl(\bar o,\sigma_{\varnothing_o}(i),b\bigr)
=
R(o,i,b)
\neq
R(o,j,b)
=
R\bigl(\bar o,\sigma_{\varnothing_o}(j),b\bigr).
\]
The empty path has length zero, which is at most $t=0$. Thus the desired statement holds at the initial refinement level.

Now fix an integer $t\geq0$ and assume that the statement holds for $\Pi^{(t)}$. Let $o\in\cO$ and $i,j\in[n]$ satisfy
\[
i\not\equiv_o^{\Pi^{(t+1)}}j.
\]
Because the refinement sequence is defined recursively by
\[
\Pi^{(t+1)}
=
\mathcal T\bigl(\Pi^{(t)}\bigr),
\]
this is equivalent to
\[
i\not\equiv_o^{\mathcal T(\Pi^{(t)})}j.
\]
Definition~\ref{def:stable_family_main} states that
\[
i\equiv_o^{\mathcal T(\Pi^{(t)})}j
\]
holds if and only if both of the following conditions hold:
\[
R(o,i,a)=R(o,j,a)
\qquad
\text{for every }a\in\cA,
\]
and
\[
\sigma_{o,a,o'}(i)
\equiv_{o'}^{\Pi^{(t)}}
\sigma_{o,a,o'}(j)
\]
for every pair $(a,o')$ satisfying
\[
P_O(o'\mid o,a)>0.
\]
Since $i$ and $j$ are not equivalent in
$\mathcal T(\Pi^{(t)})$, at least one of these two requirements fails.

Suppose first that the immediate-reward requirement fails. Then there exists an action $b\in\cA$ such that
\[
R(o,i,b)\neq R(o,j,b).
\]
As in the base case, choose the empty path
\[
\omega=\varnothing_o
\]
and set $\bar o=o$. We then have
\[
\begin{aligned}
R\bigl(\bar o,\sigma_\omega(i),b\bigr)
&=
R\bigl(o,\sigma_{\varnothing_o}(i),b\bigr)\\
&=
R(o,i,b)\\
&\neq
R(o,j,b)\\
&=
R\bigl(o,\sigma_{\varnothing_o}(j),b\bigr)\\
&=
R\bigl(\bar o,\sigma_\omega(j),b\bigr).
\end{aligned}
\]
The length of this path is zero and therefore is at most $t+1$.

It remains to consider the case in which the immediate rewards agree under every action, but the successor-equivalence requirement fails. In that case there exist an action $a\in\cA$ and an observation $o_1\in\cO$ such that
\[
P_O(o_1\mid o,a)>0
\]
and
\[
\sigma_{o,a,o_1}(i)
\not\equiv_{o_1}^{\Pi^{(t)}}
\sigma_{o,a,o_1}(j).
\]
The positivity condition means that
\[
e=(o\xrightarrow{a}o_1)
\]
is an edge of the directed support graph $X$. Define the two successor layers
\[
i_1:=\sigma_{o,a,o_1}(i),
\qquad
j_1:=\sigma_{o,a,o_1}(j).
\]
The preceding failure of successor equivalence is exactly
\[
i_1\not\equiv_{o_1}^{\Pi^{(t)}}j_1.
\]
The induction hypothesis may therefore be applied at observation $o_1$ to the pair $i_1,j_1$. It yields a directed path $\nu$ from $o_1$ to some observation $\bar o$, of length at most $t$, and an action $b\in\cA$ such that
\[
R\bigl(\bar o,\sigma_\nu(i_1),b\bigr)
\neq
R\bigl(\bar o,\sigma_\nu(j_1),b\bigr).
\]

Let $\omega$ be the directed path obtained by first traversing the edge $e$ and then traversing $\nu$. Thus, if
\[
\nu=e_2\cdots e_k,
\]
we set
\[
\omega=e\,e_2\cdots e_k.
\]
This is a directed path from $o$ to $\bar o$, because the terminal observation of $e$ is $o_1$, which is the initial observation of $\nu$. By Definition~\ref{def:directed_transport_main}, composition follows execution order, and therefore
\[
\sigma_\omega
=
\sigma_\nu\circ\sigma_e,
\qquad
\sigma_e=\sigma_{o,a,o_1}.
\]
Applying this identity to $i$ gives
\[
\begin{aligned}
\sigma_\omega(i)
&=
(\sigma_\nu\circ\sigma_e)(i)\\
&=
\sigma_\nu\bigl(\sigma_e(i)\bigr)\\
&=
\sigma_\nu\bigl(\sigma_{o,a,o_1}(i)\bigr)\\
&=
\sigma_\nu(i_1).
\end{aligned}
\]
Likewise,
\[
\begin{aligned}
\sigma_\omega(j)
&=
(\sigma_\nu\circ\sigma_e)(j)\\
&=
\sigma_\nu\bigl(\sigma_e(j)\bigr)\\
&=
\sigma_\nu\bigl(\sigma_{o,a,o_1}(j)\bigr)\\
&=
\sigma_\nu(j_1).
\end{aligned}
\]
Substituting these identities into the reward inequality supplied by the induction hypothesis yields
\[
\begin{aligned}
R\bigl(\bar o,\sigma_\omega(i),b\bigr)
&=
R\bigl(\bar o,\sigma_\nu(i_1),b\bigr)\\
&\neq
R\bigl(\bar o,\sigma_\nu(j_1),b\bigr)\\
&=
R\bigl(\bar o,\sigma_\omega(j),b\bigr).
\end{aligned}
\]
The path $\omega$ consists of the initial edge $e$ followed by $\nu$. Since $\nu$ has length at most $t$, the length of $\omega$ is at most
\[
1+t=t+1.
\]
Thus the stronger statement holds for $\Pi^{(t+1)}$. By induction, it holds for every $t\geq0$.

We now apply this statement to the stable partition. Fix an observation $o\in\cO$ and two distinct stable classes
\[
c,c'\in\cC_o,
\qquad
c\neq c'.
\]
Recall that
\[
\cC_o=[n]/\Pi_o^\star.
\]
Choose representatives
\[
i\in c,
\qquad
j\in c'.
\]
Equivalently,
\[
c=[i]_o,
\qquad
c'=[j]_o.
\]
Because $c$ and $c'$ are distinct blocks of the partition $\Pi_o^\star$, their representatives are not equivalent under that partition:
\[
i\not\equiv_o^{\Pi^\star}j.
\]
By the definition of the stabilization depth,
\[
L_\star
=
\min\left\{
L\geq0:
\Pi^{(L)}=\Pi^\star
\right\},
\]
and hence
\[
\Pi^{(L_\star)}=\Pi^\star.
\]
It follows that
\[
i\not\equiv_o^{\Pi^{(L_\star)}}j.
\]
Applying the stronger statement with $t=L_\star$, we obtain a directed path $\omega$ from $o$ to some observation $\bar o$, of length at most $L_\star$, and an action $b\in\cA$ such that
\[
R\bigl(\bar o,\sigma_\omega(i),b\bigr)
\neq
R\bigl(\bar o,\sigma_\omega(j),b\bigr).
\]

It remains only to express this raw-layer distinction in quotient notation. By Corollary~\ref{cor:path_transport_descends_main},
\[
\tau_\omega([i]_o)
=
[\sigma_\omega(i)]_{\bar o}.
\]
Since $c=[i]_o$, this gives
\[
\tau_\omega(c)
=
[\sigma_\omega(i)]_{\bar o}.
\]
Applying the same corollary to $j$ and using $c'=[j]_o$ gives
\[
\tau_\omega(c')
=
[\sigma_\omega(j)]_{\bar o}.
\]
Theorem~\ref{thm:markovization_main} defines the quotient reward by
\[
\bar R(\bar o,[r]_{\bar o},b)
=
R(\bar o,r,b)
\]
for every $r\in[n]$, and proves that this value is independent of the chosen representative of the stable class. Therefore,
\[
\begin{aligned}
\bar R\bigl(\bar o,\tau_\omega(c),b\bigr)
&=
\bar R\bigl(
\bar o,
[\sigma_\omega(i)]_{\bar o},
b
\bigr)\\
&=
R\bigl(\bar o,\sigma_\omega(i),b\bigr),
\end{aligned}
\]
whereas
\[
\begin{aligned}
\bar R\bigl(\bar o,\tau_\omega(c'),b\bigr)
&=
\bar R\bigl(
\bar o,
[\sigma_\omega(j)]_{\bar o},
b
\bigr)\\
&=
R\bigl(\bar o,\sigma_\omega(j),b\bigr).
\end{aligned}
\]
Combining these identities with the raw-layer reward inequality gives
\[
\bar R\bigl(\bar o,\tau_\omega(c),b\bigr)
\neq
\bar R\bigl(\bar o,\tau_\omega(c'),b\bigr).
\]
The constructed path has length at most $L_\star$, and when the distinction is already present in the immediate rewards, the construction correctly permits $\omega$ to be the empty path. This proves the lemma.
\end{proof}

\subsection{Proof of Theorem~\ref{thm:local_class_inf_main}}

\begin{proof}
Fix an observation $o\in\cO$ and a latent state $(o,i)$, and write
\[
c=[i]_o\in\cC_o.
\]
If $\cC_o$ contains only one stable class, then Assumption~\ref{ass:local_proto_main} implies that $\widehat{\cC}_o$ also contains only one element, because
\[
\lambda_o:\widehat{\cC}_o\to\cC_o
\]
is a bijection. In that case the nearest-prototype rule necessarily returns the unique local label, whose image under $\lambda_o$ is $c$, and hence
\[
\Pr\!\left(
\lambda_o(\widehat c)\neq c
\mid
s_{\mathrm{start}}=(o,i)
\right)
=
0.
\]
It therefore suffices to consider a non-singleton observation fiber. By Definition~\ref{def:class_fp_main}, this implies $d_o\geq1$.

Because $\lambda_o$ is bijective, there exists a unique local label
\[
\widehat c^\star
:=
\lambda_o^{-1}(c)
\in\widehat{\cC}_o
\]
corresponding to the true stable class. For each diagnostic coordinate $\ell\in\{1,\ldots,d_o\}$, define
\[
G_\ell
:=
G(o,i;u_\ell)
\]
and
\[
\widehat G_\ell
:=
\widehat G(o,i;u_\ell)
=
\frac{1}{m}
\sum_{s=1}^{m}
W^{(s)}(o,i;u_\ell).
\]
Assumption~\ref{ass:diag_class_stable_main} ensures that the diagnostic mean depends on the raw layer $i$ only through its stable class $c=[i]_o$. Therefore
\[
G_\ell
=
\bar G(o,c;u_\ell)
=
\bigl(\phi_o(c)\bigr)_\ell.
\]
Consequently,
\[
\widehat G_\ell-G_\ell
=
\bigl(\widehat\phi_o(i)-\phi_o(c)\bigr)_\ell.
\]

We first derive the required concentration inequality for one fixed coordinate. Define the centered repetition noise
\[
Z_{\ell,s}
:=
W^{(s)}(o,i;u_\ell)-G(o,i;u_\ell),
\qquad
s\in\{1,\ldots,m\}.
\]
Conditional on the event
\[
s_{\mathrm{start}}=(o,i),
\]
Assumption~\ref{ass:diag_repeat_main} states that the repetitions are independent, while Assumption~\ref{ass:diag_subg_main} gives, for every $\eta\in\mathbb R$,
\[
\bE\!\left[
\exp(\eta Z_{\ell,s})
\middle|
s_{\mathrm{start}}=(o,i)
\right]
\leq
\exp\!\left(
\frac{\eta^2\nu^2}{2}
\right).
\]
Since
\[
\widehat G_\ell-G_\ell
=
\frac{1}{m}
\sum_{s=1}^{m}Z_{\ell,s},
\]
conditional independence yields, for every $\eta\in\mathbb R$,
\[
\begin{aligned}
&\bE\!\left[
\exp\!\left(
\eta(\widehat G_\ell-G_\ell)
\right)
\middle|
s_{\mathrm{start}}=(o,i)
\right]\\
&\qquad=
\bE\!\left[
\exp\!\left(
\frac{\eta}{m}
\sum_{s=1}^{m}Z_{\ell,s}
\right)
\middle|
s_{\mathrm{start}}=(o,i)
\right]\\
&\qquad=
\prod_{s=1}^{m}
\bE\!\left[
\exp\!\left(
\frac{\eta}{m}Z_{\ell,s}
\right)
\middle|
s_{\mathrm{start}}=(o,i)
\right]\\
&\qquad\leq
\prod_{s=1}^{m}
\exp\!\left(
\frac{\eta^2\nu^2}{2m^2}
\right)\\
&\qquad=
\exp\!\left(
\frac{\eta^2\nu^2}{2m}
\right).
\end{aligned}
\]
Thus, conditionally on the fixed latent start state, the empirical-mean error $\widehat G_\ell-G_\ell$ is $(\nu^2/m)$-sub-Gaussian.

For any $\varepsilon>0$ and any $\eta>0$, Markov's inequality gives
\[
\begin{aligned}
&\Pr\!\left(
\widehat G_\ell-G_\ell\geq\varepsilon
\middle|
s_{\mathrm{start}}=(o,i)
\right)\\
&\qquad=
\Pr\!\left(
\exp\!\left(
\eta(\widehat G_\ell-G_\ell)
\right)
\geq
e^{\eta\varepsilon}
\middle|
s_{\mathrm{start}}=(o,i)
\right)\\
&\qquad\leq
e^{-\eta\varepsilon}
\bE\!\left[
\exp\!\left(
\eta(\widehat G_\ell-G_\ell)
\right)
\middle|
s_{\mathrm{start}}=(o,i)
\right]\\
&\qquad\leq
\exp\!\left(
-\eta\varepsilon
+
\frac{\eta^2\nu^2}{2m}
\right).
\end{aligned}
\]
The exponent on the right-hand side is minimized at
\[
\eta^\star
=
\frac{m\varepsilon}{\nu^2}.
\]
Substituting this value gives
\[
\Pr\!\left(
\widehat G_\ell-G_\ell\geq\varepsilon
\middle|
s_{\mathrm{start}}=(o,i)
\right)
\leq
\exp\!\left(
-\frac{m\varepsilon^2}{2\nu^2}
\right).
\]
Applying the same calculation to $G_\ell-\widehat G_\ell$ yields
\[
\Pr\!\left(
G_\ell-\widehat G_\ell\geq\varepsilon
\middle|
s_{\mathrm{start}}=(o,i)
\right)
\leq
\exp\!\left(
-\frac{m\varepsilon^2}{2\nu^2}
\right).
\]
Combining the two one-sided inequalities, we obtain
\[
\Pr\!\left(
\left|
\widehat G_\ell-G_\ell
\right|
\geq\varepsilon
\middle|
s_{\mathrm{start}}=(o,i)
\right)
\leq
2\exp\!\left(
-\frac{m\varepsilon^2}{2\nu^2}
\right).
\]

Set
\[
\varepsilon
:=
\frac{\Delta}{8\sqrt{d_o}}.
\]
For each $\ell\in\{1,\ldots,d_o\}$, the preceding bound becomes
\[
\begin{aligned}
&\Pr\!\left(
\left|
\widehat G_\ell-G_\ell
\right|
\geq
\frac{\Delta}{8\sqrt{d_o}}
\middle|
s_{\mathrm{start}}=(o,i)
\right)\\
&\qquad\leq
2\exp\!\left(
-\frac{m}{2\nu^2}
\frac{\Delta^2}{64d_o}
\right)\\
&\qquad=
2\exp\!\left(
-\frac{m\Delta^2}{128\nu^2d_o}
\right).
\end{aligned}
\]
Define the event
\[
\mathcal E
:=
\left\{
\left\|
\widehat\phi_o(i)-\phi_o(c)
\right\|_\infty
<
\frac{\Delta}{8\sqrt{d_o}}
\right\}.
\]
Since
\[
\left\|
\widehat\phi_o(i)-\phi_o(c)
\right\|_\infty
=
\max_{1\leq\ell\leq d_o}
\left|
\widehat G_\ell-G_\ell
\right|,
\]
the union bound gives
\[
\begin{aligned}
&\Pr\!\left(
\mathcal E^{\mathsf c}
\middle|
s_{\mathrm{start}}=(o,i)
\right)\\
&\qquad\leq
\sum_{\ell=1}^{d_o}
\Pr\!\left(
\left|
\widehat G_\ell-G_\ell
\right|
\geq
\frac{\Delta}{8\sqrt{d_o}}
\middle|
s_{\mathrm{start}}=(o,i)
\right)\\
&\qquad\leq
2d_o
\exp\!\left(
-\frac{m\Delta^2}{128\nu^2d_o}
\right).
\end{aligned}
\]
No independence between different diagnostic coordinates is needed for this union bound.

On the event $\mathcal E$, the Euclidean error of the empirical fingerprint satisfies
\[
\begin{aligned}
\left\|
\widehat\phi_o(i)-\phi_o(c)
\right\|_2
&=
\left(
\sum_{\ell=1}^{d_o}
\left|
\widehat G_\ell-G_\ell
\right|^2
\right)^{1/2}\\
&\leq
\left(
\sum_{\ell=1}^{d_o}
\left\|
\widehat\phi_o(i)-\phi_o(c)
\right\|_\infty^2
\right)^{1/2}\\
&=
\sqrt{d_o}
\left\|
\widehat\phi_o(i)-\phi_o(c)
\right\|_\infty\\
&<
\sqrt{d_o}
\frac{\Delta}{8\sqrt{d_o}}\\
&=
\frac{\Delta}{8}.
\end{aligned}
\]
The calibrated prototype corresponding to the true class satisfies, by Assumption~\ref{ass:local_proto_main},
\[
\begin{aligned}
\left\|
\widetilde\phi_o(\widehat c^\star)-\phi_o(c)
\right\|_2
&=
\left\|
\widetilde\phi_o(\widehat c^\star)
-
\phi_o\bigl(\lambda_o(\widehat c^\star)\bigr)
\right\|_2\\
&\leq
\frac{\Delta}{8},
\end{aligned}
\]
because
\[
\lambda_o(\widehat c^\star)=c.
\]
The triangle inequality therefore implies that, on $\mathcal E$,
\[
\begin{aligned}
\left\|
\widehat\phi_o(i)
-
\widetilde\phi_o(\widehat c^\star)
\right\|_2
&\leq
\left\|
\widehat\phi_o(i)-\phi_o(c)
\right\|_2
+
\left\|
\phi_o(c)-\widetilde\phi_o(\widehat c^\star)
\right\|_2\\
&<
\frac{\Delta}{8}
+
\frac{\Delta}{8}\\
&=
\frac{\Delta}{4}.
\end{aligned}
\]

Now fix any competing local label
\[
\widehat c'\in\widehat{\cC}_o,
\qquad
\widehat c'\neq\widehat c^\star,
\]
and define its associated canonical stable class by
\[
c'
:=
\lambda_o(\widehat c').
\]
Since $\lambda_o$ is injective and $\widehat c'\neq\widehat c^\star$,
\[
c'
=
\lambda_o(\widehat c')
\neq
\lambda_o(\widehat c^\star)
=
c.
\]
Fingerprint separability in Assumption~\ref{ass:fp_sep_main} consequently gives
\[
\left\|
\phi_o(c)-\phi_o(c')
\right\|_2
\geq
\Delta.
\]
Moreover, prototype calibration gives
\[
\left\|
\widetilde\phi_o(\widehat c')
-
\phi_o(c')
\right\|_2
=
\left\|
\widetilde\phi_o(\widehat c')
-
\phi_o\bigl(\lambda_o(\widehat c')\bigr)
\right\|_2
\leq
\frac{\Delta}{8}.
\]
Using the triangle inequality in the form
\[
\|x-z\|_2
\geq
\|y-w\|_2
-
\|x-y\|_2
-
\|z-w\|_2,
\]
with
$
x=\widehat\phi_o(i),
\qquad
y=\phi_o(c),
\qquad
z=\widetilde\phi_o(\widehat c'),
\qquad
w=\phi_o(c'),
$
we obtain, on $\mathcal E$,
\[
\begin{aligned}
\left\|
\widehat\phi_o(i)
-
\widetilde\phi_o(\widehat c')
\right\|_2
&\geq
\left\|
\phi_o(c)-\phi_o(c')
\right\|_2\\
&\qquad-
\left\|
\widehat\phi_o(i)-\phi_o(c)
\right\|_2\\
&\qquad-
\left\|
\widetilde\phi_o(\widehat c')-\phi_o(c')
\right\|_2\\
&>
\Delta
-
\frac{\Delta}{8}
-
\frac{\Delta}{8}\\
&=
\frac{3\Delta}{4}.
\end{aligned}
\]
Thus, on $\mathcal E$, the empirical fingerprint is strictly closer to the correct calibrated prototype than to every incorrect calibrated prototype:
\[
\left\|
\widehat\phi_o(i)
-
\widetilde\phi_o(\widehat c^\star)
\right\|_2
<
\frac{\Delta}{4}
<
\frac{3\Delta}{4}
<
\left\|
\widehat\phi_o(i)
-
\widetilde\phi_o(\widehat c')
\right\|_2
\]
for every $\widehat c'\neq\widehat c^\star$. Hence the minimizer in Definition~\ref{def:emp_fp_main} is unique on $\mathcal E$, and the nearest-prototype rule must return
\[
\widehat c=\widehat c^\star.
\]
It follows that
\[
\lambda_o(\widehat c)
=
\lambda_o(\widehat c^\star)
=
c
\]
on $\mathcal E$. Equivalently,
\[
\left\{
\lambda_o(\widehat c)\neq c
\right\}
\subseteq
\mathcal E^{\mathsf c}.
\]
Taking conditional probabilities and using the bound for $\mathcal E^{\mathsf c}$ gives
\[
\begin{aligned}
&\Pr\!\left(
\lambda_o(\widehat c)\neq c
\middle|
s_{\mathrm{start}}=(o,i)
\right)\\
&\qquad\leq
\Pr\!\left(
\mathcal E^{\mathsf c}
\middle|
s_{\mathrm{start}}=(o,i)
\right)\\
&\qquad\leq
2d_o
\exp\!\left(
-\frac{m\Delta^2}{128\nu^2d_o}
\right),
\end{aligned}
\]
which is the claimed result.
\end{proof}


\subsection{Finite-sample and eventual diagnostic consequences}


\begin{proof}
Fix a non-singleton observation fiber $\cC_o$, a latent start state
$(o,i)$, and a confidence level $\delta\in(0,1)$. Since $\cC_o$ is
non-singleton, Definition~\ref{def:class_fp_main} gives
\[
d_o\geq 1.
\]
Moreover, Assumption~\ref{ass:fp_sep_main} gives $\Delta>0$, and
Assumption~\ref{ass:diag_subg_main} gives $\nu>0$. Hence
\[
\frac{\Delta^2}{128\nu^2d_o}>0,
\]
and every expression appearing below is well defined.

Assume that
\[
m
\geq
\frac{128\nu^2d_o}{\Delta^2}
\log\!\left(
\frac{2d_o}{\delta}
\right).
\]
Multiplying both sides by the strictly positive quantity
$\Delta^2/(128\nu^2d_o)$ preserves the direction of the inequality and
yields
\[
\begin{aligned}
\frac{m\Delta^2}{128\nu^2d_o}
&\geq
\frac{\Delta^2}{128\nu^2d_o}
\cdot
\frac{128\nu^2d_o}{\Delta^2}
\log\!\left(
\frac{2d_o}{\delta}
\right)\\
&=
\log\!\left(
\frac{2d_o}{\delta}
\right).
\end{aligned}
\]
Multiplying by $-1$ reverses the inequality, so
\[
-\frac{m\Delta^2}{128\nu^2d_o}
\leq
-\log\!\left(
\frac{2d_o}{\delta}
\right).
\]
Because the exponential function is strictly increasing, exponentiating
both sides gives
\[
\exp\!\left(
-\frac{m\Delta^2}{128\nu^2d_o}
\right)
\leq
\exp\!\left(
-\log\!\left(
\frac{2d_o}{\delta}
\right)
\right).
\]
Since $d_o\geq1$ and $\delta\in(0,1)$, the quantity $2d_o/\delta$ is
strictly positive. Therefore,
\[
\begin{aligned}
\exp\!\left(
-\log\!\left(
\frac{2d_o}{\delta}
\right)
\right)
&=
\exp\!\left(
\log\!\left(
\frac{\delta}{2d_o}
\right)
\right)\\
&=
\frac{\delta}{2d_o}.
\end{aligned}
\]
It follows that
\[
\begin{aligned}
2d_o
\exp\!\left(
-\frac{m\Delta^2}{128\nu^2d_o}
\right)
&\leq
2d_o
\exp\!\left(
-\log\!\left(
\frac{2d_o}{\delta}
\right)
\right)\\
&=
2d_o\cdot\frac{\delta}{2d_o}\\
&=
\delta.
\end{aligned}
\]

Applying Theorem~\ref{thm:local_class_inf_main} to the latent start state
$(o,i)$, whose true stable class is $[i]_o$, gives
\[
\Pr\!\left(
\lambda_o(\widehat c)\neq[i]_o
\middle|
s_{\mathrm{start}}=(o,i)
\right)
\leq
2d_o
\exp\!\left(
-\frac{m\Delta^2}{128\nu^2d_o}
\right).
\]
Combining this inequality with the preceding bound yields
\[
\Pr\!\left(
\lambda_o(\widehat c)\neq[i]_o
\middle|
s_{\mathrm{start}}=(o,i)
\right)
\leq
\delta,
\]
which proves the claimed sample-complexity guarantee. Since the number
of diagnostic repetitions must be an integer, it is sufficient in
particular to choose
\[
m
=
\left\lceil
\frac{128\nu^2d_o}{\Delta^2}
\log\!\left(
\frac{2d_o}{\delta}
\right)
\right\rceil.
\]
For a singleton fiber, the local label set is also a singleton by the
bijection in Assumption~\ref{ass:local_proto_main}, so the inference
error is identically zero and no diagnostic repetitions are required.
\end{proof}

\subsection{Proof of Corollary~\ref{cor:eventual_local_class_main}}


\begin{proof}
Fix an observation $o\in\cO$. If the stable fiber $\cC_o$ is a singleton, then Assumption~\ref{ass:local_proto_main} implies that the local label set $\widehat{\cC}_o$ is also a singleton, because
\[
\lambda_o:\widehat{\cC}_o\to\cC_o
\]
is a bijection. Hence every classification event at $o$ returns the unique local label, whose image under $\lambda_o$ is the unique stable class in $\cC_o$. Therefore,
\[
\Pr\!\left(
\lambda_o(\widehat c_{o,k})\neq c_{o,k}
\right)
=
0
\]
for every classification event $k$. Thus the conclusion is immediate for a singleton fiber. In the remainder of the proof, assume that $\cC_o$ is non-singleton. By Definition~\ref{def:class_fp_main}, this implies
\[
d_o\geq1,
\]
so every quantity appearing in the stated schedule is well defined.

For each classification event $k\geq2$, let $I_{o,k}\in[n]$ denote the latent layer at the resettable start checkpoint of that event. The corresponding latent start state and true stable class are
\[
S_{o,k}:=(o,I_{o,k})
\qquad\text{and}\qquad
c_{o,k}:=[I_{o,k}]_o.
\]
Define the error event
\[
E_{o,k}
:=
\left\{
\lambda_o(\widehat c_{o,k})\neq c_{o,k}
\right\}.
\]
Conditional on the event $\{I_{o,k}=i\}$, the diagnostic procedure starts from the fixed latent state $(o,i)$, uses $m_{o,k}$ repetitions for each diagnostic coordinate, and has true stable class $[i]_o$. Therefore, Theorem~\ref{thm:local_class_inf_main} gives, for every $i\in[n]$ such that $\Pr(I_{o,k}=i)>0$,
\[
\Pr\!\left(
E_{o,k}
\middle|
I_{o,k}=i
\right)
\leq
2d_o
\exp\!\left(
-\frac{m_{o,k}\Delta^2}{128\nu^2d_o}
\right).
\]
The right-hand side is independent of the particular value of $i$. Since $I_{o,k}$ takes values in the finite set $[n]$, the law of total probability yields
\[
\begin{aligned}
\Pr(E_{o,k})
&=
\sum_{\substack{i\in[n]:\\ \Pr(I_{o,k}=i)>0}}
\Pr\!\left(
E_{o,k}
\middle|
I_{o,k}=i
\right)
\Pr(I_{o,k}=i)\\
&\leq
\sum_{\substack{i\in[n]:\\ \Pr(I_{o,k}=i)>0}}
2d_o
\exp\!\left(
-\frac{m_{o,k}\Delta^2}{128\nu^2d_o}
\right)
\Pr(I_{o,k}=i)\\
&=
2d_o
\exp\!\left(
-\frac{m_{o,k}\Delta^2}{128\nu^2d_o}
\right)
\sum_{\substack{i\in[n]:\\ \Pr(I_{o,k}=i)>0}}
\Pr(I_{o,k}=i)\\
&=
2d_o
\exp\!\left(
-\frac{m_{o,k}\Delta^2}{128\nu^2d_o}
\right).
\end{aligned}
\]

Assume now that, for some $\eta>0$,
\[
m_{o,k}
\geq
\frac{128\nu^2d_o}{\Delta^2}
\left[
(1+\eta)\log k+\log(2d_o)
\right].
\]
Since $\Delta>0$, $\nu>0$, and $d_o\geq1$, the factor
\[
\frac{\Delta^2}{128\nu^2d_o}
\]
is strictly positive. Multiplying the preceding inequality by this positive factor preserves the direction of the inequality and gives
\[
\begin{aligned}
\frac{m_{o,k}\Delta^2}{128\nu^2d_o}
&\geq
\frac{\Delta^2}{128\nu^2d_o}
\cdot
\frac{128\nu^2d_o}{\Delta^2}
\left[
(1+\eta)\log k+\log(2d_o)
\right]\\
&=
(1+\eta)\log k+\log(2d_o).
\end{aligned}
\]
Multiplying by $-1$ reverses the inequality:
\[
-\frac{m_{o,k}\Delta^2}{128\nu^2d_o}
\leq
-(1+\eta)\log k-\log(2d_o).
\]
Because the exponential function is increasing,
\[
\exp\!\left(
-\frac{m_{o,k}\Delta^2}{128\nu^2d_o}
\right)
\leq
\exp\!\left(
-(1+\eta)\log k-\log(2d_o)
\right).
\]
Using
\[
\exp(x+y)=\exp(x)\exp(y),
\exp(-\log x)=\frac{1}{x}
\quad\text{for }x>0,
\]
we obtain
\[
\begin{aligned}
&2d_o
\exp\!\left(
-\frac{m_{o,k}\Delta^2}{128\nu^2d_o}
\right)\\
&\leq
2d_o
\exp\!\left(
-(1+\eta)\log k-\log(2d_o)
\right)\\
&=
2d_o
\exp\!\left(
-(1+\eta)\log k
\right)
\exp\!\left(
-\log(2d_o)
\right)\\
&=
2d_o
k^{-(1+\eta)}
\frac{1}{2d_o}\\
&=
k^{-(1+\eta)}.
\end{aligned}
\]
Combining this inequality with the unconditional error bound derived above gives
\[
\Pr\!\left(
\lambda_o(\widehat c_{o,k})\neq c_{o,k}
\right)
=
\Pr(E_{o,k})
\leq
k^{-(1+\eta)}.
\]

Since $\eta>0$, we have $1+\eta>1$. Hence the corresponding $p$-series converges:
\[
\sum_{k=2}^{\infty}k^{-(1+\eta)}<\infty.
\]
For completeness, this can also be verified by the integral comparison
\[
\begin{aligned}
\sum_{k=2}^{\infty}k^{-(1+\eta)}
&\leq
\int_{1}^{\infty}x^{-(1+\eta)}\,dx\\
&=
\left[
-\frac{x^{-\eta}}{\eta}
\right]_{1}^{\infty}\\
&=
\frac{1}{\eta}\\
&<
\infty.
\end{aligned}
\]
Therefore,
\[
\sum_{k=2}^{\infty}\Pr(E_{o,k})
\leq
\sum_{k=2}^{\infty}k^{-(1+\eta)}
<
\infty.
\]
The first Borel--Cantelli lemma now implies
\[
\Pr\!\left(
E_{o,k}\ \text{infinitely often}
\right)
=
0.
\]
No independence assumption among the events $E_{o,k}$ is required for this implication. Equivalently, with probability one, there exists a finite random index $K_o$ such that
\[
\lambda_o(\widehat c_{o,k})
=
c_{o,k}
\qquad
\text{for every }k\geq K_o.
\]
Thus, almost surely, only finitely many local classification errors occur at the fixed observation $o$. If only finitely many classification events occur at $o$, then the number of errors at $o$ is trivially finite, so the same conclusion remains valid.

Finally, suppose that every required observation follows such a schedule. For each $o\in\cO$, define
\[
B_o
:=
\left\{
E_{o,k}\ \text{occurs infinitely often}
\right\}.
\]
The preceding argument gives
\[
\Pr(B_o)=0
\]
for every $o\in\cO$. Since the observation space $\cO$ is finite,
\[
\Pr\!\left(
\bigcup_{o\in\cO}B_o
\right)
\leq
\sum_{o\in\cO}\Pr(B_o)
=
0.
\]
Hence, with probability one, no observation incurs infinitely many classification errors. On this probability-one event, every observation contributes only finitely many errors, and because $\cO$ is finite, their total number is finite:
\[
\sum_{o\in\cO}
\sum_{k=2}^{\infty}
\mathbf{1}\!\left\{
\lambda_o(\widehat c_{o,k})\neq c_{o,k}
\right\}
<
\infty
\qquad
\text{almost surely}.
\]
Therefore, there are almost surely only finitely many local classification errors system-wide.
\end{proof}

\subsection{Proof of Theorem~\ref{thm:transport_local_main}}

\begin{proof}
Fix an arbitrary feasible edge
$e=(o\xrightarrow{a}o')$
and an arbitrary local source label
$\hat c\in\widehat{\cC}_o$.
Set
\[
c
:=
\lambda_o(\hat c)
\in
\cC_o
\]
and define the corresponding local target label by
\[
\hat c^{\star}
:=
\tau_e^\lambda(\hat c)
=
\lambda_{o'}^{-1}
\bigl(
\tau_e(\lambda_o(\hat c))
\bigr)
=
\lambda_{o'}^{-1}
\bigl(
\tau_e(c)
\bigr).
\]
By the prototype-gauge construction, $\lambda_o$ and $\lambda_{o'}$ are bijections, so both $c$ and $\hat c^{\star}$ are uniquely determined. Index the transition-anchored samples associated with the fixed edge $e$ by $r=1,2,\ldots$. For the $r$-th such sample, let $(o,i_r)$ be its exact source checkpoint, let $(o',\sigma_e(i_r))$ be the paired target checkpoint supplied by Assumption~\ref{ass:transition_anchor_main}, and write
\[
c_r^{-}
:=
[i_r]_o,
\qquad
c_r^{+}
:=
[\sigma_e(i_r)]_{o'}.
\]
By the definition of the deterministic quotient transport in Theorem~\ref{thm:markovization_main},
\[
\begin{aligned}
c_r^{+}
&=
[\sigma_e(i_r)]_{o'}\\
&=
\tau_e([i_r]_o)\\
&=
\tau_e(c_r^{-}).
\end{aligned}
\]
Let
$\widehat c_r^{-}\in\widehat{\cC}_o$
and
$\widehat c_r^{+}\in\widehat{\cC}_{o'}$
denote the source and target labels inferred from these two paired checkpoints. The source inference is correct when
\[
\lambda_o(\widehat c_r^{-})
=
c_r^{-},
\]
and the target inference is correct when
\[
\lambda_{o'}(\widehat c_r^{+})
=
c_r^{+}.
\]
By hypothesis, only finitely many source- or target-class inference errors occur almost surely. In particular, after restricting attention to the samples associated with the fixed edge $e$, there are still almost surely only finitely many indices at which at least one of the two equalities above fails. Therefore, on an event of probability one, there exists a finite random index
$R_{\mathrm{err}}\geq1$
such that, for every $r\geq R_{\mathrm{err}}$,
\[
\lambda_o(\widehat c_r^{-})
=
c_r^{-}
\qquad\text{and}\qquad
\lambda_{o'}(\widehat c_r^{+})
=
c_r^{+}.
\]
We work on this probability-one event throughout the remainder of the proof.

Assumption~\ref{ass:edge_cov_main}, applied to the fixed edge $e$ and the true source class
$c=\lambda_o(\hat c)$,
states that the protocol produces infinitely many transition-anchored samples satisfying
\[
c_r^{-}
=
c.
\]
Because only finitely many sample indices precede $R_{\mathrm{err}}$, the set
\[
\mathcal I
:=
\left\{
r\geq R_{\mathrm{err}}
:
c_r^{-}=c
\right\}
\]
is also infinite. Fix any $r\in\mathcal I$. Since the source inference is correct at every index not smaller than $R_{\mathrm{err}}$,
\[
\begin{aligned}
\lambda_o(\widehat c_r^{-})
&=
c_r^{-}\\
&=
c\\
&=
\lambda_o(\hat c).
\end{aligned}
\]
The map $\lambda_o$ is injective, and hence
\[
\widehat c_r^{-}
=
\hat c.
\]
For the same anchored sample, target correctness and the deterministic quotient-transport identity give
\[
\begin{aligned}
\lambda_{o'}(\widehat c_r^{+})
&=
c_r^{+}\\
&=
\tau_e(c_r^{-})\\
&=
\tau_e(c)\\
&=
\lambda_{o'}\!\left(
\lambda_{o'}^{-1}(\tau_e(c))
\right)\\
&=
\lambda_{o'}(\hat c^{\star}).
\end{aligned}
\]
Since $\lambda_{o'}$ is injective, it follows that
\[
\widehat c_r^{+}
=
\hat c^{\star}.
\]
Thus every sample indexed by $r\in\mathcal I$ produces the paired inferred labels
\[
(\widehat c_r^{-},\widehat c_r^{+})
=
(\hat c,\hat c^{\star})
\]
and therefore increments exactly the count
\[
N_e[\hat c,\hat c^{\star}].
\]

For clarity, let $N_e^{(r)}$ denote the count matrix after the first $r$ transition-anchored samples associated with edge $e$ have been incorporated. For every $r\geq R_{\mathrm{err}}$, the update rule in Definition~\ref{def:rowwise_transport_main} gives
\[
N_e^{(r)}[\hat c,\hat c^{\star}]
\geq
N_e^{(R_{\mathrm{err}}-1)}
[\hat c,\hat c^{\star}]
+
\left|
\mathcal I
\cap
\{R_{\mathrm{err}},\ldots,r\}
\right|.
\]
Because $\mathcal I$ is infinite,
\[
\left|
\mathcal I
\cap
\{R_{\mathrm{err}},\ldots,r\}
\right|
\longrightarrow
\infty
\qquad
\text{as }r\longrightarrow\infty.
\]
Consequently,
\[
N_e^{(r)}[\hat c,\hat c^{\star}]
\longrightarrow
\infty
\qquad
\text{as }r\longrightarrow\infty.
\]

We next show that every competing entry in the row indexed by $\hat c$ stops changing after the finite index $R_{\mathrm{err}}-1$. Let
\[
\hat d
\in
\widehat{\cC}_{o'}
\setminus
\{\hat c^{\star}\}
\]
be an arbitrary competing target label. Suppose that some sample with index
$r\geq R_{\mathrm{err}}$
increments an entry in the row indexed by $\hat c$. By the estimator update rule, this means that its inferred source label satisfies
\[
\widehat c_r^{-}
=
\hat c.
\]
Because the source inference is correct at this index,
\[
\begin{aligned}
c_r^{-}
&=
\lambda_o(\widehat c_r^{-})\\
&=
\lambda_o(\hat c)\\
&=
c.
\end{aligned}
\]
The paired target checkpoint therefore has true quotient class
\[
\begin{aligned}
c_r^{+}
&=
\tau_e(c_r^{-})\\
&=
\tau_e(c).
\end{aligned}
\]
Target correctness now implies
\[
\begin{aligned}
\lambda_{o'}(\widehat c_r^{+})
&=
c_r^{+}\\
&=
\tau_e(c)\\
&=
\lambda_{o'}(\hat c^{\star}).
\end{aligned}
\]
Injectivity of $\lambda_{o'}$ yields
\[
\widehat c_r^{+}
=
\hat c^{\star}.
\]
Hence, after $R_{\mathrm{err}}-1$, every sample that increments the row indexed by $\hat c$ must increment its correct column $\hat c^{\star}$. No sample with index
$r\geq R_{\mathrm{err}}$
can increment the competing entry
$N_e[\hat c,\hat d]$.
It follows that, for every
$\hat d\neq\hat c^{\star}$
and every
$r\geq R_{\mathrm{err}}$,
\[
N_e^{(r)}[\hat c,\hat d]
=
N_e^{(R_{\mathrm{err}}-1)}
[\hat c,\hat d].
\]

If
$\widehat{\cC}_{o'}=\{\hat c^{\star}\}$,
then $\hat c^{\star}$ is the only target label and is therefore the unique maximizer in the row indexed by $\hat c$ at every time, so the result is immediate. Suppose henceforth that
$\widehat{\cC}_{o'}\setminus\{\hat c^{\star}\}$
is nonempty. Since $\widehat{\cC}_{o'}$ is finite and every count at the finite index
$R_{\mathrm{err}}-1$
is finite, the random quantity
\[
B
:=
\max_{
\hat d\in
\widehat{\cC}_{o'}
\setminus
\{\hat c^{\star}\}
}
N_e^{(R_{\mathrm{err}}-1)}
[\hat c,\hat d]
\]
is finite. Moreover, the preceding stabilization identity implies that, for every
$r\geq R_{\mathrm{err}}$
and every
$\hat d\neq\hat c^{\star}$,
\[
N_e^{(r)}[\hat c,\hat d]
\leq
B.
\]
On the other hand,
\[
N_e^{(r)}[\hat c,\hat c^{\star}]
\longrightarrow
\infty.
\]
Therefore, there exists a finite random index
$R_{\star}\geq R_{\mathrm{err}}$
such that
\[
N_e^{(R_{\star})}
[\hat c,\hat c^{\star}]
>
B.
\]
The entry
$N_e^{(r)}[\hat c,\hat c^{\star}]$
is nondecreasing in $r$, because every estimator update only adds one to a count and never decreases any count. Thus, for every
$r\geq R_{\star}$,
\[
N_e^{(r)}
[\hat c,\hat c^{\star}]
\geq
N_e^{(R_{\star})}
[\hat c,\hat c^{\star}]
>
B.
\]
Combining this inequality with the bound on every competing column gives, for every
$r\geq R_{\star}$
and every
$\hat d\in\widehat{\cC}_{o'}$
with
$\hat d\neq\hat c^{\star}$,
\[
N_e^{(r)}
[\hat c,\hat c^{\star}]
>
B
\geq
N_e^{(r)}
[\hat c,\hat d].
\]
Hence $\hat c^{\star}$ is the unique maximizer of the row
\[
\hat d
\longmapsto
N_e^{(r)}[\hat c,\hat d]
\]
for every $r\geq R_{\star}$. In particular, the deterministic tie-breaking convention is no longer relevant after $R_{\star}$, because no tie involving the correct column is possible. Definition~\ref{def:rowwise_transport_main} therefore gives
\[
\begin{aligned}
\widehat\tau_e(\hat c)
&=
\hat c^{\star}\\
&=
\tau_e^\lambda(\hat c)\\
&=
\lambda_{o'}^{-1}
\bigl(
\tau_e(\lambda_o(\hat c))
\bigr)
\end{aligned}
\]
after the $R_{\star}$-th anchored sample of edge $e$ has been processed.

The index $R_{\star}$ is almost surely finite, so the corresponding anchored sample occurs at an almost surely finite protocol time. Between two updates of $N_e$, the estimator $\widehat\tau_e$ remains unchanged, and every subsequent update preserves the strict dominance of the correct column established above. The displayed equality therefore holds permanently after that finite time. Since the feasible edge $e$ and the source label
$\hat c\in\widehat{\cC}_o$
were arbitrary, the conclusion holds for every feasible edge and every local source label.
\end{proof}

\subsection{Proof of Lemma~\ref{lem:exact_local_tracking_main}}

\begin{proof}
Fix an arbitrary realization of the process on the reset-free interval beginning at time $t_0$. The claim is pathwise and therefore does not require any probabilistic qualification. We prove by induction over the time indices contained in this interval that
\[
\lambda_{o_t}(\widehat c_t)=c_t.
\]
At the initial time $t=t_0$, this identity is exactly the assumed initialization condition,
\[
\lambda_{o_{t_0}}(\widehat c_{t_0})
=
c_{t_0}.
\]
Now fix any time $t\geq t_0$ such that both $t$ and $t+1$ belong to the same reset-free interval, and suppose that
\[
\lambda_{o_t}(\widehat c_t)
=
c_t.
\]
Let
\[
e_t
:=
(o_t\xrightarrow{a_t}o_{t+1})
\]
be the realized edge at time $t$. Since this edge is realized by the process, it is feasible, and by its definition,
\[
\src(e_t)=o_t
\qquad\text{and}\qquad
\tgt(e_t)=o_{t+1}.
\]
Because no reset occurs between times $t$ and $t+1$, Definition~\ref{def:local_class_tracker_main} updates the local label by propagation rather than by a new diagnostic initialization. Hence
\[
\widehat c_{t+1}
=
\widehat\tau_{e_t}(\widehat c_t).
\]
The hypothesis of the lemma states that every edge used on the interval has the correct local-coordinate transport. Applied to the realized edge $e_t$, this gives the equality of maps
\[
\widehat\tau_{e_t}
=
\lambda_{\tgt(e_t)}^{-1}
\circ
\tau_{e_t}
\circ
\lambda_{\src(e_t)}.
\]
Evaluating both sides at $\widehat c_t$ and using
$\src(e_t)=o_t$
and
$\tgt(e_t)=o_{t+1}$,
we obtain
\[
\begin{aligned}
\widehat c_{t+1}
&=
\widehat\tau_{e_t}(\widehat c_t)\\
&=
\left(
\lambda_{\tgt(e_t)}^{-1}
\circ
\tau_{e_t}
\circ
\lambda_{\src(e_t)}
\right)(\widehat c_t)\\
&=
\lambda_{o_{t+1}}^{-1}
\left(
\tau_{e_t}
\left(
\lambda_{o_t}(\widehat c_t)
\right)
\right).
\end{aligned}
\]
Applying $\lambda_{o_{t+1}}$ to this identity yields
\[
\begin{aligned}
\lambda_{o_{t+1}}(\widehat c_{t+1})
&=
\lambda_{o_{t+1}}
\left[
\lambda_{o_{t+1}}^{-1}
\left(
\tau_{e_t}
\left(
\lambda_{o_t}(\widehat c_t)
\right)
\right)
\right]\\
&=
\tau_{e_t}
\left(
\lambda_{o_t}(\widehat c_t)
\right).
\end{aligned}
\]
The cancellation in the second equality is valid because
\[
\lambda_{o_{t+1}}
:
\widehat{\cC}_{o_{t+1}}
\longrightarrow
\cC_{o_{t+1}}
\]
is a bijection. By the induction hypothesis,
\[
\lambda_{o_t}(\widehat c_t)=c_t,
\]
and therefore
\[
\lambda_{o_{t+1}}(\widehat c_{t+1})
=
\tau_{e_t}(c_t).
\]

It remains only to identify the right-hand side with the actual successor class. By definition,
\[
c_t=[i_t]_{o_t}
\qquad\text{and}\qquad
c_{t+1}=[i_{t+1}]_{o_{t+1}}.
\]
Along the realized edge
$e_t=(o_t\xrightarrow{a_t}o_{t+1})$,
the hidden layer evolves according to
\[
i_{t+1}
=
\sigma_{e_t}(i_t).
\]
The quotient transport induced by this raw hidden transition is
\[
\tau_{e_t}([i]_{o_t})
=
[\sigma_{e_t}(i)]_{o_{t+1}}.
\]
Substituting $i=i_t$ gives
\[
\begin{aligned}
\tau_{e_t}(c_t)
&=
\tau_{e_t}([i_t]_{o_t})\\
&=
[\sigma_{e_t}(i_t)]_{o_{t+1}}\\
&=
[i_{t+1}]_{o_{t+1}}\\
&=
c_{t+1}.
\end{aligned}
\]
Combining the preceding identities, we conclude that
\[
\lambda_{o_{t+1}}(\widehat c_{t+1})
=
\tau_{e_t}(c_t)
=
c_{t+1}.
\]
Thus correctness at time $t$ implies correctness at time $t+1$ whenever both times lie in the same reset-free interval. Since correctness holds at the initial time $t_0$, induction yields
\[
\lambda_{o_t}(\widehat c_t)
=
c_t
\]
for every $t\geq t_0$ belonging to that interval.
\end{proof}

\subsection{Proof of Proposition~\ref{prop:eventual_sync_state_main}}

\begin{proof}
Let
\[
0\leq \zeta_0<\zeta_1<\zeta_2<\cdots
\]
be the successive exogenous initialization times. By assumption, these times
are almost surely unbounded, so the sequence is almost surely infinite and
$\zeta_j\to\infty$.

For each $j\geq0$, let
\[
E_j^{\mathrm{init}}
:=
\left\{
\lambda_{o_{\zeta_j}}(\widehat c_{\zeta_j})
\neq
c_{\zeta_j}
\right\}
\]
be the event that the diagnostic initialization at time $\zeta_j$ is
incorrect. The assumed summability of the initialization error probabilities
gives
\[
\sum_{j=0}^{\infty}
\bP(E_j^{\mathrm{init}})
<
\infty.
\]
The first Borel--Cantelli lemma therefore implies
\[
\bP\!\left(
E_j^{\mathrm{init}}
\text{ occurs infinitely often}
\right)
=0.
\]
Hence, on an event of probability one, there exists a finite random index
$J_{\mathrm{cls}}$ such that
\[
\lambda_{o_{\zeta_j}}(\widehat c_{\zeta_j})
=
c_{\zeta_j}
\qquad
\text{for every }j\geq J_{\mathrm{cls}}.
\]
Set
\[
T_{\mathrm{cls}}
:=
\zeta_{J_{\mathrm{cls}}}.
\]
Then every exogenous initialization occurring at or after
$T_{\mathrm{cls}}$ is correct.

We next obtain a common time after which every transport entry that can be used
by the tracker is correct. Let
\[
\mathcal P
:=
\left\{
(e,\hat c):
e=(o\xrightarrow{a}o')\in E,\;
\hat c\in\widehat{\cC}_o
\right\}.
\]
The set $\mathcal P$ is finite because $E$ and every local label set are
finite. Write $\widehat\tau_e^{(t)}$ for the transport estimate available to
the tracker at time $t$. By the assumed eventual recovery of every required
transport entry, for each $(e,\hat c)\in\mathcal P$ there exists an almost
surely finite random time $T_{e,\hat c}$ such that, for every
$t\geq T_{e,\hat c}$,
\[
\widehat\tau_e^{(t)}(\hat c)
=
\lambda_{\tgt(e)}^{-1}\!\left(
\tau_e\bigl(\lambda_{\src(e)}(\hat c)\bigr)
\right),
\]
and the equality remains true permanently. Since $\mathcal P$ is finite, the
random time
\[
T_{\mathrm{tr}}
:=
\max_{(e,\hat c)\in\mathcal P}
T_{e,\hat c}
\]
is almost surely finite. Consequently, for every $t\geq T_{\mathrm{tr}}$ and
every feasible edge $e$,
\[
\widehat\tau_e^{(t)}
=
\lambda_{\tgt(e)}^{-1}
\circ
\tau_e
\circ
\lambda_{\src(e)}
\]
as maps on $\widehat{\cC}_{\src(e)}$.

Let
\[
T_\star
:=
\max\{T_{\mathrm{cls}},T_{\mathrm{tr}}\}.
\]
Because the initialization times are unbounded, the random index
\[
J
:=
\min\{j\geq0:\zeta_j\geq T_\star\}
\]
is almost surely finite. Define
\[
T:=\zeta_J.
\]
Then $T\geq T_{\mathrm{cls}}$, so the initialization at time $T$ is correct:
\[
\lambda_{o_T}(\widehat c_T)=c_T.
\]
Also $T\geq T_{\mathrm{tr}}$, so every transport used at or after time $T$ is
the true quotient transport expressed in the fixed local gauges.

Consider the reset-free interval beginning at $\zeta_J=T$. Its initial label
is correct, and every transport used on the interval is correct.
Lemma~\ref{lem:exact_local_tracking_main} therefore gives
\[
\lambda_{o_t}(\widehat c_t)=c_t
\]
throughout that interval. Now let $j>J$. Since
$\zeta_j\geq T\geq T_{\mathrm{cls}}$, the initialization at $\zeta_j$ is also
correct; since $\zeta_j\geq T\geq T_{\mathrm{tr}}$, all transports used on the
subsequent reset-free interval are correct. Applying
Lemma~\ref{lem:exact_local_tracking_main} again yields exact class tracking on
that entire interval.

The reset-free intervals beginning at
$\zeta_J,\zeta_{J+1},\zeta_{J+2},\ldots$ cover every time $t\geq T$.
Therefore
\[
\lambda_{o_t}(\widehat c_t)=c_t
\qquad
\text{for every }t\geq T.
\]
All random times used above are finite on an event of probability one, which
proves the proposition.
\end{proof}

\subsection{Proof of Lemma~\ref{lem:gauge_policy_transfer_main}}

\begin{proof}
Define
\[
\Lambda_\lambda(o,\hat c)
:=
(o,\lambda_o(\hat c)).
\]
Because every
$\lambda_o:\widehat{\cC}_o\to\cC_o$ is bijective,
$\Lambda_\lambda:\widehat{\cS}\to\bar{\cS}$ is bijective, with inverse
\[
\Lambda_\lambda^{-1}(o,c)
=
(o,\lambda_o^{-1}(c)).
\]
The policy transformation
\[
\bar\pi(a\mid o,c)
=
\widehat\pi(a\mid o,\lambda_o^{-1}(c))
\]
is therefore well defined and bijective between stationary policies on
$\widehat{\cS}$ and stationary policies on $\bar{\cS}$. In particular,
\[
\bar\pi\bigl(a\mid\Lambda_\lambda(o,\hat c)\bigr)
=
\widehat\pi(a\mid o,\hat c).
\]

By Definition~\ref{def:quotient_lifted_state_main}, rewards are preserved:
\[
\widehat R^\lambda(o,\hat c,a)
=
\bar R\bigl(\Lambda_\lambda(o,\hat c),a\bigr).
\]
The transition kernels are preserved as well. For
$\hat z=(o,\hat c)$ and $\hat z'=(o',\hat c')$, if
$P_O(o'\mid o,a)=0$, both transition probabilities are zero. If
$P_O(o'\mid o,a)>0$, then
\[
\begin{aligned}
\widehat P^\lambda(\hat z'\mid\hat z,a)
&=
P_O(o'\mid o,a)
\mathbf 1\!\left\{
\hat c'=\tau_{o,a,o'}^\lambda(\hat c)
\right\}\\
&=
P_O(o'\mid o,a)
\mathbf 1\!\left\{
\lambda_{o'}(\hat c')
=
\tau_{o,a,o'}\bigl(\lambda_o(\hat c)\bigr)
\right\}\\
&=
\bar P\!\left(
\Lambda_\lambda(\hat z')
\mid
\Lambda_\lambda(\hat z),a
\right).
\end{aligned}
\]

For a bounded function $f:\bar{\cS}\to\mathbb R$, define its pullback by
\[
(U_\lambda f)(\hat z)
:=
f(\Lambda_\lambda(\hat z)).
\]
Let $\bar T_{\bar\pi}$ and $\widehat T_{\widehat\pi}$ be the
policy-evaluation Bellman operators of the canonical and local-coordinate
quotient MDPs, respectively. Using the policy, reward, and transition
identities above, and then changing variables through the bijection
$\Lambda_\lambda$, gives, for every bounded $f$,
\[
\widehat T_{\widehat\pi}(U_\lambda f)
=
U_\lambda(\bar T_{\bar\pi}f).
\]
Both Bellman operators are $\gamma$-contractions on their finite state spaces.
Their unique fixed points therefore satisfy
\[
\widehat V_{\widehat\pi}
=
U_\lambda\bar V_{\bar\pi}.
\]
Equivalently, for every $(o,\hat c)\in\widehat{\cS}$,
\[
\widehat V_{\widehat\pi}(o,\hat c)
=
\bar V_{\bar\pi}\bigl(o,\lambda_o(\hat c)\bigr).
\]

The same argument applies to the optimality operators. If $\bar T$ and
$\widehat T$ denote the two Bellman optimality operators, then
\[
\widehat T(U_\lambda f)
=
U_\lambda(\bar T f)
\]
for every bounded $f$, because the action set is unchanged and the maximum is
taken over the same actions on both sides. Uniqueness of the optimal fixed
points yields
\[
\widehat V^*
=
U_\lambda\bar V^*.
\]
Hence
\[
\widehat V^*(o,\hat c)
=
\bar V^*\bigl(o,\lambda_o(\hat c)\bigr).
\]

Finally, the policy transformation is bijective and preserves every statewise
value. Thus $\widehat\pi$ attains $\widehat V^*$ at every local state if and
only if $\bar\pi$ attains $\bar V^*$ at every corresponding canonical state.
Therefore $\widehat\pi$ is optimal if and only if $\bar\pi$ is optimal.
\end{proof}

\subsection{Proof of Theorem~\ref{thm:eventual_exact_reduction_main}}

\begin{proof}
Fix a deterministic time $t$ and a candidate successor
$\hat z'=(o',\hat c')\in\widehat{\cS}$. By definition, $A_t$ is determined by
the current latent quotient class, the current learned tracker state, the
transport tables available at time $t$, and the selected action $a_t$.
Therefore
\[
A_t\in\mathcal G_t^-.
\]
On $A_t$,
\[
c_t=[i_t]_{o_t}
=
\lambda_{o_t}(\widehat c_t).
\]

Suppose first that $P_O(o'\mid o_t,a_t)>0$, and write
\[
e=(o_t\xrightarrow{a_t}o').
\]
On the event $A_t\cap\{o_{t+1}=o'\}$, the tracker recursion and the
definition of $A_t$ give
\[
\begin{aligned}
\widehat c_{t+1}
&=
\widehat\tau_e(\widehat c_t)\\
&=
\tau_e^\lambda(\widehat c_t)\\
&=
\lambda_{o'}^{-1}\!\left(
\tau_e\bigl(\lambda_{o_t}(\widehat c_t)\bigr)
\right)\\
&=
\lambda_{o'}^{-1}\!\left(
\tau_e(c_t)
\right).
\end{aligned}
\]
By Theorem~\ref{thm:markovization_main}, on $\{o_{t+1}=o'\}$,
\[
c_{t+1}=\tau_e(c_t).
\]
Hence, on $A_t\cap\{o_{t+1}=o'\}$,
\[
\widehat c_{t+1}
=
\lambda_{o'}^{-1}(c_{t+1}).
\]
Equivalently,
\[
\begin{aligned}
&\mathbf 1_{A_t}
\mathbf 1\!\left\{
\widehat z_{t+1}=(o',\hat c')
\right\}\\
&\qquad=
\mathbf 1_{A_t}
\mathbf 1\{o_{t+1}=o'\}
\mathbf 1\!\left\{
\hat c'
=
\tau_{o_t,a_t,o'}^\lambda(\widehat c_t)
\right\}.
\end{aligned}
\]
Taking conditional expectation with respect to $\mathcal G_t^-$ is valid
because $A_t\in\mathcal G_t^-$. The HCDP transition law gives
\[
\bP(o_{t+1}=o'\mid\mathcal G_t^-)
=
P_O(o'\mid o_t,a_t),
\]
so
\[
\begin{aligned}
&\mathbf 1_{A_t}
\bP\!\left(
\widehat z_{t+1}=(o',\hat c')
\mid
\mathcal G_t^-
\right)\\
&\qquad=
\mathbf 1_{A_t}
P_O(o'\mid o_t,a_t)
\mathbf 1\!\left\{
\hat c'
=
\tau_{o_t,a_t,o'}^\lambda(\widehat c_t)
\right\}\\
&\qquad=
\mathbf 1_{A_t}
\widehat P^\lambda\!\left(
(o',\hat c')
\mid
\widehat z_t,a_t
\right).
\end{aligned}
\]
If $P_O(o'\mid o_t,a_t)=0$, then
\[
\bP(o_{t+1}=o'\mid\mathcal G_t^-)=0,
\]
and both transition probabilities are zero by the zero-support branch in
Definition~\ref{def:quotient_lifted_state_main}. Thus, for every
$\hat z'\in\widehat{\cS}$,
\[
\bP\!\left(
\widehat z_{t+1}=\hat z'
\mid
\mathcal G_t^-
\right)
=
\widehat P^\lambda\!\left(
\hat z'
\mid
\widehat z_t,a_t
\right)
\]
almost surely on $A_t$.

We next verify the reward law. Assumption~\ref{ass:hcdp_noise_main} gives
\[
Y_t=R(o_t,i_t,a_t)+\xi_t,
\qquad
\bE[\xi_t\mid\mathcal G_t^-]=0,
\]
and, for every $\eta\in\mathbb R$,
\[
\bE\!\left[
\exp(\eta\xi_t)
\mid
\mathcal G_t^-
\right]
\leq
\exp\!\left(
\frac{\eta^2\sigma_Y^2}{2}
\right).
\]
On $A_t$, reward consistency of the stable quotient implies
\[
\begin{aligned}
R(o_t,i_t,a_t)
&=
\bar R(o_t,c_t,a_t)\\
&=
\bar R\bigl(
o_t,\lambda_{o_t}(\widehat c_t),a_t
\bigr)\\
&=
\widehat R^\lambda(\widehat z_t,a_t).
\end{aligned}
\]
Therefore
\[
\bE[Y_t\mid\mathcal G_t^-]
=
\widehat R^\lambda(\widehat z_t,a_t)
\]
almost surely on $A_t$. Moreover, on $A_t$,
\[
Y_t-\widehat R^\lambda(\widehat z_t,a_t)=\xi_t,
\]
so, for every $\eta\in\mathbb R$,
\[
\bE\!\left[
\exp\!\left(
\eta\bigl[
Y_t-\widehat R^\lambda(\widehat z_t,a_t)
\bigr]
\right)
\middle|
\mathcal G_t^-
\right]
\leq
\exp\!\left(
\frac{\eta^2\sigma_Y^2}{2}
\right)
\]
almost surely on $A_t$.

It remains to prove the model isomorphism. Define
\[
\Lambda_\lambda(o,\hat c)
:=
(o,\lambda_o(\hat c)).
\]
Because each $\lambda_o$ is bijective, $\Lambda_\lambda$ is a bijection from
$\widehat{\cS}$ to $\bar{\cS}$. For every
$\hat z=(o,\hat c)$ and action $a$,
\[
\widehat R^\lambda(\hat z,a)
=
\bar R\bigl(\Lambda_\lambda(\hat z),a\bigr).
\]
For $\hat z'=(o',\hat c')$, if $P_O(o'\mid o,a)=0$, both kernels are zero. If
$P_O(o'\mid o,a)>0$, then
\[
\begin{aligned}
\widehat P^\lambda(\hat z'\mid\hat z,a)
&=
P_O(o'\mid o,a)
\mathbf 1\!\left\{
\hat c'=\tau_{o,a,o'}^\lambda(\hat c)
\right\}\\
&=
P_O(o'\mid o,a)
\mathbf 1\!\left\{
\lambda_{o'}(\hat c')
=
\tau_{o,a,o'}\bigl(\lambda_o(\hat c)\bigr)
\right\}\\
&=
\bar P\!\left(
\Lambda_\lambda(\hat z')
\mid
\Lambda_\lambda(\hat z),a
\right).
\end{aligned}
\]
The initial laws satisfy
\[
\widehat\rho_0^\lambda(\hat z)
=
\bar\rho_0\bigl(\Lambda_\lambda(\hat z)\bigr).
\]
Thus $\Lambda_\lambda$ preserves the action set, rewards, transition kernel,
discount factor, and canonical initial law, and is therefore an MDP
isomorphism.

Now suppose that an almost surely finite $T_0$ satisfies $A_t$ for every
$t\geq T_0$. For each deterministic $t$, the transition, reward, and
sub-Gaussian identities above are equalities of conditional random variables
on the $\mathcal G_t^-$-measurable event $A_t$. Since
\[
\{T_0\leq t\}\subseteq A_t,
\]
the same identities hold on $\{T_0\leq t\}$. This establishes the post-$T_0$
conclusions relative to the ambient filtration; no stopping-time property of
$T_0$ is needed for this implication because the conditional identities were
first proved on $A_t$ itself.

Assume additionally that $T_0$ is a stopping time for
$(\widehat{\mathcal F}_t^-)$. Let
\[
B_t:=\{T_0\leq t\}.
\]
Then
\[
B_t\in\widehat{\mathcal F}_t^-
\subseteq
\mathcal G_t^-.
\]
Since the right-hand sides of the transition and reward identities are
$\widehat{\mathcal F}_t^-$-measurable, the tower property gives
\[
\begin{aligned}
&\mathbf 1_{B_t}
\bP\!\left(
\widehat z_{t+1}=\hat z'
\mid
\widehat{\mathcal F}_t^-
\right)\\
&\quad=
\bE\!\left[
\mathbf 1_{B_t}
\bP\!\left(
\widehat z_{t+1}=\hat z'
\mid
\mathcal G_t^-
\right)
\middle|
\widehat{\mathcal F}_t^-
\right]\\
&\quad=
\mathbf 1_{B_t}
\widehat P^\lambda\!\left(
\hat z'
\mid
\widehat z_t,a_t
\right),
\end{aligned}
\]
and similarly
\[
\mathbf 1_{B_t}
\bE[Y_t\mid\widehat{\mathcal F}_t^-]
=
\mathbf 1_{B_t}
\widehat R^\lambda(\widehat z_t,a_t).
\]
Applying the same tower-property argument to the exponential moment yields
\[
\begin{aligned}
&\mathbf 1_{B_t}
\bE\!\left[
\exp\!\left(
\eta\bigl[
Y_t-\widehat R^\lambda(\widehat z_t,a_t)
\bigr]
\right)
\middle|
\widehat{\mathcal F}_t^-
\right]\\
&\qquad\leq
\mathbf 1_{B_t}
\exp\!\left(
\frac{\eta^2\sigma_Y^2}{2}
\right).
\end{aligned}
\]
Hence the same controlled-MDP laws hold after $T_0$ relative to the learned
information filtration.

Finally, $\widehat\rho_0^\lambda$ is the pullback of the canonical quotient
law at the original time zero, whereas the law of $\widehat z_{T_0}$ is
induced by the controlled trajectory and the synchronization time. In
general,
\[
\mathcal L(\widehat z_{T_0})
\neq
\widehat\rho_0^\lambda,
\]
and the theorem does not require these distributions to coincide.
\end{proof}

\subsection{Proof of Corollary~\ref{cor:diag_to_exact_reduction_main}}

\begin{proof}
Under Assumption~\ref{ass:local_proto_main}, for every observation
$o\in\cO$ the diagnostic interface provides a finite local label set
$\widehat{\cC}_o$ together with an unknown bijection
\[
\lambda_o:
\widehat{\cC}_o
\longrightarrow
\cC_o.
\]
These objects are fixed components of the calibrated diagnostic interface: successive diagnostic calls at observation $o$ return labels in the same set $\widehat{\cC}_o$, and correctness is always evaluated through the same map $\lambda_o$. In particular, the local label names may be arbitrary and may be reused at different observations, but their observation-wise gauges do not vary with time. Since $\cO$ is finite and every $\widehat{\cC}_o$ is finite, the local-coordinate state space
\[
\widehat{\cS}
=
\bigsqcup_{o\in\cO}
\bigl(\{o\}\times\widehat{\cC}_o\bigr)
\]
is finite.

The diagnostic assumptions of Section~\ref{sec:ident} include the class-stability, fingerprint-separation, resettable-repeatability, sub-Gaussian-noise, and calibrated-prototype conditions required by Theorem~\ref{thm:local_class_inf_main}, together with the summable diagnostic schedules of Corollary~\ref{cor:eventual_local_class_main}. Therefore, at every required observation, the probability of an incorrect diagnostic classification is summable over successive classification events. Corollary~\ref{cor:eventual_local_class_main} consequently implies that, almost surely, only finitely many diagnostic calls satisfy
\[
\lambda_o(\widehat c)\neq c.
\]
Because the observation space is finite, this conclusion holds system-wide: on an event $\Omega_{\mathrm{diag}}$ with
\[
\bP(\Omega_{\mathrm{diag}})=1,
\]
there are only finitely many incorrectly classified initialization checkpoints and only finitely many incorrectly classified source or target checkpoints used by the edge-transport estimator.

Under Assumptions~\ref{ass:transition_anchor_main} and~\ref{ass:edge_cov_main}, the latter finite-error property is precisely the classification condition required by Theorem~\ref{thm:transport_local_main}. Hence, for every feasible edge
\[
e=(o\xrightarrow{a}o')
\]
and every source label
$\hat c\in\widehat{\cC}_o$, there exists an almost surely finite random time after which
\[
\widehat\tau_e(\hat c)
=
\lambda_{o'}^{-1}
\left(
\tau_e\bigl(\lambda_o(\hat c)\bigr)
\right)
\]
holds permanently. Equivalently, after that time the recovered row agrees with the true quotient transport expressed in the fixed local gauges. Since there are only finitely many feasible edges and finitely many local source labels, these row-wise recovery events may be intersected and their recovery times may be maximized. Thus there is an event $\Omega_{\mathrm{tr}}$ of probability one and an almost surely finite random time $T_{\mathrm{tr}}$ such that, on $\Omega_{\mathrm{tr}}$, every transport required by the tracker satisfies
\[
\widehat\tau_e
=
\lambda_{\tgt(e)}^{-1}
\circ
\tau_e
\circ
\lambda_{\src(e)}
\]
for all times at least $T_{\mathrm{tr}}$, and every such equality remains valid thereafter.

Every exogenous initialization of HMRL-D uses the diagnostic schedule above, and between exogenous resets its label is propagated according to Definition~\ref{def:local_class_tracker_main}. The hypotheses of Proposition~\ref{prop:eventual_sync_state_main} are therefore satisfied. That proposition provides an event $\Omega_{\mathrm{sync}}$ with
\[
\bP(\Omega_{\mathrm{sync}})=1
\]
and an almost surely finite random time $T$ such that, on $\Omega_{\mathrm{sync}}$,
\[
\lambda_{o_t}(\widehat c_t)
=
c_t
\qquad
\text{for every }t\geq T.
\]
By the definition of the canonical quotient state,
\[
c_t=[i_t]_{o_t}.
\]
Consequently, on the probability-one event
\[
\Omega_\star
:=
\Omega_{\mathrm{diag}}
\cap
\Omega_{\mathrm{tr}}
\cap
\Omega_{\mathrm{sync}},
\]
we have
\[
\lambda_{o_t}(\widehat c_t)
=
c_t
=
[i_t]_{o_t}
\qquad
\text{for every }t\geq T.
\]
The intersection still has probability one because it is a finite intersection of probability-one events:
\[
\begin{aligned}
\bP(\Omega_\star^c)
&=
\bP\!\left(
\Omega_{\mathrm{diag}}^c
\cup
\Omega_{\mathrm{tr}}^c
\cup
\Omega_{\mathrm{sync}}^c
\right)\\
&\leq
\bP(\Omega_{\mathrm{diag}}^c)
+
\bP(\Omega_{\mathrm{tr}}^c)
+
\bP(\Omega_{\mathrm{sync}}^c)\\
&=
0.
\end{aligned}
\]

We now set
\[
T_0:=\max\{T,T_{\mathrm{tr}}\}.
\]
This time is almost surely finite. For every $t\geq T_0$, the class-tracking
identity gives
$\lambda_{o_t}(\widehat c_t)=c_t$, while the definition of
$T_{\mathrm{tr}}$ gives
\[
\widehat\tau_e
=
\lambda_{\tgt(e)}^{-1}\circ\tau_e\circ\lambda_{\src(e)}
\]
for every feasible edge whose table may be used at time $t$. Therefore
$\{T_0\leq t\}\subseteq A_t$ for every $t$. The eventual clause of
Theorem~\ref{thm:eventual_exact_reduction_main} now applies and yields the
local-coordinate quotient transition law, conditional reward mean, and, under
Assumption~\ref{ass:hcdp_noise_main}, the conditional sub-Gaussian noise bound.
The map $\Lambda_\lambda$ is only a proof-level relabeling: HMRL-D supplies
$(o_t,\widehat c_t)$ directly to its policy and update rules and never needs a
cross-observation synchronization of numerical class names. This proves the
corollary.
\end{proof}

\subsection{Proof of Corollary~\ref{cor:pac_calibration_main}}

\begin{proof}
If $\cO_+=\varnothing$, every quotient fiber is a singleton, so no
classification or transport-row ambiguity remains and the conclusion is
immediate. Assume henceforth that $\cO_+\neq\varnothing$ and condition on
$\mathcal E_{\rm proto}$.

For a diagnostic classification performed at observation $o$,
Theorem~\ref{thm:local_class_inf_main} gives
\[
\Pr(\text{classification error})
\leq
2d_o\exp\!\left(
-\frac{m\Delta^2}{128\nu^2d_o}
\right).
\]
Since $d_o\leq d$ and the function
$x\mapsto x\exp(-c/x)$ is increasing on $(0,\infty)$ for every $c>0$, the
chosen value of $m$ implies
\[
\Pr(\text{classification error})
\leq
2d\exp\!\left(
-\frac{m\Delta^2}{128\nu^2d}
\right)
\leq
\frac{\delta_{\rm cls}}{K}.
\]
A union bound over the $K$ classifications therefore shows that, conditional on $\mathcal E_{\rm proto}$, all classifications are correct with probability at least $1-\delta_{\rm cls}$. Hence the joint event consisting of prototype accuracy and correct classification has probability at least $1-\delta_{\rm proto}-\delta_{\rm cls}$.

Condition on this event. Because every calibration table starts at zero and
every source and target label in every anchored pair is correct, all samples in
the row associated with a source label $\hat c$ have the same target label
\[
\lambda_{\tgt(e)}^{-1}
\bigl(\tau_e(\lambda_{\src(e)}(\hat c))\bigr).
\]
Coverage supplies at least one such sample for every feasible edge--class row.
Therefore the row-wise majority rule of
Definition~\ref{def:rowwise_transport_main} recovers
\[
\widehat\tau_e
=
\lambda_{\tgt(e)}^{-1}\circ\tau_e\circ\lambda_{\src(e)}
\]
on every row. Every episode initialization counted among the $K$
classifications is also correct, so
Lemma~\ref{lem:exact_local_tracking_main} gives exact tracking throughout each
post-calibration episode. Hence $A_t$ holds at every post-calibration time.
Theorem~\ref{thm:eventual_exact_reduction_main} and
Corollary~\ref{cor:backbone_transfer_main} then give the desired backbone
guarantee on the all-correct event. Union bounding its original failure event
with the calibration failure event increases the failure probability by at
most $\delta$.
\end{proof}

\subsection{Proof of Corollary~\ref{cor:backbone_transfer_main}}

\begin{proof}

Let $T$ be the learned-information stopping time in the corollary. Since
$A_t$ holds for every $t\geq T$,
Theorem~\ref{thm:eventual_exact_reduction_main} gives, relative to
$(\widehat{\mathcal F}_t^-)$ after $T$, the transition kernel
$\widehat P^\lambda$, conditional reward mean $\widehat R^\lambda$, and,
when required, the conditional sub-Gaussian bound with variance proxy
$\sigma_Y^2$. Thus the post-$T$ stream satisfies the MDP-side probabilistic
conditions assumed by the backbone theorem.

This conclusion does not by itself verify exploration, visitation, sampling,
step-size, optimization, realizability, concentrability, or
function-approximation assumptions; those method-specific conditions remain
in force.

Suppose first that $\mathfrak A$ is initialized or restarted at a time $T\geq T_0$ and that all algorithmic information produced before $T$ is discarded. In particular, any parameter values, value tables, replay data, transition counts, eligibility traces, target-network states, optimizer states, or step-size counters that could have been affected by the pre-$T$ representation are either reinitialized or excluded from subsequent updates. For $t\geq T$, define the post-restart history
\[
\mathcal H_{T,t}^{\widehat z}
:=
\sigma
\bigl(
\widehat z_T,
a_T,
Y_T,
\widehat z_{T+1},
\ldots,
a_{t-1},
Y_{t-1},
\widehat z_t
\bigr).
\]
Because
$\mathcal H_{T,t}^{\widehat z}\subseteq\mathcal F_t^{\widehat z}$,
the tower property and the transition identity above imply
\[
\begin{aligned}
&\bP
\left(
\widehat z_{t+1}=\hat z'
\middle|
\mathcal H_{T,t}^{\widehat z},a_t
\right)\\
&\quad=
\bE
\left[
\bP
\left(
\widehat z_{t+1}=\hat z'
\middle|
\mathcal F_t^{\widehat z},a_t
\right)
\middle|
\mathcal H_{T,t}^{\widehat z},a_t
\right]\\
&\quad=
\bE
\left[
\widehat P^\lambda
\left(
\hat z'
\middle|
\widehat z_t,a_t
\right)
\middle|
\mathcal H_{T,t}^{\widehat z},a_t
\right]\\
&\quad=
\widehat P^\lambda
\left(
\hat z'
\middle|
\widehat z_t,a_t
\right),
\end{aligned}
\]
where the last equality holds because
$\widehat P^\lambda(\hat z'\mid\widehat z_t,a_t)$
is measurable with respect to
$\sigma(\widehat z_t,a_t)\subseteq
\sigma(\mathcal H_{T,t}^{\widehat z},a_t)$.
The same argument for the reward gives
\[
\begin{aligned}
\bE
\left[
Y_t
\middle|
\mathcal H_{T,t}^{\widehat z},a_t
\right]
&=
\bE
\left[
\bE
\left[
Y_t
\middle|
\mathcal F_t^{\widehat z},a_t
\right]
\middle|
\mathcal H_{T,t}^{\widehat z},a_t
\right]\\
&=
\bE
\left[
\widehat R^\lambda
\left(
\widehat z_t,a_t
\right)
\middle|
\mathcal H_{T,t}^{\widehat z},a_t
\right]\\
&=
\widehat R^\lambda
\left(
\widehat z_t,a_t
\right).
\end{aligned}
\]
If the conditional sub-Gaussian condition is required, another application of the same tower argument yields
\[
\begin{aligned}
&\bE
\left[
\exp
\left(
\eta
\left[
Y_t-
\widehat R^\lambda(\widehat z_t,a_t)
\right]
\right)
\middle|
\mathcal H_{T,t}^{\widehat z},a_t
\right]\\
&\quad=
\bE
\left[
\bE
\left[
\exp
\left(
\eta
\left[
Y_t-
\widehat R^\lambda(\widehat z_t,a_t)
\right]
\right)
\middle|
\mathcal F_t^{\widehat z},a_t
\right]
\middle|
\mathcal H_{T,t}^{\widehat z},a_t
\right]\\
&\quad\leq
\exp
\left(
\frac{\eta^2\sigma_Y^2}{2}
\right).
\end{aligned}
\]
Consequently, from the viewpoint of the restarted backbone, the complete data stream beginning at time $T$ is a valid controlled trajectory of
$\widehat{\mathsf M}^{\lambda}$.
Its initial distribution is the actual law of $\widehat z_T$, which need not be the canonical law
$\widehat\rho_0^\lambda$.
Therefore, if the original theorem for $\mathfrak A$ is uniform over initial distributions, it applies immediately; if it assumes a particular class of initial distributions, the actual post-restart law must belong to that class as part of the retained method-specific assumptions. Since no pre-$T$ data or update remains in the restarted algorithmic state, the hypotheses of the finite-MDP guarantee apply exactly to this post-$T$ execution, and the conclusion of that guarantee follows for
$\widehat{\mathsf M}^{\lambda}$.

Suppose next that the backbone is not restarted, but its original theorem is invariant to an arbitrary finite prefix of corrupted states, transitions, or updates. On the probability-one event on which $T_0<\infty$, the tuples and updates with indices
\[
0,1,\ldots,T_0-1
\]
form a finite prefix, while every tuple from index $T_0$ onward satisfies the transition and reward conditions of
$\widehat{\mathsf M}^{\lambda}$.
The prefix may have changed the internal state of the algorithm in an arbitrary way; this is precisely why invariance must cover corrupted updates and not merely the deletion of finitely many observations from a statistical record. By the assumed finite-prefix invariance, replacing the first $T_0$ states, transitions, rewards as assigned to learned states, or updates by arbitrary quantities does not alter the conclusion of the original theorem once the remaining stream is valid and all of its other assumptions hold. Hence the same guarantee again follows for
$\widehat{\mathsf M}^{\lambda}$.

The randomness of $T_0$ creates no additional difficulty because time is discrete and
\[
\{T_0<\infty\}
=
\bigcup_{m=0}^{\infty}
\{T_0=m\}
\]
has probability one. On the event $\{T_0=m\}$, only the first $m$ indices can be corrupted, and the assumed invariance applies to that finite value of $m$. For example, if the original theorem has a failure probability at most $\delta$ uniformly over the length and contents of the corrupted prefix, then
\[
\begin{aligned}
\bP(\text{failure})
&=
\sum_{m=0}^{\infty}
\bP
\left(
\text{failure}
\middle|
T_0=m
\right)
\bP(T_0=m)\\
&\leq
\sum_{m=0}^{\infty}
\delta\,\bP(T_0=m)\\
&=
\delta.
\end{aligned}
\]
An almost-sure conclusion is handled similarly by applying finite-prefix invariance on each event $\{T_0=m\}$ and then taking their countable union.

It remains to transfer the resulting guarantee from local coordinates to the canonical quotient coordinates. Define
\[
\Lambda_\lambda(o,\hat c)
=
\bigl(o,\lambda_o(\hat c)\bigr).
\]
By Theorem~\ref{thm:eventual_exact_reduction_main}, this map is an MDP isomorphism from
$\widehat{\mathsf M}^{\lambda}$
to
$\overline{\mathsf M}$: it is bijective and satisfies
\[
\widehat R^\lambda(\hat z,a)
=
\bar R
\left(
\Lambda_\lambda(\hat z),a
\right)
\]
and
\[
\widehat P^\lambda
\left(
\hat z'
\middle|
\hat z,a
\right)
=
\bar P
\left(
\Lambda_\lambda(\hat z')
\middle|
\Lambda_\lambda(\hat z),a
\right).
\]
If the backbone produces a stationary policy $\widehat\pi$ on
$\widehat{\cS}$, define its gauge-transformed policy on
$\bar{\cS}$ by
\[
\bar\pi(a\mid z)
:=
\widehat\pi
\left(
a
\middle|
\Lambda_\lambda^{-1}(z)
\right).
\]
Lemma~\ref{lem:gauge_policy_transfer_main} gives, for every
$\hat z\in\widehat{\cS}$,
\[
\widehat V_{\widehat\pi}(\hat z)
=
\bar V_{\bar\pi}
\left(
\Lambda_\lambda(\hat z)
\right)
\]
and
\[
\widehat V^\star(\hat z)
=
\bar V^\star
\left(
\Lambda_\lambda(\hat z)
\right).
\]
In particular,
\[
\begin{aligned}
\left\|
\widehat V_{\widehat\pi}
-
\widehat V^\star
\right\|_\infty
&=
\max_{\hat z\in\widehat{\cS}}
\left|
\bar V_{\bar\pi}
\left(
\Lambda_\lambda(\hat z)
\right)
-
\bar V^\star
\left(
\Lambda_\lambda(\hat z)
\right)
\right|\\
&=
\max_{z\in\bar{\cS}}
\left|
\bar V_{\bar\pi}(z)
-
\bar V^\star(z)
\right|\\
&=
\left\|
\bar V_{\bar\pi}
-
\bar V^\star
\right\|_\infty,
\end{aligned}
\]
where the middle equality uses bijectivity of $\Lambda_\lambda$. More generally, the isomorphism preserves actions, reward observations, transition probabilities, occupancy measures after the corresponding pushforward of the state distribution, and the ordering of policy values. Therefore any guarantee stated in terms of optimality, value error, return, regret, visitation, or sample complexity transfers after the states and policies are relabeled through $\Lambda_\lambda$, provided that every algorithm-specific assumption appearing in the original theorem is also preserved or separately verified in the local coordinates. This proves the positive transfer claims in both cases~(i) and~(ii).

Finally, eventual correctness alone cannot remove an error already stored permanently in the internal state of a backbone. To see this explicitly, consider the finite discounted MDP with
\[
\cS_0
=
\{s_0,s_1\},
\qquad
\cA_0
=
\{a^\star,a^-\},
\]
deterministic self-loop transitions
\[
P_0(s_j\mid s_j,a)
=
1
\qquad
\text{for every }j\in\{0,1\}
\text{ and }a\in\cA_0,
\]
and deterministic rewards
\[
R_0(s_j,a^\star)
=
1,
\qquad
R_0(s_j,a^-)
=
0.
\]
The action $a^\star$ is uniquely optimal at both states, and its value is
\[
V^\star(s_j)
=
\sum_{k=0}^{\infty}\gamma^k
=
\frac{1}{1-\gamma}.
\]
Define a backbone that selects $a^\star$ for its first action and then stores one permanent bit. It sets this bit to ``good'' if its first recorded update has the valid form
\[
(s_j,a^\star,1,s_j)
\]
for some $j\in\{0,1\}$, and otherwise sets the bit to ``bad.'' It selects $a^\star$ forever in the good mode and $a^-$ forever in the bad mode. On every uncorrupted trajectory of this MDP, the first recorded update necessarily has the valid form above. The backbone therefore remains in the good mode, selects the unique optimal action forever, and has a correct finite-MDP optimality guarantee on the uncorrupted model.

Now let the true quotient state be $s_0$ at every time, but suppose that the learned representation is incorrect only at time zero, with
\[
\widehat z_0=s_1,
\qquad
\widehat z_t=s_0
\quad
\text{for every }t\geq1.
\]
Thus one may take $T_0=1$, and the representation is exact permanently after that time. The true first action, reward, and successor state are
$a^\star$, $1$, and $s_0$, respectively, but the first update recorded by the backbone is
\[
(s_1,a^\star,1,s_0),
\]
which is not a valid self-loop tuple. The backbone enters the bad mode and thereafter selects $a^-$ forever, even though every state and transition presented from time one onward is correct. Its post-transient value at the true state is then
\[
V(s_0)
=
\sum_{k=0}^{\infty}
\gamma^k\,0
=
0
<
\frac{1}{1-\gamma}
=
V^\star(s_0).
\]
A restart after time one would delete the poisoned bit, and a theorem invariant to finite corrupted updates would exclude this failure, but eventual representation correctness by itself does neither. Hence, without condition~(i), condition~(ii), or some other assumption that removes or neutralizes all effects of the finite corrupted prefix, the finite-MDP guarantee need not transfer.
\end{proof}

\subsection{Proof of Theorem~\ref{thm:passive_conditional_reduction_main}}

\begin{proof}
Fix a deterministic time $t\geq0$ and let
\[
B_t:=\{T_0\leq t\}.
\]
Because $T_0$ is a stopping time for $(\mathcal G_t^-)$,
\[
B_t\in\mathcal G_t^-.
\]
On $B_t$, eventual class correctness holds at both $t$ and $t+1$:
\[
\widehat c_t
=
\lambda_{o_t}^{-1}(c_t),
\qquad
\widehat c_{t+1}
=
\lambda_{o_{t+1}}^{-1}(c_{t+1}).
\]
Fix a candidate successor
\[
\hat z'=(o',\hat c')\in\widehat{\cS}.
\]
If $P_O(o'\mid o_t,a_t)>0$, then on
$B_t\cap\{o_{t+1}=o'\}$, the quotient transition identity gives
\[
\begin{aligned}
\widehat c_{t+1}
&=
\lambda_{o'}^{-1}(c_{t+1})\\
&=
\lambda_{o'}^{-1}\!\left(
\tau_{o_t,a_t,o'}(c_t)
\right)\\
&=
\lambda_{o'}^{-1}\!\left(
\tau_{o_t,a_t,o'}
\bigl(\lambda_{o_t}(\widehat c_t)\bigr)
\right)\\
&=
\tau_{o_t,a_t,o'}^\lambda(\widehat c_t).
\end{aligned}
\]
Therefore
\[
\begin{aligned}
&\mathbf 1_{B_t}
\mathbf 1\!\left\{
\widehat z_{t+1}=(o',\hat c')
\right\}\\
&\qquad=
\mathbf 1_{B_t}
\mathbf 1\{o_{t+1}=o'\}
\mathbf 1\!\left\{
\hat c'
=
\tau_{o_t,a_t,o'}^\lambda(\widehat c_t)
\right\}.
\end{aligned}
\]
Taking conditional expectation with respect to $\mathcal G_t^-$ and using
$B_t\in\mathcal G_t^-$ and
\[
\bP(o_{t+1}=o'\mid\mathcal G_t^-)
=
P_O(o'\mid o_t,a_t)
\]
yields
\[
\begin{aligned}
&\mathbf 1_{B_t}
\bP\!\left(
\widehat z_{t+1}=(o',\hat c')
\mid
\mathcal G_t^-
\right)\\
&\qquad=
\mathbf 1_{B_t}
\widehat P^\lambda\!\left(
(o',\hat c')
\mid
\widehat z_t,a_t
\right).
\end{aligned}
\]
If $P_O(o'\mid o_t,a_t)=0$, both sides are zero. Thus the transition identity
holds on $B_t$ for every candidate successor.

Under Assumption~\ref{ass:hcdp_noise_main},
\[
Y_t=R(o_t,i_t,a_t)+\xi_t,
\qquad
\bE[\xi_t\mid\mathcal G_t^-]=0.
\]
On $B_t$, reward consistency and
$c_t=\lambda_{o_t}(\widehat c_t)$ imply
\[
R(o_t,i_t,a_t)
=
\bar R(o_t,c_t,a_t)
=
\widehat R^\lambda(\widehat z_t,a_t).
\]
Since $B_t\in\mathcal G_t^-$,
\[
\mathbf 1_{B_t}
\bE[Y_t\mid\mathcal G_t^-]
=
\mathbf 1_{B_t}
\widehat R^\lambda(\widehat z_t,a_t).
\]
Moreover, on $B_t$,
\[
Y_t-\widehat R^\lambda(\widehat z_t,a_t)=\xi_t,
\]
so the conditional sub-Gaussian assumption gives, for every
$\eta\in\mathbb R$,
\[
\begin{aligned}
&\mathbf 1_{B_t}
\bE\!\left[
\exp\!\left(
\eta\bigl[
Y_t-\widehat R^\lambda(\widehat z_t,a_t)
\bigr]
\right)
\middle|
\mathcal G_t^-
\right]\\
&\qquad\leq
\mathbf 1_{B_t}
\exp\!\left(
\frac{\eta^2\sigma_Y^2}{2}
\right).
\end{aligned}
\]

The map
\[
\Lambda_\lambda(o,\hat c)
=
(o,\lambda_o(\hat c))
\]
is an MDP isomorphism by the same algebraic argument as in
Theorem~\ref{thm:eventual_exact_reduction_main}: it is bijective and preserves
rewards and transition probabilities under the exact local transports
$\tau_e^\lambda$.

If $T_0$ is also a stopping time for
$(\widehat{\mathcal F}_t^-)$, then
\[
B_t\in\widehat{\mathcal F}_t^-.
\]
Applying the tower property to the three ambient-filtration identities above
yields the corresponding transition, conditional-reward, and conditional
sub-Gaussian identities relative to $\widehat{\mathcal F}_t^-$ on $B_t$.

This argument does not require the passive learner to recover or explicitly
store the edge maps; it uses only eventual exact class labels together with
the stopping-time condition. Conversely, without eventual exact labels, the
learned representation may merge histories from different stable classes.
Without the stopping-time condition, pathwise eventual agreement alone does
not justify conditioning on the random post-synchronization regime. Therefore
the structural quotient results alone do not give an unconditional Markov or
value-preservation guarantee for passive learning.
\end{proof}

\subsection{Proof of Theorem~\ref{thm:abelian_barrier_main}}

\begin{proof}
Let $H_w$ and $H_{w'}$ denote reachable observable histories that start from the fixed reachable quotient class $c_0$ at the base observation $o$ and execute the loop words $w$ and $w'$, respectively. Every generator and inverse generator is represented by an executable directed closed walk based at $o$, so both histories terminate again at the same visible observation $o$. By the definition of the free-group transport representation and of $c_w$ and $c_{w'}$, their terminal canonical quotient states are
$
(o,c_w)
=
\bigl(o,\rho_o(w)(c_0)\bigr)
\qquad\text{and}\qquad
(o,c_{w'})
=
\bigl(o,\rho_o(w')(c_0)\bigr).
$
Since the loop-memory surrogate factors through abelianization, there exists a map
$f:\mathbb Z^r\to\mathcal M$
such that
$S=f\circ\ab$.
The assumed equality
$\ab(w)=\ab(w')$
therefore gives
\[
\begin{aligned}
S(w)
&=
(f\circ\ab)(w)\\
&=
f\bigl(\ab(w)\bigr)\\
&=
f\bigl(\ab(w')\bigr)\\
&=
(f\circ\ab)(w')\\
&=
S(w').
\end{aligned}
\]
Both histories end at the same observation $o$, and hence their terminal memories satisfy
\[
\begin{aligned}
m(w)
&=
M\bigl(o,S(w)\bigr)\\
&=
M\bigl(o,S(w')\bigr)\\
&=
m(w').
\end{aligned}
\]
Write this common memory value as $m_\star$. Under the finite-memory controller convention of Section~\ref{sec:prelim}, the action readout depends on a history only through its terminal observation and terminal memory. Consequently, after either history the controller uses the same action distribution
\[
p(a)
:=
\pi_{\cM}(a\mid o,m_\star),
\qquad
a\in\cA.
\]
Because
$\pi_{\cM}(\cdot\mid o,m_\star)$
is a probability distribution, we have
\[
p\in\Delta(\cA).
\]

For an arbitrary distribution
$q\in\Delta(\cA)$,
define the one-decision optimality gaps at the two terminal quotient classes by
\[
L_w(q)
:=
\bar V^\star(o,c_w)
-
\sum_{a\in\cA}
q(a)\bar Q^\star(o,c_w,a)
\]
and
\[
L_{w'}(q)
:=
\bar V^\star(o,c_{w'})
-
\sum_{a\in\cA}
q(a)\bar Q^\star(o,c_{w'},a).
\]
For every quotient class $c\in\cC_o$, the optimal value and optimal state--action value satisfy
\[
\bar V^\star(o,c)
=
\max_{a\in\cA}
\bar Q^\star(o,c,a).
\]
It follows that, for every
$q\in\Delta(\cA)$,
\[
\begin{aligned}
\sum_{a\in\cA}
q(a)\bar Q^\star(o,c,a)
&\leq
\sum_{a\in\cA}
q(a)\bar V^\star(o,c)\\
&=
\bar V^\star(o,c)
\sum_{a\in\cA}q(a)\\
&=
\bar V^\star(o,c).
\end{aligned}
\]
Therefore
\[
L_w(q)\geq0
\qquad\text{and}\qquad
L_{w'}(q)\geq0
\]
for every
$q\in\Delta(\cA)$.

We now apply the definition of the common-decision loss to the particular distribution $p$ selected by the controller. Since the minimum over all action distributions cannot exceed the value of the same objective at the particular distribution $p$, we have
\[
\begin{aligned}
\max
\bigl\{
L_w(p),L_{w'}(p)
\bigr\}
&\geq
\min_{q\in\Delta(\cA)}
\max
\bigl\{
L_w(q),L_{w'}(q)
\bigr\}\\
&=
\varepsilon_{\mathrm{dec}}
\bigl(o;c_w,c_{w'}\bigr).
\end{aligned}
\]
Hence at least one of the inequalities
\[
L_w(p)
\geq
\varepsilon_{\mathrm{dec}}
\bigl(o;c_w,c_{w'}\bigr)
\]
or
\[
L_{w'}(p)
\geq
\varepsilon_{\mathrm{dec}}
\bigl(o;c_w,c_{w'}\bigr)
\]
must hold.

For completeness, the lower bound is strictly positive because the two terminal classes are decision separated. The simplex
$\Delta(\cA)$
is compact because $\cA$ is finite, and the function
\[
q
\longmapsto
\max
\bigl\{
L_w(q),L_{w'}(q)
\bigr\}
\]
is continuous. Its minimum is therefore attained. Suppose, for contradiction, that
\[
\varepsilon_{\mathrm{dec}}
\bigl(o;c_w,c_{w'}\bigr)
=
0.
\]
Then there exists
$q_\star\in\Delta(\cA)$
such that
\[
\max
\bigl\{
L_w(q_\star),L_{w'}(q_\star)
\bigr\}
=
0.
\]
Both losses are nonnegative, so the preceding equality implies
\[
L_w(q_\star)=0
\qquad\text{and}\qquad
L_{w'}(q_\star)=0.
\]
For the first equality, we may write
\[
\begin{aligned}
0
&=
L_w(q_\star)\\
&=
\bar V^\star(o,c_w)
-
\sum_{a\in\cA}
q_\star(a)\bar Q^\star(o,c_w,a)\\
&=
\sum_{a\in\cA}
q_\star(a)
\left[
\bar V^\star(o,c_w)
-
\bar Q^\star(o,c_w,a)
\right].
\end{aligned}
\]
For every action $a$,
\[
\bar V^\star(o,c_w)
-
\bar Q^\star(o,c_w,a)
\geq0.
\]
Thus the last expression is a sum of nonnegative terms. It can equal zero only if every term having positive weight is zero. Consequently,
\[
q_\star(a)>0
\quad\Longrightarrow\quad
\bar Q^\star(o,c_w,a)
=
\bar V^\star(o,c_w),
\]
and hence
\[
q_\star(a)>0
\quad\Longrightarrow\quad
a\in\cA^\star(o,c_w).
\]
Equivalently,
\[
\operatorname{supp}(q_\star)
\subseteq
\cA^\star(o,c_w).
\]
Applying the same reasoning to
$L_{w'}(q_\star)=0$
gives
\[
\operatorname{supp}(q_\star)
\subseteq
\cA^\star(o,c_{w'}).
\]
Therefore
\[
\operatorname{supp}(q_\star)
\subseteq
\cA^\star(o,c_w)
\cap
\cA^\star(o,c_{w'}).
\]
The classes $c_w$ and $c_{w'}$ are decision separated, so
\[
\cA^\star(o,c_w)
\cap
\cA^\star(o,c_{w'})
=
\varnothing.
\]
It would follow that
\[
\operatorname{supp}(q_\star)
=
\varnothing,
\]
which is impossible because
$q_\star$
is a probability distribution and therefore satisfies
\[
\sum_{a\in\cA}q_\star(a)=1.
\]
This contradiction proves that
\[
\varepsilon_{\mathrm{dec}}
\bigl(o;c_w,c_{w'}\bigr)
>
0.
\]

We next relate the preceding one-decision gaps to the controller's actual conditional continuation values. Let
\[
J_{\cM}(H_w)
\]
denote the controller's expected discounted return conditional on reaching the terminal decision point of $H_w$, with the discount restarted at that decision point. More explicitly, if $t_w$ is the terminal time of $H_w$, then
\[
J_{\cM}(H_w)
:=
\bE
\left[
\sum_{k=0}^{\infty}
\gamma^k
R_{t_w+k}
\middle|
H_w
\right],
\]
where the future actions are selected by the controller under consideration. Conditional on selecting action $a$ at the terminal state $(o,c_w)$, the controller's subsequent behavior is one admissible continuation strategy from that quotient state. By the definition of the optimal quotient state--action value,
\[
\bar Q^\star(o,c_w,a)
\]
is the supremal expected discounted return among all admissible continuations that first take action $a$. Therefore the controller's conditional return after first selecting $a$ cannot exceed this quantity. Averaging over its first-action distribution $p$ gives
\[
J_{\cM}(H_w)
\leq
\sum_{a\in\cA}
p(a)\bar Q^\star(o,c_w,a).
\]
Subtracting both sides from
$\bar V^\star(o,c_w)$
yields
\[
\begin{aligned}
\bar V^\star(o,c_w)
-
J_{\cM}(H_w)
&\geq
\bar V^\star(o,c_w)
-
\sum_{a\in\cA}
p(a)\bar Q^\star(o,c_w,a)\\
&=
L_w(p).
\end{aligned}
\]
The identical argument at the history $H_{w'}$ yields
\[
\bar V^\star(o,c_{w'})
-
J_{\cM}(H_{w'})
\geq
L_{w'}(p).
\]
Combining these two inequalities gives
\[
\begin{aligned}
&\max
\Bigl\{
\bar V^\star(o,c_w)-J_{\cM}(H_w),
\bar V^\star(o,c_{w'})-J_{\cM}(H_{w'})
\Bigr\}\\
&\quad\geq
\max
\bigl\{
L_w(p),L_{w'}(p)
\bigr\}\\
&\quad\geq
\varepsilon_{\mathrm{dec}}
\bigl(o;c_w,c_{w'}\bigr).
\end{aligned}
\]
It follows that at least one of the two histories satisfies
\[
\bar V^\star(o,c_w)-J_{\cM}(H_w)
\geq
\varepsilon_{\mathrm{dec}}
\bigl(o;c_w,c_{w'}\bigr)
\]
or
\[
\bar V^\star(o,c_{w'})
-
J_{\cM}(H_{w'})
\geq
\varepsilon_{\mathrm{dec}}
\bigl(o;c_w,c_{w'}\bigr),
\]
respectively. Thus at least one history incurs conditional optimality loss at least
\[
\varepsilon_{\mathrm{dec}}
\bigl(o;c_w,c_{w'}\bigr).
\]

Finally, suppose that a terminal memory factoring through abelianization were exact for quotient-optimal continuation control on a reachable history set containing both $H_w$ and $H_{w'}$. Exactness would require the controller to attain the quotient-optimal continuation value after each of these histories, namely
\[
J_{\cM}(H_w)
=
\bar V^\star(o,c_w)
\]
and
\[
J_{\cM}(H_{w'})
=
\bar V^\star(o,c_{w'}).
\]
Both conditional optimality losses would then be zero, so
\[
\max
\Bigl\{
\bar V^\star(o,c_w)-J_{\cM}(H_w),
\bar V^\star(o,c_{w'})-J_{\cM}(H_{w'})
\Bigr\}
=
0.
\]
This contradicts
$
\max
\Bigl\{
\bar V^\star(o,c_w)-J_{\cM}(H_w),
\bar V^\star(o,c_{w'})-J_{\cM}(H_{w'})
\Bigr\}
\geq
\varepsilon_{\mathrm{dec}}
\bigl(o;c_w,c_{w'}\bigr)
>
0.
$
Therefore no terminal memory of the form
$m(w)=M(o,S(w))$
with
$S=f\circ\ab$
can be exact for quotient-optimal continuation control on any reachable history set containing both histories.
\end{proof}

\subsection{Proof of Corollary~\ref{cor:commutator_barrier_main}}

\begin{proof}
Let
\[
w=[u,v]=uvu^{-1}v^{-1}
\qquad\text{and}\qquad
w'=e.
\]
Because the abelianization map
\(\ab:F_r\to\mathbb Z^r\)
is a group homomorphism from the multiplicative group \(F_r\) to the additive group \(\mathbb Z^r\), it satisfies
\[
\ab(zz')
=
\ab(z)+\ab(z')
\qquad\text{and}\qquad
\ab(z^{-1})
=
-\ab(z)
\]
for every \(z,z'\in F_r\). Applying these identities successively to the commutator gives
\[
\begin{aligned}
\ab([u,v])
&=
\ab\bigl(uvu^{-1}v^{-1}\bigr)\\
&=
\ab(u)
+
\ab(v)
+
\ab(u^{-1})
+
\ab(v^{-1})\\
&=
\ab(u)
+
\ab(v)
-
\ab(u)
-
\ab(v)\\
&=
0.
\end{aligned}
\]
Moreover, every group homomorphism maps the identity element of its domain to the identity element of its codomain. Since the identity element of the additive group \(\mathbb Z^r\) is \(0\), we also have
\[
\ab(e)=0.
\]
Consequently,
\[
\ab([u,v])=\ab(e).
\]

Now let
\(S:F_r\to\mathcal M\)
be an arbitrary abelianized loop-memory surrogate. By Definition~\ref{def:abelianized_memory_main}, there exists a map
\(f:\mathbb Z^r\to\mathcal M\)
such that
\[
S=f\circ\ab.
\]
It follows that
\[
\begin{aligned}
S([u,v])
&=
(f\circ\ab)([u,v])\\
&=
f\bigl(\ab([u,v])\bigr)\\
&=
f(0)\\
&=
f\bigl(\ab(e)\bigr)\\
&=
(f\circ\ab)(e)\\
&=
S(e).
\end{aligned}
\]
Because \(c_0\) is reachable and every generator and inverse generator in the executable reversible loop alphabet is realized by an executable directed closed walk based at \(o\), the word \([u,v]\) determines a reachable loop history starting from \((o,c_0)\). The identity word \(e\) determines the corresponding reachable empty-loop history. Both histories terminate at the same visible observation \(o\). Therefore, for any terminal controller memory of the form
\[
m(z)=M\bigl(o,S(z)\bigr),
\]
we obtain
\[
\begin{aligned}
m([u,v])
&=
M\bigl(o,S([u,v])\bigr)\\
&=
M\bigl(o,S(e)\bigr)\\
&=
m(e).
\end{aligned}
\]
Thus the controller receives exactly the same terminal observation and terminal memory after the commutator history and after the empty-loop history, and hence it must use the same action distribution after both histories.

It remains to identify the canonical quotient states reached by these two histories. Since
\[
\rho_o:F_r\to\Sym(\cC_o)
\]
is a group homomorphism, it maps the identity word \(e\) to the identity permutation on \(\cC_o\). Hence
\[
\rho_o(e)=\Id_{\cC_o},
\]
and therefore
\[
\begin{aligned}
\rho_o(e)(c_0)
&=
\Id_{\cC_o}(c_0)\\
&=
c_0.
\end{aligned}
\]
Accordingly, the terminal quotient classes of the two histories are
\[
c_{[u,v]}
=
\rho_o([u,v])(c_0)
\]
and
\[
c_e
=
\rho_o(e)(c_0)
=
c_0.
\]
The hypothesis of the corollary states precisely that
\[
\rho_o([u,v])(c_0)
\quad\text{is decision separated from}\quad
c_0.
\]
Equivalently,
\[
\cA^\star
\bigl(o,\rho_o([u,v])(c_0)\bigr)
\cap
\cA^\star(o,c_0)
=
\varnothing.
\]
We have therefore verified all hypotheses of Theorem~\ref{thm:abelian_barrier_main} with
\[
w=[u,v]
\qquad\text{and}\qquad
w'=e.
\]
That theorem implies that every controller whose terminal memory factors through abelianization must incur conditional optimality loss on at least one of the two histories of at least
\[
\varepsilon_{\mathrm{dec}}
\bigl(
o;
\rho_o([u,v])(c_0),
c_0
\bigr).
\]
Because the two terminal quotient classes are decision separated, Definition~\ref{def:decision_separated_main} gives
\[
\varepsilon_{\mathrm{dec}}
\bigl(
o;
\rho_o([u,v])(c_0),
c_0
\bigr)
>
0.
\]
If the abelianized memory were exact at both resulting quotient states, then the controller would attain the quotient-optimal continuation value after both histories, so its conditional optimality loss would be zero after each history. This would contradict the strictly positive lower bound above. Hence no abelianized loop-memory surrogate can be exact at both
\[
\bigl(o,\rho_o([u,v])(c_0)\bigr)
\qquad\text{and}\qquad
(o,c_0).
\]
\end{proof}

\subsection{Proof of Corollary~\ref{cor:additive_barrier_main}}

\begin{proof}
Let $\mathcal V$ denote the additive space in which the vectors
$v_1,\ldots,v_r$ take values. Since the abelianization map records the signed exponent sums of the free generators, for every word $w\in F_r$ we have
\[
\ab(w)
=
\bigl(N_1(w),\ldots,N_r(w)\bigr)
\in
\mathbb Z^r.
\]
Define a map
\[
\Phi:\mathbb Z^r\longrightarrow\mathcal V
\]
by
\[
\Phi(n_1,\ldots,n_r)
:=
\sum_{j=1}^{r}n_jv_j.
\]
This map is well defined. Moreover, if
$n=(n_1,\ldots,n_r)$
and
$n'=(n'_1,\ldots,n'_r)$
are arbitrary elements of $\mathbb Z^r$, then
\[
\begin{aligned}
\Phi(n+n')
&=
\Phi(n_1+n'_1,\ldots,n_r+n'_r)\\
&=
\sum_{j=1}^{r}(n_j+n'_j)v_j\\
&=
\sum_{j=1}^{r}n_jv_j
+
\sum_{j=1}^{r}n'_jv_j\\
&=
\Phi(n)+\Phi(n').
\end{aligned}
\]
Thus $\Phi$ is a homomorphism from the additive group $\mathbb Z^r$ into the additive group underlying $\mathcal V$. Injectivity of $\Phi$ is neither assumed nor needed. Indeed, if two distinct vectors in $\mathbb Z^r$ are mapped to the same element of $\mathcal V$, then the count-additive statistic discards even more information than abelianization; such additional collisions cannot recover the ordering information that abelianization has already removed.

For every $w\in F_r$, substituting the coordinates of $\ab(w)$ into the definition of $\Phi$ gives
\[
\begin{aligned}
(\Phi\circ\ab)(w)
&=
\Phi\bigl(\ab(w)\bigr)\\
&=
\Phi\bigl(N_1(w),\ldots,N_r(w)\bigr)\\
&=
\sum_{j=1}^{r}N_j(w)v_j\\
&=
A(w).
\end{aligned}
\]
Therefore
\[
A=\Phi\circ\ab.
\]
By assumption, the readout of the encoder depends on the loop word $w$ only through the value of $A(w)$. Consequently, there exists a map
\[
\psi:\mathcal V\longrightarrow\mathcal M
\]
such that the resulting loop-memory surrogate satisfies
\[
S(w)
=
\psi\bigl(A(w)\bigr)
\qquad
\text{for every }w\in F_r.
\]
Combining this representation with the identity
$A=\Phi\circ\ab$
yields
\[
\begin{aligned}
S(w)
&=
\psi\bigl(A(w)\bigr)\\
&=
\psi\bigl((\Phi\circ\ab)(w)\bigr)\\
&=
(\psi\circ\Phi)\bigl(\ab(w)\bigr).
\end{aligned}
\]
Define
\[
f
:=
\psi\circ\Phi
:
\mathbb Z^r
\longrightarrow
\mathcal M.
\]
Then, for every $w\in F_r$,
\[
S(w)
=
f\bigl(\ab(w)\bigr),
\]
and hence
\[
S=f\circ\ab.
\]
It follows directly from Definition~\ref{def:abelianized_memory_main} that $S$ is an abelianized loop-memory surrogate.

The same conclusion can also be verified directly from equality of signed generator counts. Let $w,w'\in F_r$ satisfy
\[
\ab(w)=\ab(w').
\]
Since equality in $\mathbb Z^r$ is coordinatewise and
$
\ab(w)
=
\bigl(N_1(w),\ldots,N_r(w)\bigr),
\qquad
\ab(w')
=
\bigl(N_1(w'),\ldots,N_r(w')\bigr),
$
we obtain
\[
N_j(w)=N_j(w')
\qquad
\text{for every }j\in\{1,\ldots,r\}.
\]
Therefore
\[
\begin{aligned}
A(w)
&=
\sum_{j=1}^{r}N_j(w)v_j\\
&=
\sum_{j=1}^{r}N_j(w')v_j\\
&=
A(w').
\end{aligned}
\]
Because the readout depends only on the value of the count-additive statistic, this equality further implies
\[
\begin{aligned}
S(w)
&=
\psi\bigl(A(w)\bigr)\\
&=
\psi\bigl(A(w')\bigr)\\
&=
S(w').
\end{aligned}
\]
Thus no readout applied only after $A(w)$ can distinguish two loop histories having the same abelianization.

Now fix a reachable class
$c_0\in\cC_o$
and an executable reversible loop alphabet, and suppose that
$w,w'\in F_r$
satisfy
$
\ab(w)=\ab(w'),
\qquad
c_w=\rho_o(w)(c_0),
\qquad
c_{w'}=\rho_o(w')(c_0),
$
where $c_w$ and $c_{w'}$ are decision separated. Since the loop-memory surrogate has the form
\[
S=f\circ\ab,
\]
all of the hypotheses of Theorem~\ref{thm:abelian_barrier_main} concerning the memory representation are satisfied. More explicitly, for any terminal controller memory of the form
\[
m(z)
=
M\bigl(o,S(z)\bigr),
\]
the equality of the two abelianizations gives
\[
\begin{aligned}
m(w)
&=
M\bigl(o,S(w)\bigr)\\
&=
M\bigl(o,f(\ab(w))\bigr)\\
&=
M\bigl(o,f(\ab(w'))\bigr)\\
&=
M\bigl(o,S(w')\bigr)\\
&=
m(w').
\end{aligned}
\]
Both loop histories terminate at the same base observation $o$, and the controller therefore receives the same terminal observation and the same terminal memory after both histories. It must consequently use the same action distribution after the terminal quotient states
\[
(o,c_w)
=
\bigl(o,\rho_o(w)(c_0)\bigr)
\]
and
\[
(o,c_{w'})
=
\bigl(o,\rho_o(w')(c_0)\bigr).
\]
Since these quotient classes are decision separated, Theorem~\ref{thm:abelian_barrier_main} implies that at least one of the two histories incurs conditional optimality loss at least
\[
\varepsilon_{\mathrm{dec}}
\bigl(o;c_w,c_{w'}\bigr).
\]
By Definition~\ref{def:decision_separated_main}, decision separation further implies
\[
\varepsilon_{\mathrm{dec}}
\bigl(o;c_w,c_{w'}\bigr)
>
0.
\]
If the count-additive encoder were exact for quotient-optimal continuation control on a reachable history set containing both histories, then the conditional optimality loss after each history would be zero. This would contradict the strictly positive lower bound above. Hence every encoder of the form
\[
A(w)
=
\sum_{j=1}^{r}N_j(w)v_j
\]
whose readout depends only on $A(w)$ is subject to the barrier of Theorem~\ref{thm:abelian_barrier_main} whenever the theorem's equal-abelianization and decision-separation conditions hold.
\end{proof}

\subsection{Proof of Corollary~\ref{cor:decision_lower_bound_main}}

\begin{proof}
For each $c\in\cC_o$, choose a positive-probability history $H_c$ that ends at
observation $o$ with quotient class $c$, and let $F(H_c)\in\cM$ be the memory
state induced by the controller. Suppose, toward a contradiction, that two
distinct classes $c,c'$ satisfy
\[
F(H_c)=F(H_{c'}).
\]
Because both histories end at the same observation, the controller uses the
same action distribution
\[
p(a):=\pi_{\cM}(a\mid o,F(H_c))
=
\pi_{\cM}(a\mid o,F(H_{c'}))
\]
after both histories.

Let $V_{\mathrm{ctrl}}(H_c)$ denote the controller's conditional continuation
value after $H_c$. Conditional on its first action, the controller's subsequent
behavior is only one admissible continuation policy, whereas
$\bar Q^\star(o,c,a)$ uses an optimal continuation. Hence
\[
V_{\mathrm{ctrl}}(H_c)
\leq
\sum_{a\in\cA}p(a)\bar Q^\star(o,c,a),
\]
and analogously
\[
V_{\mathrm{ctrl}}(H_{c'})
\leq
\sum_{a\in\cA}p(a)\bar Q^\star(o,c',a).
\]
By Definition~\ref{def:decision_separated_main}, pairwise decision separation
implies
$
\max\!\{
\bar V^\star(o,c)-\sum_a p(a)\bar Q^\star(o,c,a),
\bar V^\star(o,c')-\sum_a p(a)\bar Q^\star(o,c',a)
\}
\geq
\varepsilon_{\mathrm{dec}}(o;c,c')
>
0.
$
Therefore the controller's continuation value is strictly below the optimum
after at least one of $H_c$ and $H_{c'}$, contradicting the assumed exactness
after every positive-probability history. Thus the memory states
$\{F(H_c):c\in\cC_o\}$ are pairwise distinct, and
$|\cM|\geq|\cC_o|$. If the hypotheses hold at an observation $o^\star$ with
$|\cC_{o^\star}|=m_{\mathrm{exact}}^\star$, the same inequality gives
$|\cM|\geq m_{\mathrm{exact}}^\star$.
\end{proof}

\subsection{Proof of Proposition~\ref{prop:min_nonabelian_main}}

\begin{proof}
Because there is only one observation and
\(P_O(o\mid o,a)=P_O(o\mid o,b)=1\), every finite action sequence is an executable directed closed walk based at \(o\). Let
\[
\alpha:=\sigma_a=(12)
\qquad\text{and}\qquad
\beta:=\sigma_b=(123).
\]
The reward vector of a layer is the vector of its immediate rewards under the two available actions. For the three layers these vectors are
\[
\bigl(R(o,1,a),R(o,1,b)\bigr)=(1,0),
\]
\[
\bigl(R(o,2,a),R(o,2,b)\bigr)=(0,1),
\]
and
\[
\bigl(R(o,3,a),R(o,3,b)\bigr)=(0,0).
\]
They are pairwise distinct. Hence two layers \(i,j\in\{1,2,3\}\) have equal rewards under every action if and only if \(i=j\). It follows that the reward partition appearing in Lemma~\ref{lem:stable_fp_main} is
\[
\Pi_o^{(0)}
=
\bigl\{\{1\},\{2\},\{3\}\bigr\}.
\]
We next verify that this partition is already stable. If
\(i\equiv_o^{\mathcal T(\Pi^{(0)})}j\), then the first condition in Definition~\ref{def:stable_family_main} requires
\[
R(o,i,a)=R(o,j,a)
\qquad\text{and}\qquad
R(o,i,b)=R(o,j,b).
\]
The pairwise distinctness of the three reward vectors then implies \(i=j\). Conversely, every layer is equivalent to itself under
\(\mathcal T(\Pi^{(0)})\), because equality of rewards and equality of successor blocks are reflexive. Therefore
\[
\mathcal T(\Pi^{(0)})
=
\Pi^{(0)}.
\]
Thus \(\Pi^{(0)}\) is a fixed point of the stability operator. Since the refinement sequence defining the stable quotient begins at \(\Pi^{(0)}\), its fixed point is
\[
\Pi_o^\star
=
\Pi_o^{(0)}
=
\bigl\{\{1\},\{2\},\{3\}\bigr\}.
\]
Consequently, the stable quotient has exactly the three singleton classes
\[
c_1:=\{1\},
\qquad
c_2:=\{2\},
\qquad
c_3:=\{3\}.
\]
Because every stable class is a singleton, the quotient transports induced by actions \(a\) and \(b\) are identified with the corresponding raw permutations. More precisely, Theorem~\ref{thm:markovization_main} gives
\[
\tau_a(c_i)
=
c_{\alpha(i)}
\qquad\text{and}\qquad
\tau_b(c_i)
=
c_{\beta(i)}
\]
for every \(i\in\{1,2,3\}\). Hence, after identifying \(c_i\) with \(i\),
\[
\tau_a=(12)=\alpha
\qquad\text{and}\qquad
\tau_b=(123)=\beta.
\]
Let
\[
G:=\langle\tau_a,\tau_b\rangle
\leq
\Sym(\{c_1,c_2,c_3\})
\cong S_3.
\]
The element \(\tau_a\) has order \(2\), while \(\tau_b\) has order \(3\). By Lagrange's theorem, both \(2\) and \(3\) divide \(|G|\), and therefore \(6\) divides \(|G|\). On the other hand, \(G\) is a subgroup of \(S_3\), so \(|G|\) divides \(|S_3|=6\). Hence
\[
|G|=6,
\]
and therefore
\[
G=S_3.
\]
This proves that the loop transports of the stable quotient generate the nonabelian group \(S_3\).

We now construct an executable reversible loop alphabet in the sense of Definition~\ref{def:reversible_loop_alphabet_main}. Let \(\ell_1^+\) and \(\ell_1^-\) both be the one-edge closed walk labeled by action \(a\), and let \(\ell_2^+\) be the one-edge closed walk labeled by action \(b\), while \(\ell_2^-\) is the two-edge closed walk whose consecutive action labels are \(b,b\). Since
\[
\alpha^{-1}=\alpha
\]
and
\[
\beta^{-1}=\beta^2=\beta\circ\beta,
\]
the corresponding quotient transports satisfy
\[
\tau_{\ell_1^-}
=
\alpha
=
\alpha^{-1}
=
\bigl(\tau_{\ell_1^+}\bigr)^{-1}
\]
and
\[
\tau_{\ell_2^-}
=
\beta\circ\beta
=
\beta^2
=
\beta^{-1}
=
\bigl(\tau_{\ell_2^+}\bigr)^{-1}.
\]
Thus
\(\{\ell_1^+,\ell_1^-,\ell_2^+,\ell_2^-\}\)
is an executable reversible loop alphabet. Let
\(F_2=\langle x_1,x_2\rangle\)
be the associated free group, so that
$
\rho_o(x_1)=\alpha,
\qquad
\rho_o(x_2)=\beta,
\qquad
\rho_o(x_1^{-1})=\alpha^{-1},
\qquad
\rho_o(x_2^{-1})=\beta^{-1}.
$
For the commutator
\[
[x_1,x_2]
=
x_1x_2x_1^{-1}x_2^{-1},
\]
the homomorphism convention
\(\rho_o(uv)=\rho_o(u)\circ\rho_o(v)\)
gives
\[
\rho_o([x_1,x_2])
=
\alpha\circ\beta\circ\alpha^{-1}\circ\beta^{-1}.
\]
Set
\[
\kappa
:=
\alpha\circ\beta\circ\alpha^{-1}\circ\beta^{-1}.
\]
Using
\(\alpha^{-1}=\alpha=(12)\)
and
\(\beta^{-1}=(132)\), and remembering that the rightmost permutation acts first, we obtain
\[
\begin{aligned}
\kappa(1)
&=
\alpha\!\left(
\beta\!\left(
\alpha^{-1}\!\left(
\beta^{-1}(1)
\right)
\right)
\right)\\
&=
\alpha\!\left(
\beta\!\left(
\alpha^{-1}(3)
\right)
\right)\\
&=
\alpha\!\left(
\beta(3)
\right)\\
&=
\alpha(1)\\
&=2,
\end{aligned}
\]
\[
\begin{aligned}
\kappa(2)
&=
\alpha\!\left(
\beta\!\left(
\alpha^{-1}\!\left(
\beta^{-1}(2)
\right)
\right)
\right)\\
&=
\alpha\!\left(
\beta\!\left(
\alpha^{-1}(1)
\right)
\right)\\
&=
\alpha\!\left(
\beta(2)
\right)\\
&=
\alpha(3)\\
&=3,
\end{aligned}
\]
and
\[
\begin{aligned}
\kappa(3)
&=
\alpha\!\left(
\beta\!\left(
\alpha^{-1}\!\left(
\beta^{-1}(3)
\right)
\right)
\right)\\
&=
\alpha\!\left(
\beta\!\left(
\alpha^{-1}(2)
\right)
\right)\\
&=
\alpha\!\left(
\beta(1)
\right)\\
&=
\alpha(2)\\
&=1.
\end{aligned}
\]
Therefore
\[
\rho_o([x_1,x_2])
=
\kappa
=
(123)
=
\beta.
\]
In particular,
\[
\rho_o([x_1,x_2])(c_1)
=
c_2.
\]
This commutator is executable using only the original actions. Indeed, the inverse letter \(x_1^{-1}\) is implemented by one use of action \(a\), and the inverse letter \(x_2^{-1}\) is implemented by two consecutive uses of action \(b\). Because the rightmost factor in the commutator transport acts first, the chronological action sequence
\[
b,b,a,b,a
\]
implements \([x_1,x_2]\). If \(\omega_{\mathrm{com}}\) denotes this directed closed walk, then Definition~\ref{def:directed_transport_main} gives
\[
\begin{aligned}
\tau_{\omega_{\mathrm{com}}}
&=
\alpha\circ\beta\circ\alpha\circ\beta\circ\beta\\
&=
\alpha\circ\beta\circ\alpha^{-1}\circ\beta^{-1}\\
&=
\rho_o([x_1,x_2]).
\end{aligned}
\]

We also verify that \(c_1\) is reachable, as required by Corollary~\ref{cor:commutator_barrier_main}. Since \(\rho_0\) is a probability distribution on \(\{o\}\times\{1,2,3\}\), there exists at least one layer \(i_0\in\{1,2,3\}\) such that
\[
\rho_0(o,i_0)>0.
\]
The permutation \(\beta=(123)\) cycles through all three layers. Hence there exists
\(k\in\{0,1,2\}\)
such that
\[
\beta^k(i_0)=1.
\]
Executing action \(b\) exactly \(k\) times has visible probability
\[
\prod_{t=1}^{k}P_O(o\mid o,b)=1
\]
and transports layer \(i_0\) to layer \(1\). It follows that the latent state \((o,1)\), and therefore the stable quotient class \(c_1\), is reachable with positive probability. Starting from this reachable class, the action sequence implementing the commutator is also executable with probability one.

It remains to prove the asserted decision separation. Write
\[
\bar V_i^\star
:=
\bar V^\star(o,c_i),
\qquad
i\in\{1,2,3\}.
\]
Every one-step reward in this HCDP lies in the interval \([0,1]\). Therefore, under every admissible policy and from every quotient state, the discounted return lies between
\(0\) and
\(\sum_{t=0}^{\infty}\gamma^t\). Taking the supremum over policies gives
\[
0
\leq
\bar V_i^\star
\leq
\sum_{t=0}^{\infty}\gamma^t
=
\frac{1}{1-\gamma}
\qquad
\text{for every }i\in\{1,2,3\}.
\]
At class \(c_1\), both actions have the same successor class, because
\[
\alpha(1)=2
\qquad\text{and}\qquad
\beta(1)=2.
\]
Consequently, the optimal quotient state--action values satisfy
\[
\bar Q^\star(o,c_1,a)
=
1+\gamma\bar V_2^\star
\]
and
\[
\bar Q^\star(o,c_1,b)
=
0+\gamma\bar V_2^\star.
\]
Subtracting the second equality from the first yields
\[
\begin{aligned}
\bar Q^\star(o,c_1,a)
-
\bar Q^\star(o,c_1,b)
&=
\bigl(1+\gamma\bar V_2^\star\bigr)
-
\bigl(\gamma\bar V_2^\star\bigr)\\
&=1\\
&>0.
\end{aligned}
\]
Thus action \(a\) is uniquely optimal at \(c_1\), and hence
\[
\cA^\star(o,c_1)=\{a\}.
\]
At class \(c_2\), action \(a\) transports the class to \(c_1\), whereas action \(b\) transports it to \(c_3\), because
\[
\alpha(2)=1
\qquad\text{and}\qquad
\beta(2)=3.
\]
The corresponding optimal quotient state--action values are therefore
\[
\bar Q^\star(o,c_2,a)
=
0+\gamma\bar V_1^\star
\]
and
\[
\bar Q^\star(o,c_2,b)
=
1+\gamma\bar V_3^\star.
\]
Their difference satisfies
\[
\begin{aligned}
\bar Q^\star(o,c_2,b)
-
\bar Q^\star(o,c_2,a)
&=
1+\gamma\bar V_3^\star
-
\gamma\bar V_1^\star\\
&\geq
1
-
\gamma\bar V_1^\star\\
&\geq
1
-
\frac{\gamma}{1-\gamma}\\
&=
\frac{1-2\gamma}{1-\gamma}.
\end{aligned}
\]
The assumption
\(0<\gamma<1/2\)
implies
\[
1-2\gamma>0
\qquad\text{and}\qquad
1-\gamma>0,
\]
so
\[
\frac{1-2\gamma}{1-\gamma}>0.
\]
It follows that
\[
\bar Q^\star(o,c_2,b)
>
\bar Q^\star(o,c_2,a).
\]
Thus action \(b\) is uniquely optimal at \(c_2\), and hence
\[
\cA^\star(o,c_2)=\{b\}.
\]
Therefore
\[
\cA^\star(o,c_1)
\cap
\cA^\star(o,c_2)
=
\{a\}\cap\{b\}
=
\varnothing,
\]
which means, by Definition~\ref{def:decision_separated_main}, that \(c_1\) and \(c_2\) are decision separated.

We have exhibited an executable reversible loop alphabet, a reachable class \(c_0=c_1\), and the words
\[
u=x_1
\qquad\text{and}\qquad
v=x_2
\]
such that
\[
\rho_o([u,v])(c_0)
=
\rho_o([x_1,x_2])(c_1)
=
c_2,
\]
while \(c_2\) is decision separated from \(c_1\). These are exactly the hypotheses of Corollary~\ref{cor:commutator_barrier_main}. Hence the commutator and empty-loop histories have the same abelianized memory but terminate at decision-separated quotient states, completing the proof.
\end{proof}

\subsection{Proof of Proposition~\ref{prop:ordered_tracking_nonabelian_main}}

\begin{proof}
Fix a base observation $o\in\cO$, an executable reversible loop alphabet
$\{\ell_j^+,\ell_j^-\}_{j=1}^r$ based at $o$, and a known initial stable class
$c\in\cC_o$. Exact quotient tracking stores the current stable class and, after every realized feasible observation transition
$(o_t,a_t,o_{t+1})$, applies the deterministic update
\[
c_{t+1}
=
\tau_{o_t,a_t,o_{t+1}}(c_t).
\]
By Lemma~\ref{lem:quotient_memory_update_main}, if the stored class is correct before a feasible transition, then the class produced by this update is the true stable class after that transition. Consequently, once the tracker is initialized with the correct class, its correctness is preserved along every finite executable directed path.

Consider first one letter
\[
z\in\{x_1^{\pm1},\ldots,x_r^{\pm1}\}.
\]
If $z=x_j$, its executable realization is the directed closed walk $\ell_j^+$, whereas if $z=x_j^{-1}$, its executable realization is the directed closed walk $\ell_j^-$. Denote the corresponding closed walk by
\[
\ell(z)=e_1e_2\cdots e_m,
\]
where every $e_q$ is a feasible directed edge and the walk begins and ends at $o$. Starting from an arbitrary known class $d\in\cC_o$, repeated application of the exact one-edge update gives, after the first edge,
\[
d^{(1)}
=
\tau_{e_1}(d),
\]
and, after the second edge,
\[
\begin{aligned}
d^{(2)}
&=
\tau_{e_2}(d^{(1)})\\
&=
\tau_{e_2}\bigl(\tau_{e_1}(d)\bigr)\\
&=
(\tau_{e_2}\circ\tau_{e_1})(d).
\end{aligned}
\]
Continuing recursively through all $m$ edges yields
\[
\begin{aligned}
d^{(m)}
&=
\tau_{e_m}(d^{(m-1)})\\
&=
\tau_{e_m}\circ\tau_{e_{m-1}}\circ\cdots\circ\tau_{e_1}(d)\\
&=
\tau_{\ell(z)}(d).
\end{aligned}
\]
The defining assignment of the free-group transport representation gives
\[
\tau_{\ell(z)}
=
\rho_o(z).
\]
Indeed, when $z=x_j$,
\[
\tau_{\ell(z)}
=
\tau_{\ell_j^+}
=
\rho_o(x_j),
\]
whereas when $z=x_j^{-1}$,
\[
\tau_{\ell(z)}
=
\tau_{\ell_j^-}
=
\rho_o(x_j^{-1}).
\]
Therefore, after completing the executable loop associated with a single letter $z$, exact quotient tracking transforms every known class $d\in\cC_o$ according to
\[
d
\longmapsto
\rho_o(z)(d).
\]

Now let $w\in F_r$. If $w=e$ is the empty word, no loop is executed, and therefore exact tracking leaves the initial class unchanged. Since $\rho_o$ is a group homomorphism, it maps the identity of $F_r$ to the identity permutation of $\cC_o$, so
\[
\rho_o(e)
=
\Id_{\cC_o}.
\]
Consequently, the terminal tracked class for the empty word is
\[
c
=
\Id_{\cC_o}(c)
=
\rho_o(e)(c).
\]
Thus the asserted identity holds when $w=e$.

Suppose next that $w$ has a reduced expression
\[
w
=
z_1z_2\cdots z_k,
\qquad
z_q\in\{x_1^{\pm1},\ldots,x_r^{\pm1}\},
\qquad
k\geq1.
\]
The convention in Section~\ref{sec:nonabelian} is
\[
\rho_o(uv)
=
\rho_o(u)\circ\rho_o(v),
\]
so the rightmost factor acts first. Accordingly, the executable realization of the written word
$w=z_1\cdots z_k$
performs the loop associated with $z_k$ first, then the loop associated with $z_{k-1}$, and continues in this order until the loop associated with $z_1$ is performed last. Every one of these loops begins and ends at $o$, so their concatenation is executable, and the tracked class again belongs to $\cC_o$ after each completed letter.

Let
\[
d_0:=c,
\]
and, for every $t\in\{1,\ldots,k\}$, define
\[
d_t
:=
\rho_o\bigl(z_{k-t+1}\bigr)(d_{t-1}).
\]
We show that $d_t$ is exactly the stable class stored after the first $t$ loops in this chronological execution. For $t=0$, the statement follows from the initialization
\[
d_0=c.
\]
Suppose that the statement holds after $t-1$ completed loops. The next executed letter is
\[
z_{k-t+1}.
\]
By the one-letter calculation above, completing the executable loop associated with this letter transforms the correctly stored class $d_{t-1}$ into
\[
\rho_o\bigl(z_{k-t+1}\bigr)(d_{t-1})
=
d_t.
\]
Lemma~\ref{lem:quotient_memory_update_main} guarantees that this recursively produced class remains the true stable class throughout the constituent edges of the loop. Hence the assertion holds after $t$ loops. By induction, it holds for every
$t\in\{0,\ldots,k\}$.

At the terminal index, repeated expansion of the recursion gives
\[
\begin{aligned}
d_k
&=
\rho_o(z_1)(d_{k-1})\\
&=
\rho_o(z_1)\Bigl(\rho_o(z_2)(d_{k-2})\Bigr)\\
&=
\rho_o(z_1)\Bigl(
\rho_o(z_2)\bigl(
\cdots
\rho_o(z_k)(c)
\cdots
\bigr)
\Bigr)\\
&=
\bigl(
\rho_o(z_1)
\circ
\rho_o(z_2)
\circ
\cdots
\circ
\rho_o(z_k)
\bigr)(c).
\end{aligned}
\]
Repeated use of the homomorphism identity gives
\[
\begin{aligned}
\rho_o(w)
&=
\rho_o(z_1z_2\cdots z_k)\\
&=
\rho_o(z_1)\circ\rho_o(z_2\cdots z_k)\\
&=
\rho_o(z_1)\circ\rho_o(z_2)\circ\rho_o(z_3\cdots z_k)\\
&=
\cdots\\
&=
\rho_o(z_1)
\circ
\rho_o(z_2)
\circ
\cdots
\circ
\rho_o(z_k).
\end{aligned}
\]
Combining the preceding two displays yields
\[
d_k
=
\rho_o(w)(c).
\]
Therefore exact quotient tracking maps the loop word $w$ to the terminal stable class
\[
c_w
:=
\rho_o(w)(c).
\]

This terminal class is well defined as a function of the free-group element $w$, rather than of a particular spelling of that element. To see this explicitly, insertion or deletion of an adjacent inverse pair does not change the terminal transport, because
\[
\begin{aligned}
\rho_o(x_j)\circ\rho_o(x_j^{-1})
&=
\rho_o(x_jx_j^{-1})\\
&=
\rho_o(e)\\
&=
\Id_{\cC_o},
\end{aligned}
\]
and similarly
\[
\begin{aligned}
\rho_o(x_j^{-1})\circ\rho_o(x_j)
&=
\rho_o(x_j^{-1}x_j)\\
&=
\rho_o(e)\\
&=
\Id_{\cC_o}.
\end{aligned}
\]
Thus free reduction does not alter the recursively tracked terminal class.

Now let $w,w'\in F_r$ satisfy
\[
\rho_o(w)(c)
\neq
\rho_o(w')(c).
\]
Applying the identity proved above separately to the two executable loop histories gives
\[
c_w
=
\rho_o(w)(c)
\]
and
\[
c_{w'}
=
\rho_o(w')(c).
\]
Hence
\[
c_w
\neq
c_{w'}.
\]
Since the canonical exact quotient tracker stores the stable class itself, its two terminal memory states are different. More generally, if the stable classes are represented by an injective class encoder
\[
\eta_o:\cC_o\to\cM,
\]
then injectivity implies
\[
\eta_o(c_w)
\neq
\eta_o(c_{w'}).
\]
Therefore the two loop histories remain distinguished by the exact quotient memory. This conclusion is independent of their abelianized summaries and, in particular, remains valid when
\[
\ab(w)
=
\ab(w').
\]
The reason is that exact quotient tracking recursively composes the quotient transports in their prescribed order, whereas abelianization retains only the signed number of occurrences of each generator and discards their order.

Finally, let $u,v\in F_r$ and suppose that the commutator
\[
[u,v]
=
uvu^{-1}v^{-1}
\]
acts nontrivially on the known class $c$, so that
\[
\rho_o([u,v])(c)
\neq
c.
\]
The empty word leaves the class unchanged:
\[
\begin{aligned}
c_e
&=
\rho_o(e)(c)\\
&=
\Id_{\cC_o}(c)\\
&=
c.
\end{aligned}
\]
By contrast, exact quotient tracking after the commutator history produces
\[
c_{[u,v]}
=
\rho_o([u,v])(c).
\]
The assumed nontriviality therefore gives
\[
c_{[u,v]}
\neq
c_e.
\]
Nevertheless, because abelianization is a homomorphism from the multiplicative group $F_r$ to the additive abelian group $\mathbb Z^r$, we have
\[
\ab(z^{-1})
=
-\ab(z)
\]
for every $z\in F_r$, and consequently
\[
\begin{aligned}
\ab([u,v])
&=
\ab(uvu^{-1}v^{-1})\\
&=
\ab(u)
+
\ab(v)
+
\ab(u^{-1})
+
\ab(v^{-1})\\
&=
\ab(u)
+
\ab(v)
-
\ab(u)
-
\ab(v)\\
&=
0.
\end{aligned}
\]
Moreover,
\[
\ab(e)=0,
\]
and hence
\[
\ab([u,v])
=
\ab(e).
\]
Thus the commutator and empty-loop histories have identical abelianized summaries, but exact ordered stable-class tracking assigns them different terminal classes whenever the commutator acts nontrivially on $c$. Therefore exact quotient tracking preserves every nontrivial commutator effect on the known stable class while storing only the current finite stable class rather than the complete loop word.
\end{proof}

\subsection{Proof of Lemma~\ref{lem:history_support_path}}

\begin{proof}
For every $j\in\{0,\ldots,t\}$, let
\[
H_j
=
(o_0,a_0,o_1,\ldots,a_{j-1},o_j)
\]
denote the length-$j$ observable prefix of $H_t$. To distinguish random variables from their realized values, write $S_j$, $O_j$, and $A_j$ for the latent state, observation, and action at time $j$, and let
$
\mathsf{H}_t
:=
\{O_0=o_0,A_0=a_0,O_1=o_1,\ldots,A_{t-1}=a_{t-1},O_t=o_t\}
$
be the event that the realized observable history equals $H_t$. Thus saying that $H_t$ has positive probability means that
$\bP^{\boldsymbol{\pi}}(\mathsf{H}_t)>0$.
Fix an arbitrary latent-state sequence
\[
(s_0,s_1,\ldots,s_t)
\in
\cS^{t+1}.
\]
Because the observation map is deterministic, the chain rule and the definition of a history-dependent policy give
$
\bP^{\boldsymbol{\pi}}
\bigl(
S_0=s_0,S_1=s_1,\ldots,S_t=s_t,\mathsf{H}_t
\bigr)
\quad=
\rho_0(s_0)
\mathbf{1}\{h(s_0)=o_0\}
\prod_{j=0}^{t-1}
\Bigl[
\pi_j(a_j\mid H_j)
P(s_{j+1}\mid s_j,a_j)
\mathbf{1}\{h(s_{j+1})=o_{j+1}\}
\Bigr].
$
When $t=0$, the product is the empty product and is equal to one. The events obtained by varying $(s_0,\ldots,s_t)\in\cS^{t+1}$ are pairwise disjoint and their union is $\mathsf{H}_t$. Summing over all latent-state sequences therefore gives
\[
\begin{aligned}
&\bP^{\boldsymbol{\pi}}(\mathsf{H}_t)\\
&=
\sum_{(s_0,\ldots,s_t)\in\cS^{t+1}}
\rho_0(s_0)
\mathbf{1}\{h(s_0)=o_0\}\\
&\quad\times
\prod_{j=0}^{t-1}
\Bigl[
\pi_j(a_j\mid H_j)
P(s_{j+1}\mid s_j,a_j)
\mathbf{1}\{h(s_{j+1})=o_{j+1}\}
\Bigr].
\end{aligned}
\]
The sum is finite because $\cS$ is finite, and every summand is nonnegative. Suppose first that
\[
\bP^{\boldsymbol{\pi}}(\mathsf{H}_t)>0.
\]
A finite sum of nonnegative numbers is strictly positive only if at least one summand is strictly positive. Hence there exists
\[
(s_0^\star,s_1^\star,\ldots,s_t^\star)
\in
\cS^{t+1}
\]
such that
\[
\begin{aligned}
0
&<
\rho_0(s_0^\star)
\mathbf{1}\{h(s_0^\star)=o_0\}\\
&\quad\times
\prod_{j=0}^{t-1}
\Bigl[
\pi_j(a_j\mid H_j)
P(s_{j+1}^\star\mid s_j^\star,a_j)
\mathbf{1}\{h(s_{j+1}^\star)=o_{j+1}\}
\Bigr].
\end{aligned}
\]
Every factor in this product is nonnegative. Strict positivity of the product therefore implies
\[
\rho_0(s_0^\star)>0,
\]
forces all observation-consistency indicators to be equal to one, and gives
\[
P(s_{j+1}^\star\mid s_j^\star,a_j)>0
\qquad
\text{for every }j\in\{0,\ldots,t-1\}.
\]
The indicator equalities are precisely
\[
h(s_j^\star)=o_j
\qquad
\text{for every }j\in\{0,\ldots,t\},
\]
so
\[
s_j^\star\in\cS_{o_j}
\qquad
\text{for every }j\in\{0,\ldots,t\}.
\]
For every $j\in\{0,\ldots,t-1\}$, the states
\[
s_j^\star\in\cS_{o_j}
\qquad\text{and}\qquad
s_{j+1}^\star\in\cS_{o_{j+1}}
\]
together with
\[
P(s_{j+1}^\star\mid s_j^\star,a_j)>0
\]
witness, by the definition of $E$, that
\[
(o_j,a_j,o_{j+1})\in E.
\]
Thus
\[
\omega(H_t)
=
(o_0\xrightarrow{a_0}o_1)
\cdots
(o_{t-1}\xrightarrow{a_{t-1}}o_t)
\]
is a directed support path in $X$. More importantly, the same sequence
$(s_0^\star,\ldots,s_t^\star)$ simultaneously realizes all of its edges, rather than providing unrelated witnesses edge by edge, and its initial state satisfies
$\rho_0(s_0^\star)>0$. Therefore $\omega(H_t)$ is $\rho_0$-realizable. When $t=0$, the argument states simply that there exists
$s_0^\star\in\cS_{o_0}$ with $\rho_0(s_0^\star)>0$, which is exactly $\rho_0$-realizability of the empty path based at $o_0$.

Conversely, suppose that $\omega(H_t)$ is $\rho_0$-realizable. By definition, there exist latent states
\[
(s_0^\star,s_1^\star,\ldots,s_t^\star)
\]
such that
\[
s_j^\star\in\cS_{o_j}
\qquad
\text{for every }j\in\{0,\ldots,t\},
\]
\[
\rho_0(s_0^\star)>0,
\]
and
\[
P(s_{j+1}^\star\mid s_j^\star,a_j)>0
\qquad
\text{for every }j\in\{0,\ldots,t-1\}.
\]
Since
\[
\cS_{o_j}=h^{-1}(\{o_j\}),
\]
the membership $s_j^\star\in\cS_{o_j}$ implies
\[
h(s_j^\star)=o_j
\qquad
\text{for every }j\in\{0,\ldots,t\}.
\]
Hence every observation-consistency indicator in the trajectory-probability formula is equal to one. The additional assumption of the lemma gives
\[
\pi_j(a_j\mid H_j)>0
\qquad
\text{for every }j\in\{0,\ldots,t-1\}.
\]
It follows that the summand associated with the particular compatible latent sequence
$(s_0^\star,\ldots,s_t^\star)$ is
\[
\begin{aligned}
&\rho_0(s_0^\star)
\mathbf{1}\{h(s_0^\star)=o_0\}\\
&\prod_{j=0}^{t-1}
\Bigl[
\pi_j(a_j\mid H_j)
P(s_{j+1}^\star\mid s_j^\star,a_j)
\mathbf{1}\{h(s_{j+1}^\star)=o_{j+1}\}
\Bigr]\\
&\quad=
\rho_0(s_0^\star)
\prod_{j=0}^{t-1}
\Bigl[
\pi_j(a_j\mid H_j)
P(s_{j+1}^\star\mid s_j^\star,a_j)
\Bigr]\\
&\quad>0.
\end{aligned}
\]
Since $\bP^{\boldsymbol{\pi}}(\mathsf{H}_t)$ is the sum of this strictly positive term and other nonnegative terms, we conclude that
\[
\bP^{\boldsymbol{\pi}}(\mathsf{H}_t)>0.
\]
This also covers $t=0$, because in that case the displayed quantity reduces to
$\rho_0(s_0^\star)>0$.

It remains to justify why membership of every edge in $E$ does not by itself imply latent realizability of the whole path. The definition of $E$ contains an existential quantifier separately for each edge, so the terminal latent-state witness chosen for one edge need not coincide with the initial latent-state witness chosen for the next edge. Consider, for example,
\[
\cO=\{o_0,o_1,o_2\},
\qquad
\cA=\{a,b\},
\qquad
\cS=\{s_0,p,q,s_2\},
\]
with deterministic observation map
\[
h(s_0)=o_0,
\qquad
h(p)=h(q)=o_1,
\qquad
h(s_2)=o_2.
\]
Choose a transition kernel satisfying
$
P(p\mid s_0,a)=1,
\qquad
P(p\mid p,b)=1,
\qquad
P(s_2\mid q,b)=1,
$
and let every remaining unspecified state--action pair be a deterministic self-loop. The first transition shows that
\[
(o_0,a,o_1)\in E,
\]
with witnesses $s_0\in\cS_{o_0}$ and $p\in\cS_{o_1}$, while the third transition shows that
\[
(o_1,b,o_2)\in E,
\]
with witnesses $q\in\cS_{o_1}$ and $s_2\in\cS_{o_2}$. Therefore
\[
(o_0\xrightarrow{a}o_1)
(o_1\xrightarrow{b}o_2)
\]
is a directed support path in $X$. Nevertheless, any latent realization of the first edge must arrive at $p$, because $s_0$ is the only latent state over $o_0$ and
$P(p\mid s_0,a)=1$. From $p$, action $b$ returns to $p$ with probability one, since
\[
P(p\mid p,b)=1.
\]
It therefore has zero probability of reaching any latent state in $\cS_{o_2}$. The state $q$ that witnesses the second edge cannot be used as the terminal state of the first edge, so the two edge-wise witnesses cannot be joined into a single latent-state sequence. Hence this directed support path is not latent-realizable, proving the final assertion.
\end{proof}

\subsection{Proof of Lemma~\ref{lem:finite_memory_bellman_eval_app}}

\begin{proof}
Let
$
\mathcal{B}(\cS\times\cM)
:=
\left\{
V:\cS\times\cM\to\mathbb{R}
:
\|V\|_\infty<\infty
\right\},
\qquad
\|V\|_\infty
:=
\max_{(s,m)\in\cS\times\cM}|V(s,m)|.
$
Because \(\cS\) and \(\cM\) are finite, every real-valued function on \(\cS\times\cM\) is bounded, and \(\mathcal{B}(\cS\times\cM)\), equipped with the sup norm, is complete. Since the reward function is bounded, define
\[
R_{\max}
:=
\max_{(s,a)\in\cS\times\cA}|r(s,a)|
<
\infty.
\]
Fix \(V\in\mathcal{B}(\cS\times\cM)\) and \((s,m)\in\cS\times\cM\). Using the triangle inequality, nonnegativity of the action and transition probabilities, and the fact that both probability distributions sum to one, we obtain
\[
\begin{aligned}
&\bigl|(T_{\pi_{\cM},U}V)(s,m)\bigr|\\
&=
\bigl|
\sum_{a\in\cA}
\pi_{\cM}(a\mid h(s),m)
[
r(s,a)\\
&+
\gamma
\sum_{s'\in\cS}
P(s'\mid s,a)
V\bigl(s',U(m,h(s),a,h(s'))\bigr)
]
\bigr|\\
&\leq
\sum_{a\in\cA}
\pi_{\cM}(a\mid h(s),m)
\bigl[
|r(s,a)|\\
&+
\gamma
\sum_{s'\in\cS}
P(s'\mid s,a)
\left|
V\bigl(s',U(m,h(s),a,h(s'))\bigr)
\right|
\bigr]\\
&\leq
\sum_{a\in\cA}
\pi_{\cM}(a\mid h(s),m)
\left[
R_{\max}
+
\gamma
\sum_{s'\in\cS}
P(s'\mid s,a)
\|V\|_\infty
\right]\\
&=
\sum_{a\in\cA}
\pi_{\cM}(a\mid h(s),m)
\left[
R_{\max}
+
\gamma\|V\|_\infty
\right]\\
&=
R_{\max}
+
\gamma\|V\|_\infty.
\end{aligned}
\]
Taking the maximum over \((s,m)\) gives
\[
\|T_{\pi_{\cM},U}V\|_\infty
\leq
R_{\max}
+
\gamma\|V\|_\infty
<
\infty,
\]
so the operator maps bounded functions to bounded functions.

Now fix \(V,W\in\mathcal{B}(\cS\times\cM)\). The reward terms cancel after subtraction, and therefore, for every \((s,m)\in\cS\times\cM\),
\[
\begin{aligned}
&\bigl|
(T_{\pi_{\cM},U}V)(s,m)
-
(T_{\pi_{\cM},U}W)(s,m)
\bigr|\\
&\quad=
\gamma
\bigl|
\sum_{a\in\cA}
\pi_{\cM}(a\mid h(s),m)\\
&\sum_{s'\in\cS}
P(s'\mid s,a)
[
V\bigl(s',U(m,h(s),a,h(s'))\bigr)\\
&-
W\bigl(s',U(m,h(s),a,h(s'))\bigr)
]
\bigr|\\
&\quad\leq
\gamma
\sum_{a\in\cA}
\pi_{\cM}(a\mid h(s),m)\\
&\sum_{s'\in\cS}
P(s'\mid s,a)
|
V\bigl(s',U(m,h(s),a,h(s'))\bigr)\\
&-
W\bigl(s',U(m,h(s),a,h(s'))\bigr)
|\\
&\quad\leq
\gamma
\sum_{a\in\cA}
\pi_{\cM}(a\mid h(s),m)
\sum_{s'\in\cS}
P(s'\mid s,a)
\|V-W\|_\infty\\
&\quad=
\gamma\|V-W\|_\infty.
\end{aligned}
\]
Taking the maximum over \((s,m)\) yields
\[
\|T_{\pi_{\cM},U}V-T_{\pi_{\cM},U}W\|_\infty
\leq
\gamma\|V-W\|_\infty.
\]
Since \(0<\gamma<1\), the operator is a \(\gamma\)-contraction. The Banach fixed-point theorem therefore gives a unique function
\[
V_{\pi_{\cM},U}
\in
\mathcal{B}(\cS\times\cM)
\]
such that
\[
V_{\pi_{\cM},U}
=
T_{\pi_{\cM},U}V_{\pi_{\cM},U}.
\]
The fixed-point identity and the bound already proved further imply
\[
\begin{aligned}
\|V_{\pi_{\cM},U}\|_\infty
&=
\|T_{\pi_{\cM},U}V_{\pi_{\cM},U}\|_\infty\\
&\leq
R_{\max}
+
\gamma\|V_{\pi_{\cM},U}\|_\infty,
\end{aligned}
\]
so
\[
\|V_{\pi_{\cM},U}\|_\infty
\leq
\frac{R_{\max}}{1-\gamma}.
\]

We next identify this fixed point with the discounted return of the given controller. For an arbitrary pair \((s,m)\in\cS\times\cM\), initialize the augmented process at
\[
(S_0,M_0)=(s,m)
\]
and generate it according to
\[
A_t\sim\pi_{\cM}(\cdot\mid h(S_t),M_t),
\qquad
S_{t+1}\sim P(\cdot\mid S_t,A_t),
\]
\[
M_{t+1}
=
U\bigl(M_t,h(S_t),A_t,h(S_{t+1})\bigr).
\]
By Lemma~\ref{lem:memory_aug_mdp_main}, \((S_t,M_t)\) is a time-homogeneous Markov process under this fixed controller. For every integer \(N\geq0\), define
\[
V_N(s,m)
:=
\mathbb{E}_{s,m}
\left[
\sum_{t=0}^{N-1}
\gamma^t r(S_t,A_t)
\right],
\]
where the sum is empty when \(N=0\). Thus
\[
V_0(s,m)=0
\qquad
\text{for every }(s,m)\in\cS\times\cM.
\]
Fix \(N\geq0\) and \((s,m)\in\cS\times\cM\). Separating the reward at time zero and reindexing the remaining terms gives
\[
\begin{aligned}
V_{N+1}(s,m)
&=
\mathbb{E}_{s,m}
\left[
r(S_0,A_0)
+
\sum_{t=1}^{N}
\gamma^t r(S_t,A_t)
\right]\\
&=
\mathbb{E}_{s,m}
\left[
r(s,A_0)
+
\gamma
\sum_{q=0}^{N-1}
\gamma^q r(S_{q+1},A_{q+1})
\right].
\end{aligned}
\]
For any action \(a\in\cA\) and any next state \(s'\in\cS\) satisfying
\[
P(s'\mid s,a)>0,
\]
the deterministic update rule gives
\[
M_1
=
U(m,h(s),a,h(s')).
\]
By the time-homogeneous Markov property of the augmented process, conditional on
\[
S_0=s,
\qquad
M_0=m,
\qquad
A_0=a,
\qquad
S_1=s',
\]
the expected discounted reward accumulated over the next \(N\) stages, with discount restarted at time \(1\), is
\[
V_N\Bigl(
s',
U(m,h(s),a,h(s'))
\Bigr).
\]
States \(s'\) with \(P(s'\mid s,a)=0\) contribute zero to the following sum. Hence the law of total expectation gives
\[
\begin{aligned}
&V_{N+1}(s,m)\\
&=
\sum_{a\in\cA}
\pi_{\cM}(a\mid h(s),m)
\Biggl[
r(s,a)\\
&+
\gamma
\sum_{s'\in\cS}
P(s'\mid s,a)
V_N\Bigl(
s',
U(m,h(s),a,h(s'))
\Bigr)
\Biggr]\\
&=
(T_{\pi_{\cM},U}V_N)(s,m).
\end{aligned}
\]
Therefore
\[
V_{N+1}
=
T_{\pi_{\cM},U}V_N
\qquad
\text{for every }N\geq0.
\]
Since \(V_0=0\), induction gives
\[
V_N
=
T_{\pi_{\cM},U}^{N}0
\qquad
\text{for every }N\geq0.
\]
The contraction property and the fixed-point identity imply
\[
\begin{aligned}
\|V_N-V_{\pi_{\cM},U}\|_\infty
&=
\left\|
T_{\pi_{\cM},U}^{N}0
-
T_{\pi_{\cM},U}^{N}V_{\pi_{\cM},U}
\right\|_\infty\\
&\leq
\gamma^N
\|V_{\pi_{\cM},U}\|_\infty\\
&\leq
\frac{R_{\max}\gamma^N}{1-\gamma}.
\end{aligned}
\]
Hence \(V_N\) converges uniformly to \(V_{\pi_{\cM},U}\).

On the other hand, boundedness of the reward implies that the infinite discounted series converges absolutely on every trajectory. Define
\[
V_\infty(s,m)
:=
\mathbb{E}_{s,m}
\left[
\sum_{t=0}^{\infty}
\gamma^t r(S_t,A_t)
\right].
\]
For every \((s,m)\) and every \(N\geq0\),
\[
\begin{aligned}
|V_\infty(s,m)-V_N(s,m)|
&=
\left|
\mathbb{E}_{s,m}
\left[
\sum_{t=N}^{\infty}
\gamma^t r(S_t,A_t)
\right]
\right|\\
&\leq
\mathbb{E}_{s,m}
\left[
\sum_{t=N}^{\infty}
\gamma^t|r(S_t,A_t)|
\right]\\
&\leq
R_{\max}
\sum_{t=N}^{\infty}
\gamma^t\\
&=
\frac{R_{\max}\gamma^N}{1-\gamma}.
\end{aligned}
\]
Thus \(V_N\) also converges uniformly to \(V_\infty\). A sequence in a normed space has at most one limit, so
\[
V_\infty
=
V_{\pi_{\cM},U}.
\]
Equivalently, for every augmented initial state \((s,m)\),
\[
V_{\pi_{\cM},U}(s,m)
=
\mathbb{E}_{s,m}
\left[
\sum_{t=0}^{\infty}
\gamma^t r(S_t,A_t)
\right].
\]

Finally, under the original initialization, the latent state satisfies
\[
S_0\sim\rho_0,
\]
the initial observation is \(h(S_0)\), and the memory is initialized deterministically as
\[
M_0
=
\iota(h(S_0)).
\]
Lemma~\ref{lem:memory_aug_mdp_main} states that the original finite-memory controller and the augmented process have the same expected discounted return, and that the augmented initial distribution is
\[
\widetilde\rho_0(s,m)
=
\rho_0(s)
\mathbf{1}\{m=\iota(h(s))\}.
\]
Using the value interpretation just proved and averaging over this initial distribution, we obtain
\[
\begin{aligned}
J(\pi_{\cM},U,\iota)
&=
\sum_{s\in\cS}
\sum_{m\in\cM}
\widetilde\rho_0(s,m)
V_{\pi_{\cM},U}(s,m)\\
&=
\sum_{s\in\cS}
\sum_{m\in\cM}
\rho_0(s)
\mathbf{1}\{m=\iota(h(s))\}
V_{\pi_{\cM},U}(s,m)\\
&=
\sum_{s\in\cS}
\rho_0(s)
V_{\pi_{\cM},U}
\bigl(s,\iota(h(s))\bigr).
\end{aligned}
\]
Renaming the dummy variable \(s\) as \(s_0\) gives
\[
J(\pi_{\cM},U,\iota)
=
\sum_{s_0\in\cS}
\rho_0(s_0)
V_{\pi_{\cM},U}
\bigl(s_0,\iota(h(s_0))\bigr),
\]
which completes the proof.
\end{proof}

\subsection{Proof of Lemma~\ref{lem:subg_mean_app}}

\begin{proof}
Set
\[
\overline{\xi}_m
:=
\frac{1}{m}\sum_{i=1}^{m}\xi_i.
\]
Since the random variables are identically distributed and each of them has mean \(\mu\), linearity of expectation gives
\[
\begin{aligned}
\bE[\overline{\xi}_m]
&=
\bE\left[
\frac{1}{m}\sum_{i=1}^{m}\xi_i
\right]\\
&=
\frac{1}{m}\sum_{i=1}^{m}\bE[\xi_i]\\
&=
\frac{1}{m}\sum_{i=1}^{m}\mu\\
&=
\frac{m\mu}{m}\\
&=
\mu.
\end{aligned}
\]
We first derive a sub-Gaussian moment-generating-function bound for the empirical mean. Fix an arbitrary \(\lambda\in\mathbb R\). By the definition of \(\overline{\xi}_m\),
\[
\begin{aligned}
\lambda(\overline{\xi}_m-\mu)
&=
\lambda\left(
\frac{1}{m}\sum_{i=1}^{m}\xi_i-\mu
\right)\\
&=
\lambda\left(
\frac{1}{m}\sum_{i=1}^{m}\xi_i
-
\frac{1}{m}\sum_{i=1}^{m}\mu
\right)\\
&=
\frac{\lambda}{m}
\sum_{i=1}^{m}(\xi_i-\mu).
\end{aligned}
\]
Exponentiating this identity and using
\(\exp(\sum_i x_i)=\prod_i\exp(x_i)\), we obtain
\[
\begin{aligned}
\exp\!\bigl(\lambda(\overline{\xi}_m-\mu)\bigr)
&=
\exp\!\left(
\frac{\lambda}{m}
\sum_{i=1}^{m}(\xi_i-\mu)
\right)\\
&=
\prod_{i=1}^{m}
\exp\!\left(
\frac{\lambda}{m}(\xi_i-\mu)
\right).
\end{aligned}
\]
The variables \(\xi_1,\ldots,\xi_m\) are independent, and each factor in the last product is a measurable function of only the corresponding variable \(\xi_i\). Therefore,
\[
\begin{aligned}
\bE\!\left[
\exp\!\bigl(\lambda(\overline{\xi}_m-\mu)\bigr)
\right]
&=
\bE\!\left[
\prod_{i=1}^{m}
\exp\!\left(
\frac{\lambda}{m}(\xi_i-\mu)
\right)
\right]\\
&=
\prod_{i=1}^{m}
\bE\!\left[
\exp\!\left(
\frac{\lambda}{m}(\xi_i-\mu)
\right)
\right].
\end{aligned}
\]
By the definition of a \(\sigma^2\)-sub-Gaussian random variable, applied to \(\xi_i\) with the real parameter \(\lambda/m\), we have
\[
\begin{aligned}
\bE\!\left[
\exp\!\left(
\frac{\lambda}{m}(\xi_i-\mu)
\right)
\right]
&\leq
\exp\!\left(
\frac{1}{2}
\left(\frac{\lambda}{m}\right)^2
\sigma^2
\right)\\
&=
\exp\!\left(
\frac{\lambda^2\sigma^2}{2m^2}
\right)
\end{aligned}
\]
for every \(i\in\{1,\ldots,m\}\). Multiplying these \(m\) inequalities gives
\[
\begin{aligned}
\bE\!\left[
\exp\!\bigl(\lambda(\overline{\xi}_m-\mu)\bigr)
\right]
&\leq
\prod_{i=1}^{m}
\exp\!\left(
\frac{\lambda^2\sigma^2}{2m^2}
\right)\\
&=
\exp\!\left(
\sum_{i=1}^{m}
\frac{\lambda^2\sigma^2}{2m^2}
\right)\\
&=
\exp\!\left(
\frac{m\lambda^2\sigma^2}{2m^2}
\right)\\
&=
\exp\!\left(
\frac{\lambda^2\sigma^2}{2m}
\right).
\end{aligned}
\]
Thus the centered empirical mean satisfies
$
\bE\!\left[
\exp\!\bigl(\lambda(\overline{\xi}_m-\mu)\bigr)
\right]
\leq
\exp\!\left(
\frac{\lambda^2(\sigma^2/m)}{2}
\right)
\qquad
\text{for every }\lambda\in\mathbb R,
$
so \(\overline{\xi}_m\) is \((\sigma^2/m)\)-sub-Gaussian.

We next bound the upper tail. Fix \(\varepsilon>0\) and an arbitrary \(\lambda>0\). Because the exponential function is strictly increasing,
\[
\left\{
\overline{\xi}_m-\mu\geq\varepsilon
\right\}
=
\left\{
\exp\!\bigl(\lambda(\overline{\xi}_m-\mu)\bigr)
\geq
\exp(\lambda\varepsilon)
\right\}.
\]
The random variable
\(\exp(\lambda(\overline{\xi}_m-\mu))\)
is nonnegative. Hence Markov's inequality implies
\[
\begin{aligned}
\Pr\!\left(
\overline{\xi}_m-\mu\geq\varepsilon
\right)
&=
\Pr\!\left(
\exp\!\bigl(\lambda(\overline{\xi}_m-\mu)\bigr)
\geq
\exp(\lambda\varepsilon)
\right)\\
&\leq
\frac{
\bE\!\left[
\exp\!\bigl(\lambda(\overline{\xi}_m-\mu)\bigr)
\right]
}{
\exp(\lambda\varepsilon)
}\\
&=
\exp(-\lambda\varepsilon)
\bE\!\left[
\exp\!\bigl(\lambda(\overline{\xi}_m-\mu)\bigr)
\right]\\
&\leq
\exp(-\lambda\varepsilon)
\exp\!\left(
\frac{\lambda^2\sigma^2}{2m}
\right)\\
&=
\exp\!\left(
-\lambda\varepsilon
+
\frac{\lambda^2\sigma^2}{2m}
\right).
\end{aligned}
\]
This bound holds for every \(\lambda>0\). Consider the exponent
\[
q(\lambda)
:=
-\lambda\varepsilon
+
\frac{\lambda^2\sigma^2}{2m}.
\]
Its derivative is
\[
q'(\lambda)
=
-\varepsilon
+
\frac{\lambda\sigma^2}{m},
\]
and its second derivative is
\[
q''(\lambda)
=
\frac{\sigma^2}{m}
>
0.
\]
Therefore \(q\) is strictly convex, and its unique minimizer is determined by
\[
q'(\lambda)=0.
\]
Solving this equation gives
\[
-\varepsilon+\frac{\lambda\sigma^2}{m}=0,
\]
and hence
\[
\lambda_\star
=
\frac{m\varepsilon}{\sigma^2}.
\]
Because \(\varepsilon>0\), this choice satisfies \(\lambda_\star>0\). Substituting it into the exponent gives
\[
\begin{aligned}
q(\lambda_\star)
&=
-\left(
\frac{m\varepsilon}{\sigma^2}
\right)\varepsilon
+
\frac{\sigma^2}{2m}
\left(
\frac{m\varepsilon}{\sigma^2}
\right)^2\\
&=
-\frac{m\varepsilon^2}{\sigma^2}
+
\frac{\sigma^2}{2m}
\frac{m^2\varepsilon^2}{\sigma^4}\\
&=
-\frac{m\varepsilon^2}{\sigma^2}
+
\frac{m\varepsilon^2}{2\sigma^2}\\
&=
-\frac{m\varepsilon^2}{2\sigma^2}.
\end{aligned}
\]
Consequently,
\[
\Pr\!\left(
\overline{\xi}_m-\mu\geq\varepsilon
\right)
\leq
\exp\!\left(
-\frac{m\varepsilon^2}{2\sigma^2}
\right).
\]

We now bound the lower tail. Again fix an arbitrary \(\lambda>0\). Since
\[
\mu-\overline{\xi}_m
=
-(\overline{\xi}_m-\mu),
\]
the moment-generating-function estimate established above, applied with the parameter \(-\lambda\), yields
\[
\begin{aligned}
\bE\!\left[
\exp\!\bigl(\lambda(\mu-\overline{\xi}_m)\bigr)
\right]
&=
\bE\!\left[
\exp\!\bigl(-\lambda(\overline{\xi}_m-\mu)\bigr)
\right]\\
&\leq
\exp\!\left(
\frac{(-\lambda)^2\sigma^2}{2m}
\right)\\
&=
\exp\!\left(
\frac{\lambda^2\sigma^2}{2m}
\right).
\end{aligned}
\]
Moreover,
\[
\left\{
\overline{\xi}_m-\mu\leq-\varepsilon
\right\}
=
\left\{
\mu-\overline{\xi}_m\geq\varepsilon
\right\}.
\]
Using monotonicity of the exponential function followed by Markov's inequality, we obtain
\[
\begin{aligned}
\Pr\!\left(
\overline{\xi}_m-\mu\leq-\varepsilon
\right)
&=
\Pr\!\left(
\mu-\overline{\xi}_m\geq\varepsilon
\right)\\
&=
\Pr\!\left(
\exp\!\bigl(\lambda(\mu-\overline{\xi}_m)\bigr)
\geq
\exp(\lambda\varepsilon)
\right)\\
&\leq
\exp(-\lambda\varepsilon)
\bE\!\left[
\exp\!\bigl(\lambda(\mu-\overline{\xi}_m)\bigr)
\right]\\
&\leq
\exp\!\left(
-\lambda\varepsilon
+
\frac{\lambda^2\sigma^2}{2m}
\right).
\end{aligned}
\]
The exponent is the same function \(q(\lambda)\) considered for the upper tail. Choosing
\[
\lambda
=
\lambda_\star
=
\frac{m\varepsilon}{\sigma^2}
\]
therefore gives
\[
\Pr\!\left(
\overline{\xi}_m-\mu\leq-\varepsilon
\right)
\leq
\exp\!\left(
-\frac{m\varepsilon^2}{2\sigma^2}
\right).
\]

Finally, the two-sided deviation event satisfies
\[
\begin{aligned}
\left\{
|\overline{\xi}_m-\mu|\geq\varepsilon
\right\}
&=
\left\{
\overline{\xi}_m-\mu\geq\varepsilon
\right\}
\cup
\left\{
\overline{\xi}_m-\mu\leq-\varepsilon
\right\}.
\end{aligned}
\]
Applying the union bound together with the two one-sided estimates yields
\[
\begin{aligned}
\Pr\!\left(
|\overline{\xi}_m-\mu|\geq\varepsilon
\right)
&\leq
\Pr\!\left(
\overline{\xi}_m-\mu\geq\varepsilon
\right)
+
\Pr\!\left(
\overline{\xi}_m-\mu\leq-\varepsilon
\right)\\
&\leq
\exp\!\left(
-\frac{m\varepsilon^2}{2\sigma^2}
\right)
+
\exp\!\left(
-\frac{m\varepsilon^2}{2\sigma^2}
\right)\\
&=
2\exp\!\left(
-\frac{m\varepsilon^2}{2\sigma^2}
\right).
\end{aligned}
\]
Recalling that
\[
\overline{\xi}_m
=
\frac{1}{m}\sum_{i=1}^{m}\xi_i
\]
proves
\[
\Pr\!\left(
\left|
\frac{1}{m}\sum_{i=1}^{m}\xi_i-\mu
\right|
\geq\varepsilon
\right)
\leq
2\exp\!\left(
-\frac{m\varepsilon^2}{2\sigma^2}
\right),
\]
as required.
\end{proof}

\subsection{Proof of Corollary~\ref{cor:vector_subg_mean_app}}

\begin{proof}
Define the empirical mean vector
\[
\overline X_m
:=
\frac{1}{m}\sum_{i=1}^{m}X_i.
\]
For each coordinate \(j\in\{1,\ldots,d\}\), write
\[
\overline X_{m,j}
:=
(\overline X_m)_j
\qquad\text{and}\qquad
\mu_j
:=
(\mu)_j.
\]
Because vector addition and scalar multiplication in \(\mathbb{R}^d\) are defined coordinatewise, we have
\[
\begin{aligned}
\overline X_{m,j}
&=
\left(
\frac{1}{m}\sum_{i=1}^{m}X_i
\right)_j\\
&=
\frac{1}{m}\sum_{i=1}^{m}(X_i)_j.
\end{aligned}
\]
Moreover, since \(\bE[X_i]=\mu\), equality of vectors implies equality of their coordinates, and therefore
\[
\begin{aligned}
\bE[(X_i)_j]
&=
(\bE[X_i])_j\\
&=
(\mu)_j\\
&=
\mu_j
\end{aligned}
\]
for every \(i\in\{1,\ldots,m\}\) and every \(j\in\{1,\ldots,d\}\).

Fix an arbitrary coordinate \(j\in\{1,\ldots,d\}\). Since the random vectors
\[
X_1,\ldots,X_m
\]
are independent, applying the measurable coordinate projection
\[
p_j:\mathbb{R}^d\to\mathbb{R},
\qquad
p_j(x)=x_j,
\]
shows that
\[
p_j(X_1),\ldots,p_j(X_m)
\]
are independent. Since the vectors are also identically distributed, their coordinate projections are identically distributed. Thus
\[
(X_1)_j,\ldots,(X_m)_j
\]
are independent and identically distributed scalar random variables. Each of them has mean \(\mu_j\), and, by the assumption of the corollary, each of them is \(\sigma^2\)-sub-Gaussian. Hence all assumptions of Lemma~\ref{lem:subg_mean_app} are satisfied by the scalar sequence
\[
\xi_i:=(X_i)_j,
\qquad
i=1,\ldots,m.
\]
Applying that lemma gives
\[
\Pr\!\left(
\left|
\frac{1}{m}\sum_{i=1}^{m}(X_i)_j-\mu_j
\right|
\geq\varepsilon
\right)
\leq
2\exp\!\left(
-\frac{m\varepsilon^2}{2\sigma^2}
\right).
\]
Using the coordinate identity for \(\overline X_m\), this can equivalently be written as
\[
\Pr\!\left(
\left|
\overline X_{m,j}-\mu_j
\right|
\geq\varepsilon
\right)
\leq
2\exp\!\left(
-\frac{m\varepsilon^2}{2\sigma^2}
\right).
\]
Since the coordinate \(j\) was arbitrary, the same estimate holds for every
\(j\in\{1,\ldots,d\}\). No independence between different coordinates is required here. For each fixed coordinate, independence across the sample index \(i\) follows from independence of the random vectors, and the coordinate events will be combined below only by the union bound.

By the definition of the \(\ell_\infty\) norm on \(\mathbb{R}^d\),
\[
\begin{aligned}
\left\|
\overline X_m-\mu
\right\|_\infty
&=
\max_{1\leq j\leq d}
\left|
(\overline X_m-\mu)_j
\right|\\
&=
\max_{1\leq j\leq d}
\left|
(\overline X_m)_j-(\mu)_j
\right|\\
&=
\max_{1\leq j\leq d}
\left|
\overline X_{m,j}-\mu_j
\right|.
\end{aligned}
\]
Because this maximum is taken over the finite set
\(\{1,\ldots,d\}\), it is at least \(\varepsilon\) if and only if at least one of its coordinate terms is at least \(\varepsilon\). Therefore,
\[
\begin{aligned}
\left\{
\left\|
\overline X_m-\mu
\right\|_\infty
\geq\varepsilon
\right\}
&=
\left\{
\max_{1\leq j\leq d}
\left|
\overline X_{m,j}-\mu_j
\right|
\geq\varepsilon
\right\}\\
&=
\bigcup_{j=1}^{d}
\left\{
\left|
\overline X_{m,j}-\mu_j
\right|
\geq\varepsilon
\right\}.
\end{aligned}
\]
Taking probabilities of both sides gives
\[
\begin{aligned}
\Pr\!\left(
\left\|
\overline X_m-\mu
\right\|_\infty
\geq\varepsilon
\right)
&=
\Pr\!\left(
\bigcup_{j=1}^{d}
\left\{
\left|
\overline X_{m,j}-\mu_j
\right|
\geq\varepsilon
\right\}
\right).
\end{aligned}
\]
Applying the union bound, which does not require the coordinate events to be independent, yields
\[
\begin{aligned}
\Pr\!\left(
\left\|
\overline X_m-\mu
\right\|_\infty
\geq\varepsilon
\right)
&\leq
\sum_{j=1}^{d}
\Pr\!\left(
\left|
\overline X_{m,j}-\mu_j
\right|
\geq\varepsilon
\right).
\end{aligned}
\]
Substituting the coordinatewise estimate established above into every term of the sum, we obtain
\[
\begin{aligned}
\Pr\!\left(
\left\|
\overline X_m-\mu
\right\|_\infty
\geq\varepsilon
\right)
&\leq
\sum_{j=1}^{d}
2\exp\!\left(
-\frac{m\varepsilon^2}{2\sigma^2}
\right)\\
&=
2
\left(
\sum_{j=1}^{d}1
\right)
\exp\!\left(
-\frac{m\varepsilon^2}{2\sigma^2}
\right)\\
&=
2d\exp\!\left(
-\frac{m\varepsilon^2}{2\sigma^2}
\right).
\end{aligned}
\]
Finally, recalling that
\[
\overline X_m
=
\frac{1}{m}\sum_{i=1}^{m}X_i,
\]
the preceding inequality becomes
\[
\Pr\!\left(
\left\|
\frac{1}{m}\sum_{i=1}^{m}X_i-\mu
\right\|_\infty
\geq\varepsilon
\right)
\leq
2d\exp\!\left(
-\frac{m\varepsilon^2}{2\sigma^2}
\right),
\]
which is the desired result.
\end{proof}

\subsection{Proof of Corollary~\ref{cor:finite_family_union_app}}

\begin{proof}
For every $t\geq 1$ and every $\alpha\in\cI_t$, define the individual failure event
\[
E_t^\alpha
:=
\left\{
\|\hat\mu_t^\alpha-\mu^\alpha\|_\infty
\geq
\varepsilon_t
\right\}.
\]
Also define the uniform failure event at round $t$ by
\[
E_t
:=
\left\{
\exists\,\alpha\in\cI_t:
\|\hat\mu_t^\alpha-\mu^\alpha\|_\infty
\geq
\varepsilon_t
\right\}.
\]
By the meaning of the existential quantifier over the finite index set $\cI_t$, this event is exactly the union of the individual failure events:
\[
E_t
=
\bigcup_{\alpha\in\cI_t}E_t^\alpha.
\]
If $\cI_t=\varnothing$, then the union is empty, and therefore
\[
\Pr(E_t)
=
0
=
|\cI_t|\delta_t.
\]
Thus the desired inequality is immediate in this case. Suppose henceforth that $\cI_t$ is nonempty. Since $\cI_t$ is finite, it can be written as
\[
\cI_t
=
\{\alpha_1,\ldots,\alpha_{N_t}\},
\qquad
N_t
:=
|\cI_t|.
\]
Using the representation of $E_t$ as a finite union and applying finite subadditivity of probability, we obtain
\[
\begin{aligned}
\Pr(E_t)
&=
\Pr\!\left(
\bigcup_{k=1}^{N_t}E_t^{\alpha_k}
\right)\\
&\leq
\sum_{k=1}^{N_t}
\Pr\!\left(
E_t^{\alpha_k}
\right).
\end{aligned}
\]
For every $k\in\{1,\ldots,N_t\}$, the assumed pointwise concentration inequality gives
\[
\begin{aligned}
\Pr\!\left(
E_t^{\alpha_k}
\right)
&=
\Pr\!\left(
\|\hat\mu_t^{\alpha_k}-\mu^{\alpha_k}\|_\infty
\geq
\varepsilon_t
\right)\\
&\leq
\delta_t.
\end{aligned}
\]
Substituting these bounds into the preceding finite sum yields
\[
\begin{aligned}
\Pr(E_t)
&\leq
\sum_{k=1}^{N_t}\delta_t\\
&=
N_t\delta_t\\
&=
|\cI_t|\delta_t.
\end{aligned}
\]
Recalling the definition of $E_t$, we have therefore proved
\[
\Pr\!\left(
\exists\,\alpha\in\cI_t:
\|\hat\mu_t^\alpha-\mu^\alpha\|_\infty
\geq
\varepsilon_t
\right)
\leq
|\cI_t|\delta_t.
\]

Now suppose that
\[
\sum_{t\geq1}
|\cI_t|\delta_t
<
\infty.
\]
The bound established above implies, for every $t\geq1$, that
\[
0
\leq
\Pr(E_t)
\leq
|\cI_t|\delta_t.
\]
Since all terms are nonnegative, comparison of the corresponding series gives
\[
\begin{aligned}
\sum_{t\geq1}\Pr(E_t)
&\leq
\sum_{t\geq1}
|\cI_t|\delta_t\\
&<
\infty.
\end{aligned}
\]
Lemma~\ref{lem:bc_wrapper_main}, equivalently the first Borel--Cantelli lemma, therefore implies
\[
\Pr(E_t\ \mathrm{i.o.})
=
0.
\]
The event that the uniform failure events occur infinitely often is the limsup event
\[
\{E_t\ \mathrm{i.o.}\}
=
\bigcap_{T=1}^{\infty}
\bigcup_{t\geq T}E_t.
\]
Its complement is consequently
\[
\begin{aligned}
\{E_t\ \mathrm{i.o.}\}^{\mathsf c}
&=
\left(
\bigcap_{T=1}^{\infty}
\bigcup_{t\geq T}E_t
\right)^{\mathsf c}\\
&=
\bigcup_{T=1}^{\infty}
\left(
\bigcup_{t\geq T}E_t
\right)^{\mathsf c}\\
&=
\bigcup_{T=1}^{\infty}
\bigcap_{t\geq T}E_t^{\mathsf c},
\end{aligned}
\]
where the final two equalities follow from De Morgan's laws. Since
\[
\Pr(E_t\ \mathrm{i.o.})=0,
\]
we have
\[
\Pr\!\left(
\bigcup_{T=1}^{\infty}
\bigcap_{t\geq T}E_t^{\mathsf c}
\right)
=
1.
\]
Thus, with probability one, there exists a finite random index $T$ such that
\[
E_t^{\mathsf c}
\qquad
\text{for every }t\geq T.
\]
For each fixed $t$, the complement of the uniform failure event can be expanded as
\[
\begin{aligned}
E_t^{\mathsf c}
&=
\left(
\bigcup_{\alpha\in\cI_t}
E_t^\alpha
\right)^{\mathsf c}\\
&=
\bigcap_{\alpha\in\cI_t}
(E_t^\alpha)^{\mathsf c}\\
&=
\bigcap_{\alpha\in\cI_t}
\left\{
\|\hat\mu_t^\alpha-\mu^\alpha\|_\infty
<
\varepsilon_t
\right\}.
\end{aligned}
\]
It follows that, with probability one, there exists a finite random index $T$ such that, simultaneously for every $t\geq T$ and every $\alpha\in\cI_t$,
\[
\|\hat\mu_t^\alpha-\mu^\alpha\|_\infty
<
\varepsilon_t.
\]
Equivalently, the corresponding uniform failure event occurs only finitely often almost surely. Neither the finite union bound nor the first Borel--Cantelli implication requires independence among the estimators indexed by $\alpha$ or among the failure events corresponding to different rounds $t$.
\end{proof}

\subsection{Proof of Lemma~\ref{lem:bc_wrapper_main}}

\begin{proof}
For every integer \(N\geq 1\), define the tail-union event
\[
B_N
:=
\bigcup_{t=N}^{\infty}E_t.
\]
An outcome belongs to \(B_N\) precisely when at least one of the events \(E_t\) occurs at an index \(t\geq N\). Consequently, an outcome belongs to infinitely many of the events \(E_t\) if and only if it belongs to \(B_N\) for every \(N\geq 1\). In event notation,
\[
\{E_t\ \mathrm{i.o.}\}
=
\bigcap_{N=1}^{\infty}B_N
=
\bigcap_{N=1}^{\infty}
\bigcup_{t=N}^{\infty}E_t.
\]
Indeed, suppose that an outcome belongs to infinitely many of the events \(E_t\). Then, for every \(N\geq 1\), there exists an index \(t\geq N\) for which that outcome belongs to \(E_t\). Hence the outcome belongs to \(B_N\) for every \(N\), and therefore it belongs to
\(\bigcap_{N\geq 1}B_N\). Conversely, suppose that an outcome belongs to only finitely many of the events \(E_t\). Then there exists an integer \(N_0\geq 1\) such that none of the events \(E_t\) with \(t\geq N_0\) occurs for that outcome. It follows that the outcome does not belong to
\[
B_{N_0}
=
\bigcup_{t=N_0}^{\infty}E_t,
\]
and hence it cannot belong to
\(\bigcap_{N\geq 1}B_N\). This proves the displayed identity.

By countable subadditivity of probability, for every \(N\geq 1\),
\[
\begin{aligned}
\Pr(B_N)
&=
\Pr\!\left(
\bigcup_{t=N}^{\infty}E_t
\right)\\
&\leq
\sum_{t=N}^{\infty}\Pr(E_t).
\end{aligned}
\]
By assumption,
\[
\sum_{t=1}^{\infty}\Pr(E_t)
<
\infty.
\]
Because this is a convergent series of nonnegative real numbers, its tails converge to zero. To write this explicitly, define
\[
S_0:=0,
\qquad
S_N
:=
\sum_{t=1}^{N}\Pr(E_t)
\quad
\text{for }N\geq 1,
\]
and let
\[
S
:=
\sum_{t=1}^{\infty}\Pr(E_t).
\]
The convergence of the series means that
\[
S_N\longrightarrow S
\qquad
\text{as }N\longrightarrow\infty.
\]
For every \(N\geq 1\), the corresponding tail can be written as
\[
\begin{aligned}
\sum_{t=N}^{\infty}\Pr(E_t)
&=
\sum_{t=1}^{\infty}\Pr(E_t)
-
\sum_{t=1}^{N-1}\Pr(E_t)\\
&=
S-S_{N-1}.
\end{aligned}
\]
Since \(S_{N-1}\to S\), it follows that
\[
\sum_{t=N}^{\infty}\Pr(E_t)
\longrightarrow
0
\qquad
\text{as }N\longrightarrow\infty.
\]

Moreover, for every fixed \(N\geq 1\),
\[
\{E_t\ \mathrm{i.o.}\}
=
\bigcap_{K=1}^{\infty}B_K
\subseteq
B_N.
\]
Monotonicity of probability therefore gives
\[
\begin{aligned}
0
&\leq
\Pr(E_t\ \mathrm{i.o.})\\
&\leq
\Pr(B_N)\\
&\leq
\sum_{t=N}^{\infty}\Pr(E_t).
\end{aligned}
\]
This inequality holds for every \(N\geq 1\). Letting \(N\to\infty\) and using the convergence of the tail sums yields
\[
0
\leq
\Pr(E_t\ \mathrm{i.o.})
\leq
0.
\]
Consequently,
\[
\Pr(E_t\ \mathrm{i.o.})=0.
\]

The preceding conclusion can equivalently be expressed as an eventual-correctness statement. By De Morgan's laws,
\[
\begin{aligned}
\{E_t\ \mathrm{i.o.}\}^{\mathsf c}
&=
\left(
\bigcap_{N=1}^{\infty}
\bigcup_{t=N}^{\infty}E_t
\right)^{\mathsf c}\\
&=
\bigcup_{N=1}^{\infty}
\left(
\bigcup_{t=N}^{\infty}E_t
\right)^{\mathsf c}\\
&=
\bigcup_{N=1}^{\infty}
\bigcap_{t=N}^{\infty}E_t^{\mathsf c}.
\end{aligned}
\]
Since
\[
\Pr(E_t\ \mathrm{i.o.})=0,
\]
we have
\[
\Pr\!\left(
\bigcup_{N=1}^{\infty}
\bigcap_{t=N}^{\infty}E_t^{\mathsf c}
\right)
=
1.
\]
Thus, with probability one, there exists a finite random index \(N\) such that
\[
E_t^{\mathsf c}
\qquad
\text{for every }t\geq N.
\]
Equivalently, with probability one, the random set
\[
\{t\geq 1:E_t\text{ occurs}\}
\]
is finite.

For the final assertion, let \(F_t\) denote the event that the finite identification procedure is incorrect at round \(t\). The assumed per-round failure guarantee states that
\[
\Pr(F_t)
\leq
\delta_t
\qquad
\text{for every }t\geq 1.
\]
Since all terms are nonnegative, comparison of series gives
\[
\begin{aligned}
\sum_{t=1}^{\infty}\Pr(F_t)
&\leq
\sum_{t=1}^{\infty}\delta_t\\
&<
\infty.
\end{aligned}
\]
Applying the result just proved to the sequence of failure events
\((F_t)_{t\geq 1}\) gives
\[
\Pr(F_t\ \mathrm{i.o.})=0.
\]
Taking complements and expanding the infinitely-often event as above, we obtain
\[
\Pr\!\left(
\bigcup_{T=1}^{\infty}
\bigcap_{t=T}^{\infty}F_t^{\mathsf c}
\right)
=
1.
\]
Therefore, with probability one, there exists a finite random round \(T\) such that
\[
F_t^{\mathsf c}
\qquad
\text{for every }t\geq T.
\]
Since \(F_t^{\mathsf c}\) is exactly the event that the identification procedure is correct at round \(t\), this means that
$
\Pr\!(
\exists\,T<\infty
\text{ such that the identification procedure is correct for every }
t\geq T
)
=
1.
$
Hence only finitely many rounds are incorrect almost surely. No independence assumption among the events \(E_t\), or among the identification failure events \(F_t\), is required.
\end{proof}

\subsection{Proof of Proposition~\ref{prop:raw_topological_holonomy_app}}

\begin{proof}
Fix a base observation $o\in\cO$. Every based edge loop in $|X_{\mathrm{und}}|$ can be represented by an oriented edge word
\[
w=e_1e_2\cdots e_k,
\]
where
$
\src(e_1)=o,
\qquad
\tgt(e_j)=\src(e_{j+1})
\quad
\text{for }j=1,\ldots,k-1,
\qquad
\tgt(e_k)=o.
$
The empty word $\varnothing_o$ represents the constant loop at $o$. Since every $\sigma_e$ is an element of $S_n$, the composition
\[
\sigma_w
=
\sigma_{e_k}\circ\cdots\circ\sigma_{e_1}
\]
is again an element of $S_n$, while the empty word has transport
\[
\sigma_{\varnothing_o}
=
\Id_{[n]}.
\]
We first prove that $\sigma_w$ is unchanged by the elementary reductions defining edge-path homotopy in the geometric realization of a graph. Suppose that an oriented edge word contains an adjacent backtracking pair and can be written as
\[
w=u\,e\,\bar e\,v,
\]
where $u$ and $v$ are possibly empty oriented edge words and all displayed concatenations are composable. Traversal occurs from left to right, so the transport of the complete word is obtained by composing the transports in the reverse written order. Hence
\[
\begin{aligned}
\sigma_w
&=
\sigma_v\circ\sigma_{\bar e}\circ\sigma_e\circ\sigma_u.
\end{aligned}
\]
Assumption~\ref{ass:inverse_consistency_app} gives
\[
\sigma_{\bar e}
=
\sigma_e^{-1},
\]
and therefore
\[
\begin{aligned}
\sigma_{\bar e}\circ\sigma_e
&=
\sigma_e^{-1}\circ\sigma_e\\
&=
\Id_{[n]}.
\end{aligned}
\]
Substituting this identity into the preceding expression yields
\[
\begin{aligned}
\sigma_w
&=
\sigma_v\circ
\bigl(\sigma_{\bar e}\circ\sigma_e\bigr)
\circ\sigma_u\\
&=
\sigma_v\circ\Id_{[n]}\circ\sigma_u\\
&=
\sigma_v\circ\sigma_u.
\end{aligned}
\]
The word obtained by deleting the adjacent pair $e\bar e$ is $uv$, and its transport is exactly
\[
\sigma_{uv}
=
\sigma_v\circ\sigma_u.
\]
Thus
\[
\sigma_{u e\bar e v}
=
\sigma_{uv}.
\]
The other possible immediate backtrack is handled in the same way. If
\[
w=u\,\bar e\,e\,v,
\]
then
\[
\begin{aligned}
\sigma_w
&=
\sigma_v\circ\sigma_e\circ\sigma_{\bar e}\circ\sigma_u\\
&=
\sigma_v\circ
\bigl(\sigma_e\circ\sigma_e^{-1}\bigr)
\circ\sigma_u\\
&=
\sigma_v\circ\Id_{[n]}\circ\sigma_u\\
&=
\sigma_{uv}.
\end{aligned}
\]
Consequently, insertion or deletion of either adjacent pair
\[
e\bar e
\qquad\text{or}\qquad
\bar e e
\]
does not change the raw transport.

For the geometric realization of a graph, based edge-path homotopy is generated by precisely these insertions and deletions of immediate backtracks, together with insertion or deletion of constant vertex pauses, which do not introduce an edge and hence have identity transport. Equivalently, two based oriented edge loops represent the same element of
$\pi_1(|X_{\mathrm{und}}|,o)$
if and only if one can pass from one edge word to the other by a finite sequence of these elementary moves. Let $w$ and $w'$ represent the same based homotopy class. Then there exist oriented edge words
\[
w=w^{(0)},w^{(1)},\ldots,w^{(L)}=w'
\]
such that, for every $\ell\in\{0,\ldots,L-1\}$, the word $w^{(\ell+1)}$ is obtained from $w^{(\ell)}$ by one elementary insertion or deletion. The calculation above gives
\[
\sigma_{w^{(\ell)}}
=
\sigma_{w^{(\ell+1)}}
\qquad
\text{for every }\ell\in\{0,\ldots,L-1\}.
\]
Applying transitivity of equality along this finite sequence gives
\[
\begin{aligned}
\sigma_w
&=
\sigma_{w^{(0)}}\\
&=
\sigma_{w^{(1)}}\\
&=\cdots\\
&=
\sigma_{w^{(L)}}\\
&=
\sigma_{w'}.
\end{aligned}
\]
Therefore the value assigned to $[w]$ depends only on the based homotopy class and not on the chosen oriented edge-word representative. Hence
\[
\Holo_o^{\mathrm{raw}}([w])
:=
\sigma_w
\]
defines a well-defined map
\[
\Holo_o^{\mathrm{raw}}:
\pi_1(|X_{\mathrm{und}}|,o)
\longrightarrow
S_n.
\]

We next verify compatibility with the group operations. Let $w_1$ and $w_2$ be based oriented edge loops at $o$, where $w_1$ is traversed first and $w_2$ is traversed second. Under the path-composition convention used throughout the paper, their product in the fundamental group is represented by
\[
w_2\circ w_1.
\]
If
\[
w_1=e_1\cdots e_k
\qquad\text{and}\qquad
w_2=f_1\cdots f_\ell,
\]
then the chronological edge word for $w_2\circ w_1$ is
\[
e_1\cdots e_k f_1\cdots f_\ell.
\]
Its raw transport is therefore
\[
\begin{aligned}
\sigma_{w_2\circ w_1}
&=
\sigma_{f_\ell}\circ\cdots\circ\sigma_{f_1}
\circ
\sigma_{e_k}\circ\cdots\circ\sigma_{e_1}\\
&=
\bigl(\sigma_{f_\ell}\circ\cdots\circ\sigma_{f_1}\bigr)
\circ
\bigl(\sigma_{e_k}\circ\cdots\circ\sigma_{e_1}\bigr)\\
&=
\sigma_{w_2}\circ\sigma_{w_1}.
\end{aligned}
\]
It follows that
\[
\begin{aligned}
\Holo_o^{\mathrm{raw}}
\bigl([w_2]\,[w_1]\bigr)
&=
\Holo_o^{\mathrm{raw}}
\bigl([w_2\circ w_1]\bigr)\\
&=
\sigma_{w_2\circ w_1}\\
&=
\sigma_{w_2}\circ\sigma_{w_1}\\
&=
\Holo_o^{\mathrm{raw}}([w_2])
\circ
\Holo_o^{\mathrm{raw}}([w_1]).
\end{aligned}
\]
Thus $\Holo_o^{\mathrm{raw}}$ preserves multiplication. It also preserves the identity, because the identity element of the fundamental group is represented by the constant loop and
\[
\begin{aligned}
\Holo_o^{\mathrm{raw}}([\varnothing_o])
&=
\sigma_{\varnothing_o}\\
&=
\Id_{[n]},
\end{aligned}
\]
which is the identity element of $S_n$.

For completeness, the inverse operation is also represented correctly. If
\[
w=e_1e_2\cdots e_k,
\]
then the reversed loop is represented by
\[
w^{-1}
=
\bar e_k\bar e_{k-1}\cdots\bar e_1.
\]
Its transport is
\[
\begin{aligned}
\sigma_{w^{-1}}
&=
\sigma_{\bar e_1}\circ\sigma_{\bar e_2}\circ\cdots\circ\sigma_{\bar e_k}\\
&=
\sigma_{e_1}^{-1}\circ\sigma_{e_2}^{-1}\circ\cdots\circ\sigma_{e_k}^{-1}\\
&=
\bigl(
\sigma_{e_k}\circ\sigma_{e_{k-1}}\circ\cdots\circ\sigma_{e_1}
\bigr)^{-1}\\
&=
\sigma_w^{-1}.
\end{aligned}
\]
Therefore
\[
\begin{aligned}
\Holo_o^{\mathrm{raw}}([w]^{-1})
&=
\Holo_o^{\mathrm{raw}}([w^{-1}])\\
&=
\sigma_{w^{-1}}\\
&=
\sigma_w^{-1}\\
&=
\Holo_o^{\mathrm{raw}}([w])^{-1}.
\end{aligned}
\]
Hence $\Holo_o^{\mathrm{raw}}$ is a group homomorphism.

Finally, define
\[
G_o^{\mathrm{raw}}
:=
\Holo_o^{\mathrm{raw}}
\bigl(\pi_1(|X_{\mathrm{und}}|,o)\bigr).
\]
This set is a subgroup of $S_n$. Indeed,
\[
\Id_{[n]}
=
\Holo_o^{\mathrm{raw}}([\varnothing_o])
\in
G_o^{\mathrm{raw}}.
\]
If $g_1,g_2\in G_o^{\mathrm{raw}}$, then there exist based loop classes $[w_1]$ and $[w_2]$ such that
\[
g_1
=
\Holo_o^{\mathrm{raw}}([w_1])
\qquad\text{and}\qquad
g_2
=
\Holo_o^{\mathrm{raw}}([w_2]).
\]
Using the homomorphism property,
\[
\begin{aligned}
g_2\circ g_1
&=
\Holo_o^{\mathrm{raw}}([w_2])
\circ
\Holo_o^{\mathrm{raw}}([w_1])\\
&=
\Holo_o^{\mathrm{raw}}([w_2]\,[w_1])\\
&\in
G_o^{\mathrm{raw}}.
\end{aligned}
\]
Likewise,
\[
\begin{aligned}
g_1^{-1}
&=
\Holo_o^{\mathrm{raw}}([w_1])^{-1}\\
&=
\Holo_o^{\mathrm{raw}}([w_1]^{-1})\\
&\in
G_o^{\mathrm{raw}}.
\end{aligned}
\]
Thus $G_o^{\mathrm{raw}}\leq S_n$, and by definition it is the raw topological holonomy group at the base observation $o$.
\end{proof}

\subsection{Proof of Lemma~\ref{lem:quotient_inverse_consistency_app}}

\begin{proof}
Fix an arbitrary directed edge $e\in E$, and write
\[
o:=\src(e)
\qquad\text{and}\qquad
o':=\tgt(e).
\]
By Assumption~\ref{ass:inverse_consistency_app}, the paired edge $\bar e$ also belongs to $E$ and satisfies
\[
\src(\bar e)=o',
\qquad
\tgt(\bar e)=o,
\qquad
\sigma_{\bar e}=\sigma_e^{-1}.
\]
Since both $e$ and $\bar e$ are edges of the directed support graph, both are feasible. Theorem~\ref{thm:markovization_main} therefore gives well-defined quotient transports
\[
\tau_e:\cC_o\longrightarrow\cC_{o'}
\qquad\text{and}\qquad
\tau_{\bar e}:\cC_{o'}\longrightarrow\cC_o
\]
satisfying
\[
\tau_e([i]_o)=[\sigma_e(i)]_{o'}
\qquad
\text{for every }i\in[n]
\]
and
\[
\tau_{\bar e}([j]_{o'})=[\sigma_{\bar e}(j)]_o
\qquad
\text{for every }j\in[n].
\]
In particular, the codomain of $\tau_e$ is the domain of $\tau_{\bar e}$, and the codomain of $\tau_{\bar e}$ is the domain of $\tau_e$, so both compositions appearing in the statement are well defined.

Let $c\in\cC_o$ be arbitrary. Because
\[
\cC_o=[n]/\Pi_o^\star,
\]
there exists a layer $i\in[n]$ such that
\[
c=[i]_o.
\]
Using first the quotient-transport formula for $e$ and then the corresponding formula for $\bar e$, we obtain
\[
\begin{aligned}
(\tau_{\bar e}\circ\tau_e)(c)
&=
(\tau_{\bar e}\circ\tau_e)([i]_o)\\
&=
\tau_{\bar e}\bigl(\tau_e([i]_o)\bigr)\\
&=
\tau_{\bar e}\bigl([\sigma_e(i)]_{o'}\bigr)\\
&=
[\sigma_{\bar e}(\sigma_e(i))]_o.
\end{aligned}
\]
Assumption~\ref{ass:inverse_consistency_app} gives
\[
\sigma_{\bar e}=\sigma_e^{-1},
\]
and hence
\[
\begin{aligned}
\sigma_{\bar e}(\sigma_e(i))
&=
\sigma_e^{-1}(\sigma_e(i))\\
&=
(\sigma_e^{-1}\circ\sigma_e)(i)\\
&=
\Id_{[n]}(i)\\
&=
i.
\end{aligned}
\]
Substituting this identity into the preceding quotient-class calculation yields
\[
\begin{aligned}
(\tau_{\bar e}\circ\tau_e)(c)
&=
[\sigma_{\bar e}(\sigma_e(i))]_o\\
&=
[i]_o\\
&=
c.
\end{aligned}
\]
Since $c\in\cC_o$ was arbitrary, the two maps agree at every element of their common domain, and therefore
\[
\tau_{\bar e}\circ\tau_e
=
\Id_{\cC_o}
=
\Id_{\cC_{\src(e)}}.
\]

We prove the reverse composition in the same explicit manner. Let $c'\in\cC_{o'}$ be arbitrary. Since
\[
\cC_{o'}=[n]/\Pi_{o'}^\star,
\]
there exists $j\in[n]$ such that
\[
c'=[j]_{o'}.
\]
Applying the quotient-transport formulas first to $\bar e$ and then to $e$ gives
\[
\begin{aligned}
(\tau_e\circ\tau_{\bar e})(c')
&=
(\tau_e\circ\tau_{\bar e})([j]_{o'})\\
&=
\tau_e\bigl(\tau_{\bar e}([j]_{o'})\bigr)\\
&=
\tau_e\bigl([\sigma_{\bar e}(j)]_o\bigr)\\
&=
[\sigma_e(\sigma_{\bar e}(j))]_{o'}.
\end{aligned}
\]
Again using $\sigma_{\bar e}=\sigma_e^{-1}$, we have
\[
\begin{aligned}
\sigma_e(\sigma_{\bar e}(j))
&=
\sigma_e(\sigma_e^{-1}(j))\\
&=
(\sigma_e\circ\sigma_e^{-1})(j)\\
&=
\Id_{[n]}(j)\\
&=
j.
\end{aligned}
\]
Consequently,
\[
\begin{aligned}
(\tau_e\circ\tau_{\bar e})(c')
&=
[\sigma_e(\sigma_{\bar e}(j))]_{o'}\\
&=
[j]_{o'}\\
&=
c'.
\end{aligned}
\]
Because $c'\in\cC_{o'}$ was arbitrary, it follows that
\[
\tau_e\circ\tau_{\bar e}
=
\Id_{\cC_{o'}}
=
\Id_{\cC_{\tgt(e)}}.
\]

The two identities show that $\tau_{\bar e}$ is simultaneously a left inverse and a right inverse of $\tau_e$. We now verify explicitly that this makes $\tau_e$ bijective. Suppose that $c_1,c_2\in\cC_o$ satisfy
\[
\tau_e(c_1)=\tau_e(c_2).
\]
Applying $\tau_{\bar e}$ to both sides and using the first composition identity gives
\[
\begin{aligned}
c_1
&=
(\tau_{\bar e}\circ\tau_e)(c_1)\\
&=
\tau_{\bar e}(\tau_e(c_1))\\
&=
\tau_{\bar e}(\tau_e(c_2))\\
&=
(\tau_{\bar e}\circ\tau_e)(c_2)\\
&=
c_2.
\end{aligned}
\]
Thus $\tau_e$ is injective. Next, let $c'\in\cC_{o'}$ be arbitrary. The second composition identity gives
\[
\begin{aligned}
c'
&=
(\tau_e\circ\tau_{\bar e})(c')\\
&=
\tau_e\bigl(\tau_{\bar e}(c')\bigr).
\end{aligned}
\]
Therefore $c'$ is the image under $\tau_e$ of the class
\[
\tau_{\bar e}(c')\in\cC_o,
\]
so $\tau_e$ is surjective. Hence $\tau_e$ is a bijection. Since $\tau_{\bar e}$ is both its left and right inverse, uniqueness of the inverse map yields
\[
\tau_{\bar e}
=
\tau_e^{-1}.
\]
The same argument, with $e$ and $\bar e$ interchanged, shows that $\tau_{\bar e}$ is also a bijection and that
\[
\tau_e
=
\tau_{\bar e}^{-1}.
\]
Every oriented traversal of an edge of $X_{\mathrm{und}}$ is represented by one member of a paired set
\[
\{e,\bar e\}.
\]
Since the quotient transport in either orientation is bijective, every quotient transport along an edge of $X_{\mathrm{und}}$ is a bijection.
\end{proof}

\subsection{Proof of Proposition~\ref{prop:quotient_topological_holonomy_app}}

\begin{proof}
Fix an arbitrary base observation $o\in\cO$. Let
\[
w=e_1e_2\cdots e_k
\]
be an oriented edge loop in $X_{\mathrm{und}}$ based at $o$. Thus there exist observations
\[
o_0,o_1,\ldots,o_k
\]
with
\[
o_0=o_k=o,
\]
such that, for every $j\in\{1,\ldots,k\}$,
\[
\src(e_j)=o_{j-1}
\qquad\text{and}\qquad
\tgt(e_j)=o_j.
\]
The associated quotient transport is
\[
\tau_w
=
\tau_{e_k}\circ\tau_{e_{k-1}}\circ\cdots\circ\tau_{e_1}
:
\cC_o\longrightarrow\cC_o.
\]
By Lemma~\ref{lem:quotient_inverse_consistency_app}, every edge transport
\[
\tau_{e_j}:\cC_{o_{j-1}}\longrightarrow\cC_{o_j}
\]
is a bijection, with inverse
\[
\tau_{e_j}^{-1}
=
\tau_{\bar e_j}.
\]
A finite composition of bijections is a bijection. Hence $\tau_w$ is a bijection from $\cC_o$ to itself, and therefore
\[
\tau_w\in\Sym(\cC_o).
\]
For the empty loop $\varnothing_o$, we use the empty-composition convention
\[
\tau_{\varnothing_o}
=
\Id_{\cC_o}.
\]
Thus the proposed assignment takes values in $\Sym(\cC_o)$ for every based oriented edge loop, including the empty loop.

We next prove that the transport is invariant under the elementary edge-word reductions representing based homotopy in the geometric realization of the graph. Suppose first that a based oriented edge loop contains an immediate backtrack and can be written as
\[
w=u\,e\,\bar e\,v.
\]
Let
\[
p:=\src(e)
\qquad\text{and}\qquad
q:=\tgt(e).
\]
Then $u$ is a path from $o$ to $p$, the edge $e$ is traversed from $p$ to $q$, the paired edge $\bar e$ is traversed from $q$ back to $p$, and $v$ is a path from $p$ to $o$. Consequently, the corresponding quotient transports have the compatible types
\[
\begin{aligned}
\tau_u&:\cC_o\to\cC_p,
&
\tau_e&:\cC_p\to\cC_q,\\
\tau_{\bar e}&:\cC_q\to\cC_p,
&
\tau_v&:\cC_p\to\cC_o.
\end{aligned}
\]
Because the edges are executed in the written order while maps compose from right to left, the transport of $w$ is
\[
\begin{aligned}
\tau_w
&=
\tau_v\circ\tau_{\bar e}\circ\tau_e\circ\tau_u.
\end{aligned}
\]
Lemma~\ref{lem:quotient_inverse_consistency_app} gives
\[
\tau_{\bar e}\circ\tau_e
=
\Id_{\cC_p}.
\]
Substituting this identity into the preceding expression yields
\[
\begin{aligned}
\tau_w
&=
\tau_v\circ
\bigl(\tau_{\bar e}\circ\tau_e\bigr)
\circ\tau_u\\
&=
\tau_v\circ\Id_{\cC_p}\circ\tau_u\\
&=
\tau_v\circ\tau_u.
\end{aligned}
\]
Deleting the adjacent pair $e\bar e$ produces the based loop $uv$. Since $u$ is executed first and $v$ is executed second, its transport is
\[
\tau_{uv}
=
\tau_v\circ\tau_u.
\]
Therefore
\[
\tau_{u e\bar e v}
=
\tau_{uv}.
\]

The other orientation of immediate backtracking is handled similarly. Suppose
\[
w=u\,\bar e\,e\,v.
\]
In this case, $u$ ends at $q=\tgt(e)$, the edge $\bar e$ moves from $q$ to $p=\src(e)$, the edge $e$ returns from $p$ to $q$, and $v$ begins at $q$. Thus
\[
\tau_w
=
\tau_v\circ\tau_e\circ\tau_{\bar e}\circ\tau_u.
\]
The second identity in Lemma~\ref{lem:quotient_inverse_consistency_app} gives
\[
\tau_e\circ\tau_{\bar e}
=
\Id_{\cC_q}.
\]
It follows that
\[
\begin{aligned}
\tau_w
&=
\tau_v\circ
\bigl(\tau_e\circ\tau_{\bar e}\bigr)
\circ\tau_u\\
&=
\tau_v\circ\Id_{\cC_q}\circ\tau_u\\
&=
\tau_v\circ\tau_u\\
&=
\tau_{uv}.
\end{aligned}
\]
Hence insertion or deletion of either adjacent pair
\[
e\bar e
\qquad\text{or}\qquad
\bar e e
\]
does not change the quotient transport. Insertion or deletion of a constant pause at a vertex also does not change the transport, because the empty path at an observation $p$ acts by
\[
\tau_{\varnothing_p}
=
\Id_{\cC_p}.
\]

For the geometric realization of a graph, based edge-path homotopy is generated by these elementary insertions and deletions of immediate backtracks, together with constant vertex pauses. Equivalently, two based oriented edge loops represent the same element of
$\pi_1(|X_{\mathrm{und}}|,o)$
precisely when their edge words are connected by a finite sequence of such elementary moves. Let $w$ and $w'$ represent the same based homotopy class. Then there exist based oriented edge loops
\[
w=w^{(0)},w^{(1)},\ldots,w^{(L)}=w'
\]
such that each $w^{(\ell+1)}$ is obtained from $w^{(\ell)}$ by one elementary insertion or deletion. The calculations above imply
\[
\tau_{w^{(\ell)}}
=
\tau_{w^{(\ell+1)}}
\qquad
\text{for every }\ell\in\{0,\ldots,L-1\}.
\]
Therefore
\[
\begin{aligned}
\tau_w
&=
\tau_{w^{(0)}}\\
&=
\tau_{w^{(1)}}\\
&=\cdots\\
&=
\tau_{w^{(L)}}\\
&=
\tau_{w'}.
\end{aligned}
\]
Thus the value $\tau_w$ depends only on the based homotopy class $[w]$, not on the chosen oriented edge-word representative. Consequently,
\[
\Holo_o^{\mathrm{quot}}([w])
:=
\tau_w
\]
defines a well-defined map
\[
\Holo_o^{\mathrm{quot}}:
\pi_1(|X_{\mathrm{und}}|,o)
\longrightarrow
\Sym(\cC_o).
\]

We now verify that this map preserves the group operation. Let $w_1$ and $w_2$ be arbitrary oriented edge loops based at $o$, with $w_1$ executed first and $w_2$ executed second. Write
\[
w_1=e_1e_2\cdots e_k
\qquad\text{and}\qquad
w_2=f_1f_2\cdots f_\ell.
\]
Under the path-composition convention used in Definition~\ref{def:directed_transport_main}, the concatenated loop is denoted by
\[
w_2\circ w_1,
\]
and its chronological edge word is
\[
e_1e_2\cdots e_k f_1f_2\cdots f_\ell.
\]
Applying the definition of quotient path transport gives
\[
\begin{aligned}
\tau_{w_2\circ w_1}
&=
\tau_{f_\ell}\circ\cdots\circ\tau_{f_1}
\circ
\tau_{e_k}\circ\cdots\circ\tau_{e_1}\\
&=
\bigl(
\tau_{f_\ell}\circ\cdots\circ\tau_{f_1}
\bigr)
\circ
\bigl(
\tau_{e_k}\circ\cdots\circ\tau_{e_1}
\bigr)\\
&=
\tau_{w_2}\circ\tau_{w_1}.
\end{aligned}
\]
The multiplication in the fundamental group follows the same convention, so the class product of $[w_1]$ followed by $[w_2]$ is represented by $w_2\circ w_1$. Hence
\[
\begin{aligned}
\Holo_o^{\mathrm{quot}}
\bigl([w_2][w_1]\bigr)
&=
\Holo_o^{\mathrm{quot}}
\bigl([w_2\circ w_1]\bigr)\\
&=
\tau_{w_2\circ w_1}\\
&=
\tau_{w_2}\circ\tau_{w_1}\\
&=
\Holo_o^{\mathrm{quot}}([w_2])
\circ
\Holo_o^{\mathrm{quot}}([w_1]).
\end{aligned}
\]
Thus $\Holo_o^{\mathrm{quot}}$ preserves multiplication. It also preserves the identity element, because the identity of
$\pi_1(|X_{\mathrm{und}}|,o)$
is represented by the constant loop $\varnothing_o$, and
\[
\begin{aligned}
\Holo_o^{\mathrm{quot}}([\varnothing_o])
&=
\tau_{\varnothing_o}\\
&=
\Id_{\cC_o}.
\end{aligned}
\]

For completeness, we also verify explicitly that loop reversal is sent to the inverse permutation. Let
\[
w=e_1e_2\cdots e_k.
\]
The reversed loop is represented by
\[
w^{-1}
=
\bar e_k\bar e_{k-1}\cdots\bar e_1.
\]
By the definition of transport along the reversed edge word,
\[
\begin{aligned}
\tau_{w^{-1}}
&=
\tau_{\bar e_1}\circ\tau_{\bar e_2}
\circ\cdots\circ\tau_{\bar e_k}.
\end{aligned}
\]
Using Lemma~\ref{lem:quotient_inverse_consistency_app} at every edge gives
\[
\begin{aligned}
\tau_{w^{-1}}
&=
\tau_{e_1}^{-1}\circ\tau_{e_2}^{-1}
\circ\cdots\circ\tau_{e_k}^{-1}\\
&=
\bigl(
\tau_{e_k}\circ\tau_{e_{k-1}}
\circ\cdots\circ\tau_{e_1}
\bigr)^{-1}\\
&=
\tau_w^{-1}.
\end{aligned}
\]
Therefore
\[
\begin{aligned}
\Holo_o^{\mathrm{quot}}([w]^{-1})
&=
\Holo_o^{\mathrm{quot}}([w^{-1}])\\
&=
\tau_{w^{-1}}\\
&=
\tau_w^{-1}\\
&=
\Holo_o^{\mathrm{quot}}([w])^{-1}.
\end{aligned}
\]
This is consistent with, and also follows from, the multiplication and identity identities above. Hence
\[
\Holo_o^{\mathrm{quot}}
\]
is a group homomorphism.

Finally, define
\[
G_o^{\mathrm{quot}}
:=
\Holo_o^{\mathrm{quot}}
\bigl(
\pi_1(|X_{\mathrm{und}}|,o)
\bigr).
\]
Every element of this image lies in $\Sym(\cC_o)$. Moreover,
\[
\Id_{\cC_o}
=
\Holo_o^{\mathrm{quot}}([\varnothing_o])
\in
G_o^{\mathrm{quot}}.
\]
Let $g_1,g_2\in G_o^{\mathrm{quot}}$. By the definition of the image, there exist based loop classes $[w_1]$ and $[w_2]$ such that
\[
g_1
=
\Holo_o^{\mathrm{quot}}([w_1])
\qquad\text{and}\qquad
g_2
=
\Holo_o^{\mathrm{quot}}([w_2]).
\]
The homomorphism property gives
\[
\begin{aligned}
g_2\circ g_1
&=
\Holo_o^{\mathrm{quot}}([w_2])
\circ
\Holo_o^{\mathrm{quot}}([w_1])\\
&=
\Holo_o^{\mathrm{quot}}
\bigl([w_2][w_1]\bigr)\\
&\in
G_o^{\mathrm{quot}}.
\end{aligned}
\]
Likewise,
\[
\begin{aligned}
g_1^{-1}
&=
\Holo_o^{\mathrm{quot}}([w_1])^{-1}\\
&=
\Holo_o^{\mathrm{quot}}([w_1]^{-1})\\
&\in
G_o^{\mathrm{quot}}.
\end{aligned}
\]
Thus the image contains the identity and is closed under composition and inverses. Therefore
\[
G_o^{\mathrm{quot}}
\leq
\Sym(\cC_o),
\]
and, by definition, it is the quotient topological holonomy group at the base observation $o$.
\end{proof}

\subsection{Proof of Proposition~\ref{prop:topological_basepoint_change_app}}

\begin{proof}
Fix an oriented path $\eta$ from $o$ to $o'$. Write
\[
\eta=e_1e_2\cdots e_k,
\]
where there are observations
\[
o=o_0,o_1,\ldots,o_k=o'
\]
such that
\[
\src(e_j)=o_{j-1}
\qquad\text{and}\qquad
\tgt(e_j)=o_j
\]
for every $j\in\{1,\ldots,k\}$. Under Assumption~\ref{ass:inverse_consistency_app}, the reversed oriented path is
\[
\eta^{-1}
=
\bar e_k\bar e_{k-1}\cdots\bar e_1,
\]
which begins at $o'$ and ends at $o$. By the definition of raw path transport,
\[
\sigma_\eta
=
\sigma_{e_k}\circ\sigma_{e_{k-1}}\circ\cdots\circ\sigma_{e_1}.
\]
The transport along the reversed path is
\[
\begin{aligned}
\sigma_{\eta^{-1}}
&=
\sigma_{\bar e_1}\circ\sigma_{\bar e_2}\circ\cdots\circ\sigma_{\bar e_k}\\
&=
\sigma_{e_1}^{-1}\circ\sigma_{e_2}^{-1}\circ\cdots\circ\sigma_{e_k}^{-1}\\
&=
\bigl(
\sigma_{e_k}\circ\sigma_{e_{k-1}}\circ\cdots\circ\sigma_{e_1}
\bigr)^{-1}\\
&=
\sigma_\eta^{-1},
\end{aligned}
\]
where the second equality follows from
\[
\sigma_{\bar e_j}=\sigma_{e_j}^{-1}
\]
for every $j\in\{1,\ldots,k\}$. Thus the transport along the reversed reference path is exactly the inverse of the transport along the reference path.

The same conclusion holds on the stable quotient. By Lemma~\ref{lem:quotient_inverse_consistency_app}, every edge transport
\[
\tau_{e_j}:\cC_{o_{j-1}}\longrightarrow\cC_{o_j}
\]
is a bijection and satisfies
\[
\tau_{\bar e_j}=\tau_{e_j}^{-1}.
\]
Consequently,
\[
\tau_\eta
=
\tau_{e_k}\circ\tau_{e_{k-1}}\circ\cdots\circ\tau_{e_1}
:
\cC_o\longrightarrow\cC_{o'}
\]
is a bijection. Moreover,
\[
\begin{aligned}
\tau_{\eta^{-1}}
&=
\tau_{\bar e_1}\circ\tau_{\bar e_2}\circ\cdots\circ\tau_{\bar e_k}\\
&=
\tau_{e_1}^{-1}\circ\tau_{e_2}^{-1}\circ\cdots\circ\tau_{e_k}^{-1}\\
&=
\bigl(
\tau_{e_k}\circ\tau_{e_{k-1}}\circ\cdots\circ\tau_{e_1}
\bigr)^{-1}\\
&=
\tau_\eta^{-1}.
\end{aligned}
\]
In particular, conjugation by $\tau_\eta$ maps permutations of $\cC_o$ to permutations of $\cC_{o'}$.

We first prove the raw-layer identity. We use the notation
\[
\sigma_\eta
G_o^{\mathrm{raw}}
\sigma_\eta^{-1}
:=
\left\{
\sigma_\eta\circ g\circ\sigma_\eta^{-1}
:
g\in G_o^{\mathrm{raw}}
\right\}.
\]
Let
\[
g\in G_o^{\mathrm{raw}}
\]
be arbitrary. By the definition of the raw topological holonomy group, there exists a based loop class
\[
[w]\in\pi_1(|X_{\mathrm{und}}|,o)
\]
with an oriented edge-loop representative $w$ such that
\[
g
=
\Holo_o^{\mathrm{raw}}([w])
=
\sigma_w.
\]
Consider the oriented path
\[
\lambda
:=
\eta\circ w\circ\eta^{-1}.
\]
Under the path-composition convention of Definition~\ref{def:directed_transport_main}, the rightmost path is executed first. Hence $\lambda$ first follows $\eta^{-1}$ from $o'$ to $o$, then traverses the loop $w$ based at $o$, and finally follows $\eta$ from $o$ back to $o'$. Therefore $\lambda$ is a loop based at $o'$, and
\[
[\lambda]\in\pi_1(|X_{\mathrm{und}}|,o').
\]
Repeated use of the path-composition identity gives
\[
\begin{aligned}
\sigma_\lambda
&=
\sigma_{\eta\circ w\circ\eta^{-1}}\\
&=
\sigma_\eta\circ\sigma_{w\circ\eta^{-1}}\\
&=
\sigma_\eta\circ\sigma_w\circ\sigma_{\eta^{-1}}\\
&=
\sigma_\eta\circ g\circ\sigma_\eta^{-1}.
\end{aligned}
\]
Since $\lambda$ is a loop based at $o'$, its transport belongs to the raw holonomy group at $o'$. More explicitly,
\[
\begin{aligned}
\sigma_\eta\circ g\circ\sigma_\eta^{-1}
&=
\sigma_\lambda\\
&=
\Holo_{o'}^{\mathrm{raw}}([\lambda])\\
&\in
G_{o'}^{\mathrm{raw}}.
\end{aligned}
\]
Since $g\in G_o^{\mathrm{raw}}$ was arbitrary, this proves
\[
\sigma_\eta
G_o^{\mathrm{raw}}
\sigma_\eta^{-1}
\subseteq
G_{o'}^{\mathrm{raw}}.
\]

For the reverse inclusion, let
\[
g'\in G_{o'}^{\mathrm{raw}}
\]
be arbitrary. There exists a based loop class
\[
[w']\in\pi_1(|X_{\mathrm{und}}|,o')
\]
with an oriented edge-loop representative $w'$ satisfying
\[
g'
=
\Holo_{o'}^{\mathrm{raw}}([w'])
=
\sigma_{w'}.
\]
Define
\[
\kappa
:=
\eta^{-1}\circ w'\circ\eta.
\]
This path first follows $\eta$ from $o$ to $o'$, then traverses $w'$ at $o'$, and finally follows $\eta^{-1}$ from $o'$ back to $o$. Hence $\kappa$ is a loop based at $o$. Its raw transport is
\[
\begin{aligned}
\sigma_\kappa
&=
\sigma_{\eta^{-1}\circ w'\circ\eta}\\
&=
\sigma_{\eta^{-1}}\circ\sigma_{w'}\circ\sigma_\eta\\
&=
\sigma_\eta^{-1}\circ g'\circ\sigma_\eta.
\end{aligned}
\]
Because $\kappa$ is a loop based at $o$,
\[
\sigma_\kappa
\in
G_o^{\mathrm{raw}}.
\]
Set
\[
g
:=
\sigma_\kappa
=
\sigma_\eta^{-1}\circ g'\circ\sigma_\eta.
\]
Then $g\in G_o^{\mathrm{raw}}$. Conjugating this identity by $\sigma_\eta$ gives
\[
\begin{aligned}
\sigma_\eta\circ g\circ\sigma_\eta^{-1}
&=
\sigma_\eta\circ
\bigl(
\sigma_\eta^{-1}\circ g'\circ\sigma_\eta
\bigr)
\circ\sigma_\eta^{-1}\\
&=
\bigl(
\sigma_\eta\circ\sigma_\eta^{-1}
\bigr)
\circ g'\circ
\bigl(
\sigma_\eta\circ\sigma_\eta^{-1}
\bigr)\\
&=
\Id_{[n]}\circ g'\circ\Id_{[n]}\\
&=
g'.
\end{aligned}
\]
Therefore
\[
g'
\in
\sigma_\eta
G_o^{\mathrm{raw}}
\sigma_\eta^{-1}.
\]
Since $g'\in G_{o'}^{\mathrm{raw}}$ was arbitrary, it follows that
\[
G_{o'}^{\mathrm{raw}}
\subseteq
\sigma_\eta
G_o^{\mathrm{raw}}
\sigma_\eta^{-1}.
\]
Combining the two inclusions yields
\[
G_{o'}^{\mathrm{raw}}
=
\sigma_\eta
G_o^{\mathrm{raw}}
\sigma_\eta^{-1}.
\]

We now prove the corresponding quotient-layer identity. Define
\[
\tau_\eta
G_o^{\mathrm{quot}}
\tau_\eta^{-1}
:=
\left\{
\tau_\eta\circ q\circ\tau_\eta^{-1}
:
q\in G_o^{\mathrm{quot}}
\right\}.
\]
The domains and codomains in this composition are compatible. Indeed,
\[
\tau_\eta^{-1}:\cC_{o'}\longrightarrow\cC_o,
\]
every
\[
q\in G_o^{\mathrm{quot}}
\]
is a permutation from $\cC_o$ to itself, and
\[
\tau_\eta:\cC_o\longrightarrow\cC_{o'}.
\]
Thus
\[
\tau_\eta\circ q\circ\tau_\eta^{-1}
\]
is a permutation of $\cC_{o'}$.

Let
\[
q\in G_o^{\mathrm{quot}}
\]
be arbitrary. By the definition of the quotient topological holonomy group, there exists a based loop class
\[
[w]\in\pi_1(|X_{\mathrm{und}}|,o)
\]
with an oriented edge-loop representative $w$ such that
\[
q
=
\Holo_o^{\mathrm{quot}}([w])
=
\tau_w.
\]
Using the loop
\[
\lambda
=
\eta\circ w\circ\eta^{-1}
\]
based at $o'$, quotient path composition gives
\[
\begin{aligned}
\tau_\lambda
&=
\tau_{\eta\circ w\circ\eta^{-1}}\\
&=
\tau_\eta\circ\tau_{w\circ\eta^{-1}}\\
&=
\tau_\eta\circ\tau_w\circ\tau_{\eta^{-1}}\\
&=
\tau_\eta\circ q\circ\tau_\eta^{-1}.
\end{aligned}
\]
Since $\lambda$ is a loop based at $o'$,
\[
\begin{aligned}
\tau_\eta\circ q\circ\tau_\eta^{-1}
&=
\tau_\lambda\\
&=
\Holo_{o'}^{\mathrm{quot}}([\lambda])\\
&\in
G_{o'}^{\mathrm{quot}}.
\end{aligned}
\]
As $q\in G_o^{\mathrm{quot}}$ was arbitrary, this proves
\[
\tau_\eta
G_o^{\mathrm{quot}}
\tau_\eta^{-1}
\subseteq
G_{o'}^{\mathrm{quot}}.
\]

Conversely, let
\[
q'\in G_{o'}^{\mathrm{quot}}
\]
be arbitrary. There exists a based loop class
\[
[w']\in\pi_1(|X_{\mathrm{und}}|,o')
\]
with an oriented edge-loop representative $w'$ such that
\[
q'
=
\Holo_{o'}^{\mathrm{quot}}([w'])
=
\tau_{w'}.
\]
For the loop
\[
\kappa
=
\eta^{-1}\circ w'\circ\eta
\]
based at $o$, quotient path composition gives
\[
\begin{aligned}
\tau_\kappa
&=
\tau_{\eta^{-1}\circ w'\circ\eta}\\
&=
\tau_{\eta^{-1}}\circ\tau_{w'}\circ\tau_\eta\\
&=
\tau_\eta^{-1}\circ q'\circ\tau_\eta.
\end{aligned}
\]
Since $\kappa$ is a loop based at $o$,
\[
\tau_\kappa
\in
G_o^{\mathrm{quot}}.
\]
Define
\[
q
:=
\tau_\kappa
=
\tau_\eta^{-1}\circ q'\circ\tau_\eta.
\]
Then $q\in G_o^{\mathrm{quot}}$, and
\[
\begin{aligned}
\tau_\eta\circ q\circ\tau_\eta^{-1}
&=
\tau_\eta\circ
\bigl(
\tau_\eta^{-1}\circ q'\circ\tau_\eta
\bigr)
\circ\tau_\eta^{-1}\\
&=
\bigl(
\tau_\eta\circ\tau_\eta^{-1}
\bigr)
\circ q'\circ
\bigl(
\tau_\eta\circ\tau_\eta^{-1}
\bigr)\\
&=
\Id_{\cC_{o'}}\circ q'\circ\Id_{\cC_{o'}}\\
&=
q'.
\end{aligned}
\]
Therefore
\[
q'
\in
\tau_\eta
G_o^{\mathrm{quot}}
\tau_\eta^{-1}.
\]
Since $q'\in G_{o'}^{\mathrm{quot}}$ was arbitrary, we obtain
\[
G_{o'}^{\mathrm{quot}}
\subseteq
\tau_\eta
G_o^{\mathrm{quot}}
\tau_\eta^{-1}.
\]
Combining the two inclusions gives
\[
G_{o'}^{\mathrm{quot}}
=
\tau_\eta
G_o^{\mathrm{quot}}
\tau_\eta^{-1}.
\]
Therefore both the raw and stable-quotient topological holonomy groups are transported from one base observation to another by conjugation through the corresponding reference-path transport. In particular, changing the base observation preserves the holonomy structure up to the relabeling induced by the chosen path $\eta$.
\end{proof}

\subsection{Proof of Proposition~\ref{prop:sync_tree_main}}

\begin{proof}
Work in the observation-wise local class coordinates after exact recovery of the selected tree-edge maps. Thus, for each observation $o$ in the chosen connected component, the recovered class set is $\widehat{\cC}_o$, and, for every selected feasible directed tree edge
\[
e=(p\xrightarrow{a}q),
\]
the recovered map satisfies
\[
\widehat\tau_e
=
\tau_e^\lambda
=
\lambda_q^{-1}\circ\tau_e\circ\lambda_p
:
\widehat{\cC}_p
\longrightarrow
\widehat{\cC}_q.
\]
The maps $\lambda_p$ and $\lambda_q$ are bijections, and the quotient transport $\tau_e$ is bijective by the hypothesis on the tree scaffold. Hence $\widehat\tau_e$ is also bijective. Indeed, the map
\[
\lambda_p^{-1}\circ\tau_e^{-1}\circ\lambda_q
:
\widehat{\cC}_q
\longrightarrow
\widehat{\cC}_p
\]
is a two-sided inverse, because
\[
\begin{aligned}
&\bigl(
\lambda_p^{-1}\circ\tau_e^{-1}\circ\lambda_q
\bigr)
\circ
\bigl(
\lambda_q^{-1}\circ\tau_e\circ\lambda_p
\bigr)\\
&\qquad=
\lambda_p^{-1}\circ\tau_e^{-1}
\circ
\bigl(
\lambda_q\circ\lambda_q^{-1}
\bigr)
\circ\tau_e\circ\lambda_p\\
&\qquad=
\lambda_p^{-1}\circ\tau_e^{-1}
\circ\tau_e\circ\lambda_p\\
&\qquad=
\lambda_p^{-1}\circ\Id_{\cC_p}\circ\lambda_p\\
&\qquad=
\Id_{\widehat{\cC}_p},
\end{aligned}
\]
and similarly
\[
\begin{aligned}
&\bigl(
\lambda_q^{-1}\circ\tau_e\circ\lambda_p
\bigr)
\circ
\bigl(
\lambda_p^{-1}\circ\tau_e^{-1}\circ\lambda_q
\bigr)\\
&\qquad=
\lambda_q^{-1}\circ\tau_e
\circ
\bigl(
\lambda_p\circ\lambda_p^{-1}
\bigr)
\circ\tau_e^{-1}\circ\lambda_q\\
&\qquad=
\lambda_q^{-1}\circ\tau_e
\circ\tau_e^{-1}\circ\lambda_q\\
&\qquad=
\lambda_q^{-1}\circ\Id_{\cC_q}\circ\lambda_q\\
&\qquad=
\Id_{\widehat{\cC}_q}.
\end{aligned}
\]
Consequently, the inverse of the recovered map is well defined and is given by
\[
\widehat\tau_e^{-1}
=
\lambda_p^{-1}\circ\tau_e^{-1}\circ\lambda_q.
\]
The synchronization construction below uses only the recovered maps and their inverses; the unknown proof-level gauges $\lambda_o$ are not required by the procedure itself.

For each undirected tree adjacency $\{p,q\}\in E(T)$, let $e_{\{p,q\}}$ denote the directed edge selected in the hypothesis. Define the transport used for an ordered traversal from $p$ to $q$ by
\[
\Psi_{p\to q}
:=
\begin{cases}
\widehat\tau_{e_{\{p,q\}}},
&
\text{if }
\src(e_{\{p,q\}})=p
\text{ and }
\tgt(e_{\{p,q\}})=q,
\\[1ex]
\widehat\tau_{e_{\{p,q\}}}^{-1},
&
\text{if }
\src(e_{\{p,q\}})=q
\text{ and }
\tgt(e_{\{p,q\}})=p.
\end{cases}
\]
In either case,
\[
\Psi_{p\to q}
:
\widehat{\cC}_p
\longrightarrow
\widehat{\cC}_q
\]
is a bijection. Reversing the ordered traversal exchanges the two cases in the definition, and therefore
\[
\Psi_{q\to p}
=
\Psi_{p\to q}^{-1}.
\]
In particular, adjacent observations in $T$ have class sets of equal cardinality. Since $T$ is connected, repeated application along a tree path shows that
\[
|\widehat{\cC}_o|
=
|\widehat{\cC}_{o_{\mathrm{root}}}|
\qquad
\text{for every }o\in V(T).
\]

Fix an arbitrary observation $o\in V(T)$. Because $T$ is a tree, there exists a unique simple path from the root to $o$. Write its ordered vertex sequence as
\[
o_0=o_{\mathrm{root}},
\quad
o_1,
\quad
\ldots,
\quad
o_k=o.
\]
Define the corresponding root-to-$o$ synchronization map by
\[
\Psi_o
:=
\Psi_{o_{k-1}\to o_k}
\circ
\Psi_{o_{k-2}\to o_{k-1}}
\circ\cdots\circ
\Psi_{o_0\to o_1}
:
\widehat{\cC}_{o_{\mathrm{root}}}
\longrightarrow
\widehat{\cC}_o.
\]
For the root itself, define
\[
\Psi_{o_{\mathrm{root}}}
:=
\Id_{\widehat{\cC}_{o_{\mathrm{root}}}}.
\]
Every factor in the composition defining $\Psi_o$ is a bijection, so $\Psi_o$ is a bijection. Its inverse is obtained by traversing the unique tree path in the reverse direction:
\[
\Psi_o^{-1}
=
\Psi_{o_1\to o_0}
\circ
\Psi_{o_2\to o_1}
\circ\cdots\circ
\Psi_{o_k\to o_{k-1}}.
\]
Indeed, using
\[
\Psi_{o_j\to o_{j-1}}
=
\Psi_{o_{j-1}\to o_j}^{-1}
\]
and successively cancelling adjacent inverse pairs gives
\[
\begin{aligned}
\Psi_o^{-1}\circ\Psi_o
&=
\Psi_{o_1\to o_0}
\circ\cdots\circ
\Psi_{o_k\to o_{k-1}}\\
&\circ
\Psi_{o_{k-1}\to o_k}
\circ\cdots\circ
\Psi_{o_0\to o_1}\\
&=
\Psi_{o_1\to o_0}
\circ\cdots\circ
\Psi_{o_{k-1}\to o_{k-2}}\\
&\circ
\Id_{\widehat{\cC}_{o_{k-1}}}
\circ
\Psi_{o_{k-2}\to o_{k-1}}
\circ\cdots\circ
\Psi_{o_0\to o_1}\\
&=
\cdots\\
&=
\Psi_{o_1\to o_0}
\circ
\Psi_{o_0\to o_1}\\
&=
\Id_{\widehat{\cC}_{o_{\mathrm{root}}}}.
\end{aligned}
\]
The reverse composition satisfies
\[
\Psi_o\circ\Psi_o^{-1}
=
\Id_{\widehat{\cC}_o}
\]
by the same cancellation argument. The uniqueness of the simple root-to-$o$ path is precisely what makes $\Psi_o$ unambiguous; no choice among multiple tree paths is required.

Let $\mathcal L$ be any finite label set satisfying
\[
|\mathcal L|
=
|\widehat{\cC}_{o_{\mathrm{root}}}|,
\]
and choose an arbitrary root labeling, that is, an arbitrary bijection
\[
\ell_{o_{\mathrm{root}}}
:
\widehat{\cC}_{o_{\mathrm{root}}}
\longrightarrow
\mathcal L.
\]
For every $o\in V(T)$, define
\[
\ell_o
:=
\ell_{o_{\mathrm{root}}}
\circ
\Psi_o^{-1}
:
\widehat{\cC}_o
\longrightarrow
\mathcal L.
\]
Since both factors are bijections, each $\ell_o$ is a bijection. This family is the desired root-relative coordinate convention. Explicitly, if a root class $\hat c_{\mathrm{root}}$ is transported along the unique tree path to
\[
\hat c_o
=
\Psi_o(\hat c_{\mathrm{root}}),
\]
then its label at $o$ is
\[
\begin{aligned}
\ell_o(\hat c_o)
&=
\ell_{o_{\mathrm{root}}}
\bigl(
\Psi_o^{-1}(\hat c_o)
\bigr)\\
&=
\ell_{o_{\mathrm{root}}}
\bigl(
\Psi_o^{-1}(\Psi_o(\hat c_{\mathrm{root}}))
\bigr)\\
&=
\ell_{o_{\mathrm{root}}}
\bigl(
\Id_{\widehat{\cC}_{o_{\mathrm{root}}}}
(\hat c_{\mathrm{root}})
\bigr)\\
&=
\ell_{o_{\mathrm{root}}}(\hat c_{\mathrm{root}}).
\end{aligned}
\]
Thus the root label is preserved when it is propagated along the tree.

The same statement can be written locally on each rooted tree edge. Let $p$ be the parent of $q$ in the tree rooted at $o_{\mathrm{root}}$. The unique root-to-$q$ path is obtained by appending the traversal $p\to q$ to the unique root-to-$p$ path, and hence
\[
\Psi_q
=
\Psi_{p\to q}\circ\Psi_p.
\]
Taking inverses gives
\[
\begin{aligned}
\Psi_q^{-1}
&=
\bigl(
\Psi_{p\to q}\circ\Psi_p
\bigr)^{-1}\\
&=
\Psi_p^{-1}
\circ
\Psi_{p\to q}^{-1}.
\end{aligned}
\]
Therefore
\[
\begin{aligned}
\ell_q\circ\Psi_{p\to q}
&=
\ell_{o_{\mathrm{root}}}
\circ\Psi_q^{-1}
\circ\Psi_{p\to q}\\
&=
\ell_{o_{\mathrm{root}}}
\circ\Psi_p^{-1}
\circ\Psi_{p\to q}^{-1}
\circ\Psi_{p\to q}\\
&=
\ell_{o_{\mathrm{root}}}
\circ\Psi_p^{-1}
\circ
\Id_{\widehat{\cC}_p}\\
&=
\ell_{o_{\mathrm{root}}}
\circ\Psi_p^{-1}\\
&=
\ell_p.
\end{aligned}
\]
Equivalently, for every $\hat c\in\widehat{\cC}_p$,
\[
\ell_q\bigl(\Psi_{p\to q}(\hat c)\bigr)
=
\ell_p(\hat c).
\]
Consequently, if the selected directed edge points from $p$ to $q$, labels are propagated through the recovered bijection
\(\widehat\tau_{e_{\{p,q\}}}\); if it points from $q$ to $p$, labels are propagated from $p$ to $q$ through
\(\widehat\tau_{e_{\{p,q\}}}^{-1}\). This is exactly the propagation rule stated in the proposition.

We now prove uniqueness. Suppose that
\[
\left\{
\widetilde\ell_o:
\widehat{\cC}_o\to\mathcal L
\right\}_{o\in V(T)}
\]
is another family of coordinate maps satisfying the same root labeling
\[
\widetilde\ell_{o_{\mathrm{root}}}
=
\ell_{o_{\mathrm{root}}}
\]
and the same tree-edge compatibility equations
\[
\widetilde\ell_q\circ\Psi_{p\to q}
=
\widetilde\ell_p
\]
for every parent--child pair $p,q$ in the rooted tree. We show that
\[
\widetilde\ell_o
=
\ell_o
\qquad
\text{for every }o\in V(T).
\]
The equality holds at the root by assumption. Suppose it holds at a vertex $p$, and let $q$ be any child of $p$. Since $\Psi_{p\to q}$ is bijective, the compatibility equation uniquely determines the coordinate map at $q$:
\[
\begin{aligned}
\widetilde\ell_q
&=
\widetilde\ell_p
\circ
\Psi_{p\to q}^{-1}\\
&=
\ell_p
\circ
\Psi_{p\to q}^{-1}.
\end{aligned}
\]
On the other hand, the previously established relation
\[
\ell_q\circ\Psi_{p\to q}
=
\ell_p
\]
implies
\[
\ell_q
=
\ell_p\circ\Psi_{p\to q}^{-1}.
\]
Consequently,
\[
\widetilde\ell_q
=
\ell_q.
\]
Induction on the graph distance from the root therefore gives
\[
\widetilde\ell_o
=
\ell_o
\]
for every vertex $o$ of the tree. Hence the root-relative coordinate convention is unique once the root labeling, the spanning tree, and the selected directed tree edges are fixed.

Finally, let
\[
f=(p\xrightarrow{a}q)
\]
be any feasible directed edge not used as a tree edge, and let
\[
\widehat\tau_f:
\widehat{\cC}_p
\longrightarrow
\widehat{\cC}_q
\]
denote its recovered local edge map. Its expression in the common root-relative label set is the map
\[
D_f
:=
\ell_q
\circ
\widehat\tau_f
\circ
\ell_p^{-1}
:
\mathcal L
\longrightarrow
\mathcal L.
\]
Using
\[
\ell_q
=
\ell_{o_{\mathrm{root}}}
\circ
\Psi_q^{-1}
\]
and
\[
\ell_p^{-1}
=
\Psi_p
\circ
\ell_{o_{\mathrm{root}}}^{-1},
\]
we obtain
\[
\begin{aligned}
D_f
&=
\ell_{o_{\mathrm{root}}}
\circ
\Psi_q^{-1}
\circ
\widehat\tau_f
\circ
\Psi_p
\circ
\ell_{o_{\mathrm{root}}}^{-1}.
\end{aligned}
\]
The middle map
\[
\Psi_q^{-1}
\circ
\widehat\tau_f
\circ
\Psi_p
:
\widehat{\cC}_{o_{\mathrm{root}}}
\longrightarrow
\widehat{\cC}_{o_{\mathrm{root}}}
\]
is the transport obtained by moving from the root to $p$ along the unique tree path, applying the off-tree edge $f$, and returning from $q$ to the root along the reverse of the unique tree path. Thus it is the residual transport around the closed walk determined by $f$ and the tree scaffold, and $D_f$ is precisely that residual transport written in the chosen root labels.

The off-tree edge is perfectly consistent with the identification induced by the tree if and only if this residual map is the identity. Indeed,
\[
\begin{aligned}
D_f
=
\Id_{\mathcal L}
&\iff
\ell_q
\circ
\widehat\tau_f
\circ
\ell_p^{-1}
=
\Id_{\mathcal L}\\
&\iff
\widehat\tau_f
=
\ell_q^{-1}
\circ
\ell_p\\
&\iff
\widehat\tau_f
=
\Psi_q
\circ
\ell_{o_{\mathrm{root}}}^{-1}
\circ
\ell_{o_{\mathrm{root}}}
\circ
\Psi_p^{-1}\\
&\iff
\widehat\tau_f
=
\Psi_q
\circ
\Id_{\widehat{\cC}_{o_{\mathrm{root}}}}
\circ
\Psi_p^{-1}\\
&\iff
\widehat\tau_f
=
\Psi_q\circ\Psi_p^{-1}.
\end{aligned}
\]
The map on the right is exactly the transport from $p$ to $q$ predicted by the synchronized tree coordinates. Hence $D_f$ provides a cycle-consistency diagnostic. More generally, a nonidentity $D_f$ records the residual effect around the corresponding cycle rather than being erased by the gauge choice. If $\widehat\tau_f$ is itself bijective, then
\[
D_f\in\Sym(\mathcal L),
\]
and the displayed conjugacy shows that it is the holonomy permutation of the associated closed walk, expressed in root-relative coordinates.

To make the dependence on the arbitrary root labeling explicit, suppose that the root labels are changed by a permutation
\[
\chi\in\Sym(\mathcal L),
\qquad
\ell'_{o_{\mathrm{root}}}
:=
\chi\circ\ell_{o_{\mathrm{root}}}.
\]
The propagated coordinate map at an arbitrary observation $o$ is then
\[
\begin{aligned}
\ell'_o
&=
\ell'_{o_{\mathrm{root}}}
\circ
\Psi_o^{-1}\\
&=
\chi
\circ
\ell_{o_{\mathrm{root}}}
\circ
\Psi_o^{-1}\\
&=
\chi\circ\ell_o.
\end{aligned}
\]
The corresponding off-tree diagnostic becomes
\[
\begin{aligned}
D'_f
&=
\ell'_q
\circ
\widehat\tau_f
\circ
(\ell'_p)^{-1}\\
&=
\chi
\circ
\ell_q
\circ
\widehat\tau_f
\circ
(\chi\circ\ell_p)^{-1}\\
&=
\chi
\circ
\ell_q
\circ
\widehat\tau_f
\circ
\ell_p^{-1}
\circ
\chi^{-1}\\
&=
\chi\circ D_f\circ\chi^{-1}.
\end{aligned}
\]
Thus all root-relative coordinates and off-tree diagnostics are completely determined by the chosen root labeling, the selected directed tree edges, and the tree $T$. Changing only the arbitrary root names produces the expected common relabeling, or conjugation, and introduces no additional ambiguity.
\end{proof}

\end{document}